\documentclass[twoside,11pt]{article} 

\usepackage{times}
\usepackage{amsmath, amsfonts}
\usepackage{nicefrac,enumitem}
\usepackage{booktabs,subcaption}
\usepackage{xfrac}

\usepackage{algorithm}
\usepackage{algorithmic}
\usepackage{natbib} 

\usepackage{fullpage}
\usepackage{amsthm,amssymb} 

\usepackage{adjustbox}
\usepackage{mathtools}

\usepackage{bm}

\usepackage{color}
\definecolor{puorange}{rgb}{0.80,0.20,0}
\definecolor{bluegray}{rgb}{0.04,0,0.7}
\definecolor{greengray}{rgb}{0.05,0.50,0.15}
\definecolor{darkbrown}{rgb}{0.40,0.2,0.05}
\definecolor{darkcyan}{rgb}{0,0.4,1}
\definecolor{black}{rgb}{0,0,0}
\definecolor{grey}{rgb}{0.93,0.93,0.93}

\newcommand \reals {\mathbb{R}}

\newcommand \inv {^{-1}} 
\newcommand \T {^{\top}}	

\newcommand \bigO {\mathcal{O}}

\newcommand \expect {\mathbb{E}}

\newcommand \ind {\operatorname*{\mathbb{I}}}
\usepackage{bbm}

\newcommand \prob {\operatorname*{\mathbb{P}}}

\newcommand \pow [1]{^{(#1)}}
\DeclarePairedDelimiterX{\inp}[2]{\langle}{\rangle}{#1, #2} 
\DeclarePairedDelimiterX{\norm}[1]{\Vert}{\Vert}{#1} 
\DeclarePairedDelimiterX{\normsq}[1]{\Vert}{\Vert^2}{#1} 
\DeclarePairedDelimiter\abs{\lvert}{\rvert}
\newcommand \norma [2]{\Vert #2 \Vert_{#1}}
\newcommand \normasq [2]{\Vert #2 \Vert^2_{#1}}

\newcommand \normad [2]{\Vert #2 \Vert_{#1}^{*}}

\newcommand \eps \epsilon

\newcommand \argmin {\operatorname*{arg\,min}} 
\newcommand \argmax {\operatorname*{arg\,max}} 
\newcommand \conv {\operatorname*{conv}} 
\newcommand \dom {\operatorname*{dom}} 
 
\newcommand \grad {\nabla}

\newcommand \ev {{\bm{e}}}

\newcommand \wv {{\bm{w}}}
\newcommand \xv {{\bm{x}}}
\newcommand \yv {{\bm{y}}}

\newcommand \zv {{\bm{z}}}
\newcommand \rv {{\bm{r}}}
\newcommand \gv {{\bm{g}}}

\newcommand \av {{\bm{a}}}
\newcommand \bv {{\bm{b}}}

\newcommand \uv {{\bm{u}}}
\newcommand \vv {{\bm{v}}}

\newcommand \thetav {{\bm{\theta}}}

\newcommand \zerov {{\bm{0}}}

\newcommand \Am {{\bm{A}}}

\newcommand \mcU {\mathcal U}

\newcommand \mcV {\mathcal V}

\newcommand \mcX {\mathcal X}
\newcommand \mcY {\mathcal Y}
\newcommand \mcT {\mathcal T}

\newcommand \mcB {\mathcal B}
\newcommand \mcL {\mathcal L}
\newcommand \mcM {\mathcal M}

\newcommand \mcE {\mathcal E}
\newcommand \mcF {\mathcal F}
\newcommand \mcR {\mathcal R}
\newcommand \mcG {\mathcal G}
\newcommand \mcA {\mathcal A}
\newcommand \mcD {\mathcal D}
\newcommand \mcS {\mathcal S}
\newcommand \mcC {\mathcal C}
\newcommand \mcP {\mathcal P}

\newcommand{\ms}{\scriptscriptstyle}

\newtheorem{theorem}{Theorem}
\newtheorem{example}[theorem]{Example} 
\newtheorem{lemma}[theorem]{Lemma} 
\newtheorem{proposition}[theorem]{Proposition} 
\newtheorem{remark}[theorem]{Remark}
\newtheorem{corollary}[theorem]{Corollary}
\newtheorem{definition}[theorem]{Definition}
\newtheorem{claim}[theorem]{Claim}

\newtheorem{fact}[theorem]{Fact}
\newtheorem{assumption}[theorem]{Assumption}

 \usepackage{thmtools}
 \declaretheoremstyle[
notefont=\bfseries, notebraces={}{},
bodyfont=\normalfont\itshape,
headformat=\NAME \NOTE
]{nopar}

\declaretheorem[style=nopar, name=Lemma]{lemma_unnumbered}
\declaretheorem[style=nopar, name=Corollary]{corollary_unnumbered}
\declaretheorem[style=nopar, name=Proposition]{proposition_unnumbered}

\newcommand {\maxK} [2] {\sideset{}{^{(#1)}} \max_{#2} } 

\usepackage[colorlinks,citecolor=bluegray,linkcolor=darkbrown,urlcolor=blue,breaklinks]{hyperref}

\newcommand{\casimir}{{Casimir}}

\newcommand{\nsCatalystSvrg}{{Casimir-SVRG}}

\newcommand{\nsCatalystExpt}{{Casimir-SVRG-const}}
\newcommand{\nsCatalystDecayExpt}{{Casimir-SVRG-adapt}}
\newcommand{\plcsvrg}{{PL-Casimir-SVRG}}
\newcommand{\SGD}{SGD}
\newcommand{\svrg}{{SVRG}}
\newcommand{\bcfw}{{BCFW}}

\newcommand{\proxgrad}{\bm{\varrho}}

\date{\vspace{-5ex}}
	\title{A Smoother Way to Train Structured Prediction Models}
	
	\author{
	Krishna Pillutla$^1$ \qquad 
	Vincent Roulet$^2$ \qquad
	Sham M. Kakade$^{1,2}$ \qquad
	Zaid Harchaoui$^2$
	\\ \\
	$^1$ Paul G. Allen School of Computer Science and Engineering, University of Washington \\
	$^2$ Department of Statistics, University of Washington \\
	\texttt{\{pillutla,sham\}@cs.washington.edu, \{vroulet,zaid\}@uw.edu}
	}

\begin{document}
	
\maketitle

\begin{abstract}%

We present a framework to train a structured prediction model by performing smoothing on the inference algorithm it builds upon. 
Smoothing overcomes the non-smoothness inherent to the maximum margin structured prediction objective, 
and paves the way for the use of fast primal gradient-based optimization algorithms. 
We illustrate the proposed framework by developing a novel primal incremental 
optimization algorithm for the structural support vector machine. 
The proposed algorithm blends an extrapolation scheme for acceleration and an adaptive smoothing scheme
and builds upon the stochastic variance-reduced gradient algorithm. 
We establish its worst-case global complexity bound and study several practical variants,
including extensions to deep structured prediction.
We present experimental results on two real-world problems, 
namely named entity recognition and visual object localization. The experimental results
show that the proposed framework allows us to build upon efficient inference algorithms 
to develop large-scale optimization algorithms 
for structured prediction which can achieve competitive performance on the two real-world problems.
\end{abstract}

\section{Introduction}
Consider the optimization problem arising when training
maximum margin structured prediction models: 
\begin{align} \label{eq:c:main:prob}
	\min_{\wv\in\reals^d} \left[ F(\wv) := \frac{1}{n} \sum_{i=1}^n f\pow{i}(\wv) + \frac{\lambda}{2} \normasq{2}{\wv}  \right] \,,
\end{align}
where each $f\pow{i}$ is the structural hinge loss. 
Max-margin structured prediction was designed to forecast
discrete data structures such as 
sequences and trees~\citep{taskar2004max,tsochantaridis2004support}. 

Batch non-smooth optimization algorithms such as cutting plane methods are appropriate for problems with small or moderate sample sizes~\citep{tsochantaridis2004support,joachims2009cutting}. Stochastic non-smooth optimization algorithms such as stochastic subgradient methods can tackle problems with large sample sizes~\citep{ratliff2007approximate,shalev2011pegasos}. However, both families of methods achieve the typical worst-case complexity bounds of non-smooth optimization algorithms and cannot easily leverage a possible hidden smoothness of the objective. 

Furthermore, as significant progress is being made on incremental smooth optimization algorithms for training unstructured prediction models~\citep{lin2017catalyst}, we would like to transfer such advances and design faster optimization algorithms to train structured prediction models. Indeed if each term in the finite-sum were $L$-smooth,  
incremental optimization algorithms such as 
MISO~\citep{mairal2013optimization},
SAG~\citep{roux2012stochastic,schmidt2017minimizing},
SAGA~\citep{defazio2014saga},
SDCA \citep{shalev2013stochastic}, and
SVRG~\citep{johnson2013accelerating}
could leverage the finite-sum structure of the objective~\eqref{eq:c:main:prob} 
and achieve faster convergence than batch algorithms on large-scale problems. 

Incremental optimization algorithms can be further accelerated, 
either on a case-by-case basis~\citep{shalev2014accelerated,frostig2015regularizing,allen2016katyusha,defazio2016simple}
or using the Catalyst acceleration scheme~\citep{lin2015universal,lin2017catalyst},
to achieve near-optimal convergence rates~\citep{woodworth2016tight}.
Accelerated incremental optimization algorithms demonstrate stable and fast convergence behavior on a wide range of problems, 
in particular for ill-conditioned ones. 

We introduce a general framework that allows us to bring the power of accelerated incremental optimization algorithms
to the realm of structured prediction problems. 
To illustrate our framework, we focus on the problem of training a structural support vector machine (SSVM), 
and extend the developed algorithms to deep structured prediction models with nonlinear mappings.

We seek primal optimization algorithms, as opposed to saddle-point or primal-dual optimization algorithms, 
in order to be able to tackle 
structured prediction models with affine mappings such as SSVM as well as deep structured prediction models with nonlinear mappings. 
We show how to shade off the inherent non-smoothness of the objective 
while still being able to rely on efficient inference algorithms. 

\begin{description}
\item[Smooth Inference Oracles.] We introduce a notion of smooth inference oracles that gracefully fits the framework of black-box first-order optimization.
While the exp inference oracle reveals the relationship between max-margin and probabilistic structured prediction models, 
the top-$K$ inference oracle 
can be efficiently computed using simple modifications of efficient inference algorithms in many cases of interest. 

\item[Incremental Optimization Algorithms.] We present a new algorithm built on top of SVRG, 
blending an extrapolation scheme for acceleration 
and an adaptive smoothing scheme. We establish the worst-case complexity 
bounds of the proposed algorithm and extend it to the case of non-linear mappings.
Finally, we demonstrate its effectiveness compared
to competing algorithms on two tasks, namely named entity recognition and visual object localization. 
\end{description}

The code is publicly available as a software library called \texttt{Casimir}\footnote{\url{https://github.com/krishnap25/casimir}}.
The outline of the paper is as follows: Sec.~\ref{sec:related_work} reviews related work.
Sec.~\ref{sec:setting} discusses smoothing for structured prediction followed by 
Sec.~\ref{sec:inference_oracles}, which defines and studies the properties of inference oracles
and Sec.~\ref{sec:smooth_oracle_impl}, which describes the concrete implementation of these
inference oracles in several settings of interest.
Then, we switch gears to study accelerated incremental algorithms in convex case (Sec.~\ref{sec:cvx_opt})
and their extensions to deep structured prediction (Sec.~\ref{sec:ncvx_opt}).
Finally, we evaluate the proposed algorithms on two tasks, namely 
named entity recognition and visual object localization in Sec.~\ref{sec:expt}.

\subsection{Related Work} \label{sec:related_work}

\begin{table*}[t!]
\caption{\small{Convergence rates given in terms of the number of calls to various oracles for different optimization algorithms 
 on the learning problem~\eqref{eq:c:main:prob} in 
 case of structural support vector machines~\eqref{eq:pgm:struc_hinge}.
 The rates are specified in terms of the target accuracy $\eps$, 
 the number of training examples $n$, the 
 regularization $\lambda$, the 
 size of the label space~$\scriptsize{\abs\mcY}$, the 
 max feature norm $R=\max_i \norma{2}{\Phi(\xv\pow{i},\yv) - \Phi(\xv\pow{i}, \yv\pow{i})}$ and $\widetilde R \ge R$ 
 (see Remark~\ref{remark:smoothing:l2vsEnt} for explicit form).
 The rates are specified up to constants and factors logarithmic in the 
 problem parameters. The dependence on the initial error is ignored. 
 * denotes algorithms that make $\bigO(1)$ oracle calls per iteration.
\vspace{2mm}
}}
\label{tab:rates}
\footnotesize\setlength{\tabcolsep}{2pt}
\begin{minipage}{.32\linewidth}
\centering
\begin{adjustbox}{width=1\textwidth}
\begin{tabular}{|c|c|}
\hline
	\rule{0pt}{10pt}
	\textbf{Algo.} (\textit{exp} oracle)  & \textbf{\# Oracle calls} \\[0.45ex] \hline\hline

	\rule{0pt}{15pt}
	\begin{tabular}{c} Exponentiated \\ gradient* \\ \citep{collins2008exponentiated}\end{tabular} &
		$\dfrac{(n + \log |\mcY|) R^2 }{\lambda \eps}$ \\[2.54ex] \hline

	\rule{0pt}{15pt}
	\begin{tabular}{c} Excessive gap \\ reduction \\ \citep{zhang2014accelerated} \end{tabular} &
		$n R \sqrt{\dfrac{\log |\mcY|}{\lambda \eps}}$ \\[2.54ex] \hline

	\rule{0pt}{15pt}
	\begin{tabular}{c} Prop.~\ref{prop:c:total_compl_svrg_main}*,  \\ entropy smoother \end{tabular}
	& $\sqrt{\dfrac{nR^2 \log\abs\mcY}{\lambda \eps}}$ \\[2.54ex] \hline

	\rule{0pt}{15pt}
	\begin{tabular}{c} Prop.~\ref{prop:c:total_compl_sc:dec_smoothing_main}*,  \\ entropy smoother \end{tabular}
	& $n + {\dfrac{R^2 \log\abs\mcY}{\lambda \eps}}$ \\[2.54ex] \hline
\end{tabular}
\end{adjustbox}
\end{minipage}	\hspace{2.6mm}%
\begin{minipage}{.35\linewidth}
\centering
\begin{adjustbox}{width=1\textwidth}
\begin{tabular}{|c|c|}
\hline
	\rule{0pt}{10pt}
    \textbf{Algo.} (\textit{max} oracle)  & \textbf{\# Oracle calls} \\[0.45ex] \hline\hline

    \rule{0pt}{15pt}
	\begin{tabular}{c} BMRM \\ \citep{teo2009bundle}\end{tabular} & 
		$\dfrac{n R^2}{\lambda \eps}$ \\[2ex] \hline

	\rule{0pt}{15pt}
	\begin{tabular}{c} QP 1-slack \\ \citep{joachims2009cutting} \end{tabular}& 
		$\dfrac{n R^2}{\lambda \eps}$ \\[2ex] \hline

	\rule{0pt}{15pt}
	\begin{tabular}{c} Stochastic \\ subgradient* \\ \citep{shalev2011pegasos} \end{tabular}& 
		$\dfrac{R^2}{\lambda \eps}$ \\[2ex] \hline

	\rule{0pt}{15pt}
	\begin{tabular}{c} Block-Coordinate \\ Frank-Wolfe* \\ \citep{lacoste2012block}  \end{tabular} & 
		$n + \dfrac{R^2}{\lambda \eps}$ \\[2ex] \hline
\end{tabular} 
\end{adjustbox}   
\end{minipage} \hspace{2.2mm}
\begin{minipage}{.25\linewidth}
\centering

\medskip
\begin{adjustbox}{width=1\textwidth}
\begin{tabular}{|c|c|}
\hline
	\rule{0pt}{0pt}
	\begin{tabular}{c} \textbf{Algo.}  \\ (\textit{top-$K$} oracle)   \end{tabular}
	& \textbf{\# Oracle calls} \\[0.45pt] \hline\hline

	\rule{0pt}{12pt}
	\begin{tabular}{c} Prop.~\ref{prop:c:total_compl_svrg_main}*,  \\ $\ell_2^2$ smoother \end{tabular}
	& $\sqrt{\dfrac{n{\widetilde R}^2}{\lambda \eps}}$ \\[2.54ex] \hline

	\rule{0pt}{12pt}
	\begin{tabular}{c} Prop.~\ref{prop:c:total_compl_sc:dec_smoothing_main}*,  \\ $\ell_2^2$ smoother \end{tabular}
	& $n + {\dfrac{{\widetilde R}^2}{\lambda \eps}}$ \\[2.45ex] \hline
\end{tabular}  
\end{adjustbox}    
\end{minipage}%
\end{table*}

\paragraph{Optimization for Structural Support Vector Machines}
Table~\ref{tab:rates} gives an overview of different optimization algorithms designed for structural 
support vector machines.
Early works~\citep{taskar2004max,tsochantaridis2004support,joachims2009cutting,teo2009bundle}
considered batch dual quadratic optimization (QP) algorithms.
The stochastic subgradient method operated directly 
on the non-smooth primal formulation~\citep{ratliff2007approximate,shalev2011pegasos}. 
More recently,~\citet{lacoste2012block} proposed a block coordinate Frank-Wolfe (BCFW) algorithm to 
optimize the dual formulation of structural support vector machines; see also~\citet{osokin2016minding} for variants and extensions.
Saddle-point or primal-dual approaches include
the mirror-prox algorithm~\citep{taskar2006structured,cox2014dual,he2015semi}.~\citet{palaniappan2016stochastic} propose an incremental optimization algorithm for saddle-point problems. However, it is unclear how to extend it to the structured prediction problems considered here. Incremental optimization algorithms for conditional random fields were proposed by~\citet{schmidt2015non}.
We focus here on primal optimization algorithms in order to be able to 
train structured prediction models with affine or nonlinear mappings 
with a unified approach, and on incremental optimization algorithms which can scale to large datasets. 

\paragraph{Inference}
The ideas of dynamic programming inference in tree structured graphical models have been around
since the pioneering works of \citet{pearl1988probabilistic} and \citet{dawid1992applications}.
Other techniques emerged based on graph cuts \citep{greig1989exact,ishikawa1998segmentation},  
bipartite matchings \citep{cheng1996maximum,taskar2005discriminative} and search 
algorithms \citep{daume2005learning,lampert2008beyond,lewis2014ccg,he2017deep}.
For graphical models that admit no such a discrete structure, 
techniques based on loopy belief propagation~\citep{mceliece1998turbo,murphy1999loopy},
linear programming (LP)~\citep{schlesinger1976syntactic}, dual decomposition \citep{johnson2008convex}
and variational inference \citep{wainwright2005map,wainwright2008graphical}
gained popularity. 

\paragraph{Top-$K$ Inference}
Smooth inference oracles with $\ell_2^2$ smoothing echo older heuristics 
in speech and language processing~\citep{jurafsky2014speech}. 
Combinatorial algorithms for top-$K$ inference have been studied extensively by the graphical models community under the name 
``$M$-best MAP''.
\citet{seroussi1994algorithm} and \citet{nilsson1998efficient} 
first considered the problem of finding the $K$ most probable configurations
in a tree structured graphical model.
Later, \citet{yanover2004finding} presented the Best Max-Marginal First algorithm which solves this problem with access only to 
an oracle that computes max-marginals. 
We also use this algorithm in Sec.~\ref{subsec:smooth_inference_loopy}.
\citet{fromer2009lp} study top-$K$ inference for LP relaxation, while
\citet{batra2012efficient} considers the dual problem to exploit graph structure.
\citet{flerova2016searching} study top-$K$ extensions of the popular $\text{A}^\star$ 
and branch and bound search algorithms in the context of graphical models. 
Other related approaches include diverse $K$-best solutions \citep{batra2012diverse} and 
finding $K$-most probable modes \citep{chen2013computing}.

\paragraph{Smoothing Inference}
Smoothing for inference was used to speed up iterative algorithms for continuous relaxations.
\citet{johnson2008convex} considered smoothing dual decomposition inference using the entropy smoother,
followed by \citet{jojic2010accelerated} and \citet{savchynskyy2011study} who studied its theoretical properties.
\citet{meshi2012convergence} expand on this study to include $\ell_2^2$ smoothing.
Explicitly smoothing discrete inference algorithms in order to smooth the learning problem was considered by 
\citet{zhang2014accelerated} and \citet{song2014learning} using the entropy and $\ell_2^2$ smoothers respectively.
The $\ell_2^2$ smoother was also used by \citet{martins2016softmax}.
\citet{hazan2016blending} consider the approach of blending learning and inference, instead of using inference 
algorithms as black-box procedures.

Related ideas to ours appear in the independent works~\citep{mensch2018differentiable,niculae2018sparsemap}. 
These works partially overlap with ours, but the papers choose different perspectives, 
making them complementary to each other.~\citet{mensch2018differentiable} proceed
differently when, {e.g.},~smoothing inference based on dynamic programming. 
Moreover, they do not establish complexity bounds
for optimization algorithms making calls to the resulting smooth inference oracles. 
We define smooth inference oracles in the context 
of black-box first-order optimization and 
establish worst-case complexity bounds for incremental optimization algorithms making calls to these oracles.
Indeed we relate the amount of smoothing controlled by $\mu$ to the resulting complexity of the optimization algorithms relying on smooth inference oracles. 

\paragraph{End-to-end Training of Structured Prediction}
The general framework for global training of structured prediction models was introduced by~\cite{bottou1990framework} 
and applied to handwriting recognition by~\cite{bengio1995lerec} and to document processing by~\cite{bottou1997global}. 
This approach, now called ``deep structured prediction'', was used, e.g., 
by \citet{collobert2011natural} and \citet{belanger2016structured}.

\subsection{Notation}
Vectors are denoted by bold lowercase characters as $\wv \in \reals^d$ while matrices are denoted by bold uppercase characters as $\Am \in \reals^{d \times n}$.
For a matrix $\Am \in \reals^{m \times n}$, define the norm for $\alpha,\beta \in \{1, 2, \infty\}$,
\begin{align} \label{eq:matrix_norm_defn}
\norma{\beta, \alpha}{\Am} = \max\{ \inp{\yv}{\Am\xv} \, | \, \norma{\alpha}{\yv} \le 1 \, , \, \norma{\beta}{\xv} \le 1  \}
	\,.
\end{align}
For any function $f: \reals^d \to \reals \cup \{ +\infty \}$, its convex conjugate $f^*:\reals^d \to \reals \cup \{+\infty\}$ is defined as 
\begin{align*}
	f^*(\zv) = \sup_{\wv \in \reals^d} \left\{ \inp{\zv}{\wv} - f(\wv) \right\} \, .
\end{align*}

A function $f : \reals^d \to \reals$ is said to be $L$-smooth with respect to an arbitrary norm $\norm{\cdot}$ 
if it is continuously differentiable and its gradient $\grad f$ 
is $L$-Lipschitz with respect to $\norm{\cdot}$.
When left unspecified, $\norm{\cdot}$ refers to $\norma{2}{\cdot}$.
Given a continuously differentiable map $\gv : \reals^d \to \reals^m$, 
its Jacobian $\grad \gv(\wv) \in \reals^{m \times d}$ at $\wv \in \reals^d$ 
is defined so that its $ij$th entry is $[\grad \gv(\wv)]_{ij} = \partial g_i(\wv) / w_j$ 
where $g_i$ is the $i$th element of $\gv$ and $w_j$ is the $j$th element of $\wv$.
The vector valued function $\gv : \reals^d \to \reals^m$ is said to be $L$-smooth with respect to $\norm{\cdot}$ 
if it is continuously differentiable and 
its Jacobian $\grad \gv$ is $L$-Lipschitz with respect to $\norm{\cdot}$.

For a vector $\zv \in \reals^m$, $z_{(1)} \ge \cdots \ge z_{(m)}$ refer to its components enumerated in non-increasing order
where ties are broken arbitrarily.
Further, we let $\zv_{[k]} = (z_{(1)}, \cdots, z_{(k)}) \in \reals^k$ denote the vector of the $k$ largest components of $\zv$.
We denote by $\Delta^{m-1}$ the standard probability simplex in $\reals^{m}$.
When the dimension is clear from the context, we shall simply denote it by $\Delta$.
Moreover, for a positive integer $p$, $[p]$ refers to the set $\{1, \ldots,p\}$.
Lastly, $\widetilde \bigO$ in the big-$\bigO$ notation hides factors logarithmic
in problem parameters.


\section{Smooth Structured Prediction} \label{sec:setting}

Structured prediction aims to search for {\em score} functions $\phi$ parameterized by 
$\wv \in \reals^d$ that model the compatibility of input $\xv \in \mcX$ and output $\yv \in \mcY$ 
as $\phi(\xv,\yv; \wv)$ through a graphical model.
Given a score function $\phi(\cdot, \cdot;\wv)$, predictions are made
using an {\em inference} procedure which, when given an input $\xv$, produces the 
best output
\begin{align}\label{eq:pgm:inference}
\yv^*(\xv ; \wv) \in \argmax_{\yv \in \mcY} \phi(\xv, \yv ; \wv) \,.
\end{align}
We shall return to the score functions and the inference procedures in Sec.~\ref{sec:inference_oracles}.
First, given such a score function $\phi$, we define the structural hinge loss and describe how it can be smoothed.

\subsection{Structural Hinge Loss}
On a given input-output pair $(\xv, \yv)$, the error of prediction of $\yv$ by the inference procedure with a 
score function $\phi(\cdot, \cdot; \wv)$, is measured by a task loss $\ell \big( \yv, \yv^*(\xv; \wv) \big)$  such as the Hamming loss. 
The learning procedure would then aim to find the best parameter $\wv$ that minimizes 
the loss on a given dataset of input-output training examples.
However, the resulting problem is piecewise constant and hard to optimize.
Instead, \citet{altun2003hidden,taskar2004max,tsochantaridis2004support} propose to minimize a majorizing surrogate of the task loss, 
called the structural hinge loss defined on an input-output pair $(\xv\pow{i}, \yv\pow{i})$ as
\begin{align}\label{eq:pgm:struc_hinge}
f\pow{i}(\wv) = \max_{\yv \in \mcY} 
\left\{ \phi(\xv\pow{i}, \yv ; \wv) + \ell(\yv\pow{i}, \yv) \right\}
- \phi(\xv\pow{i}, \yv\pow{i} ; \wv)  = \max_{\yv \in \mcY} \psi\pow{i}(\yv, \wv)  \, .
\end{align}
where $\psi\pow{i}(\yv ; \wv) =  \phi(\xv\pow{i}, \yv ; \wv) + \ell(\yv\pow{i}, \yv) - \phi(\xv\pow{i}, \yv\pow{i} ; \wv)$ is the augmented score function.

This approach, known as {\em max-margin structured prediction}, 
builds upon binary and multi-class support vector machines~\citep{crammer2001algorithmic}, where the term $\ell(\yv\pow{i}, \yv)$ inside the maximization in \eqref{eq:pgm:struc_hinge} 
generalizes the notion of margin.
The task loss $\ell$ is assumed to possess appropriate structure
so that the maximization inside \eqref{eq:pgm:struc_hinge}, known as {\em loss augmented inference},
is no harder than the inference problem in \eqref{eq:pgm:inference}.
When considering a fixed input-output pair $(\xv\pow{i}, \yv\pow{i}$), 
we drop the index with respect to the sample $i$ and consider the 
structural hinge loss as 
\begin{equation}\label{eq:struct_hinge}
f(\wv) = \max_{\yv \in \mathcal{\mcY}} \psi(\yv;\wv),
\end{equation}
When the map $\wv \mapsto \psi(\yv ; \wv)$ is affine, 
the structural hinge loss $f$ and the objective $F$ from \eqref{eq:c:main:prob} are both convex - 
we refer to this case as the structural support vector machine. When $\wv \mapsto \psi(\yv ; \wv)$
is a nonlinear but smooth map, then the structural hinge loss $f$ and the objective $F$ are nonconvex.

\subsection{Smoothing Strategy}
A convex, non-smooth function $h$ can be smoothed by taking its infimal convolution with a smooth function~\citep{beck2012smoothing}. 
We now recall its dual representation, which \citet{nesterov2005smooth} 
first used to relate the amount of smoothing to optimal complexity bounds.

\begin{definition} \label{defn:smoothing:inf-conv}
	For a given convex function $h:\reals^m \to \reals$, a smoothing function $\omega: \dom h^* \to \reals$ which is 
	1-strongly convex with respect to $\norma{\alpha}{\cdot}$ (for $\alpha \in \{1,2\}$), 
	and a parameter $\mu > 0$, define
	\begin{align*}
	h_{\mu \omega}(\zv) = \max_{\uv \in \dom h^*} \left\{ \inp{\uv}{\zv} -  h^*(\uv) - \mu \omega(\uv) \right\} \,.
	\end{align*}
	as the smoothing of $h$ by $\mu \omega$.
\end{definition}
\noindent
We now state a classical result showing how the parameter $\mu$ 
controls both the approximation error and the level of the smoothing.
For a proof, see \citet[Thm. 4.1, Lemma 4.2]{beck2012smoothing} or Prop.~\ref{prop:smoothing:difference_of_smoothing}
of Appendix~\ref{sec:a:smoothing}.

\begin{proposition} \label{thm:setting:beck-teboulle}
	Consider the setting of Def.~\ref{defn:smoothing:inf-conv}. 
	The smoothing $h_{\mu \omega}$ is continuously differentiable and its gradient, given by 
	\[
	\grad h_{\mu \omega}(\zv) = \argmax_{\uv \in \dom h^*} \left\{ \inp{\uv}{\zv} - h^*(\uv) - \mu \omega(\uv) \right\}
	\]
	is $1/\mu$-Lipschitz with respect to $\normad{\alpha}{\cdot}$. 
	Moreover, letting $h_{\mu \omega} \equiv h$ for $\mu = 0$, the smoothing satisfies, for all $\mu_1 \ge \mu_2 \ge 0$,
	\begin{align*}
		(\mu_1 - \mu_2) \inf_{\uv \in \dom h^*} \omega(\uv) 
		\le 
		h_{\mu_2 \omega}(\zv) - h_{\mu_1 \omega}(\zv) 
		\le 
		(\mu_1 - \mu_2) \sup_{\uv \in \dom h^*} \omega(\uv) \,.
	\end{align*}
\end{proposition}
\paragraph{Smoothing the Structural Hinge Loss}
We rewrite the structural hinge loss as a composition
\begin{equation}\label{eq:mapping_def}
\gv:\
\begin{cases}
\reals^d &\to \reals^m \\
\wv &\mapsto (\psi(\yv;\wv))_{\yv \in \mathcal{Y}},
\end{cases} \, \qquad h: \begin{cases}
\reals^{m} &\to \reals \\
\zv &\mapsto \max_{i \in [m]} z_i,
\end{cases}
\end{equation}
where $m= |\mcY|$ so that the structural hinge loss reads
\begin{align} \label{eq:pgm:struc_hinge_vec}
f(\wv) = h \circ \gv(\wv)\,.
\end{align}
We smooth the structural hinge loss~\eqref{eq:pgm:struc_hinge_vec} by simply smoothing the 
non-smooth max function $h$ as 
\begin{align*}
f_{\mu \omega} = h_{\mu \omega} \circ \gv.
\end{align*}
When $\gv$ is smooth and Lipschitz continuous, 
$f_{\mu \omega}$ is a smooth approximation of the structural hinge loss, whose gradient is readily given by the chain-rule. 
In particular, when $\gv$ is an affine map $\gv(\wv) = \Am\wv + \bv$, 
if follows that
$f_{\mu \omega}$ is $(\normasq{\beta,\alpha}{\Am} / \mu)$-smooth with respect to $\norma{\beta}{\cdot}$
(cf. Lemma~\ref{lemma:smoothing:composition} in Appendix~\ref{sec:a:smoothing}).
Furthermore, for $\mu_1 \ge \mu_2 \ge 0$, we have, 
\[
	(\mu_1 - \mu_2) \min_{\uv \in \Delta^{m-1}} \omega(\uv) \le f_{\mu_2\omega}(\wv) - f_{\mu_1 \omega}(\wv) 
	\le (\mu_1 - \mu_2) \max_{\uv \in \Delta^{m-1}} \omega(\uv) \,.
\]
\subsection{Smoothing Variants}
In the context of smoothing the max function, we now describe two popular choices for the smoothing function $\omega$, 
followed by computational considerations.
\subsubsection{Entropy and $\ell_2^2$ smoothing}
When $h$ is the max function, the smoothing operation can be computed analytically for
the \emph{entropy} smoother and the $\ell_2^2$ smoother, denoted respectively as
\begin{align*}
-H(\uv) := \inp{\uv}{\log \uv} \qquad \mbox{and} \qquad \ell_2^2(\uv) := \tfrac{1}{2}(\normasq{2}{\uv} - 1) \,.
\end{align*}
These lead respectively to the log-sum-exp function~\citep[Lemma 4]{nesterov2005smooth}
\[
h_{-\mu H}(\zv) = \mu  \log\left(\sum_{i=1}^{m}e^{z_i/\mu}\right), \quad \nabla h_{-\mu H}(\zv) = \left[\frac{e^{z_i/\mu}}{\sum_{j=1}^{m}e^{z_j/\mu}}\right]_{i=1,\ldots,m} \,,
\]
and an orthogonal projection onto the simplex, 
\[
h_{\mu \ell_2^2}(\zv) = \langle \zv, \operatorname{proj}_{\Delta^{m-1}}(\zv/\mu) \rangle 
	- \tfrac{\mu}{2}\|\operatorname{proj}_{\Delta^{m-1}}(\zv/\mu)\|^2 + \tfrac{\mu}{2},  
	\quad \nabla h_{\mu \ell_2^2}(\zv) = \operatorname{proj}_{\Delta^{m-1}}(\zv/\mu) \,.
\]
Furthermore, the following holds for all $\mu_1 \ge \mu_2 \ge 0$ from Prop.~\ref{thm:setting:beck-teboulle}: 
\[
	0 \le h_{-\mu_1 H}(\zv) - h_{-\mu_2 H}(\zv) \le (\mu_1 - \mu_2) \log m, \quad \text{and,} \quad 
	0 \le h_{\mu_1 \ell_2^2}(\zv) - h_{\mu_2 \ell_2^2}(\zv) \le \tfrac{1}{2}(\mu_1-\mu_2) \,.
\]

\subsubsection{Top-$K$ Strategy}
Though the gradient of the composition $f_{\mu \omega} = h_{\mu \omega} \circ \gv$ 
can be written using the chain rule, its actual computation for structured prediction problems 
involves computing $\grad \gv$ over all $m = \abs{\mcY}$ of its components, which may be intractable.
However, in the case of $\ell_2^2$ smoothing, projections onto the simplex are sparse, as pointed out by the following proposition.
\begin{proposition} \label{prop:smoothing:proj-simplex-1}
	Consider the Euclidean projection 
	$\uv^* = \argmin_{\uv \in \Delta^{m-1}}\normasq{2}{\uv -{\zv}/{\mu}} $ of $\zv/\mu \in \reals^m$ onto the simplex,
	where $\mu > 0$. 
	The projection $\uv^*$ has exactly $k \in [m]$ non-zeros if and only if
	\begin{align} \label{eq:smooth:proj:simplex_1_statement}
		\sum_{i=1}^k \left(z_{(i)} - z_{(k)} \right) < \mu 
			\le \sum_{i=1}^k \left(z_{(i)} - z_{(k+1)} \right) \,,
	\end{align}
	where $z_{(1)}\ge \cdots \ge z_{(m)}$ are the components of $\zv$ in non-decreasing order
	and $z_{(m+1)} := -\infty$.
	In this case, $\uv^*$ is given by
	\begin{align*}
		u_i^* = \max \bigg\{0, \,  \tfrac{1}{k\mu }\sum_{j=1}^k \big( z_i - z_{(j)} \big) + \tfrac{1}{k} \bigg\} \,.
	\end{align*}
\end{proposition}
\begin{proof}
	The projection $\uv^*$ satisfies $u^*_i = (z_i/\mu + \rho^*)_+$, 
	where $\rho^*$ is the unique solution of $\rho$ in the equation
	\begin{align} \label{eq:smooth:proj:simplex_1}
		\sum_{i=1}^m \left( \frac{z_i}{\mu} + \rho \right)_+ = 1 \,,
	\end{align}
	where $\alpha_+ = \max\{0, \alpha\}$. See, e.g.,  
	\citet{held1974validation} for a proof of this fact.
	Note that $z_{(i)}/\mu + \rho^* \le 0$ 
	implies that $z_{(j)}/\mu + \rho^* \le 0$ for all $j \ge i$.  Therefore  
	$\uv^*$ has $k$ non-zeros if and only if $z_{(k)}/\mu + \rho^* > 0$ and $z_{(k+1)}/\mu + \rho^* \le 0$.

	Now suppose that $\uv^*$ has exactly $k$ non-zeros,  we can then solve \eqref{eq:smooth:proj:simplex_1} to obtain $\rho^* = \varphi_k(\zv/\mu)$, which is defined as
	\begin{align} \label{eq:smooth:proj:simplex_1b}
		\varphi_k\left( \frac \zv \mu \right) := \frac{1}{k} - \frac{1}{k} \sum_{i=1}^k \frac{z_{(i)}}{\mu} \,.
	\end{align}
	Plugging in the value of $\rho^*$ in $z_{(k)}/\mu + \rho^* > 0$ gives $\mu > \sum_{i=1}^k \left(z_{(i)} - z_{(k)} \right)$.
	Likewise, $z_{(k+1)}/\mu + \rho^* \le 0$ gives $\mu \le \sum_{i=1}^k \left(z_{(i)} - z_{(k+1)} \right)$.

	Conversely assume~\eqref{eq:smooth:proj:simplex_1_statement} and let $\widehat \rho = \varphi_k(\zv/\mu)$. 
	Eq. \eqref{eq:smooth:proj:simplex_1_statement} can  be written as 
	$z_{(k)}/\mu + \widehat\rho > 0$ and $z_{(k+1)}/\mu + \widehat\rho \le 0$. Furthermore, we verify that 
	$\widehat\rho$ satisfies Eq.~\eqref{eq:smooth:proj:simplex_1}, and so $\widehat \rho = \rho^*$ is its unique root. 
	It follows, therefore, that the sparsity of $\uv^*$ is $k$.
\end{proof}
Thus, the projection of $\zv /\mu$ onto the simplex picks out some number $K_{\zv/\mu}$ 
of the largest entries of $\zv / \mu$ - we refer to this as the sparsity of 
$\operatorname{proj}_{\Delta^{m-1}}(\zv/\mu)$.
This fact motivates the {\em top-$K$ strategy}: given $\mu>0$, fix an integer $K$ {\em a priori} and 
consider as surrogates for $h_{\mu\ell_2^2}$ and $\grad h_{\mu\ell_2^2}$ respectively
\[
h_{\mu, K}(\zv) := \max_{\uv \in \Delta^{K-1}} \left\{ \inp*{\zv_{[K]}}{\uv} - \mu \ell_2^2(\uv) \right\}\,,
\quad \text{and,} \quad
\widetilde \grad h_{\mu, K}(\zv) :=  \Omega_K(\zv)\T\operatorname{proj}_{\Delta^{K-1}}\left( \frac{\zv_{[K]}}{\mu} \right) \,,
\]
where $\zv_{[K]}$ denotes the vector composed of the $K$ largest entries of $\zv$ and
$
	\Omega_K : \reals^m \to \{0,1\}^{K \times m} 
$
defines their extraction, i.e.,
$\Omega_K(\zv) = (\ev_{j_1}\T,  \ldots, \ev_{j_K} \T)^\top\in \{0, 1\}^{K \times m}$ 
where $j_1, \cdots, j_K$ satisfy $z_{j_1} \ge \cdots \ge z_{j_K}$ 
such that
$ \zv_{[K]} =  \Omega_K(\zv) \zv$ \,.
A surrogate of the $\ell_2^2$ smoothing is then given by 
\begin{align} \label{eq:smoothing:fmuK_defn}
f_{\mu, K} := h_{\mu, K} \circ \gv \,,
\quad\text{and,}\quad
\widetilde \grad f_{\mu, K}(\wv) := \grad \gv(\wv)\T \widetilde \grad h_{\mu, K}(\gv(\wv)) \,.
\end{align}

\paragraph{Exactness of Top-$K$ Strategy}
We say that the top-$K$ strategy is {\em exact} at $\zv$ for $\mu>0$ when it recovers the first order information 
of $h_{\mu \ell_2^2}$, i.e. when $	h_{\mu \ell_2^2}(\zv) = h_{\mu, K}(\zv)$ and 
$\grad h_{\mu \ell_2^2}(\zv) = \widetilde \grad h_{\mu, K}(\zv)$.
The next proposition outlines when this is the case. Note that if the top-$K$ strategy is exact at $\zv$ for 
a smoothing parameter $\mu>0$ then it will be exact at $\zv$ for any $\mu'<\mu$.
\begin{proposition} \label{prop:smoothing:proj-simplex-2}
	The top-$K$ strategy is exact  at $\zv$
	for $\mu>0$ if
	\begin{align} \label{eq:smooth:proj:simplex_2}
		\mu \le \sum_{i=1}^K  \left(\zv_{(i)}  - \zv_{( {\ms K}+1)} \right) \,.
	\end{align}
	Moreover, for any fixed $\zv \in \reals^m$ such that the vector $\zv_{[\ms K+1]} = \Omega_{K+1}(\zv)\zv$ 
	has at least two unique elements, the top-$K$ strategy is exact at $\zv$ for 
	all $\mu$ satisfying $0 < \mu \le z_{(1)} - z_{({\ms K}+1)}$.
\end{proposition}
\begin{proof}
	First, we note that the top-$K$ strategy is exact when the sparsity $K_{\zv/\mu}$ of the projection 
	$\operatorname{proj}_{\Delta^{m-1}}(\zv/\mu)$ satisfies $K_{\zv/\mu} \le K$.
	From Prop.~\ref{prop:smoothing:proj-simplex-1}, the condition that 
	$K_{\zv/\mu} \in \{1, 2, \cdots, K\}$ happens when 
	\begin{align*}
	\mu \in 
		\bigcup_{k=1}^K \left( \sum_{i=1}^k \left(z_{(i)} - z_{(k)} \right), \, 
		\sum_{i=1}^k \left(z_{(i)} - z_{(k+1)} \right) \right] = 
		\left( 0 ,  \sum_{i=1}^K \left(z_{(i)} - z_{({\ms K}+1)} \right) \right] \,,
	\end{align*}
	since the intervals in the union are contiguous.
	This establishes \eqref{eq:smooth:proj:simplex_2}.

	The only case when \eqref{eq:smooth:proj:simplex_2} cannot hold for any value of $\mu > 0$ is when the right hand size 
	of \eqref{eq:smooth:proj:simplex_2} is zero. In the opposite case when $\zv_{[{\ms K} + 1]}$ has at least 
	two unique components, or equivalently, $z_{(1)} - z_{({\ms K}+1)} > 0$, the condition 
	$0 < \mu \le z_{(1)} - z_{({\ms K}+1)}$ implies \eqref{eq:smooth:proj:simplex_2}.
\end{proof}

If the top-$K$ strategy is exact at $\gv(\wv)$ for $\mu$, then 
\[
	f_{\mu, K}(\wv) = f_{\mu \ell_2^2}(\wv) 
	\quad \text{and} \quad
	\widetilde \grad f_{\mu, K}(\wv) = \grad f_{\mu \ell_2^2}(\wv) \,,
\]
where the latter follows from the chain rule.
When used instead of $\ell_2^2$ smoothing in the algorithms presented in Sec.~\ref{sec:cvx_opt}, 
the top-$K$ strategy provides a computationally efficient heuristic to smooth the structural hinge loss.
Though we do not have theoretical guarantees using this surrogate, 
experiments presented in Sec.~\ref{sec:expt} show its efficiency and its robustness to the choice of $K$.

\section{Inference Oracles} \label{sec:inference_oracles}
This section studies first order oracles used in standard and smoothed structured prediction. We first describe the parameterization of the score functions through graphical models.

\subsection{Score Functions} \label{sec:inf_oracles:score_func}
Structured prediction is defined by the structure of the output $\yv$, while input $\xv \in \mcX$ can be arbitrary.
Each output $\yv \in \mcY$ is composed of $p$ components $y_1, \ldots, y_p$ that are linked through a graphical model 
$\mathcal{G} = (\mathcal{V}, \mathcal{E})$ - 
the nodes $\mcV=\{1,\cdots,p\}$ represent the components of the output $\yv$ while the edges $\mcE$ define the 
dependencies between various components.
The value of each component $y_v$ for $v \in \mcV$ represents the state of the node $v$ and takes values from a finite set $\mcY_v$.
The set of all output structures $\mcY = \mcY_1 \times \cdots \times \mcY_p$ 
is then finite yet potentially intractably large.

The structure of the graph (i.e., its edge structure) depends on the task.
For the task of sequence labeling, the graph is a chain, 
while for the task of parsing, the graph is a tree. On the other hand, the graph 
used in image segmentation is a grid.

For a given input $\xv$ and a score function $\phi(\cdot, \cdot ; \wv)$, 
the value $\phi(\xv, \yv; \wv)$ measures the compatibility of the 
output $\yv$ for the input $\xv$. 
The essential characteristic of the score function is that it decomposes over the nodes and edges of the graph as
\begin{align} \label{eq:setting:score:decomp}
\phi(\xv, \yv ; \wv) = \sum_{v \in \mcV} \phi_v(\xv, y_v; \wv)
+ \sum_{(v,v') \in \mcE} \phi_{v,v'}(\xv, y_v, y_{v'} ; \wv) \,.
\end{align}

For a fixed $\wv$, each input $\xv$ defines a specific compatibility function $\phi(\xv, \cdot\, ; \wv)$. 
The nature of the problem and the optimization algorithms we consider hinge upon whether 
$\phi$ is an affine function of $\wv$ or not. The two settings studied here are the following:

\begin{description}
	\item{\bfseries Pre-defined Feature Map.}
	In this structured prediction framework, a pre-specified feature map 
	$\Phi: \mcX \times \mcY \to \reals^d$ is employed and the score $\phi$ is then defined as the linear function
	\begin{equation}\label{eq:pre_spec_feature_map}
	\phi(\xv, \yv ; \wv) = \inp{\Phi(\xv, \yv)}{\wv} = \sum_{v \in \mcV} \inp{\Phi_v(\xv, y_v)}{\wv}
	+ \sum_{(v,v') \in \mcE} \inp{\Phi_{v,v'}(\xv, y_v, y_{v'})}{\wv}\,.
	\end{equation}
	
	\item{\bfseries Learning the Feature Map.}
	We also consider the setting where the feature map $\Phi$ is parameterized by $\wv_0$,
	for example, using a neural network, and is learned from the data. The score function can then be written as
	\begin{equation}\label{eq:deep_setting}
	\phi(\xv, \yv ; \wv) = \inp{\Phi(\xv, \yv; \wv_0)}{\wv_1}
	\end{equation}
	where $\wv = (\wv_0, \wv_1)$ and the scalar product decomposes into nodes and edges as above.
\end{description}
Note that we only need the decomposition of the score function over nodes and edges of the $\mcG$ as in 
Eq.~\eqref{eq:setting:score:decomp}. 
In particular, while Eq.~\eqref{eq:deep_setting} is helpful 
to understand the use of neural networks in structured prediction,
the optimization algorithms developed in Sec.~\ref{sec:ncvx_opt}
apply to general nonlinear but smooth score functions.

This framework captures both generative probabilistic models such as Hidden Markov Models (HMMs)
that model the joint distribution between $\xv$ and $\yv$
as well as discriminative probabilistic models, 
such as conditional random fields~\citep{lafferty2001conditional}
where dependencies among the input variables $\xv$ do not need to be explicitly represented.
In these cases, the log joint and conditional probabilities respectively 
play the role of the score $\phi$.

\begin{example}[Sequence Tagging]
\label{example:inf_oracles:viterbi_example}
Consider the task of sequence tagging in natural language processing
where each $\xv = (x_1, \cdots, x_p) \in \mcX$ is a sequence of words and
$\yv = (y_1, \cdots, y_p) \in \mcY$ is a sequence of labels, 
both of length $p$. Common examples include part of speech tagging and named entity recognition.
Each word $x_v$ in the sequence $\xv$ comes from a finite dictionary $\mcD$, 
and each tag $y_v$ in $\yv$ takes values from a finite set $\mcY_v = \mcY_{\mathrm{tag}}$. 
The corresponding graph is simply a linear chain.

The score function measures the compatibility of a sequence $\yv\in\mcY$ for the input $\xv\in\mcX$ using parameters 
$\wv = (\wv_{\mathrm{unary}}, \wv_{\mathrm{pair}})$ as, for instance,
\[
\phi(\xv, \yv; \wv) = \sum_{v =1}^p \inp{\Phi_{\mathrm{unary}}(x_v, y_v)}{\wv_{\mathrm{unary}}}
+ \sum_{v=0}^p \inp{\Phi_{\mathrm{pair}}(y_v, y_{v+1})}{\wv_{\mathrm{pair}}}\,,
\]
where, using $\wv_{\mathrm{unary}} \in \reals^{\abs\mcD\abs{\mcY_{\mathrm{tag}}}}$ 
and  $\wv_{\mathrm{pair}} \in \reals^{\abs{\mcY_{\mathrm{tag}}}^2}$ as node and edge weights respectively, 
we define for each $v \in [p]$,
\[
\inp{\Phi_{\mathrm{unary}}(x_v, y_v)}{\wv_{\mathrm{unary}}} = \sum_{x \in \mathcal{D},\, j \in \mcY_{\mathrm{tag}}} 
w_{\mathrm{unary},\, x, j} \ind(x = x_v) \ind(j = y_v) \,.
\]
The pairwise term $\inp{\Phi_{\mathrm{pair}}(y_v, y_{v+1})}{\wv_{\mathrm{pair}}}$ is analogously defined.
Here, $y_0, y_{p+1}$ are special ``start'' and ``stop'' symbols respectively.
This can be written as a dot product of $\wv$ with a pre-specified feature map as in~\eqref{eq:pre_spec_feature_map}, 
by defining
\[
\Phi(\xv, \yv) = \big(\sum_{v=1}^p \ev_{x_v} \otimes \ev_{y_v} \big) 
\oplus \big(\sum_{v=0}^p \ev_{y_v} \otimes \ev_{y_{v+1}} \big) \,,
\]
where $\ev_{x_v}$ is the unit vector $(\ind(x = x_v))_{x \in \mcD} \in \reals^{\abs\mcD}$, 
$ \ev_{y_v}$ is the unit vector $(\ind(j=y_v))_{j \in \mcY_{\mathrm{tag}}} \in \reals^{\abs{\mcY_{\mathrm{tag}}}}$,
$\otimes$ denotes the Kronecker product between vectors and $\oplus$ denotes vector concatenation.
\end{example}

\subsection{Inference Oracles}
We define now inference oracles as first order oracles in structured prediction.
These are used later
to understand the information-based complexity of optimization algorithms.

\subsubsection{First Order Oracles in Structured Prediction}
A first order oracle for a function $f :\reals^d \to \reals$ is a routine which,
given a point $\wv \in \reals^d$, returns on output a value $f(\wv)$ and 
a (sub)gradient $\vv \in \partial f(\wv)$, where $\partial f$ is the 
Fr\'echet (or regular) subdifferential
\citep[Def. 8.3]{rockafellar2009variational}.
We now define inference oracles as first order oracles for the structural hinge loss 
$f$ and its smoothed variants $f_{\mu \omega}$.
Note that these definitions are independent of the graphical structure.
However, as we shall see, the graphical structure plays a crucial role in the implementation of 
the inference oracles. 
\begin{definition} \label{defn:inf-oracles-all}
	Consider an augmented score function $\psi$, 
	a level of smoothing $\mu > 0$
	and the structural hinge loss $f(\wv) = \max_{\yv \in \mathcal{\mcY}} \psi(\yv;\wv)$. For a given $\wv \in \reals^d$,
	\begin{enumerate}[label={\upshape(\roman*)}, align=left, widest=iii, leftmargin=*]
		\item the {\em max oracle}
			returns $f(\wv)$ and $\vv \in \partial f(\wv)$.
		\item the {\em exp oracle} 
			returns $f_{-\mu H}(\wv)$ and $\grad f_{-\mu H}(\wv)$.
		\item the {\em top-$K$ oracle} 
			returns $f_{\mu, K}(\wv)$ and $\widetilde \grad f_{\mu, K}(\wv)$ as surrogates for 
			$f_{\mu \ell_2^2}(\wv)$ and $\grad f_{\mu \ell_2^2}(\wv)$ respectively.
	\end{enumerate}
\end{definition}
\noindent
Note that the exp oracle gets its name since it can be written as an expectation 
over all $\yv$, as revealed by the next lemma, 
which gives analytical expressions for the gradients returned by the oracles.
\begin{lemma} \label{lemma:smoothing:first-order-oracle}
	Consider the setting of Def.~\ref{defn:inf-oracles-all}. We have the following:
	\begin{enumerate}[label={\upshape(\roman*)}, align=left, widest=iii, leftmargin=*]
		\item \label{lem:foo:max}
			For any $\yv^* \in \argmax_{\yv \in \mathcal{\mcY}} \psi(\yv;\wv)$, we have that
			$\grad_\wv \psi(\yv^* ; \wv) \in \partial f(\wv)$. That is, the max oracle can be implemented
			by inference.
		\item 
			The output of the exp oracle satisfies
			 $\grad f_{-\mu H}(\wv) = \sum_{\yv \in \mcY} P_{\psi, \mu}(\yv ; \wv) \grad \psi(\yv ; \wv)$, 
			where 
			\[P_{\psi, \mu}(\yv ; \wv) 
				= \frac{
				\exp\left(\tfrac{1}{\mu}\psi(\yv ; \wv)\right)}
				{\sum_{\yv' \in \mcY }\exp\left(\tfrac{1}{\mu}\psi(\yv' ; \wv)\right)} \,.
			\]
			\label{lem:foo:exp}
		\item \label{lem:foo:l2}
			The output of the top-$K$ oracle satisfies 
			$
			\widetilde \grad f_{\mu, K}(\wv) = \sum_{i=1}^K u_{\psi, \mu, i}^*(\wv) \grad \psi(\yv_{(i)}; \wv) \,,
			$
			where $Y_K = \left\{\yv_{(1)}, \cdots, \yv_{(K)} \right\}$ is the set of $K$ largest scoring outputs
			satisfying
			\[
				\psi(\yv_{(1)} ; \wv) \ge \cdots \ge \psi(\yv_{(K)} ; \wv) \ge \max_{\yv \in \mcY \setminus Y_K} \psi(\yv ; \wv)\,,
			\]
			and $
			\uv^*_{\psi, \mu} = \operatorname{proj}_{\Delta^{K-1}} \left( \left[\psi(\yv_{(1)} ; \wv), \cdots, 
				\psi(\yv_{(K)} ; \wv) \right]\T \right)$.

	\end{enumerate}
\end{lemma}
\begin{proof}
	Part~\ref{lem:foo:exp} deals with the composition of differentiable 
	functions, and follows from the chain rule. Part~\ref{lem:foo:l2} follows from the definition in Eq.~\eqref{eq:smoothing:fmuK_defn}.
	The proof of Part~\ref{lem:foo:max} follows from the chain rule for Fr\'echet subdifferentials of compositions 
	\citep[Theorem 10.6]{rockafellar2009variational}
	together with the fact that by convexity and 
	Danskin's theorem \citep[Proposition B.25]{bertsekas1999nonlinear}, 
	the subdifferential of the max function is given by 
	$\partial h(\zv) = \conv \{ \ev_i \, |  \, i \in [m] \text{ such that } z_i = h(\zv) \}$.
\end{proof}

\begin{figure*}[!t]
   \centering
   \begin{subfigure}[b]{0.28\linewidth}
   \centering
       \adjincludegraphics[width=\textwidth,trim={0.11\width 0.2\height 0.11\width 0.2\height},clip]{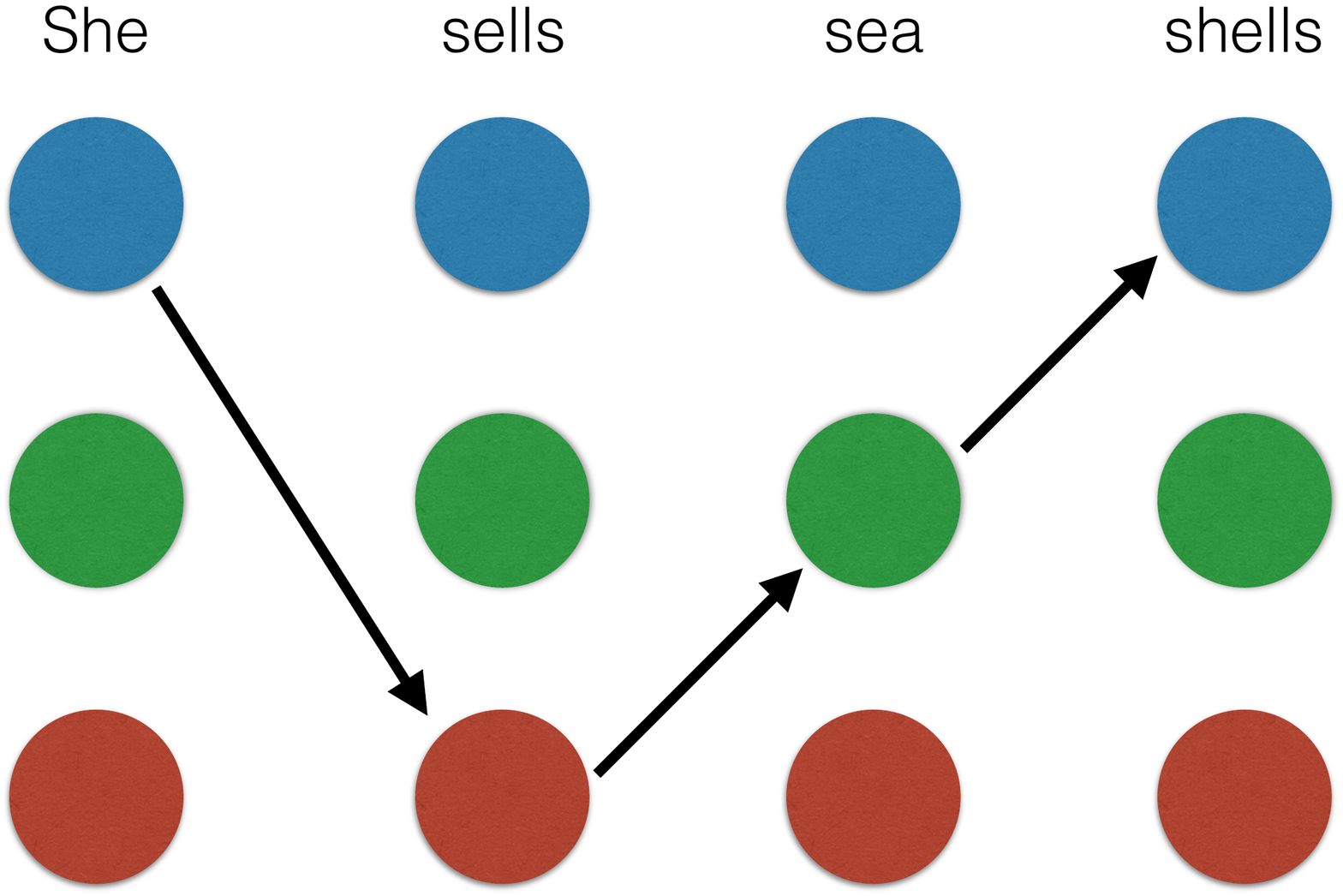}
       \caption{\small{Non-smooth.}}
       \label{subfig:viterbi:1:max}
   \end{subfigure} 
   \hspace{5mm}%
   \begin{subfigure}[b]{0.28\linewidth}
   \centering
       \adjincludegraphics[width=\textwidth,trim={0.11\width 0.2\height 0.11\width 0.2\height},clip]{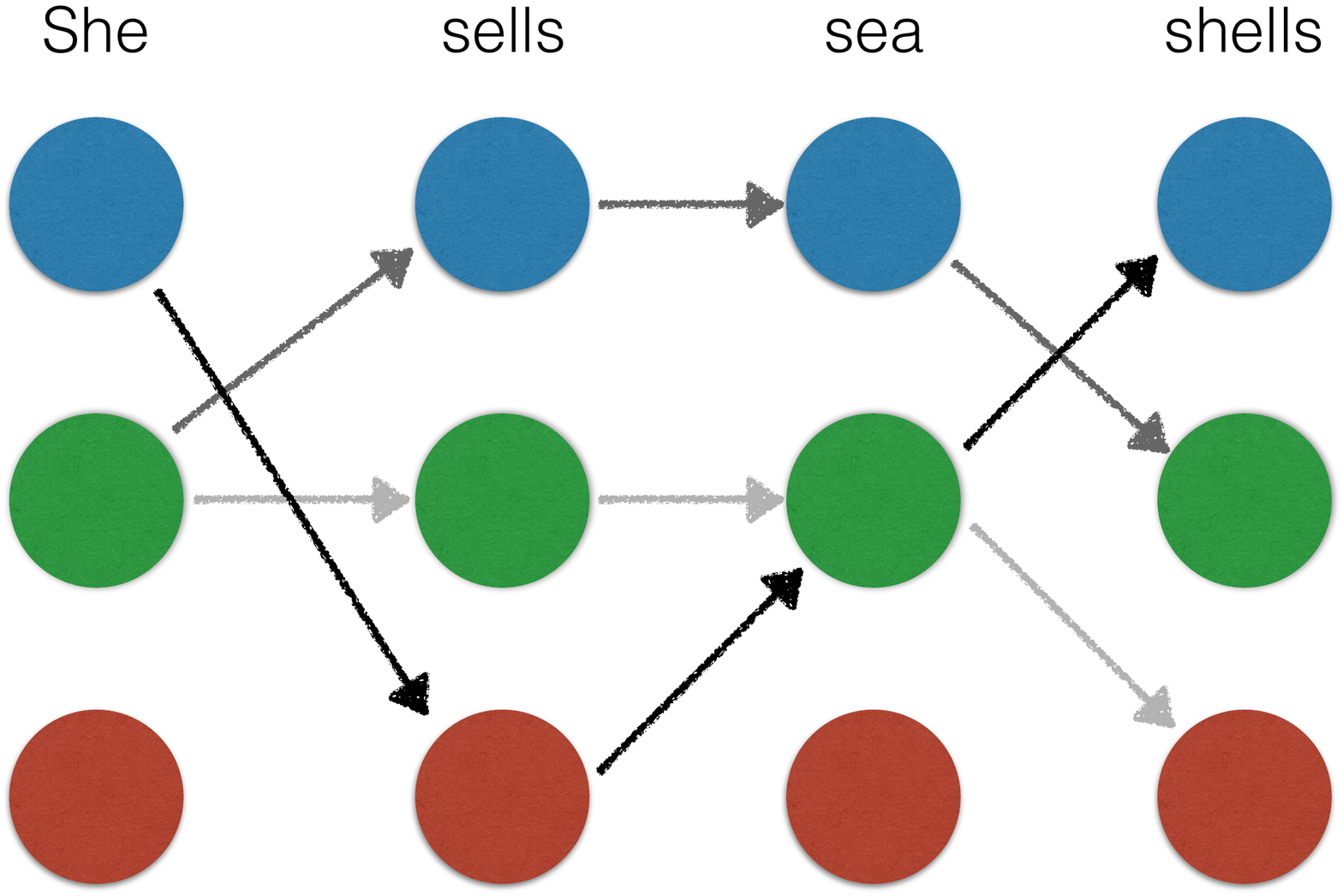}
       \caption{\small{$\ell_2^2$ smoothing.}}
       \label{subfig:viterbi:1:l2}
   \end{subfigure} 
   \hspace{5mm}%
   \begin{subfigure}[b]{0.28\linewidth}
   \centering
       \adjincludegraphics[width=\textwidth,trim={0.11\width 0.2\height 0.11\width 0.2\height},clip]{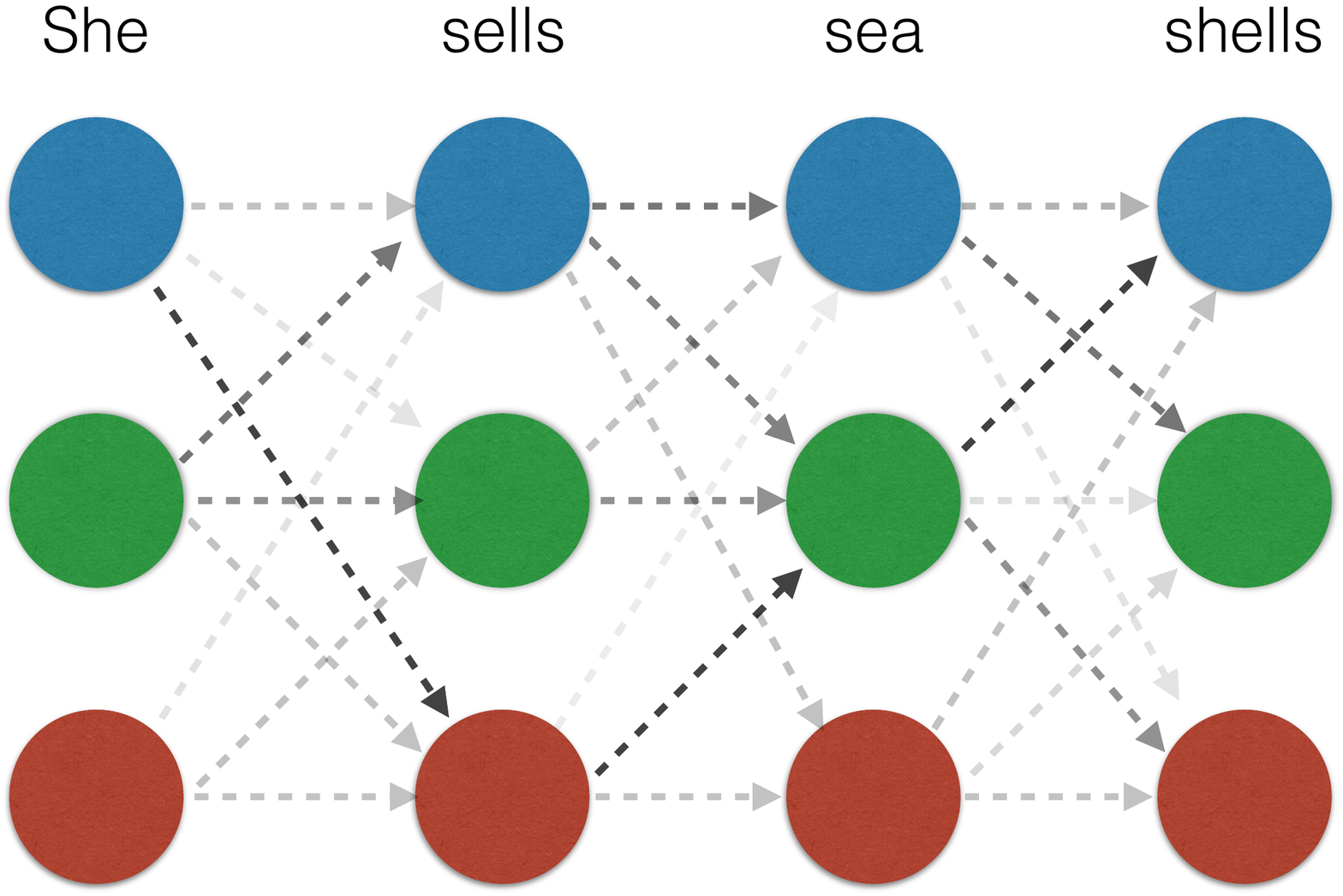}
       \caption{\small{Entropy smoothing.}}
       \label{subfig:viterbi:1:ent}
   \end{subfigure}
   \caption{\small{Viterbi trellis for a chain graph with $p=4$ nodes and 3 labels.
   }}
   \label{fig:viterbi:1}
\end{figure*}

\begin{example} \label{example:inf_oracles:viterbi_example_2}
Consider the task of sequence tagging from Example~\ref{example:inf_oracles:viterbi_example}.
The inference problem~\eqref{eq:pgm:inference} is a search over all 
$\abs{\mcY} = \abs{\mcY_{\mathrm{tag}}}^p$ label sequences. For chain graphs, this is equivalent 
to searching for the shortest path in the associated trellis, shown in Fig.~\ref{fig:viterbi:1}.
An efficient dynamic programming approach called the Viterbi algorithm \citep{viterbi1967error}
can solve this problem in space and time polynomial in $p$ and $\abs{\mcY_{\mathrm{tag}}}$.
The structural hinge loss is non-smooth because a small change in $\wv$ might lead to a radical change
in the best scoring path shown in Fig.~\ref{fig:viterbi:1}.

When smoothing $f$ with $\omega = \ell_2^2$,
the smoothed function $f_{\mu \ell_2^2}$ is given by a projection onto the simplex, 
which picks out some number $K_{\psi/\mu}$ of the highest scoring outputs $\yv \in \mcY$ or equivalently, 
$K_{\psi/\mu}$ shortest paths in the Viterbi trellis (Fig.~\ref{subfig:viterbi:1:l2}).
The top-$K$ oracle then uses the top-$K$ strategy to approximate $f_{\mu \ell_2^2}$ with $f_{\mu, K}$.

On the other hand, with entropy smoothing $\omega = -H$, we get the log-sum-exp function and 
its gradient is obtained by averaging
over paths with weights such that  
shorter paths have a larger weight (cf. Lemma~\ref{lemma:smoothing:first-order-oracle}\ref{lem:foo:exp}).
This is visualized in Fig.~\ref{subfig:viterbi:1:ent}. 
\end{example}

\subsubsection{Exp Oracles and Conditional Random Fields}
Recall that a {\em Conditional Random Field (CRF)}~\citep{lafferty2001conditional} 
with augmented score function $\psi$
and parameters $\wv \in \reals^d$ is a probabilistic model that assigns 
to output $\yv \in \mcY$ the probability
\begin{align} \label{eq:smoothing:crf:def}
\prob(\yv \mid \psi ; \wv) = \exp\left(\psi(\yv ; \wv)  - A_\psi(\wv) \right) \,,
\end{align}
where $A_\psi(\wv)$ is known as the log-partition function, 
a normalizer so that the probabilities sum to one.
Gradient-based maximum likelihood learning algorithms for CRFs require computation 
of the log-partition function $A_\psi(\wv)$ and its gradient $\grad A_\psi(\wv)$.
Next proposition relates the computational costs of the exp oracle and the log-partition function.

\begin{proposition} \label{prop:smoothing:exp-crf}
	The exp oracle for an augmented score function $\psi$ with parameters $\wv \in \reals^d$ is 
	equivalent in hardness to computing the log-partition function $A_\psi(\wv)$ 
	and its gradient $\grad A_\psi(\wv)$ for a conditional 
	random field with augmented score function $\psi$.
\end{proposition}
\begin{proof}
	Fix a smoothing parameter $\mu > 0$. 
	Consider a CRF with augmented score function 
	$\psi'(\yv ; \wv) = \mu\inv \psi(\yv ; \wv)$. Its log-partition function
	$A_{\psi'}(\wv)$ satisfies
	$\exp(A_{\psi'}(\wv)) = \sum_{\yv \in \mcY} \exp \left( \mu\inv \psi(\yv ; \wv)  \right)$.
	The claim now follows from the bijection $f_{- \mu H}(\wv) = \mu \, A_{\psi'}(\wv)$
	between $f_{-\mu H}$ and $A_{\psi'}$.
\end{proof}


\section{Implementation of Inference Oracles} \label{sec:smooth_oracle_impl}

We now turn to the concrete implementation of the inference oracles. 
This depends crucially on the structure of the graph $\mcG = (\mcV, \mcE)$.
%
If the graph $\mcG$ is a tree, 
then the inference oracles can be computed exactly with efficient procedures, 
as we shall see in in the Sec.~\ref{subsec:smooth_inference_trees}.
When the graph $\mcG$ is not a tree, we study
special cases when specific discrete structure can be exploited to efficiently implement some of the inference oracles
in Sec.~\ref{subsec:smooth_inference_loopy}. The results of this section are summarized in Table~\ref{tab:smoothing}.

\begin{table*}[t!]
   \caption{{Smooth inference oracles, algorithms and complexity. Here, $p$ is the size of each $\yv \in \mcY$.
         The time complexity is phrased in terms of the time complexity $\mcT$ of the 
         max oracle.
   }}
   \label{tab:smoothing}
   \centering
   \resizebox{0.8\textwidth}{!}{
      \begin{tabular}{c|cc|cc}
         
         \toprule
         
         \textbf{Max oracle} & 
         \multicolumn{2}{c|}{\textbf{Top-$K$ oracle}} 
         & 
         \multicolumn{2}{c}{\textbf{Exp oracle}} \\
         Algo &  Algo & Time & Algo & Time 
         \\ \hline\hline
         
         Max-product &
         \begin{tabular}{c} Top-$K$  \\ max-product  \end{tabular} & $\bigO(K \mcT \log K)$ &
         \begin{tabular}{c} Sum-Product  \end{tabular} & $\bigO(\mcT)$ \\ \hline
         
         \rule{0pt}{12pt}
         \begin{tabular}{c} Graph cut  \end{tabular} & 
         \begin{tabular}{c}  BMMF  \end{tabular}
         & $\bigO(pK \mcT)$ &
         \multicolumn{2}{c}{Intractable} \\[3pt] \hline
         
         \rule{0pt}{12pt}
         \begin{tabular}{c} Graph  matching  \end{tabular} & 
         \begin{tabular}{c}  BMMF  \end{tabular}
         & $\bigO(K \mcT)$ &
         \multicolumn{2}{c}{Intractable} \\[3pt] \hline
         
         \begin{tabular}{c} Branch and  \\ Bound search \end{tabular} & 
         \begin{tabular}{c}  Top-$K$ search   \end{tabular}
         & N/A &
         \multicolumn{2}{c}{Intractable} \\ \hline
         
         \bottomrule
   \end{tabular}
   } 
\end{table*}

Throughout this section, we fix an input-output pair $(\xv\pow{i}, \yv\pow{i})$ and consider
the augmented score function $\psi(\yv ; \wv) = 
\phi( \xv\pow{i} , \yv ; \wv) + \ell(\yv\pow{i}, \yv) - \phi( \xv\pow{i}, \yv\pow{i} ; \wv)$ it defines, where the index of the sample is dropped by convenience.  
From \eqref{eq:setting:score:decomp} and the decomposability of the loss,
we get that $\psi$ decomposes along nodes $\mcV$ and edges $\mcE$ of $\mcG$ as:
\begin{align} \label{eq:smoothing:aug_score_decomp}
   \psi(\yv ; \wv) = \sum_{v \in \mcV} \psi_v(y_v; \wv)
      + \sum_{(v,v') \in \mcE} \psi_{v,v'}(y_v, y_{v'} ; \wv) \,.
\end{align}
When $\wv$ is clear from the context, we denote $\psi(\cdot \, ; \wv)$ by $\psi$. Likewise for $\psi_v$ and $\psi_{v, v'}$.

\subsection{Inference Oracles in Trees} \label{subsec:smooth_inference_trees}
We first consider algorithms implementing the inference algorithms in trees and examine their computational complexity.

\subsubsection{Implementation of Inference Oracles}

\paragraph{Max Oracle}
In tree structured graphical models, the inference problem~\eqref{eq:pgm:inference}, and thus the max oracle
(cf. Lemma~\ref{lemma:smoothing:first-order-oracle}\ref{lem:foo:max})
can always be solved exactly in polynomial time by the max-product algorithm~\citep{pearl1988probabilistic},
which uses the technique of dynamic programming~\citep{bellman1957dynamic}. 
The Viterbi algorithm (Algo.~\ref{algo:dp:max:chain}) for chain graphs from Example~\ref{example:inf_oracles:viterbi_example_2}
is a special case. See Algo.~\ref{algo:dp:supp} in Appendix~\ref{sec:a:dp} for
the max-product algorithm in full generality.


\paragraph{Top-$K$ Oracle}
The top-$K$ oracle uses a generalization of the max-product algorithm that we name top-$K$ max-product algorithm.
Following the work of \citet{seroussi1994algorithm}, it keeps track of the $K$-best intermediate structures while 
the max-product algorithm just tracks the single best intermediate structure.
Formally, the $k$th largest element from a discrete set $S$ is defined as
\begin{align*}
\maxK{k}{x \in S} f(x) = 
\begin{cases}
\text{$k$th largest element of $\{f(y)\, |\, y \in S\} $} & k \le |S| \\
-\infty, & k > |S| \,.
\end{cases}
\end{align*}
We present the algorithm in the simple case of chain structured graphical models in
Algo.~\ref{algo:dp:topK:chain}.
The top-$K$ max-product algorithm for general trees is given in
Algo.~\ref{algo:dp:topK:main} in Appendix~\ref{sec:a:dp}.
Note that it requires $\widetilde\bigO(K)$
times the time and space of the max oracle. 

\paragraph{Exp oracle}
The relationship of the exp oracle with CRFs (Prop.~\ref{prop:smoothing:exp-crf})
leads directly to 
Algo.~\ref{algo:dp:supp_exp}, which is 
based on marginal computations from the sum-product algorithm.

\begin{algorithm}[tb]
   \caption{Max-product (Viterbi) algorithm for chain graphs
   }
   \label{algo:dp:max:chain}
\begin{algorithmic}[1]
   \STATE {\bfseries Input:} Augmented score function $\psi(\cdot, \cdot; \wv)$ defined on a chain graph $\mcG$.
   \STATE Set $\pi_1(y_1) \leftarrow \psi_1(y_1)$ for all $y_1 \in \mcY_1$.
   
   \FOR{$v= 2, \cdots p$}
      \STATE For all $y_v \in \mcY_v$, set 
      \begin{align} \label{eq:dp:viterbi:update}
         \pi_{v}(y_v) \leftarrow \psi_v(y_v) + \max_{y_{v-1} \in \mcY_{v-1}} 
            \left\{ \pi_{v-1}(y_{v-1}) + \psi_{v, v-1}(y_v, y_{v-1})  \right\} \,.
      \end{align}
      \STATE Assign to $\delta_v(y_v)$ the $y_{v-1}$ 
         that attains the $\max$ above for each $y_v \in \mcY_v$.
   \ENDFOR
   \STATE Set $\psi^* \leftarrow \max_{y_p \in \mcY_p} \pi_p(y_p)$
      and store the maximizing assignments of $y_p$ in $y_p^*$.

   \FOR{$v= p-1, \cdots, 1$}
      \STATE Set $y_v^* \leftarrow \delta_{v+1}( y_{v+1})$.
   \ENDFOR
   \RETURN $\psi^*, \yv^*:=(y_1^*, \cdots, y_p^*) $.
\end{algorithmic}
\end{algorithm}

\begin{algorithm}[tb]
   \caption{Top-$K$ max-product (top-$K$ Viterbi) algorithm for chain graphs}
   \label{algo:dp:topK:chain}
\begin{algorithmic}[1]
   \STATE {\bfseries Input:} Augmented score function $\psi(\cdot, \cdot ; \wv)$ defined on chain graph $\mcG$,
   integer $K>0$.
   \STATE For $k=1,\cdots, K$, set $\pi_1\pow{k}(y_1) \leftarrow \psi_1(y_1)$ if $k=1$ and $-\infty$ otherwise for all $y_1 \in \mcY_1$.
   
   \FOR{$v= 2, \cdots p$ and $k=1,\cdots, K$}
      \STATE For all $y_v \in \mcY_v$, set 
      \begin{align} \label{eq:dp:viterbi:topk:update}
         \pi_{v}\pow{k}(y_v) \leftarrow \psi_v(y_v) + \maxK{k}{y_{v-1} \in \mcY_{v-1}, \ell \in [K]} 
            \left\{ \pi_{v-1}\pow{\ell}(y_{v-1}) + \psi_{v, v-1}(y_v, y_{v-1})  \right\} \,.
      \end{align}
      \STATE Assign to $\delta_v\pow{k}(y_v), \kappa_v\pow{k}(y_v)$ the $y_{v-1}, \ell$ 
         that attain the $\max\pow{k}$ above for each $y_v \in \mcY_v$.
   \ENDFOR
   \STATE For $k=1,\cdots, K$, set $\psi\pow{k} \leftarrow \max\pow{k}_{y_p \in \mcY_p, k \in [K]} \pi_p\pow{k}(y_p)$
      and store  in $y_p\pow{k}, \ell\pow{k}$ respectively the maximizing assignments of $y_p , k$.

   \FOR{$v= p-1, \cdots 1$ and $k=1,\cdots, K$}
      \STATE Set $y_v\pow{k} \leftarrow \delta_{v+1}\pow{\ell\pow{k}} \big( y_{v+1}\pow{k} \big)$ and 
         $\ell\pow{k} \leftarrow \kappa_{v+1}\pow{\ell\pow{k}} \big( y_{v+1}\pow{k} \big)$.
   \ENDFOR
   \RETURN $\left\{ \psi\pow{k}, \yv\pow{k}:=(y_1\pow{k}, \cdots, y_p\pow{k}) \right\}_{k=1}^K$.
\end{algorithmic}
\end{algorithm}

\begin{algorithm}[tb]
   \caption{Entropy smoothed max-product algorithm}
   \label{algo:dp:supp_exp}
\begin{algorithmic}[1]
   \STATE {\bfseries Input:} Augmented score function $\psi(\cdot, \cdot ; \wv)$ defined on 
      tree structured graph $\mcG$, 
      $\mu > 0$.
   \STATE Compute the log-partition function and marginals using the sum-product algorithm 
      (Algo.~\ref{algo:dp:supp_sum-prod} in Appendix~\ref{sec:a:dp})
      \[
         A_{\psi/\mu}, \{P_v \text{ for } v \in \mcV\}, \{ P_{v, v'} \text{ for } (v, v') \in \mcE \}
            \leftarrow \textsc{SumProduct}\left( \tfrac{1}{\mu} \psi(\cdot \, ; \wv), \mcG \right) \,.
      \]

   \STATE Set $f_{-\mu H}(\wv) \leftarrow \mu A_{\psi /\mu}$ and 
   \[
      \grad f_{-\mu H}(\wv) \leftarrow \sum_{v \in \mcV} \sum_{y_v \in \mcY_v} P_v(y_v) \grad \psi_v(y_v ; \wv) 
         + \sum_{(v, v') \in \mcE} \sum_{y_v \in \mcY_v} \sum_{y_{v'} \in \mcY_{v'}} P_{v,v'}(y_v, y_{v'})\grad \psi_{v, v'}(y_v ; \wv) \,.
   \]
   \label{line:algo:dp:exp:gradient}
   \RETURN $f_{-\mu H}(\wv), \grad f_{-\mu H}(\wv)$.
\end{algorithmic}
\end{algorithm}

\begin{remark}
We note that clique trees allow the generalization of the 
algorithms of this section to general graphs with cycles.
However, the construction of a clique tree requires time and space 
exponential in the {\em treewidth} of the graph.
\end{remark}

\begin{example}
Consider the task of sequence tagging from Example~\ref{example:inf_oracles:viterbi_example}.
The Viterbi algorithm (Algo.~\ref{algo:dp:max:chain}) maintains 
a table $\pi_v(y_v)$, which stores the best length-$v$ prefix ending in label $y_v$.
One the other hand, the top-$K$ Viterbi algorithm (Algo.~\ref{algo:dp:topK:chain})
must store in $\pi_v\pow{k}(y_v)$ the score of $k$th best length-$v$ prefix that ends in $y_v$ for each $k \in [K]$.
In the vanilla Viterbi algorithm, the entry $\pi_v(y_v)$ is updated by looking the previous column
$\pi_{v-1}$
following~\eqref{eq:dp:viterbi:update}.
Compare this to update \eqref{eq:dp:viterbi:topk:update}
of the top-$K$ Viterbi algorithm.
In this case, the exp oracle is implemented by the forward-backward algorithm, a specialization of the 
sum-product algorithm to chain graphs.
\end{example}

\subsubsection{Complexity of Inference Oracles}
The next proposition presents the correctness guarantee and complexity of 
each of the aforementioned  algorithms. Its proof has been placed in Appendix~\ref{sec:a:dp}.

\begin{proposition} \label{prop:dp:main}
      Consider as inputs an augmented score function $\psi(\cdot, \cdot ; \wv)$ defined on a tree structured graph $\mcG$, 
      an integer $K>0$ and a smoothing parameter $\mu > 0$.
      \begin{enumerate}[label={\upshape(\roman*)}, align=left, widest=iii, leftmargin=*]
         \item The output $(\psi^*, \yv^*)$  of the max-product algorithm 
            (Algo.~\ref{algo:dp:max:chain} for the special case when $\mcG$ is chain structured 
            Algo.~\ref{algo:dp:supp} from Appendix~\ref{sec:a:dp} in general) satisfies 
            $\psi^* = \psi(\yv^* ; \wv) = \max_{\yv \in \mcY} \psi(\yv ; \wv)$.
            Thus, the pair $\big(\psi^*, \grad \psi(\yv^* ; \wv)\big)$  is a correct implementation of the max oracle.
            It requires time $\bigO(p \max_{v\in\mcV} \abs{\mcY_v}^2)$
            and space $\bigO(p \max_{v\in\mcV} \abs{\mcY_v})$.
         \item The output $\{ \psi\pow{k}, \yv\pow{k} \}_{k=1}^K$
            of the top-$K$ max-product algorithm 
            (Algo.~\ref{algo:dp:topK:chain} for the special case when $\mcG$ is chain structured 
            or Algo.~\ref{algo:dp:topK:main} from Appendix~\ref{sec:a:dp} in general)
            satisfies $\psi\pow{k} = \psi(\yv\pow{k}) = \max\pow{k}_{\yv \in \mcY} \psi(\yv)$.
            Thus, the top-$K$ max-product algorithm followed by a projection onto the simplex 
            (Algo.~\ref{algo:smoothing:top_K_oracle} in Appendix~\ref{sec:a:smoothing}) 
            is a correct implementation of the top-$K$ oracle.
            It requires time $\bigO(pK\log K \max_{v\in\mcV} \abs{\mcY_v}^2)$
            and space $\bigO(p K \max_{v\in\mcV} \abs{\mcY_v})$.
         \label{prop:dp:main:part:topk}
         \item Algo.~\ref{algo:dp:supp_exp}
            returns $\big(f_{-\mu H}(\wv), \grad f_{-\mu H}(\wv)\big)$.
            Thus, Algo.~\ref{algo:dp:supp_exp} is a correct implementation of the exp oracle.
            It requires time $\bigO(p \max_{v\in\mcV} \abs{\mcY_v}^2)$
            and space $\bigO(p \max_{v\in\mcV} \abs{\mcY_v})$.
      \end{enumerate}
\end{proposition}

\subsection{Inference Oracles in Loopy Graphs} \label{subsec:smooth_inference_loopy}
For general loopy graphs with high tree-width, 
the inference problem \eqref{eq:pgm:inference} is NP-hard \citep{cooper1990computational}.
In particular cases, graph cut, matching or search algorithms  
can be used for exact inference in dense loopy graphs, and therefore, 
to implement the max oracle as well (cf. Lemma~\ref{lemma:smoothing:first-order-oracle}\ref{lem:foo:max}).
In each of these cases, 
we find that the top-$K$ oracle can be implemented, but the exp oracle is intractable.
Appendix~\ref{sec:a:smooth:loopy} contains a review of the algorithms and guarantees referenced 
in this section.

\subsubsection{Inference Oracles using Max-Marginals}
We now define a {\em max-marginal}, 
which is a constrained maximum of the augmented score $\psi$.
\begin{definition}
The max-marginal of $\psi$ relative to a variable $y_v$ is defined, 
for $j \in \mcY_v$ as 
\begin{align}
   \psi_{v; j}(\wv) := \max_{\substack{\yv \in \mcY \,: \, y_v = j}} \psi(\yv ; \wv)\, .
\end{align}
\end{definition}
\noindent
In cases where exact inference is tractable using graph cut or matching algorithms, 
it is possible to extract max-marginals as well.
This, as we shall see next, allows the implementation of the max and top-$K$ oracles.

When the augmented score function $\psi$ is {\em unambiguous}, i.e., 
no two distinct $\yv_1, \yv_2 \in \mcY$ have the same augmented score,
the output $\yv^*(\wv)$ is unique can be decoded from the max-marginals as
(see \citet{pearl1988probabilistic,dawid1992applications} or  Thm.~\ref{thm:a:loopy:decoding}
in Appendix~\ref{sec:a:smooth:loopy})
\begin{align} \label{eq:max-marg:defn}
   y_v^*(\wv) = \argmax_{j \in \mcY_v} \psi_{v ; j}(\wv) \,.
\end{align}

If one has access to an algorithm $\mcM$ that can compute max-marginals, 
the top-$K$ oracle is also easily implemented via the {\em Best Max-Marginal First (BMMF)} 
algorithm of \citet{yanover2004finding}.
This algorithm requires computations of $2K$ sets of max-marginals, 
where a {\em set} of max-marginals refers to max-marginals for all $y_v$ in $\yv$.
Therefore, 
the BMMF algorithm followed by a projection onto the simplex 
(Algo.~\ref{algo:smoothing:top_K_oracle} in Appendix~\ref{sec:a:smoothing}) 
is a correct implementation of the top-$K$ oracle at a computational cost of 
$2K$ sets of max-marginals.
The BMMF algorithm and its guarantee are recalled in Appendix~\ref{sec:a:bmmf} for completeness.

\paragraph{Graph Cut and Matching Inference}
\citet{kolmogorov2004energy} showed that submodular energy functions \citep{lovasz1983submodular} 
over binary variables can be efficiently minimized exactly via a minimum cut algorithm.
For a class of alignment problems, e.g., \citet{taskar2005discriminative}, 
inference amounts to finding the best bipartite matching.
In both these cases, max-marginals can be computed exactly and efficiently
by combinatorial algorithms.
This gives us a way to implement the max and top-$K$ oracles.
However, in both settings, 
computing the log-partition function $A_\psi(\wv)$ of a CRF with score $\psi$
is known to be \#P-complete~\citep{jerrum1993polynomial}.
Prop.~\ref{prop:smoothing:exp-crf} immediately extends this result to the exp oracle. 
This discussion is summarized by the following proposition, whose proof is provided in Appendix~\ref{sec:a:proof-prop}.

\begin{proposition} \label{prop:smoothing:max-marg:all}
   Consider as inputs an augmented score function $\psi(\cdot, \cdot ; \wv)$, 
   an integer $K>0$ and a smoothing parameter $\mu > 0$.
   Further, suppose that $\psi$ is unambiguous, that is, 
   $\psi(\yv' ; \wv) \neq \psi(\yv'' ;\wv)$ for all distinct $\yv', \yv'' \in \mcY$.
   Consider one of the two settings:
   \begin{enumerate}[label={\upshape(\Alph*)}, align=left, leftmargin=*]
   \item the output space $\mcY_v = \{0,1\}$ for each $v \in \mcV$, and the function
      $-\psi$ is submodular (see Appendix~\ref{sec:a:graph_cuts} and, in particular, \eqref{eq:top_k_map:submodular}
      for the precise definition), or, 
      \label{part:prop:max-marg:cuts}
   \item the augmented score corresponds to an alignment task where the 
      inference problem~\eqref{eq:pgm:inference} corresponds to a 
      maximum weight bipartite matching (see Appendix~\ref{sec:a:graph_matchings} for a precise definition).
      \label{part:prop:max-marg:matching}
   \end{enumerate}
   In these cases, we have the following:
   \begin{enumerate}[label={\upshape(\roman*)}, align=left, widest=iii, leftmargin=*]
      \item The max oracle can be implemented at a 
         computational complexity of $\bigO(p)$ minimum cut computations in Case~\ref{part:prop:max-marg:cuts}, 
         and in time $\bigO(p^3)$ in Case~\ref{part:prop:max-marg:matching}.
      \item The top-$K$ oracle can be implemented at a 
         computational complexity of $\bigO(pK)$ minimum cut computations in Case~\ref{part:prop:max-marg:cuts}, 
         and in time $\bigO(p^3K)$ in Case~\ref{part:prop:max-marg:matching}.
      \item The exp oracle is \#P-complete in both cases.
   \end{enumerate}
\end{proposition}
Prop.~\ref{prop:smoothing:max-marg:all} is loose in that the max oracle can be implemented with just one 
minimum cut computation instead of $p$ in in Case~\ref{part:prop:max-marg:cuts}~\citep{kolmogorov2004energy}.
%

\subsubsection{Branch and Bound Search}
Max oracles implemented via search algorithms can often be extended to implement the top-$K$ oracle.
We restrict our attention to best-first branch and bound search such as the
celebrated Efficient Subwindow Search \citep{lampert2008beyond}.   

Branch and bound methods partition the search space into disjoint subsets, 
while keeping an upper bound 
$\widehat \psi: \mcX \times 2^{\mcY} \to \reals$, 
on the maximal augmented score for each of the subsets $\widehat \mcY \subseteq \mcY$. 
Using a best-first strategy, promising parts of the search space are explored first.
Parts of the search space whose upper bound indicates that they cannot contain the maximum
do not have to be examined further.

The top-$K$ oracle is implemented by simply 
continuing the search procedure until $K$ outputs have been produced - see
Algo.~\ref{algo:top_k:bb} in Appendix~\ref{sec:a:bb_search}.
Both the max oracle and the top-$K$ oracle 
can degenerate to an exhaustive search in the worst case, so we do not have sharp running time 
guarantees. However, we have the following correctness guarantee.

\begin{proposition} \label{prop:smoothing:bb-search}
   Consider an augmented score function $\psi(\cdot, \cdot ; \wv)$, 
   an integer $K > 0$ and a smoothing parameter $\mu > 0$.
   Suppose the upper bound function $\widehat \psi(\cdot, \cdot ; \wv): \mcX \times 2^{\mcY} \to \reals$
   satisfies the following properties:
   \begin{enumerate}[label=(\alph*), align=left, widest=a, leftmargin=*]
      \item $\widehat \psi(\widehat \mcY ; \wv)$ is finite for every $\widehat \mcY \subseteq \mcY$,
      \item $\widehat \psi(\widehat \mcY ; \wv) \ge \max_{\yv \in \widehat \mcY} \psi(\yv ; \wv)$
         for all $\widehat \mcY \subseteq \mcY$, and,
      \item $\widehat \psi(\{\yv\} ; \wv) = \psi(\yv ; \wv)$ for every $\yv \in \mcY$.
   \end{enumerate}
   Then, we have the following:
   \begin{enumerate}[label={\upshape(\roman*)}, align=left, widest=ii, leftmargin=*]
      \item Algo.~\ref{algo:top_k:bb} with $K=1$ is a correct implementation of the max oracle.
      \item Algo.~\ref{algo:top_k:bb} followed by a projection onto the simplex 
         (Algo.~\ref{algo:smoothing:top_K_oracle} in Appendix~\ref{sec:a:smoothing}) is a correct implementation of the top-$K$ oracle.
   \end{enumerate}
\end{proposition}
\noindent
See Appendix~\ref{sec:a:bb_search} for a proof.
The discrete structure that allows inference via branch and bound search cannot be 
leveraged to implement the exp oracle.

\section{The Casimir Algorithm} \label{sec:cvx_opt}
We come back to the optimization problem~\eqref{eq:c:main:prob} with $f\pow{i}$ defined in \eqref{eq:pgm:struc_hinge_vec}.
We assume in this section that the mappings $\gv\pow{i}$ 
defined in \eqref{eq:mapping_def} are affine. Problem~\eqref{eq:c:main:prob} now reads
\begin{equation} \label{eq:cvx_pb_finite_sum}
\min_{\wv\in \reals^d} \left[ F(\wv) := \frac{1}{n}\sum_{i=1}^n h(\Am\pow{i} \wv + {\bm b}\pow{i}) + \frac{\lambda}{2} \normasq{2}{\wv} \right] \,.
\end{equation}
For a single input ($n=1$), the problem reads
\begin{equation}\label{eq:cvx_pb}
\min_{\wv\in \reals^d}  h(\Am\wv + {\bm b}) + \frac{\lambda}{2} \normasq{2}{\wv}.
\end{equation}
where $h$ is a simple non-smooth convex function and $\lambda \geq 0$. 
\citet{nesterov2005smooth, nesterov2005excessive} first analyzed such setting: while the problem suffers from its non-smoothness, 
fast methods can be developed by considering smooth approximations of the objectives.
We combine this idea with the Catalyst acceleration scheme \citep{lin2017catalyst} to accelerate a linearly convergent 
smooth optimization algorithm resulting in a scheme called {\em \casimir}.

\subsection{\casimir: Catalyst with Smoothing}
The Catalyst~\citep{lin2017catalyst} approach minimizes regularized objectives centered around the current iterate. 
The algorithm proceeds by computing approximate proximal point steps instead of the classical (sub)-gradient steps.
A proximal point step from a point $\wv$ with step-size $\kappa^{-1}$ is defined as the minimizer of 
\begin{equation}\label{eq:prox_point}
\min_{\zv \in \reals^m} F(\zv) + \frac{\kappa}{2}\normasq{2}{\zv-\wv},
\end{equation}
which can also be seen as a gradient step on the Moreau envelope of $F$ - see \citet{lin2017catalyst} for a detailed discussion. 
While solving the subproblem~\eqref{eq:prox_point} might be as hard as the original problem we only 
require an approximate solution returned by a given optimization method $\mathcal{M}$.
The Catalyst approach is then an inexact accelerated proximal point algorithm 
that carefully mixes approximate proximal point steps with the extrapolation scheme of \citet{nesterov1983method}. 
The \casimir{} scheme extends this approach to non-smooth optimization.

For the overall method to be efficient, subproblems~\eqref{eq:prox_point} must have a low complexity.
That is, there must exist an optimization algorithm $\mathcal{M}$ that solves them linearly. 
For the \casimir{} approach to be able to handle non-smooth objectives, it means that we need not only to regularize the objective 
but also to smooth it. To this end we define
\[
F_{\mu \omega}(\wv) := \frac{1}{n}\sum_{i=1}^n h_{\mu \omega}(\Am\pow{i} \wv + {\bm b}\pow{i}) + \frac{\lambda}{2} \normasq{2}{\wv}
\]
as a smooth approximation of the objective $F$, and,
\[
F_{\mu \omega, \kappa}(\wv; \zv) := \frac{1}{n}\sum_{i=1}^n h_{\mu \omega}(\Am\pow{i} \wv + {\bm b}\pow{i}) + \frac{\lambda}{2} \normasq{2}{\wv} + \frac{\kappa}{2}\normasq{2}{\wv-\zv}
\]
a smooth and regularized approximation of the objective centered around a given point $\zv \in \reals^d$.
While the original Catalyst algorithm considered a fixed regularization term $\kappa$,
we vary $\kappa$ and $\mu$ along the iterations. 
This enables us to get adaptive smoothing strategies. 

The overall method is presented in Algo.~\ref{algo:catalyst}.   We first analyze in Sec.~\ref{sec:catalyst:analysis} its complexity 
for a generic linearly convergent algorithm $\mcM$.
Thereafter, in Sec.~\ref{sec:catalyst:total_compl}, we compute the total complexity with SVRG~\citep{johnson2013accelerating}
as $\mcM$.
Before that, we specify two practical aspects of the implementation: a proper stopping criterion~\eqref{eq:stopping_criterion} 
and a good initialization of subproblems (Line~\ref{line:algo:c:prox_point}).

\paragraph{Stopping Criterion}
Following~\citet{lin2017catalyst}, we 
solve subproblem $k$ in Line~\ref{line:algo:c:prox_point} to a degree of relative accuracy specified by 
$\delta_k \in [0, 1)$. 
In view of the $(\lambda+\kappa_k)$-strong convexity of $F_{\mu_k\omega, \kappa_k}(\cdot\,; \zv_{k-1})$, the functional gap can be controlled by the norm of the gradient, precisely 
it can be seen that $\normasq{2}{\grad F_{\mu_k\omega, \kappa_k}(\widehat\wv; \zv_{k-1})} 
\le (\lambda+\kappa_k)\delta_k \kappa_k \normasq{2}{\widehat \wv - \zv_{k-1}}$
is a sufficient condition for 
the stopping criterion \eqref{eq:stopping_criterion}.

A practical alternate stopping criterion proposed by \citet{lin2017catalyst} is to fix an iteration budget $T_{\mathrm{budget}}$ 
and run the inner solver $\mcM$ for exactly $T_{\mathrm{budget}}$ steps.
We do not have a theoretical analysis for this scheme but find that it works well in experiments.

\paragraph{Warm Start of Subproblems}
Rate of convergence of first order optimization algorithms depends on the initialization
and we must warm start $\mcM$ at an appropriate initial point in order to obtain 
the best convergence of subproblem~\eqref{eq:prox_point_algo} in Line~\ref{line:algo:c:prox_point} of Algo.~\ref{algo:catalyst}.
We advocate the use of the prox center $\zv_{k-1}$ in iteration $k$ as the warm start strategy.
We also experiment with other warm start strategies in Section~\ref{sec:expt}.

\begin{algorithm}[tb]
	\caption{The \casimir{} algorithm}
	\label{algo:catalyst}
	\begin{algorithmic}[1]
		\STATE {\bfseries Input:} Smoothable objective $F$ of the form \eqref{eq:cvx_pb} with $h$ simple,
		smoothing function $\omega$,
		linearly convergent algorithm $\mcM$,
		non-negative and non-increasing sequence of smoothing parameters $(\mu_k)_{k \ge 1}$,
		positive and non-decreasing sequence of regularization parameters $(\kappa_k)_{k\ge1}$, 
		non-negative sequence of relative target accuracies $(\delta_k)_{k\ge 1}$ and,
		initial point $\wv_0$,  $\alpha_0 \in (0, 1)$, 
		time horizon $K$.
		\STATE {\bfseries Initialize:} $\zv_0 = \wv_0$.
		\FOR{$k=1$ \TO $K$}
		\STATE Using $\mcM$ with $\zv_{k-1}$ as the starting point, find  \label{line:algo:c:prox_point}
		$\wv_{k} \approx \argmin_{\wv\in\reals^d} F_{\mu_k \omega, \kappa_k}(\wv; \zv_{k-1})$ where
		\begin{equation}\label{eq:prox_point_algo}
		F_{\mu_k \omega, \kappa_k}(\wv; \zv_{k-1}) := \frac{1}{n}\sum_{i=1}^n h_{\mu_k \omega}(\Am\pow{i} \wv + {\bm b}\pow{i}) + \frac{\lambda}{2} \normasq{2}{\wv} + \frac{\kappa_k}{2}\normasq{2}{\wv- \zv_{k-1}} 
		\end{equation}
		such that 
		\begin{align}\label{eq:stopping_criterion}
		F_{\mu_k\omega, \kappa_k}(\wv_k;\zv_{k-1}) - \min_\wv  F_{\mu_k\omega, \kappa_k}(\wv;\zv_{k-1})\leq \tfrac{\delta_k\kappa_k}{2} \normasq{2}{\wv_k - \zv_{k-1}}
		\end{align}
		 \STATE Solve for $\alpha_k \geq 0$
		 		\begin{align} \label{eq:c:update_alpha}
		 			\alpha_k^2 (\kappa_{k+1} + \lambda) = (1-\alpha_k) \alpha_{k-1}^2 (\kappa_k + \lambda) + \alpha_k \lambda.
		 		\end{align}
		 \STATE Set 
		 \begin{align}	\label{eq:c:update_support}
		 	\zv_k = \wv_k + \beta_k (\wv_k - \wv_{k-1}),
		 \end{align}
		  where
		 \begin{align} \label{eq:c:update_beta}
			 \beta_k = \frac{ \alpha_{k-1}(1-\alpha_{k-1}) (\kappa_k + \lambda) }
			 {  \alpha_{k-1}^2 (\kappa_k + \lambda) + \alpha_k(\kappa_{k+1} + \lambda) }.
		 \end{align}
		\ENDFOR
		\RETURN $\wv_K$.
	\end{algorithmic}
\end{algorithm}

\subsection{Convergence Analysis of Casimir} \label{sec:catalyst:analysis}
We first state the outer loop complexity results of Algo.~\ref{algo:catalyst} for any generic
linearly convergent algorithm $\mcM$ in Sec.~\ref{sec:catalyst:outer_compl}, prove it in Sec.~\ref{subsec:c:proof}. 
Then, we consider the complexity of each inner optimization problem~\eqref{eq:prox_point_algo} in Sec.~\ref{sec:catalyst:inner_compl}
based on properties of $\mcM$.

\subsubsection{Outer Loop Complexity Results} \label{sec:catalyst:outer_compl}
The following theorem states the convergence of the algorithm for general choice of parameters, where we denote $\wv^* \in \argmin_{\wv\in\reals^d} F(\wv)$ and $F^* = F(\wv^*)$.
\begin{theorem} \label{thm:catalyst:outer}
	Consider Problem~\eqref{eq:cvx_pb_finite_sum}.
	Suppose 
	$\delta_k \in [0, 1)$ for all $k \ge 1$, the sequence $(\mu_k)_{k\ge 1}$ is non-negative and non-increasing, 
	and the sequence $(\kappa_k)_{k \ge 1}$ is strictly positive and non-decreasing.
	Further, suppose the smoothing function $\omega: \dom h^* \to \reals$ satisfies 
	$-D_\omega \le \omega(\uv) \le 0$ for all $\uv \in \dom h^*$ and that 
	$\alpha_0^2 \ge \lambda / (\lambda + \kappa_1)$.
	Then, the sequence $(\alpha_k)_{k \ge 0}$ generated by Algo.~\ref{algo:catalyst}
	satisfies $0 < \alpha_k \le \alpha_{k-1} < 1$ for all $k \ge 1$.
	Furthermore, the sequence $(\wv_{k})_{k \ge 0}$
	of iterates generated by Algo.~\ref{algo:catalyst} satisfies
	\begin{align} \label{thm:c:main:main}
	F(\wv_k) - F^* \le 
	\frac{\mcA_0^{k-1}}{\mcB_1^k} \Delta_0 + \mu_k D_\omega 
	+ \sum_{j=1}^k \frac{\mcA_j^{k-1}}{\mcB_j^k} \left( \mu_{j-1} - (1-\delta_j)\mu_j \right) D_\omega
	\,,
	\end{align}
	where $\mcA_i^j := \prod_{r=i}^j (1-\alpha_r)$, 
	$\mcB_i^j  := \prod_{r=i}^j (1-\delta_r)$, 
	$\Delta_0 := F(\wv_0) - F^* + \frac{(\kappa_1 + \lambda) \alpha_0^2 - \lambda \alpha_0 } {2(1 - \alpha_0)} \normasq{2}{\wv_0 - \wv^*}$ and
	$\mu_0 := 2\mu_1$.
\end{theorem}
Before giving its proof, we present various parameters strategies as corollaries.
Table~\ref{tab:catalyst_corollaries_summary} summarizes the parameter settings and the rates obtained for each setting.
Overall, the target accuracies $\delta_k$ are chosen such that $\mcB_j^k$ is a constant and 
the parameters $\mu_k$ 
and $\kappa_k$ are then carefully chosen
for an almost parameter-free algorithm with the right rate of convergence. 
Proofs of these corollaries are provided in Appendix~\ref{subsec:c:proofs_missing_cor}.

The first corollary considers the strongly convex case ($\lambda > 0$) with constant smoothing $\mu_k=\mu$, 
assuming that $\eps$ is known {\em a priori}. We note that this is, up to constants, the same complexity obtained by
the original Catalyst scheme on a fixed smooth approximation $F_{\mu\omega}$ with $\mu = \bigO(\eps D_\omega)$.
\begin{corollary} \label{cor:c:outer_sc}
	Consider the setting of Thm.~\ref{thm:catalyst:outer}. 
	Let $q = {\lambda}/(\lambda + \kappa)$. 
	Suppose $\lambda > 0$ and $\mu_k = \mu$, $\kappa_k = \kappa$, for all $k \ge 1$. Choose  $\alpha_0 = \sqrt{q}$ and, 
	$\delta_k = {\sqrt{q}}/({2 - \sqrt{q}}) \,.$
	Then, we have,
	\begin{align*}
	F(\wv_k) - F^* \le \frac{3 - \sqrt{q}}{1 - \sqrt{q}} \mu D_\omega +  
	2 \left( 1- \frac{\sqrt q}{2} \right)^k \left( F(\wv_0) - F^* \right) \,.
	\end{align*}
\end{corollary}
\noindent
Next, we consider the strongly convex case where the target accuracy $\eps$ is not known in advance.
We let smoothing parameters $( \mu_k )_{k \ge 0}$ decrease over time to obtain an adaptive smoothing scheme
that gives progressively better surrogates of the original objective.

\begin{corollary} \label{cor:c:outer_sc:decreasing_mu_const_kappa}
	Consider the setting of Thm.~\ref{thm:catalyst:outer}. 
	Let $q = {\lambda}/(\lambda + \kappa)$ and $\eta = 1 - {\sqrt q}/{2}$. 
	Suppose $\lambda > 0$ and 
	$\kappa_k = \kappa$, for all $k \ge 1$. Choose  $\alpha_0 = \sqrt{q}$ and, 
	the sequences $(\mu_k)_{k \ge 1}$ and $(\delta_k)_{k \ge 1}$ as 
	\begin{align*}
	\mu_k = \mu \eta^{{k}/{2}} \,, \qquad \text{and,} \qquad
	\delta_k = \frac{\sqrt{q}}{2 - \sqrt{q}} \,,
	\end{align*}
	where $\mu > 0$ is any constant.
	Then, we have, 
	\begin{align*}
	F(\wv_k) - F^* \le \eta^{{k}/{2}} \left[  
	2 \left( F(\wv_0) - F^* \right) 
	+ \frac{\mu D_\omega}{1-\sqrt{q}} \left(2-\sqrt{q} + \frac{\sqrt{q}}{1 - \sqrt \eta}  \right)
	\right] \, .
	\end{align*}
\end{corollary}
\noindent
The next two corollaries consider the unregularized problem, i.e., $\lambda = 0$ with constant and adaptive smoothing respectively.
\begin{corollary} \label{cor:c:outer_smooth}
	Consider the setting of Thm.~\ref{thm:catalyst:outer}. Suppose  $\mu_k = \mu$, $\kappa_k = \kappa$, for all $k \ge 1$
	and $\lambda = 0$. Choose $\alpha_0 = (\sqrt{5}-1)/{2}$ and 
	$\delta_k = (k+1)^{-2} \,.$
	Then, we have, 
	\begin{align*}
		F(\wv_k) - F^* \le  \frac{8}{(k+2)^2} \left( F(\wv_0) - F^* + \frac{\kappa}{2} \normasq{2}{\wv_0 - \wv^*} \right) 
		+ \mu D_\omega\left( 1 +  \frac{12}{k+2} +  \frac{30}{(k+2)^2} \right) \, .
	\end{align*}
\end{corollary}
\begin{corollary} \label{cor:c:outer_smooth_dec_smoothing}
	Consider the setting of Thm.~\ref{thm:catalyst:outer} with $\lambda = 0$. 
	Choose  $\alpha_0 = (\sqrt{5}-1)/{2}$, and for some non-negative constants $\kappa, \mu$, 
	define sequences $(\kappa_k)_{k \ge 1}, (\mu_k)_{k \ge 1}, (\delta_k)_{k \ge 1}$ as 
	\begin{align*}
	\kappa_k = \kappa  \, k\,, \quad
	\mu_k = \frac{\mu}{k} \quad \text{and,} \quad
	\delta_k = \frac{1}{(k + 1)^2} \,.
	\end{align*}
	Then, for $k \ge 2$,  we have,
	\begin{align*}
	F(\wv_k) - F^* \le  
		\frac{\log(k+1)}{k+1} \left( 
		2(F(\wv_0) - F^*) + \kappa \normasq{2}{\wv_0 - \wv^*} + 27 \mu D_\omega
		\right) \,.
	\end{align*}
	For the first iteration (i.e., $k = 1$), this bound is off by a constant factor $1 / \log2$.
\end{corollary}

\begin{table*}[t!]
	\caption{\small{Summary of outer iteration complexity for Algorithm~\ref{algo:catalyst} 
			for different parameter settings. We use shorthand
			$\Delta F_0 := F(\wv_0) - F^*$ and {$\Delta_0 = \norma{2}{\wv_0 - \wv^*}$}. 
			Absolute constants are omitted from the rates.
			\vspace{2mm}
	}}
	\begin{adjustbox}{width=\textwidth}
		\label{tab:catalyst_corollaries_summary}
		\centering
		\begin{tabular}{|c||ccccc|c|c|}
			\hline
			{Cor.} & $\lambda>0$ & $\kappa_k$ & $\mu_k$  & $\delta_k$ & $\alpha_0$ & $F(\wv_k)-F^*$ & Remark \\ \hline\hline
			
			\ref{cor:c:outer_sc} & Yes & 
			$\kappa$ & $\mu$ &  $\frac{\sqrt{q}}{2-\sqrt{q}}$ & $\sqrt{q}$ &
			$\left(1- \frac{\sqrt{q}}{2}\right)^k \Delta F_0 + \frac{\mu D}{1-\sqrt{q}}$ 
			& $q = \frac{\lambda}{\lambda+\kappa}$
			\\ \hline
			\ref{cor:c:outer_sc:decreasing_mu_const_kappa} & Yes & $\kappa$ & 
			$\mu \left( 1 - \frac{\sqrt{q}}{2} \right)^{k/2}$ & 
			$\frac{\sqrt{q}}{2-\sqrt{q}}$ & $\sqrt{q}$ &
			$\left(1- \frac{\sqrt{q}}{2}\right)^{k/2} \left( \Delta F_0 + \frac{\mu D}{1-\sqrt{q}} \right)$ 
			& $q = \frac{\lambda}{\lambda+\kappa}$
			\\ \hline  
			\rule{0pt}{12pt}
			\ref{cor:c:outer_smooth} & No & 
			$\kappa$ &  $\mu$ &  $k^{-2}$ & $c$ & 
			$\frac{1}{k^2} \left(\Delta F_0 + \kappa \Delta_0^2 \right) + \mu D$ 
			& $c = (\sqrt 5 - 1)/ 2$
			\\[3pt]  
			\hline
			\rule{0pt}{12pt}
			\ref{cor:c:outer_smooth_dec_smoothing} & No & 
			 $\kappa \, k$ & $\mu /k$ & $k^{-2}$ & $c$ &
			$\frac{\log k}{k} (\Delta F_0 + \kappa \Delta_0^2 + \mu D )$ 
			& $c = (\sqrt 5 - 1)/ 2$
			\\[3pt]  
			\hline
		\end{tabular}
	\end{adjustbox}
\end{table*}

\subsubsection{Outer Loop Convergence Analysis}\label{subsec:c:proof}
We now prove Thm.~\ref{thm:catalyst:outer}. 
The proof technique largely follows that of \citet{lin2017catalyst}, with the added challenges of accounting
for smoothing and varying Moreau-Yosida regularization.
We first analyze the sequence $(\alpha_k)_{k \ge 0}$. The proof follows from 
the algebra of Eq.~\eqref{eq:c:update_alpha}
and has been given in Appendix~\ref{sec:a:c_alpha_k}.
\begin{lemma} \label{lem:c:alpha_k}
	Given a positive, non-decreasing sequence $(\kappa_k)_{k\ge 1}$ and $\lambda \ge 0$, 
	consider the sequence $(\alpha_k)_{k \ge 0}$ defined by \eqref{eq:c:update_alpha}, where
	$\alpha_0 \in (0, 1)$ such that $\alpha_0^2 \ge \lambda / (\lambda + \kappa_1)$.
	Then, we have for every $k \ge 1$ that $0< \alpha_k \le \alpha_{k-1}$ and,
	$
		\alpha_k^2 \ge {\lambda}/({\lambda + \kappa_{k+1}}) \,.
	$
\end{lemma}
\noindent
We now characterize the effect of an approximate proximal point step 
on $F_{\mu\omega}$. 
\begin{lemma} \label{lem:c:approx_descent}
	Suppose $\widehat \wv \in \reals^d$ satisfies 
	$F_{\mu\omega, \kappa}(\widehat \wv ;\zv) - \min_{\wv \in \reals^d} F_{\mu\omega, \kappa}( \wv ;\zv) \le \widehat\eps$
	for some $\widehat \eps > 0$.
	Then, for all $0 < \theta < 1$ and all $\wv \in \reals^d$, we have, 
	\begin{align} \label{eq:c:approx_descent}
		F_{\mu\omega}(\widehat\wv) + \frac{\kappa}{2} \normasq{2}{\widehat\wv - \zv} 
			+ \frac{\kappa + \lambda}{2}(1-\theta) \normasq{2}{\wv - \widehat\wv} 
		\le F_{\mu\omega}(\wv) + \frac{\kappa}{2} \normasq{2}{\wv - \zv} + \frac{\widehat\eps}{\theta} \,.
	\end{align}
\end{lemma}
\begin{proof}
Let $\widehat F^* = \min_{\wv \in \reals^d} F_{\mu\omega,\kappa}(\wv ; \zv)$. 
Let $\widehat \wv^*$ be the unique minimizer of $F_{\mu\omega,\kappa}(\cdot \,; \zv)$. 
We have, from $(\kappa + \lambda)$-strong convexity of $F_{\mu\omega,\kappa}(\cdot \,; \zv)$,
\begin{align*}
	F_{\mu\omega,\kappa}(\wv ; \zv) &\ge \widehat F^* +\frac{\kappa +\lambda}{2} \normasq{2}{\wv - \widehat \wv^*}  \\
		&\ge \left( F_{\mu\omega,\kappa}(\widehat\wv ; \zv) - \widehat\eps \right) 
			+ \frac{\kappa + \lambda}{2}(1-\theta) \normasq{2}{\wv - \widehat\wv} 
			 - \frac{\kappa + \lambda}{2} \left( \frac{1}{\theta} - 1 \right) \normasq{2}{\widehat\wv -  \widehat \wv^*} \,,
\end{align*}
where we used that $\widehat\eps$ was sub-optimality of $\widehat\wv$ and Lemma~\ref{lem:c:helper:quadratic}
from Appendix~\ref{subsec:a:catalyst:helper}.
From $(\kappa + \lambda)$-strong convexity of $F_{\mu\omega, \kappa}(\cdot ; \zv)$, 
we have, 
\begin{align*}
	\frac{\kappa + \lambda }{2} \normasq{2}{\widehat\wv -  \widehat \wv^*} \le 
		F_{\mu\omega,\kappa}(\widehat\wv ; \zv) - \widehat F^* \le \widehat\eps\,,
\end{align*}
Since $(1/\theta - 1)$ is non-negative, 
we can plug this into the previous statement to get,
\begin{align*}
	F_{\mu\omega,\kappa}(\wv ; \zv)  \ge F_{\mu\omega,\kappa}(\widehat\wv ; \zv)  
		+ \frac{\kappa + \lambda}{2} (1-\theta) \normasq{2}{\wv - \widehat\wv} - \frac{\widehat\eps}{\theta}\,.
\end{align*}
Substituting the definition of $F_{\mu\omega,\kappa}(\cdot \,; \zv)$ 
from \eqref{eq:prox_point_algo} completes the proof.
\end{proof}
We now define a few auxiliary sequences integral to the proof.
Define sequences $(\vv_k)_{k \ge 0}$, $(\gamma_k)_{k \ge 0}$, $(\eta_k)_{k \ge 0}$, and $(\rv_k)_{k \ge 1}$ as 
\begin{align}
		\label{eq:c:v_defn_base} 
		\vv_0 &= \wv_0 \, \\
		\label{eq:c:v_defn}
		\vv_k &= \wv_{k-1} + \frac{1}{\alpha_{k-1}} (\wv_k - \wv_{k-1}) \,, \, k \ge 1 \,, \\
	\label{eq:c:gamma_defn_base}
		\gamma_0 &= \frac{(\kappa_1 + \lambda) \alpha_0^2 - \lambda \alpha_0 } {1 - \alpha_0} \,, \\
		\label{eq:c:gamma_defn}
		\gamma_k &= (\kappa_k + \lambda) \alpha_{k-1}^2 \, , \, k \ge 1 \,, \\
	\label{eq:c:eta_defn}
	\eta_k &= \frac{\alpha_k \gamma_k}{\gamma_{k+1} + \alpha_k \gamma_k} \,, \, k\ge  0 \,, \\
	\label{eq:c:ly_vec_defn}
	\rv_k &= \alpha_{k-1} \wv^* + ( 1- \alpha_{k-1}) \wv_{k-1} \, , \, k \ge 1\,.
\end{align}
One might recognize $\gamma_k$ and $\vv_k$ from their resemblance to 
counterparts from the proof of \citet{nesterov2013introductory}.
Now, we claim some properties of these sequences. 
\begin{claim} \label{claim:c:sequences}
	For the sequences defined in \eqref{eq:c:v_defn_base}-\eqref{eq:c:ly_vec_defn}, we have, 
	\begin{align}
		\label{eq:c:gamma_defn_2}
		\gamma_k &= \frac{(\kappa_{k+1} + \lambda) \alpha_k^2 - \lambda \alpha_k } {1 - \alpha_k} \, , \, k \ge 0\,, \\
		\label{eq:c:gamma_defn_3}
		\gamma_{k+1} &= (1- \alpha_k) \gamma_k + \lambda \alpha_k \, , \, k \ge 0\,, \\
		\label{eq:c:eta_defn_2}
		\eta_k &= \frac{\alpha_k \gamma_k}{\gamma_k + \alpha_k \lambda} \, , \, k \ge 0 \\
		\label{eq:c:v_defn_2}
		\zv_k &= \eta_k \vv_k + (1- \eta_k) \wv_k \, , \, k \ge 0\,, \,.
	\end{align}
\end{claim}
\begin{proof}
Eq.~\eqref{eq:c:gamma_defn_2} 
follows from plugging in \eqref{eq:c:update_alpha}  in \eqref{eq:c:gamma_defn} for $k\ge 1$,
while for $k=0$, it is true by definition. 
Eq.~\eqref{eq:c:gamma_defn_3} follows from plugging \eqref{eq:c:gamma_defn} in \eqref{eq:c:gamma_defn_2}. 
Eq.~\eqref{eq:c:eta_defn_2} follows from \eqref{eq:c:gamma_defn_3} and \eqref{eq:c:eta_defn}.
	Lastly, to show \eqref{eq:c:v_defn_2}, we shall show instead that \eqref{eq:c:v_defn_2} is equivalent 
	to the update \eqref{eq:c:update_support} for $\zv_k$. We have, 
	\begin{align*}
		\zv_k &\,= \eta_k \vv_k + (1-\eta_k) \wv_k \\
			&\stackrel{\eqref{eq:c:v_defn}}{=} \eta_k \left( \wv_{k-1} 
				+ \frac{1}{\alpha_{k-1}} (\wv_k - \wv_{k-1}) \right) + (1-\eta_k) \wv_k  \\
			&\,= \wv_k + \eta_k \left(\frac{1}{\alpha_{k-1}} - 1 \right) (\wv_{k} - \wv_{k-1}) \,.
	\end{align*}
	Now, 
	\begin{align*}
		\eta_k \left(\frac{1}{\alpha_{k-1}} - 1 \right) 
		&\stackrel{\eqref{eq:c:eta_defn}}{=} \frac{\alpha_k \gamma_k}{\gamma_{k+1} + \alpha_k \gamma_k } \cdot \frac{1-\alpha_{k-1}}{\alpha_{k-1}} 
		\\& \stackrel{\eqref{eq:c:gamma_defn}}{=} \frac{\alpha_k (\kappa_k + \lambda) \alpha_{k-1}^2 }
			{ \alpha_k^2 (\kappa_{k+1} + \lambda) +\alpha_k (\kappa_k + \lambda) \alpha_{k-1}^2  } 
			\cdot \frac{1-\alpha_{k-1}}{\alpha_{k-1}} 
		\stackrel{\eqref{eq:c:update_beta}}{=} \beta_k \, ,
	\end{align*}
	completing the proof.
\end{proof}

\begin{claim} \label{claim:c:ly_sequence}
 	The sequence $(\rv_k)_{k \ge 1}$ from \eqref{eq:c:ly_vec_defn} satisfies
 	\begin{align} \label{eq:c:norm_ly_sequence}
 		\normasq{2}{\rv_k - \zv_{k-1}} \le \alpha_{k-1} (\alpha_{k-1} - \eta_{k-1}) \normasq{2}{\wv_{k-1} - \wv^*} 
 			+ \alpha_{k-1} \eta_{k-1} \normasq{2}{\vv_{k-1} - \wv^*} \,.
 	\end{align}
\end{claim}
\begin{proof}
	Notice that $\eta_k \stackrel{\eqref{eq:c:eta_defn_2}}{=} \alpha_k  \cdot \frac{\gamma_k}{\gamma_k + \alpha_k \lambda} \le \alpha_k$.
	Hence, using convexity of the squared Euclidean norm, we get, 
	\begin{align*}
		\normasq{2}{\rv_k - \zv_{k-1}} & \stackrel{\eqref{eq:c:v_defn_2}}{=} 
			\normasq{2}{ (\alpha_{k-1} - \eta_{k-1})(\wv^* - \wv_{k-1}) + \eta_{k-1}(\wv^* - \vv_{k-1}) } \\
		&\,= \alpha_{k-1}^2 \normsq*{ \left(1 - \frac{\eta_{k-1}}{\alpha_{k-1}} \right) (\wv^* - \wv_{k-1}) 
			+  \frac{\eta_{k-1}}{\alpha_{k-1}} (\wv^* - \vv_{k-1}) }_2 \\
		&\stackrel{(*)}{\le} \alpha_{k-1}^2 \left(1 - \frac{\eta_{k-1}}{\alpha_{k-1}} \right) \normasq{2}{\wv_{k-1} - \wv^*}
			+ \alpha_{k-1}^2 \frac{\eta_{k-1}}{\alpha_{k-1}}  \normasq{2}{\vv_{k-1} - \wv^*}  \\
		&\,= \alpha_{k-1} (\alpha_{k-1} - \eta_{k-1}) \normasq{2}{\wv_{k-1} - \wv^*} 
 			+ \alpha_{k-1} \eta_{k-1} \normasq{2}{\vv_{k-1} - \wv^*} \,.
	\end{align*}
\end{proof}
For all $\mu \ge \mu' \ge 0$, we know from Prop.~\ref{thm:setting:beck-teboulle} that 
\begin{align}
	0 \le F_{\mu \omega}(\wv) - F_{\mu'\omega}(\wv) \le (\mu - \mu') D_\omega \,.
	\label{asmp:c:smoothing:1}
\end{align}
We now define the sequence $( S_k )_{k\ge0}$ to play the role of a 
potential function here.
\begin{align}
	\label{eq:c:ly_fn_defn}
	\begin{split}
	S_0 &= (1 - \alpha_0) (F(\wv_0) - F(\wv^*)) + \frac{\alpha_0 \kappa_1 \eta_0}{2} \normasq{2}{\wv_0 - \wv^*}\,, \\
	S_k &= (1-\alpha_k) ( F_{\mu_k \omega}(\wv_k)  - F_{\mu_k \omega}(\wv^*)) + \frac{\alpha_k \kappa_{k+1} \eta_k}{2} \normasq{2}{\vv_k - \wv^*}\,, \, k\ge  1 \,.
	\end{split}
\end{align}
We are now ready to analyze the effect of one outer loop. This lemma is the crux of the analysis.
\begin{lemma} \label{lem:c:one_step_ly}
	Suppose $F_{\mu_k\omega, \kappa_k}(\wv_k ;\zv) - \min_{\wv\in\reals^d} F_{\mu_k\omega, \kappa_k}(\wv ; \zv) \le \eps_k$
	for some $\eps_k > 0$. The following statement holds for all $0 < \theta_k  < 1$: 
	\begin{align} \label{eq:c:one_step_ly}
		\frac{S_k}{1-\alpha_k} \le S_{k-1} + (\mu_{k-1} - \mu_k) D_\omega + \frac{\eps_k}{\theta_k}
			- \frac{\kappa_k}{2}\normasq{2}{\wv_k - \zv_{k-1}} + \frac{\kappa_{k+1}\eta_k \alpha_k \theta_k}{2(1-\alpha_k)} \normasq{2}{\vv_k - \wv^*} \,,
	\end{align}
	where we set $\mu_0 := 2 \mu_1$.
\end{lemma}
\begin{proof}
	For ease of notation, let $F_k := F_{\mu_k \omega}$, and $D := D_\omega$.
	By $\lambda$-strong convexity of $F_{\mu_k\omega}$, we have, 
	\begin{align} \label{eq:c:proof:step_sc_r_k}
		F_k(\rv_k) \le \alpha_{k-1} F_k(\wv^*) + (1-\alpha_{k-1}) F_k(\wv_{k-1}) -
			 \frac{\lambda \alpha_{k-1} (1-\alpha_{k-1})}{2} \normasq{2}{\wv_{k-1} - \wv^*} \,.
	\end{align}
	We now invoke Lemma~\ref{lem:c:approx_descent} on the function $F_{\mu_k \omega, \kappa_k}(\cdot ; \zv_{k-1})$ with 
	$\widehat \eps = \eps_k$ and $\wv = \rv_k$ to get,
	\begin{align} \label{eq:c:proof:main_eq_unsimplified}
		F_k(\wv_k) + \frac{\kappa_k}{2} \normasq{2}{\wv_k - \zv_{k-1}} + \frac{\kappa_k + \lambda}{2}(1-\theta_k) \normasq{2}{\rv_k - \wv_k} 
		\le F_k(\rv_k) + \frac{\kappa_k}{2} \normasq{2}{\rv_k - \zv_{k-1}} + \frac{\eps_k}{\theta_k} \, .
	\end{align}
	We shall separately manipulate the left and right hand sides of \eqref{eq:c:proof:main_eq_unsimplified}, 
	starting with the right hand side, which we call $\mcR$. 
	We have, using \eqref{eq:c:proof:step_sc_r_k} and~\eqref{eq:c:norm_ly_sequence},
	\begin{align*}
	\mcR
	\le& \,
	(1- \alpha_{k-1}) F_k(\wv_{k-1}) + \alpha_{k-1} F_k(\wv^*) 
			- \frac{\lambda \alpha_{k-1} (1-\alpha_{k-1})}{2} \normasq{2}{\wv_{k-1} - \wv^*} 
		\\
			&+ \frac{\kappa_k}{2} \alpha_{k-1}(\alpha_{k-1} - \eta_{k-1}) \normasq{2}{\wv_{k-1} - \wv^*}
			+ \frac{\kappa_k \alpha_{k-1} \eta_{k-1}}{2} \normasq{2}{\vv_{k-1} - \wv^*} + \frac{\eps_k}{\theta_k} \,.
	\end{align*}
	We notice now that 
	\begin{align}
		\alpha_{k-1} - \eta_{k-1} 
			&\stackrel{\eqref{eq:c:eta_defn_2}}{=} \alpha_{k-1} - \frac{\alpha_{k-1}\gamma_{k-1}}{\gamma_k + \alpha_{k-1} \gamma_{k-1}} \nonumber\\
			&\,= \alpha_{k-1} \left(  \frac{\gamma_k - \gamma_{k-1} ( 1- \alpha_{k-1})}{\gamma_k + \alpha_{k-1} \gamma_{k-1}} \right) 
				\nonumber\\
			&\stackrel{\eqref{eq:c:gamma_defn_3}}{=} \frac{\alpha_{k-1}^2 \lambda}{\gamma_{k-1} + \alpha_{k-1}\lambda} 
				\nonumber\\
			&\stackrel{\eqref{eq:c:gamma_defn_2}}{=} \frac{\alpha_{k-1}^2 \lambda (1-\alpha_{k-1})}
				{(\kappa_k + \lambda) \alpha_{k-1}^2 - \lambda \alpha_{k-1} + (1-\alpha_{k-1})\alpha_{k-1}\lambda} 
				\nonumber\\
			&\,= \frac{\lambda}{\kappa_k}(1-\alpha_{k-1}) \,, \label{eq:c:one_step_ly_pf_1}
	\end{align}
	and hence the terms containing $\normasq{2}{\wv_{k-1} - \wv^*}$ cancel out. 
	Therefore, we get,
	\begin{align} \label{eq:c:proof:main_eq:rhs:simplified}
	\mcR
	\le
	(1 - \alpha_{k-1}) F_k(\wv_{k-1}) + \alpha_{k-1} F_k(\wv^*)  
			+ \frac{\kappa_k \alpha_{k-1} \eta_{k-1}}{2} \normasq{2}{\vv_{k-1} - \wv^*} + \frac{\eps_k}{\theta_k} \,.
	\end{align}
	To move on to the left hand side, we note that
	\begin{align} \label{eq:c:one_step_ly_proof_prod}
	\alpha_k \eta_k \nonumber
		&\stackrel{\eqref{eq:c:eta_defn_2}}{=}  \frac{\alpha_k^2 \gamma_k}{\gamma_k + \alpha_k \lambda} 
		\stackrel{\eqref{eq:c:gamma_defn},\eqref{eq:c:gamma_defn_2}}{=} \frac{\alpha_k^2 \alpha_{k-1}^2 (\kappa_k + \lambda)}
			{\frac{(\kappa_{k+1} + \lambda) \alpha_k^2 - \lambda \alpha_k}{1-\alpha_k} + \alpha_k \lambda } \\
		&\,=\frac{ (1-\alpha_k)(\kappa_k + \lambda) \alpha_{k-1}^2 \alpha_k^2}{(\kappa_{k+1} + \lambda)\alpha_k^2 - \lambda \alpha_k^2} 
		= (1-\alpha_k) \alpha_{k-1}^2 \frac{\kappa_k + \lambda}{\kappa_{k+1}} \,.
	\end{align}
	Therefore, 
	\begin{align} \label{eq:c:one_step_ly_pf_2}
		F_k(\wv_k) - F_k(\wv^*) + \frac{\kappa_k + \lambda}{2} \alpha_{k-1}^2 \normasq{2}{\vv_k - \wv^*} 
		\stackrel{\eqref{eq:c:ly_fn_defn},\eqref{eq:c:one_step_ly_proof_prod}}{=} 
		\frac{S_k}{1 - \alpha_{k}} \,.
	\end{align}
	Using
	$\rv_k - \wv_k \stackrel{\eqref{eq:c:v_defn}}{=} \alpha_{k-1}(\wv^* - \vv_{k})$,
	we simplify the left hand side of \eqref{eq:c:proof:main_eq_unsimplified}, which we call $\mcL$, as
	\begin{align} \label{eq:c:proof:main_eq:lhs:simplified}
		\nonumber
		\mcL &= F_k(\wv_k)  - F_k(\wv^*) + \frac{\kappa_k}{2} \normasq{2}{\wv_k - \zv_{k-1}} + 
			\frac{\kappa_k + \lambda}{2}(1-\theta_k) \alpha_{k-1}^2 \normasq{2}{\vv_k - \wv^*}  \\
			&\stackrel{\eqref{eq:c:one_step_ly_pf_2}}{=}
				\frac{S_k}{1-\alpha_k} + F_k(\wv^*) + \frac{\kappa_k}{2} \normasq{2}{\wv_k - \zv_{k-1}} 
				- \frac{\kappa_{k+1} \alpha_k \eta_k \theta_k}{2 (1-\alpha_k)} \normasq{2}{\vv_k - \wv^*} \,.
	\end{align}
	In view of \eqref{eq:c:proof:main_eq:rhs:simplified} and \eqref{eq:c:proof:main_eq:lhs:simplified}, 
	we can simplify \eqref{eq:c:proof:main_eq_unsimplified} as
	\begin{align} \label{eq:c:one_step_ly_pf_2_int}
	\begin{aligned}
		\frac{S_k}{1-\alpha_k} & + \frac{\kappa_k}{2} \normasq{2}{\wv_k - \zv_{k-1}} 
			- \frac{\kappa_{k+1}\alpha_k\eta_k\theta_k}{2(1-\alpha_k)} \normasq{2}{\vv_k - \wv^*}
		\\&\le (1 - \alpha_{k-1})\left( F_k(\wv_{k-1}) - F_k(\wv^*) \right) 
			+ \frac{\kappa_k \alpha_{k-1} \eta_{k-1}}{2}\normasq{2}{\vv_{k-1} - \wv^*} + \frac{\eps_k}{\theta_k} \,.
	\end{aligned}
	\end{align}
	We make a distinction for $k \ge 2$ and $k=1$ here. For $k \ge 2$, 
	the condition that $\mu_{k-1} \ge \mu_k$ gives us,
	\begin{align} \label{eq:c:one_step_ly_pf_3}
		F_k(\wv_{k-1}) - F_k(\wv^*) \stackrel{\eqref{asmp:c:smoothing:1}}{\le} 
		F_{k-1}(\wv_{k-1}) - F_{k-1}(\wv^*) + (\mu_{k-1} - \mu_{k}) D \, .
	\end{align}
	The right hand side of \eqref{eq:c:one_step_ly_pf_2_int} can now be upper bounded by
	\begin{align*}
		(1 - \alpha_{k-1})(\mu_{k-1} - \mu_k) D + S_{k-1} + \frac{\eps_k}{\theta_k} \,,
	\end{align*}
	and noting that $1-\alpha_{k-1} \le 1$ yields \eqref{eq:c:one_step_ly} for $k \ge 2$. 
	
	For $k=1$, we note that $S_{k-1} (= S_0)$ is defined in terms of $F(\wv)$. So we have, 
	\begin{align*}
		F_1(\wv_0) - F_1(\wv^*) \le F(\wv_0) - F(\wv^*) + \mu_1 D =   F(\wv_0) - F(\wv^*) + (\mu_0 - \mu_1) D\,,
	\end{align*}
	because we used $\mu_0 = 2\mu_1$. This is of the same form as \eqref{eq:c:one_step_ly_pf_3}. Therefore, 
	\eqref{eq:c:one_step_ly} holds for $k=1$ as well.
\end{proof}
\noindent
We now prove Thm.~\ref{thm:catalyst:outer}.
\begin{proof} [Proof of Thm.~\ref{thm:catalyst:outer}]
	We continue to use shorthand $F_k := F_{\mu_k \omega}$, and $D := D_\omega$. 
	We now apply Lemma~\ref{lem:c:one_step_ly}.
	In order to satisfy the supposition of Lemma~\ref{lem:c:one_step_ly} that $\wv_k$ is $\eps_k$-suboptimal, 
	we make the choice $\eps_k = \frac{\delta_k \kappa_k}{2} \normasq{2}{\wv_k - \zv_{k-1}}$ (cf.~\eqref{eq:stopping_criterion}).
	Plugging this in and setting $\theta_k = \delta_k < 1$, we get from~\eqref{eq:c:one_step_ly}, 
	\begin{align*}
		\frac{S_k}{1-\alpha_k} - \frac{\kappa_{k+1} \eta_k \alpha_k \delta_k}{2(1 - \alpha_k)} \normasq{2}{\vv_k - \wv^*}
			\le S_{k-1}
			+ (\mu_{k-1} - \mu_k) D \,.
	\end{align*}
	The left hand side simplifies to $ S_k \, ({1- \delta_k})/({1-\alpha_k}) + \delta_k ( F_k(\wv_k) - F_k(\wv^*))$.
	Note that $F_k(\wv_k) - F_k(\wv^*) \stackrel{\eqref{asmp:c:smoothing:1}}{\ge}  F(\wv_k) - F(\wv^*) - \mu_k D \ge -\mu_k D$.
	From this, noting that $\alpha_k \in (0, 1)$ for all $k$, we get, 
	\begin{align*}
		S_k \left(\frac{1-\delta_k}{1-\alpha_k} \right) \le S_{k-1} + \delta_k \mu_k D + (\mu_{k-1} - \mu_k) D\,,
	\end{align*}
	or equivalently, 
	\begin{align*}
		S_k \le \left( \frac{1-\alpha_k}{1-\delta_k} \right) S_{k-1} + 
			\left( \frac{1-\alpha_k}{1-\delta_k} \right) (\mu_{k-1} - (1-\delta_k) \mu_k) D\,.
	\end{align*}
	Unrolling the recursion for $S_k$, we now have, 
	\begin{align} \label{eq:c:pf_thm_main_1}
		S_k \le \left( \prod_{j=1}^k  \frac{1-\alpha_j}{1-\delta_j} \right) S_0
			+ \sum_{j=1}^k \left( \prod_{i=j}^k  \frac{1-\alpha_i}{1-\delta_i} \right) (\mu_{j-1} - (1- \delta_j) \mu_j)  D\,.
	\end{align}
	Now, we need to reason about $S_0$ and $S_k$ to complete the proof. To this end, consider $\eta_0$:
	\begin{align}
		\eta_0 &\stackrel{\eqref{eq:c:eta_defn}}{=} \frac{\alpha_0 \gamma_0}{\gamma_1 + \alpha_0 \gamma_0} 
			\nonumber\\
			&\stackrel{\eqref{eq:c:gamma_defn_base}}{=} \frac{\alpha_0 \gamma_0}
				{(\kappa_1 + \lambda)\alpha_0^2 + \tfrac{\alpha_0}{1-\alpha_0}\left( (\kappa_1 + \lambda)\alpha_0^2 - \lambda \alpha_0 \right)} 
				\nonumber\\
			&\, = \frac{\alpha_0 \gamma_0 (1-\alpha_0)}{(\kappa_1 + \lambda)\alpha_0^2  - \lambda \alpha_0^2} 
			= (1-\alpha_0) \frac{\gamma_0}{\kappa_1 \alpha_0} \,. \label{eq:c:thm_pf_1}
	\end{align}
	With this, we can expand out $S_0$ to get
	\begin{align*}
		S_0 &\stackrel{\eqref{eq:c:ly_fn_defn}}{=} 
			(1- \alpha_0) \left(F(\wv_0) - F(\wv^*)\right) + \frac{\alpha_0 \kappa_1 \eta_0}{2} \normasq{2}{\wv_0 - \wv^*} \\
			&\stackrel{\eqref{eq:c:thm_pf_1}}{=} 
			(1- \alpha_0) \left(  F(\wv_0) - F^* + \frac{\gamma_0}{2}\normasq{2}{\wv_0 - \wv^*} \right) \,.
	\end{align*}
	Lastly, we reason about $S_k$ for $k \ge 1$ as,
	\begin{align*}
		S_k \stackrel{\eqref{eq:c:ly_fn_defn}}{\ge}  
			(1-\alpha_k) \left(F_k(\wv_k) - F_k(\wv^*) \right) 
			\stackrel{\eqref{asmp:c:smoothing:1}}{\ge} 
			(1-\alpha_k) \left( F(\wv_k) - F(\wv^*) - \mu_k D \right) \,.
	\end{align*}
	Plugging this into the left hand side of \eqref{eq:c:pf_thm_main_1} completes the proof.
\end{proof}


\subsubsection{Inner Loop Complexity} \label{sec:catalyst:inner_compl}
Consider a class $\mcF_{L, \lambda}$ of functions defined as 
	\[
		\mcF_{L, \lambda} = \left\{  
			f : \reals^d \to \reals \text{ such that $f$ is $L$-smooth and $\lambda$-strongly convex}
		\right\} \,.
	\]
We now formally define a linearly convergent algorithm on this class of functions.
\begin{definition} \label{defn:c:linearly_convergent}
	A first order algorithm $\mcM$ is said to be linearly convergent with parameters 
	$C : \reals_+ \times \reals_+ \to \reals_+$ and $\tau : \reals_+ \times \reals_+ \to (0, 1)$
	if the following holds: for all $L \ge \lambda > 0$, and every $f \in \mcF_{L, \lambda}$ and $\wv_0 \in \reals^d$, 
	$\mcM$ started at $\wv_0$ generates a sequence $(\wv_k)_{k \ge 0}$ that satisfies:
	\begin{align} \label{eq:def:linearly_convergent_2}
		\expect f(\wv_k) -  f^* \le C(L, \lambda) \left( 1 - \tau(L, \lambda) \right)^k \left( f(\wv_0) - f^* \right)\, ,
	\end{align}
	where $f^* := \min_{\wv\in\reals^d} f(\wv)$ and the expectation is over the randomness of $\mcM$.
\end{definition}
The parameter $\tau$ determines the rate of convergence of the algorithm.
For instance, batch gradient descent is a deterministic linearly convergent algorithm with $\tau(L, \lambda)\inv = L/\lambda$ and 
incremental algorithms such as SVRG and SAGA satisfy requirement~\eqref{eq:def:linearly_convergent_2} with 
$\tau(L,\lambda)\inv = c(n + \nicefrac{L}{\lambda})$ for some universal constant $c$.

The warm start strategy in
step $k$ of Algo.~\ref{algo:catalyst} is to initialize $\mcM$ at the prox center $\zv_{k-1}$.
The next proposition, due to \citet[Cor. 16]{lin2017catalyst} bounds the expected number of iterations of $\mcM$ required to
ensure that $\wv_k$ satisfies \eqref{eq:stopping_criterion}. Its proof has been given in Appendix~\ref{sec:c:proofs:inner_compl}
for completeness.
\begin{proposition} \label{prop:c:inner_loop_final}
	Consider $F_{\mu\omega, \kappa}(\cdot \, ;\zv)$ defined in Eq.~\eqref{eq:prox_point_algo},
	and a linearly convergent algorithm $\mcM$ with parameters $C$, $\tau$. 
	Let $\delta \in [0,1)$. Suppose $F_{\mu\omega}$ is $L_{\mu\omega}$-smooth and 
	$\lambda$-strongly convex. 
	Then the expected number of iterations $\expect[\widehat T]$ of $\mcM$ when started at $\zv$
	in order to obtain $\widehat \wv \in \reals^d$ that satisfies
	\begin{align*}
		F_{\mu\omega, \kappa}(\widehat\wv;\zv) - \min_\wv  F_{\mu\omega, \kappa}(\wv;\zv)\leq \tfrac{\delta\kappa}{2} \normasq{2}{\wv - \zv}
	\end{align*}
	is upper bounded by 
	\begin{align*}
		\expect[\widehat T] \le \frac{1}{\tau(L_{\mu\omega} + \kappa, \lambda + \kappa)} \log\left( 
		\frac{8 C(L_{\mu\omega} + \kappa, \lambda + \kappa)}{\tau(L_{\mu\omega} + \kappa, \lambda + \kappa)} \cdot 
		\frac{L_{\mu\omega} + \kappa}{\kappa \delta} \right)  + 1 \,.
	\end{align*}
\end{proposition}
%


\begin{table*}[t!]
\caption{\small{Summary of global complexity of \nsCatalystSvrg, i.e., Algorithm~\ref{algo:catalyst} 
with SVRG as the inner solver for various parameter settings. 
We show $\expect[N]$, the expected total number of SVRG iterations required to obtain an accuracy $\eps$,
up to constants and factors logarithmic in problem parameters.
We denote 
$\Delta F_0 := F(\wv_0) - F^*$ and {$\Delta_0 = \norma{2}{\wv_0 - \wv^*}$}. 
Constants $D, A$ are short for $D_\omega, A_\omega$ (see \eqref{eq:c:A_defn}).
\vspace{2mm}
}}
\begin{adjustbox}{width=\textwidth}
\label{tab:sc-svrg_rates_summary}
\centering
\begin{tabular}{|c||cccc|c|c|}
\hline
\rule{0pt}{12pt}
{Prop.} & $\lambda>0$ & $\mu_k$ & $\kappa_k$ & $\delta_k$ & $\expect[N]$ & Remark \\[3pt] \hline\hline

\rule{0pt}{15pt}
 \ref{prop:c:total_compl_svrg_main}
  & Yes & $\sfrac{\eps}{D}$ & 
	$\sfrac{A D}{\eps n} - \lambda$ & $\sqrt\frac{\lambda\eps n}{A D}$ &
	$n + \sqrt{\frac{A D n}{\lambda \eps}} $ & fix $\eps$ in advance
\\[5pt] \hline
\rule{0pt}{15pt}
 \ref{prop:c:total_compl_sc:dec_smoothing_main}
  & Yes & $\mu c^k $ & $\lambda$ & $c'$ & 
	$ n + \frac{A}{\lambda\eps} \frac{\Delta F_0 + \mu D}{\mu}$ &
	$c,c'<1$ are universal constants
\\[5pt] \hline
\rule{0pt}{15pt}
 \ref{prop:c:total_compl_svrg_smooth_main}
  & No &  $\sfrac{\eps}{D}$ & 
	$\sfrac{A D}{\eps n}$ & $1/k^2$ &
	$n\sqrt{\frac{\Delta F_0}{\eps}} + 
			\frac{ \sqrt{A D n} \Delta_0}{\eps} $ &  fix $\eps$ in advance
\\[5pt] \hline
\rule{0pt}{15pt}
 \ref{prop:c:total_compl_nsc:dec_smoothing_main} & No & $\sfrac{\mu}{k}$ & 
	$\kappa_0  \, k$ & $1/k^2$ &
	$\frac{\widehat\Delta_0}{\eps}  \left( n + \frac{A}{\mu \kappa_0} \right)  $ &
	$\widehat\Delta_0 = \Delta F_0 + \frac{\kappa_0}{2} \Delta_0^2 + \mu D$
\\[5pt] \hline
\end{tabular}
\end{adjustbox}
\end{table*}

\subsection{Casimir with SVRG} \label{sec:catalyst:total_compl}
We now choose SVRG \citep{johnson2013accelerating} to be the linearly convergent algorithm $\mcM$, 
resulting in an algorithm called \nsCatalystSvrg{}.
The rest of this section analyzes the total iteration complexity of 
\nsCatalystSvrg{} to solve Problem~\eqref{eq:cvx_pb_finite_sum}.
The proofs of the results from this section are calculations 
stemming from combining the outer loop complexity from
Cor.~\ref{cor:c:outer_sc} to~\ref{cor:c:outer_smooth_dec_smoothing} with 
the inner loop complexity from Prop.~\ref{prop:c:inner_loop_final},
and are relegated to Appendix~\ref{sec:c:proofs:total_compl}.
Table~\ref{tab:sc-svrg_rates_summary} summarizes the results of this section.

Recall that if $\omega$ is 1-strongly convex with respect to $\norma{\alpha}{\cdot}$, then 
$h_{\mu\omega}(\Am \wv + \bv)$ is $L_{\mu\omega}$-smooth with respect to $\norma{2}{\cdot}$,
where $L_{\mu\omega} = \normasq{2,\alpha}{\Am} / \mu$.
Therefore, the complexity of solving problem~\eqref{eq:cvx_pb_finite_sum} will depend on 
\begin{align} \label{eq:c:A_defn}
	A_\omega := \max_{i=1,\cdots,n} \normasq{2, \alpha}{\Am\pow{i}} \,.
\end{align}
\begin{remark} \label{remark:smoothing:l2vsEnt}
    We have that $\norma{2,2}{\Am} = \norma{2}{\Am}$ is the spectral norm of $\Am$ and
    $\norma{2,1}{\Am} = \max_j \norma{2}{\av_j}$ is the largest row norm, where $\av_j$ is the $j$th row of $\Am$. 
    Moreover, we have that $\norma{2,2}{\Am} \ge \norma{2,1}{\Am}$. 
\end{remark}
\noindent
We start with the strongly convex case with constant smoothing.
\begin{proposition} \label{prop:c:total_compl_svrg_sc} \label{prop:c:total_compl_svrg_main}
	Consider the setting of Thm.~\ref{thm:catalyst:outer} and 
	fix $\eps > 0$.
	If we run Algo.~\ref{algo:catalyst} with SVRG as the inner solver with parameters:
	$\mu_k = \mu = \eps / {10 D_\omega}$, $\kappa_k = k$ chosen as 
	\begin{align*}
		\kappa = 
	\begin{cases}
		\frac{A}{\mu n} - \lambda \,, \text{ if } \frac{A}{\mu n} > 4 \lambda \\
		\lambda \,, \text{ otherwise}
	\end{cases} \,,
	\end{align*}
	$q = {\lambda}/{(\lambda + \kappa)}$, $\alpha_0 = \sqrt{q}$, and
	$\delta = {\sqrt{q}}/{(2 - \sqrt{q})}$.
	Then, the number of iterations $N$ to obtain $\wv$ such that $F(\wv) - F(\wv^*) \le \eps$ is 
	bounded in expectation as 
	\begin{align*}
		\expect[N] \le \widetilde \bigO \left( 
				n + \sqrt{\frac{A_\omega D_\omega n}{\lambda \eps}} 
			\right) \,.
	\end{align*}
\end{proposition}
\noindent
Here, we note that $\kappa$ was chosen to minimize the total complexity (cf. \citet{lin2017catalyst}).
This bound is known to be tight, up to logarithmic factors \citep{woodworth2016tight}. 
\noindent
Next, we turn to the strongly convex case with decreasing smoothing.
\begin{proposition} \label{prop:c:total_compl_sc:dec_smoothing} \label{prop:c:total_compl_sc:dec_smoothing_main}
    Consider the setting of Thm.~\ref{thm:catalyst:outer}. 
    Suppose $\lambda > 0$ and $\kappa_k = \kappa$, for all $k \ge 1$ and 
    that $\alpha_0$, $(\mu_k)_{k \ge 1}$ and $(\delta_k)_{k \ge 1}$ 
    are chosen as in Cor.~\ref{cor:c:outer_sc:decreasing_mu_const_kappa},
    with $q = \lambda/(\lambda + \kappa)$ and $\eta = 1- {\sqrt q}/{2}$.
    If we run Algo.~\ref{algo:catalyst} with SVRG as the inner solver with these parameters,
    the number of iterations $N$ of SVRG required to obtain $\wv$ such that $F(\wv) - F^* \le \eps$ is 
    bounded in expectation as 
    \begin{align*}
        \expect[N] \le \widetilde \bigO \left( n 
            + \frac{A_\omega}{\mu(\lambda + \kappa)\eps} \left( F(\wv_0) - F^* + \frac{\mu D_\omega}{1-\sqrt{q}}  \right)
        \right) \,.
    \end{align*}
\end{proposition}
\noindent
Unlike the previous case, there is no obvious choice of $\kappa$, such as to minimize the global complexity.
Notice that we do not get the accelerated rate of Prop.~\ref{prop:c:total_compl_svrg_sc}.
We now turn to the case when $\lambda = 0$ and $\mu_k = \mu$ for all $k$.
\begin{proposition} \label{prop:c:total_compl_svrg_smooth} \label{prop:c:total_compl_svrg_smooth_main}
	Consider the setting of Thm.~\ref{thm:catalyst:outer} and fix $\eps > 0$.
	If we run Algo.~\ref{algo:catalyst} with SVRG as the inner solver with parameters:
	$\mu_k = \mu ={\eps}/{20 D_\omega}$, $\alpha_0 = (\sqrt{5} - 1)/{2}$, 
	$\delta_k = {1}/{(k+1)^2}$, and $\kappa_k = \kappa = {A_\omega}/{\mu(n+1)}$.
	Then, the number of iterations $N$ to get a point $\wv$ such that $F(\wv) - F^* \le \eps$ is 
	bounded in expectation as 
	\begin{align*}
		\expect[N] \le \widetilde \bigO \left( n\sqrt{\frac{F(\wv_0) - F^*}{\eps}} + 
			\sqrt{A_\omega D_\omega n} \frac{\norma{2}{\wv_0 - \wv^*}}{\eps} \right) \, .
	\end{align*}
\end{proposition}
\noindent
This rate is tight up to log factors~\citep{woodworth2016tight}.
Lastly, we consider the non-strongly convex case ($\lambda = 0$) together with decreasing smoothing.
As with Prop.~\ref{prop:c:total_compl_sc:dec_smoothing}, we do not obtain an accelerated rate here.
\begin{proposition} \label{prop:c:total_compl_nsc:dec_smoothing} \label{prop:c:total_compl_nsc:dec_smoothing_main}
    Consider the setting of Thm.~\ref{thm:catalyst:outer}. 
    Suppose $\lambda = 0$ and that $\alpha_0$, $(\mu_k)_{k\ge 1}$,$ (\kappa_k)_{k\ge 1}$ and $(\delta_k)_{k \ge 1}$ 
    are chosen as in Cor.~\ref{cor:c:outer_smooth_dec_smoothing}.
    If we run Algo.~\ref{algo:catalyst} with SVRG as the inner solver with these parameters,
    the number of iterations $N$ of SVRG required to obtain $\wv$ such that $F(\wv) - F^* \le \eps$ is 
    bounded in expectation as 
    \begin{align*}
        \expect[N] \le \widetilde\bigO \left( \frac{1}{\eps} 
        	\left( F(\wv_0) - F^* + \kappa \normasq{2}{\wv_0 - \wv^*} + \mu D \right)  
        	\left( n + \frac{A_\omega}{\mu \kappa} \right) 
            \right)   \,.
    \end{align*}
\end{proposition}

\section{Extension to Non-Convex Optimization} \label{sec:ncvx_opt}
Let us now turn to the optimization problem~\eqref{eq:c:main:prob}
in full generality where the mappings $\gv\pow{i}$
defined in \eqref{eq:mapping_def} are not constrained to be affine: 
\begin{equation}\label{eq:n-cvx_pb}
	\min_{\wv\in \reals^d} \left[ F(\wv) := \frac{1}{n}\sum_{i=1}^n h(\gv\pow{i}(\wv)) 
		+ \frac{\lambda}{2} \normasq{2}{\wv} \right]\, ,
\end{equation}
where $h$ is a simple, non-smooth, convex function, and each $\gv\pow{i}$ is 
a continuously differentiable nonlinear map and $\lambda \geq 0$. 

We describe the prox-linear algorithm in Sec.~\ref{sec:pl:pl-algo}, followed
by the convergence guarantee in Sec.~\ref{sec:pl:convergence} and the total 
complexity of using Casimir-SVRG together with the prox-linear algorithm in 
Sec.~\ref{sec:pl:total-compl}.

\subsection{The Prox-Linear Algorithm} \label{sec:pl:pl-algo}
The exact prox-linear algorithm  of \citet{burke1985descent} generalizes the 
proximal gradient algorithm (see e.g., \citet{nesterov2013introductory}) 
to compositions of convex functions with smooth mappings such as~\eqref{eq:n-cvx_pb}. 
When given a function $f=h\circ \gv$, the prox-linear algorithm defines a local convex approximation 
$f(\cdot \, ; \wv_k)$ about some point $\wv \in \reals^d$ by linearizing the smooth map $\gv$ as 
$
	f(\wv; \wv_k) := h(\gv(\wv_k) + \grad\gv(\wv_k)(\wv - \wv_k)) \, .
$
With this, it builds a convex model $F(\cdot \, ; \wv_k)$ of $F$ about $\wv_k$ as 
\[
F(\wv ; \wv_k) := \frac{1}{n}\sum_{i=1}^n h(\gv\pow{i}(\wv_k) + \grad\gv\pow{i}(\wv_k)(\wv - \wv_k)) 
		+ \frac{\lambda}{2} \normasq{2}{\wv}\,.
\]
Given a step length $\eta > 0$, each iteration of the exact prox-linear algorithm 
then minimizes the local convex model plus a proximal term as
\begin{align} \label{eq:pl:exact_pl}
	\wv_{k+1} = \argmin_{\wv \in \reals^d} \left[ F_{\eta}( \wv ; \wv_k) := F(\wv ; \wv_k) + \frac{1}{2\eta}\normasq{2}{\wv-\wv_k} \right] \,.
\end{align}

\begin{algorithm}[tb]
	\caption{(Inexact) Prox-linear algorithm: outer loop}
	\label{algo:prox-linear}
	\begin{algorithmic}[1]
		\STATE {\bfseries Input:} Smoothable objective $F$ of the form \eqref{eq:n-cvx_pb} with $h$ simple,
		step length $\eta$,
		tolerances $( \epsilon_k )_{k\ge1}$,
		initial point $\wv_0$,
		non-smooth convex optimization algorithm, $\mcM$, 
		time horizon $K$
		\FOR{$k=1$ \TO $K$}
		\STATE Using $\mcM$ with $\wv_{k-1}$ as the starting point, find
			\begin{align} \label{eq:pl:algo:update}
			\nonumber
				\widehat \wv_k \approx \argmin_{\wv} \bigg[  
					F_\eta(\wv ; \wv_{k-1}) := 
					\frac{1}{n} \sum_{i=1}^n & h\big(\gv\pow{i}(\wv_{k-1}) + \grad\gv\pow{i}(\wv_{k-1})(\wv - \wv_{k-1}) \big) \\ &+ \frac{\lambda}{2}\normasq{2}{\wv}
					+ \frac{1}{2\eta}\normasq{2}{\wv-\wv_{k-1}} \,,
				\bigg]
			\end{align}
			such that
			\begin{align} \label{eq:pl:algo:stop}
   			F_\eta(\widehat \wv_k ; \wv_{k-1}) - \min_{ \wv \in \reals^d} F_\eta(\wv ; \wv_{k-1}) \le \eps_k \,.
   			\end{align}
   			\label{line:pl:algo:subprob}
   		\STATE Set $\wv_k = \widehat \wv_k$ if $F(\widehat \wv_k) \le F(\wv_{k-1})$, else set 
   			$\wv_k = \wv_{k-1}$. 
   		\label{line:pl:algo:accept}
   		\ENDFOR
		\RETURN $\wv_K$.
	\end{algorithmic}
\end{algorithm}

Following \citet{drusvyatskiy2016efficiency}, 
we consider an inexact prox-linear algorithm, which approximately solves \eqref{eq:pl:exact_pl}
using an iterative algorithm. In particular, since the function to be minimized in \eqref{eq:pl:exact_pl}
is precisely of the form~\eqref{eq:cvx_pb}, we employ the fast convex solvers developed in the previous section
as subroutines. Concretely, the prox-linear outer loop is displayed in Algo.~\ref{algo:prox-linear}. 
We now delve into details about the algorithm and convergence guarantees. 
\subsubsection{Inexactness Criterion}
As in Section~\ref{sec:cvx_opt}, we must be prudent in choosing when to terminate the inner optimization
(Line~\ref{line:pl:algo:subprob} of Algo.~\ref{algo:prox-linear}).
Function value suboptimality is used as the inexactness criterion here. In particular, for some specified tolerance
$\eps_k > 0$, iteration $k$ of the prox-linear algorithm accepts a solution $\widehat \wv$ that satisfies
$F_\eta(\widehat \wv_k ; \wv_{k-1}) - \min_{ \wv} F_\eta(\wv ; \wv_{k-1}) \le \eps_k$.

\paragraph{Implementation} 
In view of the $(\lambda + \eta\inv)$-strong convexity of $F_\eta(\cdot \, ; \wv_{k-1})$,
it suffices to ensure that $(\lambda + \eta\inv) \normasq{2}{\vv} \le \eps_k$ for a subgradient 
$\vv \in \partial F_\eta(\widehat \wv_k ; \wv_{k-1})$.

\paragraph{Fixed Iteration Budget}
As in the convex case, we consider as a practical alternative a fixed iteration budget $T_{\mathrm{budget}}$ 
and optimize $F_\eta(\cdot\, ; \wv_k)$ for exactly $T_{\mathrm{budget}}$ iterations of $\mcM$.
Again, we do not have a theoretical analysis for this scheme but find it to be effective in practice. 

\subsubsection{Warm Start of Subproblems}
As in the convex case, we advocate the use of 
the prox center $\wv_{k-1}$ to warm start the inner optimization problem in iteration $k$
(Line~\ref{line:pl:algo:subprob} of Algo.~\ref{algo:prox-linear}).

\subsection{Convergence analysis of the prox-linear algorithm} \label{sec:pl:convergence}
We now state the assumptions and the convergence guarantee of the prox-linear algorithm.
\subsubsection{Assumptions}
For the prox-linear algorithm to work, the only requirement is that we minimize an upper model. 
The assumption below makes this concrete.
\begin{assumption} \label{asmp:pl:upper-bound}
	The map $\gv\pow{i}$ is continuously differentiable everywhere for each $i \in [n]$.
	Moreover, there exists a constant $L > 0$ such that for all $\wv, \wv' \in \reals^d$ and $i\in [n]$, it holds that
	\begin{align*}
	h\big(\gv\pow{i}(\wv') \big) \le 
	h\big(\gv\pow{i}(\wv) + \grad\gv\pow{i}(\wv) (\wv'-\wv) \big) +  \frac{L}{2}\normasq{2}{\wv'-\wv} \,.
	\end{align*}
\end{assumption}
\noindent
When $h$ is $G$-Lipschitz and each $\gv\pow{i}$ is $\widetilde L$-smooth, both with respect to
$\norma{2}{\cdot}$, then Assumption~\ref{asmp:pl:upper-bound} holds with $L = G\widetilde L$
\citep{drusvyatskiy2016efficiency}.
In the case of structured prediction,
Assumption~\ref{asmp:pl:upper-bound} holds when
the augmented score $\psi$ as a function of $\wv$ is $L$-smooth.
The next lemma makes this precise and its proof is in Appendix~\ref{sec:c:pl_struct_pred}.
\begin{lemma} \label{lem:pl:struct_pred}
	Consider the structural hinge loss $f(\wv) =  \max_{\yv \in \mcY} \psi(\yv ; \wv) = h\circ \gv(\wv)$ 
	where $h, \gv$ are as defined in \eqref{eq:mapping_def}.
	If the mapping $\wv \mapsto \psi(\yv ; \wv)$ is $L$-smooth with respect to $\norma{2}{\cdot}$ for all 
	$\yv \in \mcY$, then it holds for all $\wv, \zv \in \reals^d$ that
	\begin{align*}
	|h(\gv(\wv+\zv)) - h(\gv(\wv) + \grad\gv(\wv) \zv)| \le  \frac{L}{2}\normasq{2}{\zv}\,.
	\end{align*}
\end{lemma}

\subsubsection{Convergence Guarantee}
Convergence is measured via the norm of the {\em prox-gradient} $\proxgrad_\eta(\cdot)$,
also known as the {\em gradient mapping}, defined as 
\begin{align}
\proxgrad_\eta(\wv) = \frac{1}{\eta} \left( \wv - \argmin_{\zv \in \reals^d} F_\eta(\zv ; \wv)  \right) \,.
\end{align}
The measure of stationarity $\norm{\proxgrad_\eta(\wv)}$ turns out to be related 
to the norm of the gradient of the Moreau envelope of $F$ under certain conditions - see 
\citet[Section 4]{drusvyatskiy2016efficiency} for a discussion. 
In particular, a point $\wv$ with small $\norm{\proxgrad_\eta(\wv)}$ means that $\wv$ is close to 
$\wv' = \argmin_{\zv \in \reals^d} F_\eta(\zv ; \wv)$, which is nearly stationary for $F$.

The prox-linear outer loop shown in Algo.~\ref{algo:prox-linear} has the following convergence guarantee
\citep[Thm.~5.2]{drusvyatskiy2016efficiency}.
\begin{theorem} \label{thm:pl:outer-loop}
	Consider $F$ of the form~\eqref{eq:n-cvx_pb} that satisfies Assumption~\ref{asmp:pl:upper-bound}, 
	a step length $0 < \eta \le 1/L$ and a non-negative sequence $(\eps_k)_{k\ge1}$. 
	With these inputs, Algo.~\ref{algo:prox-linear} produces a sequence $(\wv_k)_{k \ge 0}$ that satisfies
	\begin{align*}
		\min_{k=0, \cdots, K-1} \normasq{2}{\proxgrad_\eta(\wv_k)} \le \frac{2}{\eta K} \left( F(\wv_0) - F^* + \sum_{k=1}^{K} \eps_k  \right) \,,
	\end{align*}
	where $F^* = \inf_{\wv \in \reals^d} F(\wv)$.
	In addition, we have that the sequence $(F(\wv_k))_{k\ge0}$ is non-increasing.
\end{theorem} 

\begin{remark}
	Algo.~\ref{algo:prox-linear} accepts an update only if it improves the function value (Line~\ref{line:pl:algo:accept}).
	A variant of Algo.~\ref{algo:prox-linear} which always accepts the update has a guarantee identical to 
	that of Thm.~\ref{thm:pl:outer-loop}, 
	but the sequence $(F(\wv_k))_{k\ge0}$ would not guaranteed to be non-increasing.
\end{remark}

\subsection{Prox-Linear with \nsCatalystSvrg{}} \label{sec:pl:total-compl}
We now analyze the total complexity of minimizing the finite sum problem~\eqref{eq:n-cvx_pb}
with \nsCatalystSvrg{} to approximately solve the subproblems of Algo.~\ref{algo:prox-linear}.

For the algorithm to converge, the map 
$\wv \mapsto \gv\pow{i}(\wv_k) + \grad \gv\pow{i}(\wv_k)(\wv - \wv_k)$ must be Lipschitz for each $i$ and each iterate $\wv_k$. 
To be precise, we assume that 
\begin{align}
	A_\omega := \max_{i=1,\cdots,n} \sup_{\wv \in \reals^d} \normasq{2, \alpha}{\grad \gv\pow{i}(\wv)} 
\end{align}
is finite, where $\omega$, the smoothing function, is 1-strongly convex 
with respect to $\norma{\alpha}{\cdot}$.
When $\gv\pow{i}$ is the linear map $\wv \mapsto \Am\pow{i}\wv$, this reduces to \eqref{eq:c:A_defn}.

We choose the tolerance $\eps_k$ to decrease as $1/k$. 
When using the \nsCatalystSvrg{} algorithm with constant smoothing (Prop.~\ref{prop:c:total_compl_svrg_sc})
as the inner solver, this method effectively smooths the $k$th prox-linear subproblem as $1/k$.
We have the following rate of convergence for this method, which is proved in Appendix~\ref{sec:c:pl_proofs}.
\begin{proposition} \label{prop:pl:total_compl}
	Consider the setting of Thm.~\ref{thm:pl:outer-loop}. Suppose the sequence $(\eps_k)_{k\ge 1}$
	satisfies $\eps_k = \eps_0 / k$ for some $\eps_0 > 0$ and that 
	the subproblem of Line~\ref{line:pl:algo:subprob} of Algo.~\ref{algo:prox-linear} is solved using 
	\nsCatalystSvrg{} with the settings of Prop.~\ref{prop:c:total_compl_svrg_sc}.
	Then, total number of SVRG iterations $N$ required to produce a $\wv$ such that 
	$\norma{2}{\proxgrad_\eta(\wv)} \le \eps$ is bounded as
	\begin{align*}
		\expect[N] \le \widetilde\bigO\left(
			\frac{n}{\eta \eps^2} \left(F(\wv_0) - F^* + \eps_0 \right) + 
			\frac{\sqrt{A_\omega D_\omega n \eps_0\inv}}{\eta \eps^3} \left( F(\wv_0) - F^* + \eps_0 \right)^{3/2}
			\right) \, .
	\end{align*}
\end{proposition}
%
%
\begin{remark} \label{remark:pl:choosing_eps0}
	When an estimate or an upper bound $B$ on $F(\wv_0) - F^*$, one could set 
	$\eps_0 = \bigO(B)$. This is true, for instance, in the structured prediction task where 
	$F^* \ge 0$ whenever the task loss $\ell$ is non-negative (cf.~\eqref{eq:pgm:struc_hinge}). 
\end{remark}

\section{Experiments} \label{sec:expt}
In this section, we study the experimental behavior of the proposed algorithms on 
two structured prediction tasks, namely named entity recognition and visual object localization.
Recall that given training examples $\{ (\xv\pow{i}, \yv\pow{i})\}_{i=1}^n$, we wish to solve the problem:
\begin{align*}
	\min_{\wv\in\reals^d} \left[ F(\wv) := \frac{\lambda}{2}\normasq{2}{\wv} + 
		\frac{1}{n} \sum_{i=1}^n  \max_{\yv' \in \mcY(\xv\pow{i})} \left\{  
			\phi(\xv\pow{i}, \yv' ; \wv) + \ell(\yv\pow{i}, \yv')
		\right\} - \phi(\xv\pow{i}, \yv\pow{i} ; \wv) 
	\right] \,.
\end{align*}
Note that we now allow the output space $\mcY(\xv)$ to depend on the instance $\xv$ - the analysis 
from the previous sections applies to this setting as well.
In all the plots, the shaded region represents one standard deviation over ten random runs.


We compare the performance of various optimization algorithms based on the number of 
calls to a smooth inference oracle.
Moreover, following literature for algorithms based on SVRG~\citep{schmidt2017minimizing,lin2017catalyst}, 
we exclude the cost of computing the full gradients.

The results must be interpreted keeping in mind that the running time of all inference oracles is not the same.
These choices were motivated by the following reasons, which may not be appropriate in all contexts.
	The ultimate yardstick to benchmark the performance of optimization algorithms is 
		wall clock time. However, this depends heavily on implementation, system and ambient system conditions.
		With regards to the differing running times of different oracles,
		we find that a small value of $K$, e.g., 5 suffices, 
		so that our highly optimized implementations of the top-$K$ oracle incurs negligible 
		running time penalties over the max oracle.
	Moreover, the computations of the batch gradient have been neglected as they are embarrassingly parallel. 

The outline of the rest of this section is as follows.
First, we describe the datasets and task description in Sec.~\ref{subsec:expt:task_description},
followed by methods compared in Sec.~\ref{subsec:expt:competing_methods} and 
their hyperparameter settings in Sec.~\ref{subsec:expt:hyperparam}.
Lastly, Sec.~\ref{subsec:expt:competing_results} presents the experimental studies.

\subsection{Dataset and Task Description} \label{subsec:expt:task_description}
 For each of the tasks, we specify below the following:
(a) the dataset $\{ (\xv\pow{i}, \yv\pow{i})\}_{i=1}^n$,
(b) the output structure $\mcY$,
(c) the loss function $\ell$,
(d) the score function $\phi(\xv, \yv ; \wv)$,
(e) implementation of inference oracles, and lastly,
(f) the evaluation metric used to assess the quality of predictions.

\subsubsection{CoNLL 2003: Named Entity Recognition}
Named entities are phrases that contain the names of persons, organization, locations, etc,
and the task is to predict the label (tag) of each entity.
Named entity recognition can be formulated as a sequence tagging problem where the set 
$\mcY_{\mathrm{tag}}$ of individual tags is of size 7.

Each datapoint $\xv$ is a sequence of words $\xv = (x_1, \cdots, x_p)$, 
and the label $\yv = (y_1, \cdots, y_p) \in \mcY(\xv)$ is a sequence of the same length, 
where each $y_i \in \mcY_{\mathrm{tag}}$ is a tag.

\paragraph{Loss Function}
The loss function is the Hamming Loss $\ell(\yv, \yv') = \sum_i \ind(y_i \neq y_i')$. 

\paragraph{Score Function}
We use a chain graph to represent this task. In other words, 
the observation-label dependencies are encoded as a Markov chain of order 1 to enable efficient
inference using the Viterbi algorithm.
We only consider the case of linear score $\phi(\xv, \yv ; \wv) = \inp{\wv}{\Phi(\xv, \yv)}$
for this task. The feature map $\Phi$ here is very similar to that given in 
Example~\ref{example:inf_oracles:viterbi_example}.
Following \citet{tkachenko2012named}, we use local context $\Psi_i(\xv)$ around $i$\textsuperscript{th} word $x_i$ of $\xv$. 
In particular, define $\Psi_i(\xv) = \ev_{x_{i-2}} \otimes \cdots \otimes \ev_{x_{i+2}}$, 
where $\otimes$ denotes the Kronecker product between column vectors,
and $\ev_{x_i}$ denotes a one hot encoding of word $x_i$, concatenated with the one hot encoding of its
the part of speech tag and syntactic chunk tag which are provided with the input. 
Now, we can define the feature map $\Phi$ as
\begin{align*}
    \Phi(\xv, \yv) = \left[ \sum_{v=1}^p \Psi_v(\xv) \otimes \ev_{y_v} \right] \oplus
        \left[ \sum_{i=0}^p \ev_{y_{v}} \otimes \ev_{y_{v+1}} \right] \,,
\end{align*}
where $\ev_y \in \reals^{\abs{\mcY_{\mathrm{tag}}}}$ is a one hot-encoding of $y \in \mcY_{\mathrm{tag}}$,
and $\oplus$ denotes vector concatenation. 

\paragraph{Inference}
We use the Viterbi algorithm as the max oracle (Algo.~\ref{algo:dp:max:chain}) 
and top-$K$ Viterbi algorithm (Algo.~\ref{algo:dp:topK:chain}) for the top-$K$ oracle.

\paragraph{Dataset}
The dataset used was CoNLL 2003 \citep{tjong2003introduction},
which contains about $\sim 20K$ sentences.

\paragraph{Evaluation Metric}
We follow the official CoNLL metric: the $F_1$ measure excluding the `O' tags.
In addition, we report the objective function value measured on the training set (``train loss'').

\paragraph{Other Implementation Details}
The sparse feature vectors obtained above are hashed onto $2^{16} - 1$ dimensions for efficiency.

\subsubsection{PASCAL VOC 2007: Visual Object Localization}
Given an image and an object of interest, the task is to localize the object in the given image,
i.e., determine the best bounding box around the object. A related, but harder task is object detection,
which requires identifying and localizing any number of objects of interest, if any, in the image. 
Here, we restrict ourselves to pure localization with a single instance of each object.
Given an image $\xv \in \mcX$ of size $n_1 \times n_2$, the label $\yv \in \mcY(\xv)$
is a bounding box, where $\mcY(\xv)$ is the set of all bounding boxes in an image of size $n_1 \times n_2$.
Note that $\abs{\mcY(\xv)} = \bigO(n_1^2n_2^2)$.

\paragraph{Loss Function}
The PASCAL IoU metric  \citep{everingham2010pascal} is used to measure the quality of localization. 
Given bounding boxes $\yv, \yv'$, the IoU is defined as the ratio of the intersection of the 
bounding boxes to the union:
\begin{align*}
	\mathrm{IoU}(\yv, \yv') = \frac{\mathrm{Area}(\yv \cap \yv')}{\mathrm{Area}(\yv \cup \yv')} \,.
\end{align*}
We then use the $1 - \mathrm{IoU}$ loss defined as $\ell(\yv, \yv') = 1 - \mathrm{IoU}(\yv, \yv')$.

\paragraph{Score Function}
The formulation we use is based on the popular R-CNN approach \citep{girshick2014rich}.
We consider two cases: linear score and non-linear score $\phi$, both of which are based on the following 
definition of the feature map $\Phi(\xv, \yv)$.
\begin{itemize}
	\item Consider a patch $\xv|_\yv$ of image $\xv$ cropped to box $\yv$, 
		and rescale it to $64\times 64$. 
		Call this $\Pi(\xv|_\yv)$.
	\item Consider a convolutional neural network known as AlexNet \citep{krizhevsky2012imagenet}
		pre-trained on ImageNet \citep{ILSVRC15} and 
		pass $\Pi(\xv|_\yv)$ through it. 
		Take the output of {\tt conv4}, the penultimate convolutional layer as the feature map $\Phi(\xv, \yv)$.
		It is of size $ 3 \times 3\times 256$.
\end{itemize}

In the case of linear score functions, we take $\phi(\xv, \yv ; \wv) = \inp{\wv}{\Phi(\xv, \yv)}$.
In the case of non-linear score functions, we define the score $\phi$ as the the result of 
a convolution composed with a non-linearity and followed by a linear map. Concretely, 
for $\thetav \in \reals^{H \times W \times C_1}$ and $\wv \in \reals^{C_1 \times C_2}$
let the map $\thetav \mapsto \thetav \star \wv \in \reals^{H \times W \times C_2}$ denote a 
two dimensional convolution with stride $1$ and kernel size $1$,
and $\sigma: \reals \to \reals$ denote the exponential linear unit, defined respectively as
\begin{align*}
	[\thetav \star \wv]_{ij} = \wv\T [\thetav]_{ij} \quad \text{and} 
	\quad \sigma(x) = x \, \ind(x \ge 0) + (\exp(x) - 1) \, \ind(x < 0) \,,
\end{align*}
where $[\thetav]_{ij} \in \reals^{C_1}$ is such that its $l$th entry is $\thetav_{ijl}$ 
and likewise for $[\thetav \star \wv]_{ij}$.
We overload notation to let $\sigma:\reals^d\to \reals^d$ denote the exponential linear unit applied element-wise.
Notice that $\sigma$ is smooth.
The non-linear score function $\phi$ is now defined, with 
$\wv_1 \in \reals^{256\times16}, \wv_2 \in \reals^{16\times3\times3}$ and $\wv=(\wv_1, \wv_2)$, as,
\begin{align*}
	\phi(\xv, \yv ; \wv) = \inp{\sigma(\Phi(\xv, \yv) \star \wv_1)}{\wv_2} \,.
\end{align*}

\paragraph{Inference}
For a given input image $\xv$, we follow the R-CNN approach~\citep{girshick2014rich} and use 
selective search~\citep{van2011segmentation} to prune the search space.
In particular, for an image $\xv$, we use the selective search implementation provided by OpenCV~\citep{opencv_library}
and take the top 1000 candidates returned to be the set $\widehat{\mcY}(\xv)$, 
which we use as a proxy for $\mcY(\xv)$.
The max oracle and the top-$K$ oracle are then implemented as exhaustive searches over 
this reduced set $\widehat\mcY(\xv)$.

\paragraph{Dataset}
We use the PASCAL VOC 2007 dataset \citep{everingham2010pascal}, which contains 
$\sim 5K$ annotated consumer (real world) images shared on the photo-sharing site Flickr
from 20 different object categories.
For each class, we consider all images with only a single occurrence of the object, and train 
an independent model for each class.

\paragraph{Evaluation Metric}
We keep track of two metrics.
The first is the localization accuracy, also known as CorLoc (for correct localization), 
following \citet{deselaers2010localizing}. A bounding box with IoU $> 0.5$
with the ground truth is considered correct and the localization accuracy is the fraction
of images labeled correctly. 
The second metric is average precision (AP), which
requires a confidence score for each prediction. 
We use $\phi(\xv, \yv' ; \wv)$ as the confidence score of $\yv'$.
As previously, we also plot the objective function value measured on the training examples. 

\paragraph{Other Implementation Details}
For a given input-output pair $(\xv, \yv)$ in the dataset, 
we instead use $(\xv, \widehat \yv)$ as a training example, where
$\widehat\yv = \argmax_{\yv' \in \widehat\mcY(\xv)} \mathrm{IoU}(\yv, \yv')$
is the element of $\widehat\mcY(\xv)$ which overlaps the most with the true output $\yv$.

\subsection{Methods Compared} \label{subsec:expt:competing_methods}
The experiments compare various convex stochastic and incremental 
optimization methods for structured prediction.
\begin{itemize}
  \item {\bfseries \SGD}: Stochastic subgradient method with a learning rate $\gamma_t = \gamma_0 / (1 + \lfloor t / t_0 \rfloor)$,
  		where $\eta_0, t_0$ are tuning parameters. Note that this scheme of learning rates does not have 
  		a theoretical analysis. However, the averaged iterate $\overline \wv_t = {2}/(t^2+t)\sum_{\tau=1}^t \tau \wv_\tau$
  		obtained from the related scheme
  		$\gamma_t = 1/(\lambda t)$ was shown to have a convergence rate of $\bigO((\lambda \eps)\inv)$ 
  		\citep{shalev2011pegasos,lacoste2012simpler}. It works on the non-smooth formulation directly.
  \item {\bfseries \bcfw}: The block coordinate Frank-Wolfe algorithm of \citet{lacoste2012block}.
  		We use the version that was found to work best in practice, namely,
		one that uses the weighted averaged iterate $\overline \wv_t = {2}/(t^2+t)\sum_{\tau=1}^t \tau \wv_\tau$
		(called {\tt bcfw-wavg} by the authors) 
		with optimal tuning of learning rates. This algorithm also works on the non-smooth formulation
		and does not require any tuning.
  \item {\bfseries \svrg}: The SVRG algorithm proposed by \citet{johnson2013accelerating},
	    with each epoch making one pass through the dataset and using the averaged iterate
	    to compute the full gradient and restart the next epoch. 
	    This algorithm requires smoothing.
  \item {\bfseries \nsCatalystExpt}: Algo.~\ref{algo:catalyst} with SVRG as the inner optimization algorithm.
   		The parameters $\mu_k$ and $\kappa_k$ as chosen in
   		Prop.~\ref{prop:c:total_compl_svrg_sc}, where $\mu$ and $\kappa$ are hyperparameters.
	    This algorithm requires smoothing.
   \item {\bfseries \nsCatalystDecayExpt}: Algo.~\ref{algo:catalyst} with SVRG as the inner optimization algorithm.
   		The parameters $\mu_k$ and $\kappa_k$ as chosen in
   		Prop.~\ref{prop:c:total_compl_sc:dec_smoothing}, where $\mu$ and $\kappa$ are hyperparameters.
   		This algorithm requires smoothing.
\end{itemize}
On the other hand, for non-convex structured prediction, we only have two methods:
\begin{itemize}
	\item {\bfseries \SGD}: The stochastic subgradient method \citep{davis2018stochastic}, which we call as SGD. This 
		algorithm works directly on the non-smooth formulation. We try learning rates 
		$\gamma_t = \gamma_0$, $\gamma_t = \gamma_0 /\sqrt{t}$
		and $\gamma_t = \gamma_0 / t$, where $\gamma_0$ is found by grid search in each of these cases.
		We use the names SGD-const, SGD-$t^{-1/2}$ and SGD-$t^{-1}$ respectively for these variants.
		We note that SGD-$t^{-1}$ does not have any theoretical analysis in the non-convex case. 
	\item {\bfseries \plcsvrg}: Algo.~\ref{algo:prox-linear} with \nsCatalystExpt{} as the inner solver using the settings 
		of Prop.~\ref{prop:pl:total_compl}. This algorithm requires smoothing the inner subproblem.
\end{itemize}

\subsection{Hyperparameters and Variants} \label{subsec:expt:hyperparam}

\paragraph{Smoothing}
In light of the discussion of Sec.~\ref{sec:smooth_oracle_impl}, 
we use the $\ell_2^2$ smoother $\omega(\uv) = \normasq{2}{\uv} / 2$ and use the top-$K$ strategy 
for efficient computation.
We then have $D_\omega = 1/2$.

\paragraph{Regularization} 
The regularization coefficient $\lambda$ is chosen as $\nicefrac{c}{n}$,
where $c$ is varied in $\{ 0.01, 0.1, 1, 10\}$.

\paragraph{Choice of $K$}
The experiments use $K = 5$ for named entity recognition where the performance of the top-$K$ 
oracle is $K$ times slower, 
and $K=10$ for visual object localization, where the running time of the top-$K$ oracle is independent of $K$.
We also present results for other values of $K$ in Fig.~\ref{fig:plot_ner_K} and find that 
the performance of the tested algorithms is robust to the value of $K$.

\paragraph{Tuning Criteria}
Some algorithms require tuning one or more hyperparameters such as the learning rate.
We use grid search to find the best choice of the hyperparameters using the following criteria:
For the named entity recognition experiments, the train function value and the validation $F_1$ metric 
were only weakly correlated. For instance, the 3 best learning rates in the grid in terms of $F_1$ score, 
the best $F_1$ score attained the worst train function value and vice versa.
Therefore, we choose the value of the tuning parameter that attained the best objective function value within 1\% of the 
best validation $F_1$ score in order to measure the optimization performance while still remaining relevant
to the named entity recognition task.
For the visual object localization task, 
a wide range of  hyperparameter values achieved nearly equal performance in terms of 
the best CorLoc over the given time horizon, so we choose 
the value of the hyperparameter that achieves the best objective function value within 
a given iteration budget.


\subsubsection{Hyperparameters for Convex Optimization}
This corresponds to the setting of Section~\ref{sec:cvx_opt}.

\paragraph{Learning Rate}
The algorithms \svrg{} and \nsCatalystDecayExpt{} require tuning of a learning rate, 
while SGD requires $\eta_0, t_0$ and
 \nsCatalystExpt{} requires tuning of the Lipschitz constant $L$ of $\grad F_{\mu\omega}$, 
which determines the learning rate $\gamma = 1/(L + \lambda + \kappa)$. 
Therefore, tuning the Lipschitz parameter is similar to tuning the learning rate.
For both the learning rate and Lipschitz parameter, we use grid search on a logarithmic grid,
with consecutive entries chosen a factor of two apart. 

\paragraph{Choice of $\kappa$}
For \nsCatalystExpt{}, with the Lipschitz constant in hand, the parameter $\kappa$
is chosen to minimize the overall complexity as in Prop.~\ref{prop:c:total_compl_svrg_sc}.
For \nsCatalystDecayExpt{}, we use $\kappa = \lambda$.

\paragraph{Stopping Criteria}
Following the discussion of Sec.~\ref{sec:cvx_opt}, we use 
an iteration budget of $T_{\mathrm{budget}} = n$.

\paragraph{Warm Start}
The warm start criterion determines the starting iterate of an epoch of the inner optimization algorithm. 
Recall that we solve the following subproblem using SVRG for the $k$th iterate (cf. \eqref{eq:prox_point_algo}):
\begin{align*}
	\wv_k \approx \argmin_{\wv \in \reals^d} F_{\mu_k\omega, \kappa_k}(\wv_k;\zv_{k-1}) \,.
\end{align*}
Here, we consider the following warm start strategy to choose the initial iterate $\widehat \wv_0$ for this subproblem:
\begin{itemize}
	\item {\tt Prox-center}: $\widehat \wv_0 = \zv_{k-1}$.
\end{itemize}
In addition, we also try out the following warm start strategies of \citet{lin2017catalyst}:
\begin{itemize}
	\item {\tt Extrapolation}: $\widehat \wv_0 = \wv_{k-1} + c(\zv_{k-1} - \zv_{k-2})$ where $c = \frac{\kappa}{\kappa + \lambda}$.
	\item {\tt Prev-iterate}: $\widehat \wv_0 = \wv_{k-1}$.
\end{itemize}
We use the {\tt Prox-center} strategy unless mentioned otherwise.

\paragraph{Level of Smoothing and Decay Strategy}
For \svrg{} and \nsCatalystExpt{} with constant smoothing, we try various values of the smoothing
parameter in a logarithmic grid. On the other hand, \nsCatalystDecayExpt{} is more robust to the choice of 
the smoothing parameter (Fig.~\ref{fig:plot_ner_smoothing}). 
We use the defaults of $\mu = 2$ for named entity recognition and $\mu = 10$ for 
visual object localization.

\subsubsection{Hyperparameters for Non-Convex Optimization}
This corresponds to the setting of Section~\ref{sec:ncvx_opt}.

\paragraph{Prox-Linear Learning Rate $\eta$}
We perform grid search in powers of 10 to find the best prox-linear learning rate $\eta$.
We find that the performance of the algorithm is robust to the choice of $\eta$ (Fig.~\ref{fig:ncvx:pl_lr}).

\paragraph{Stopping Criteria}
We used a fixed budget of 5 iterations of \nsCatalystExpt{}.
In Fig.~\ref{fig:ncvx:inner-iter},
we experiment with different iteration budgets.

\paragraph{Level of Smoothing and Decay Strategy}
In order to solve the $k$th prox-linear subproblem with \nsCatalystExpt{}, 
we must specify the level of smoothing $\mu_k$. We experiment with two schemes, 
(a) constant smoothing $\mu_k = \mu$, and (b) adaptive smoothing $\mu_k = \mu / k$.
Here, $\mu$ is a tuning parameters, and the adaptive smoothing scheme is designed 
based on Prop.~\ref{prop:pl:total_compl} and Remark~\ref{remark:pl:choosing_eps0}.
We use the adaptive smoothing strategy as a default, but compare the two in Fig.~\ref{fig:ncvx_smoothing}.

\paragraph{Gradient Lipschitz Parameter for Inner Optimization}
The inner optimization algorithm \nsCatalystExpt{} still requires a hyperparameter 
$L_k$ to serve as an estimate to the Lipschitz parameter of the gradient
$\grad F_{\eta, \mu_k\omega}(\cdot\,; \wv_k)$. We set this parameter as follows, 
based on the smoothing strategy:
(a) $L_k = L_0$ with the constant smoothing strategy, and 
(b) $L_k = k\, L_0$ with the adaptive smoothing strategy (cf. Prop.~\ref{thm:setting:beck-teboulle}).
We note that the latter choice has the effect of decaying the learning rate as $~1/k$ 
in the $k$th outer iteration.


\begin{figure}[t!]
    \centering
        \includegraphics[width=\textwidth]{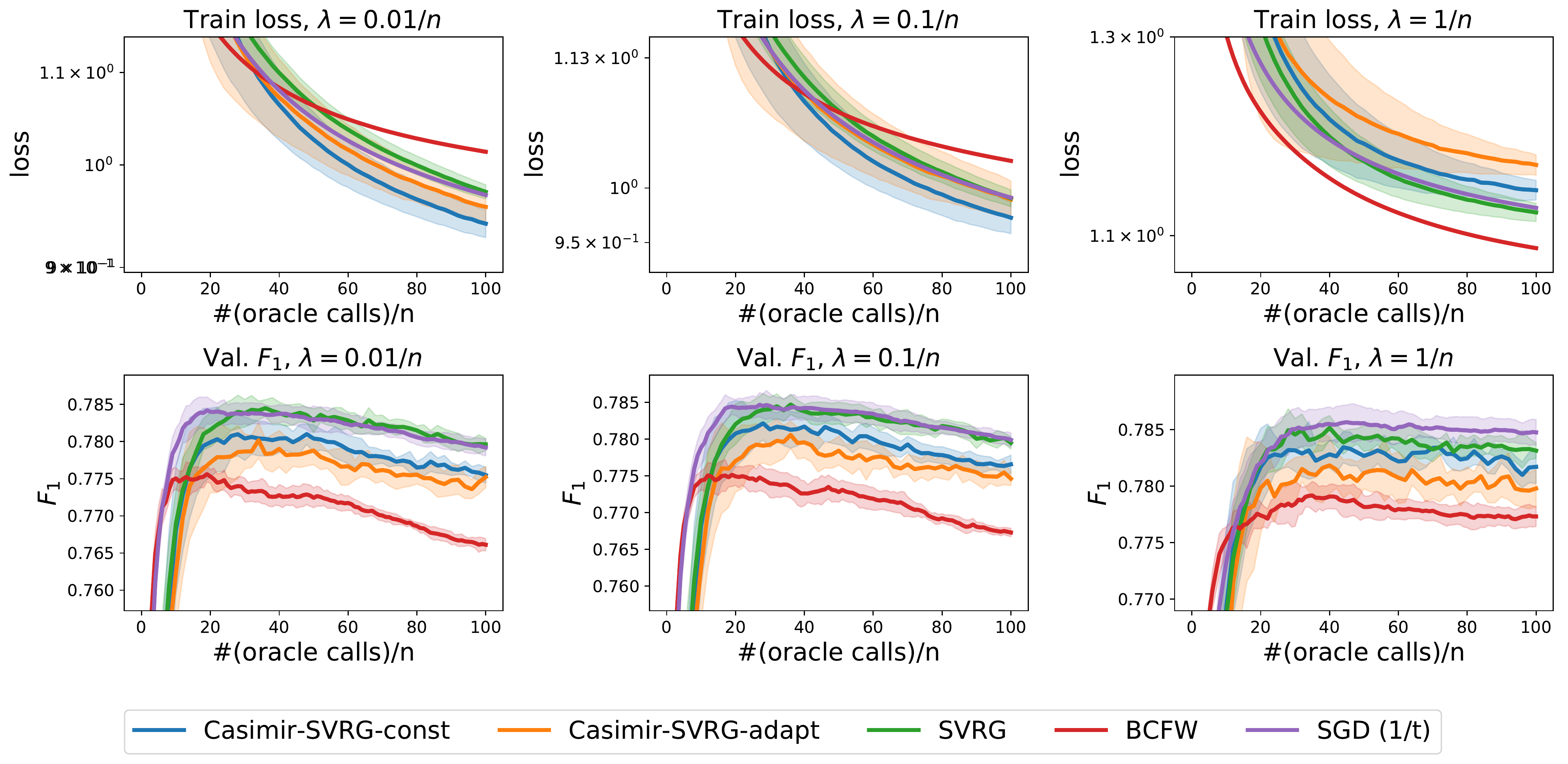}
    \caption{Comparison of convex optimization algorithms for the task of Named Entity Recognition on CoNLL 2003.}\label{fig:plot_all_ner}
\end{figure}
\begin{figure}[!thb]
    \centering
 	\includegraphics[width=0.93\textwidth]{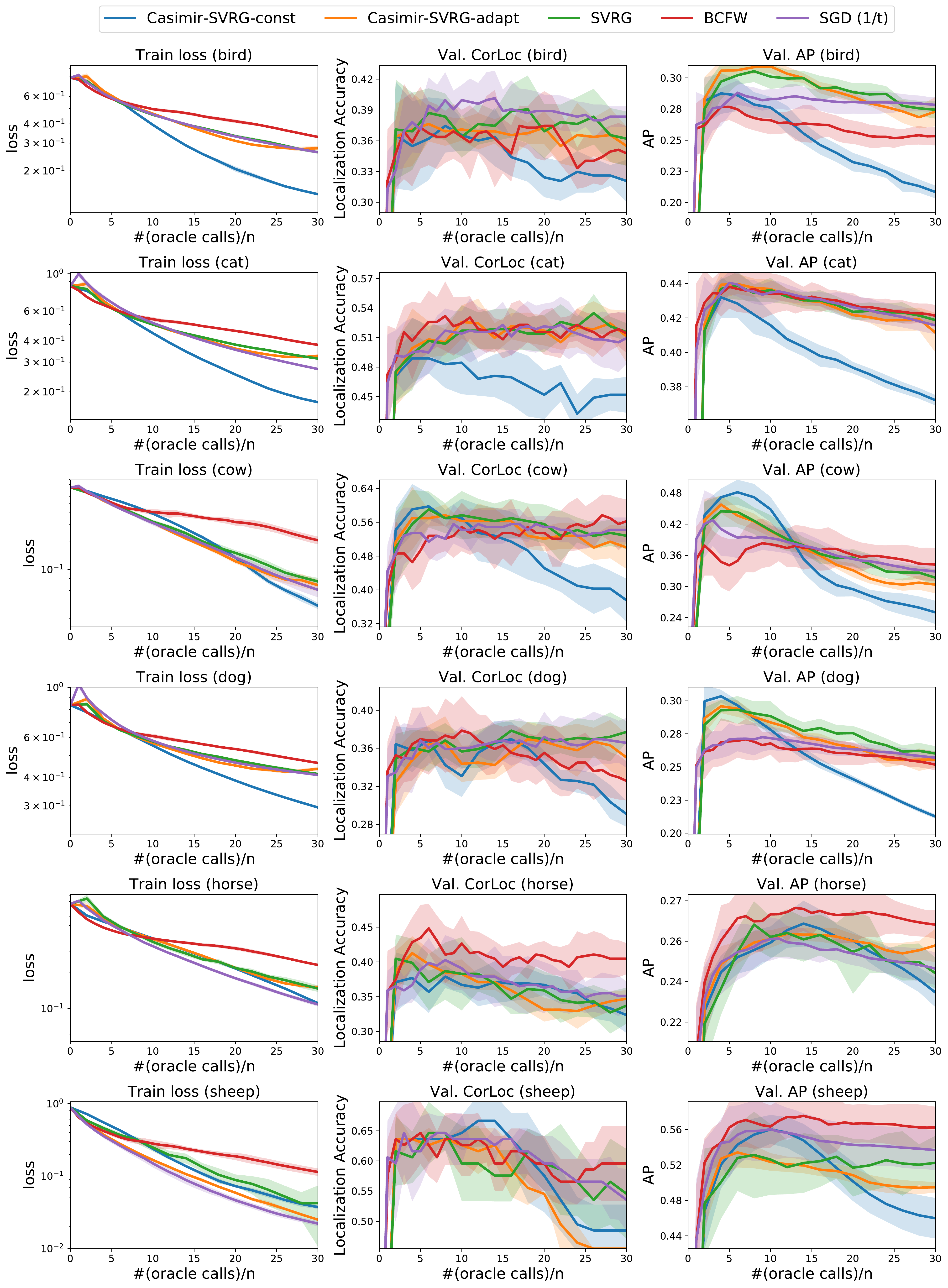}
    \caption{Comparison of convex optimization algorithms 
    	for the task of visual object localization on PASCAL VOC 2007 for $\lambda=10/n$.
    	Plots for all other classes are in 
    	Appendix~\ref{sec:a:expt}.}\label{fig:plot_all_loc}
\end{figure}

\begin{figure}[!thb]
    \centering
 	\includegraphics[width=0.93\textwidth]{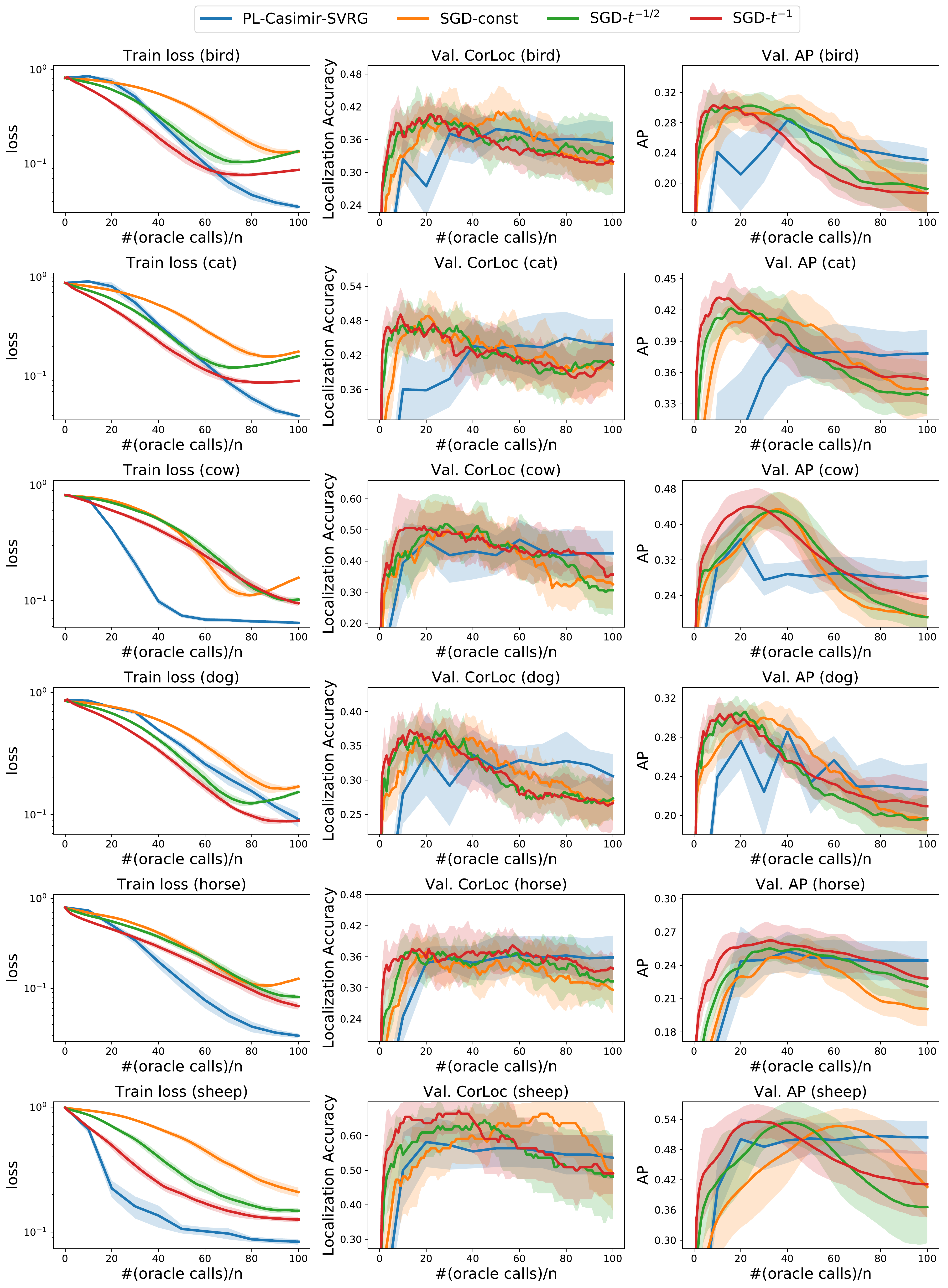}
    \caption{Comparison of non-convex optimization algorithms 
    	for the task of visual object localization on PASCAL VOC 2007 for $\lambda=1/n$.
    	Plots for all other classes are in 
    	Appendix~\ref{sec:a:expt}.}\label{fig:plot_ncvx_loc}
\end{figure}

\subsection{Experimental study of different methods} \label{subsec:expt:competing_results}
\paragraph{Convex Optimization}
For the named entity recognition task,
Fig.~\ref{fig:plot_all_ner} plots the performance of various methods on CoNLL 2003.
On the other hand, Fig.~\ref{fig:plot_all_loc} presents 
plots for various classes of PASCAL VOC 2007 for visual object localization.

The plots reveal that smoothing-based methods converge faster in terms of training error 
while achieving a competitive performance in terms of the performance metric on a held-out set. 
Furthermore, BCFW and SGD make twice as many actual passes as SVRG based algorithms. 

\paragraph{Non-Convex Optimization}
Fig.~\ref{fig:plot_ncvx_loc} plots the performance of various algorithms on the task of visual object localization
on PASCAL VOC.

\subsection{Experimental Study of Effect of Hyperparameters: Convex Optimization} 
We now study the effects of various hyperparameter choices.
\paragraph{Effect of Smoothing} 
Fig.~\ref{fig:plot_ner_smoothing} plots the effect of the level of smoothing for \nsCatalystExpt{}
and \nsCatalystDecayExpt{}. The plots reveal that, in general, small values of the smoothing parameter lead
to better optimization performance for \nsCatalystExpt. \nsCatalystDecayExpt{} is robust to the choice
of $\mu$.
Fig.~\ref{fig:plot_ner_smoothing-2} shows how the smooth optimization algorithms work when used heuristically on the 
non-smooth problem.
\begin{figure}[!thb]
    \centering
    \begin{subfigure}[b]{0.88\linewidth}
    \centering
        \includegraphics[width=\textwidth]{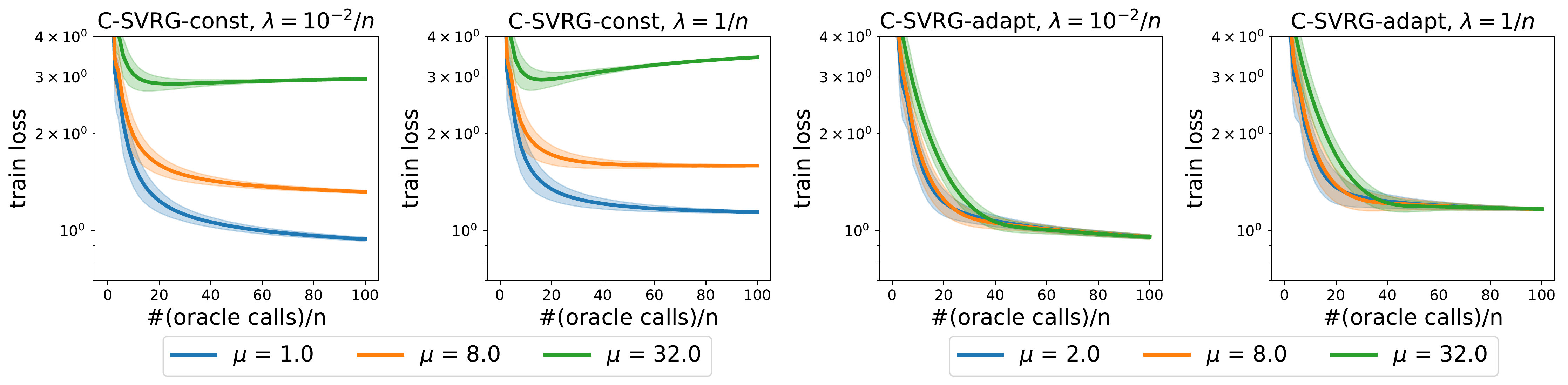}
        \caption{\small{Effect of level of smoothing.}}
        \label{fig:plot_ner_smoothing}
    \end{subfigure} 

    \begin{subfigure}[b]{0.88\linewidth}
    \centering
        \includegraphics[width=\textwidth]{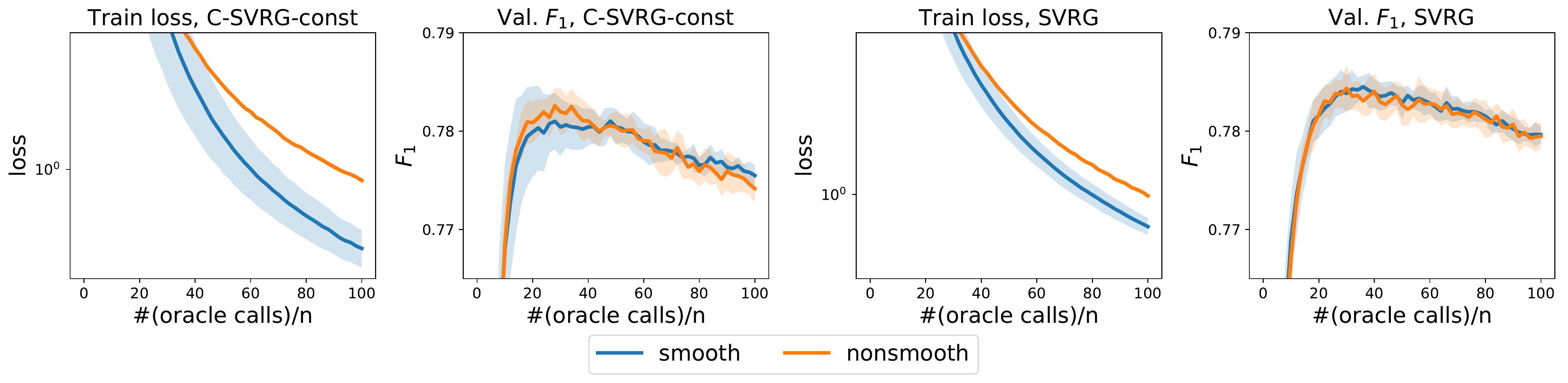}
        \caption{\small{Effect of smoothing: use of smooth optimization with smoothing (labeled ``smooth'')
        		versus the heuristic use of these
        		algorithms without smoothing (labeled ``non-smooth'') for $\lambda = 0.01/n$.}}
        \label{fig:plot_ner_smoothing-2}
    \end{subfigure} 

    \begin{subfigure}[b]{0.88\linewidth}
    \centering
        \includegraphics[width=\textwidth]{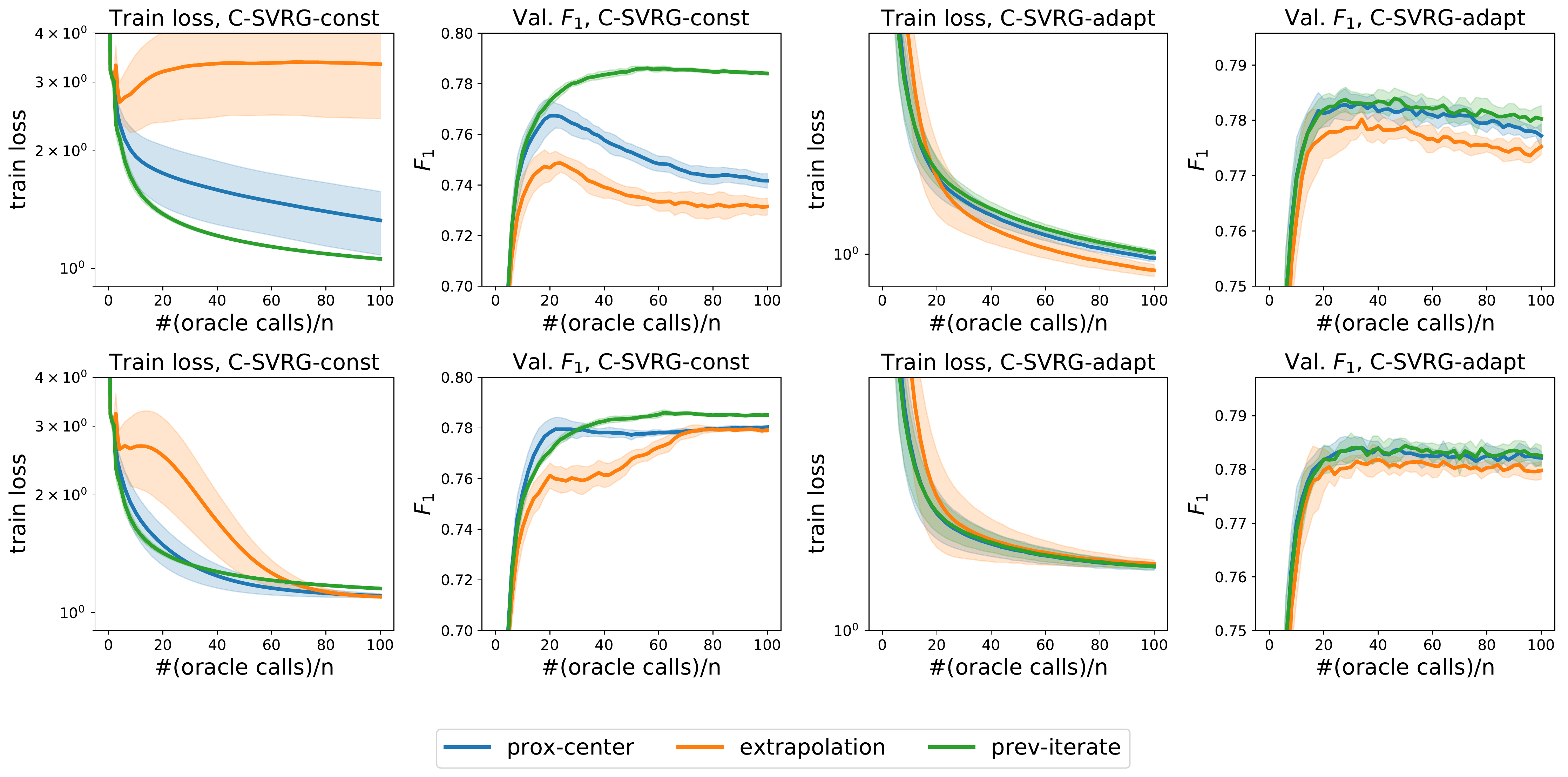}
        \caption{\small{Effect of warm start strategies for $\lambda=0.01/n$ (first row) and $\lambda = 1/n$ (second row).}}
        \label{fig:plot_ner_warm-start}
    \end{subfigure} 

    \begin{subfigure}[b]{0.88\linewidth}
    \centering
        \includegraphics[width=\textwidth]{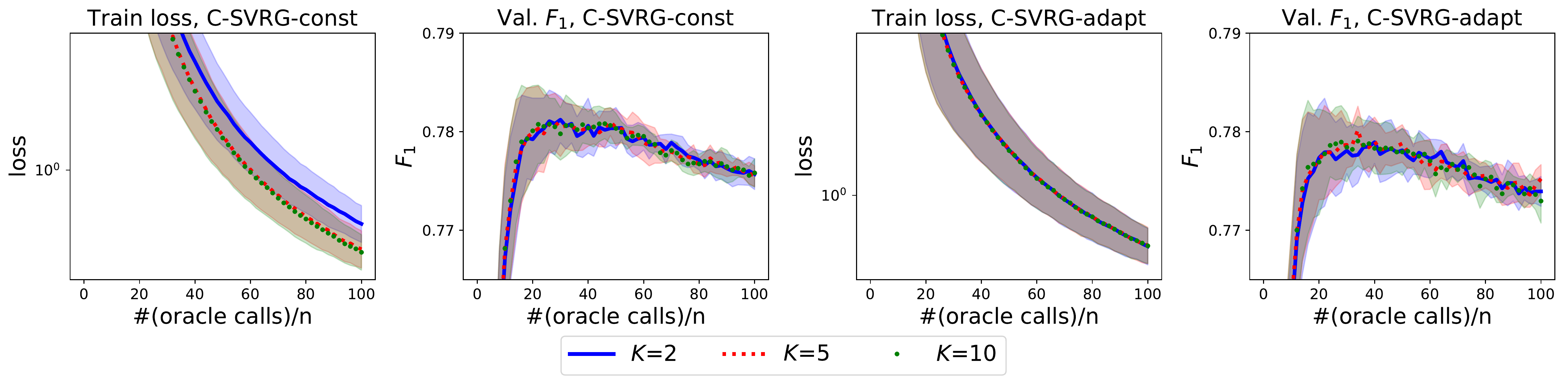}
        \caption{\small{Effect of $K$ in the top-$K$ oracle ($\lambda = 0.01/n$).}}
        \label{fig:plot_ner_K}
    \end{subfigure} 

    \caption{Effect of hyperparameters for the task of Named Entity Recognition on CoNLL 2003.
    C-SVRG stands for Casimir-SVRG in these plots.}\label{fig:plot_cvx_hyperparam}
\end{figure}

\paragraph{Effect of Warm Start Strategies} 
Fig.~\ref{fig:plot_ner_warm-start} plots different warm start strategies for \nsCatalystExpt{}
and \nsCatalystDecayExpt. 
We find that \nsCatalystDecayExpt{} is robust to the choice of the warm start strategy while \nsCatalystExpt{} is not.
For the latter, we observe that {\tt Extrapolation} is less stable (i.e., tends to diverge more) than {\tt Prox-center}, 
which is in turn less stable than {\tt Prev-iterate}, which always works (cf. Fig.~\ref{fig:plot_ner_warm-start}). 
However, when they do work, {\tt Extrapolation} and {\tt Prox-center} provide greater acceleration than {\tt Prev-iterate}.
We use {\tt Prox-center} as the default choice to trade-off between acceleration and applicability.

\paragraph{Effect of $K$}
Fig.~\ref{fig:plot_ner_K} illustrates the robustness of the method to choice of $K$: we observe that
the results are all within one standard deviation of each other.

\subsection{Experimental Study of Effect of Hyperparameters: Non-Convex Optimization}
We now study the effect of various hyperparameters for the non-convex optimization algorithms.
All of these comparisons have been made for $\lambda = 1/n$.

\paragraph{Effect of Smoothing} 
Fig.~\ref{fig:ncvx_smoothing:1} compares the adaptive and constant smoothing strategies. 
Fig.~\ref{fig:ncvx_smoothing:2} and Fig.~\ref{fig:ncvx_smoothing:3} compare the effect of the level of smoothing on the 
the both of these.
As previously, the adaptive smoothing strategy is more robust to the choice of the smoothing parameter.

\begin{figure}[!thb]
    \centering
    \begin{subfigure}[b]{\linewidth}
    \centering
    	\includegraphics[width=\textwidth]{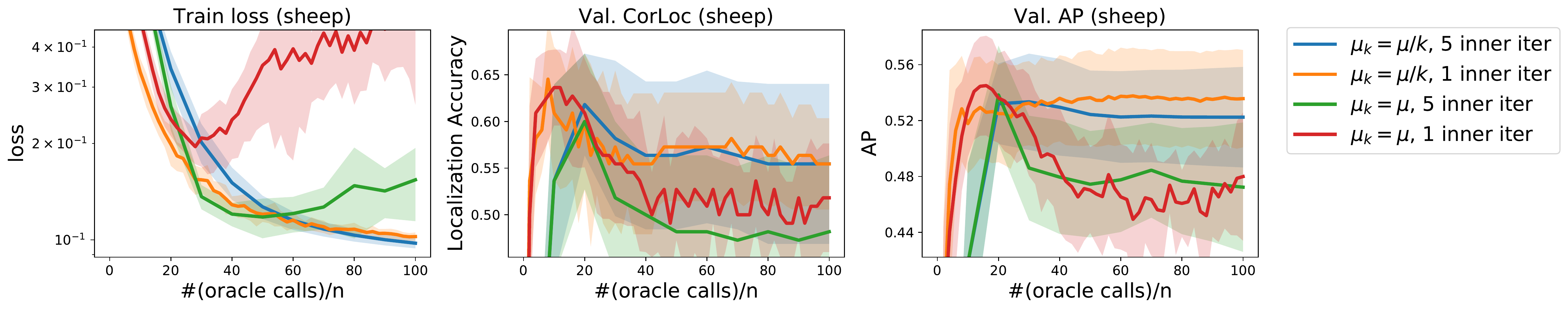}
        \caption{Comparison of adaptive and constant smoothing strategies.}
        \label{fig:ncvx_smoothing:1}
    \end{subfigure} 

     \begin{subfigure}[b]{\linewidth}
    \centering
        \includegraphics[width=\textwidth]{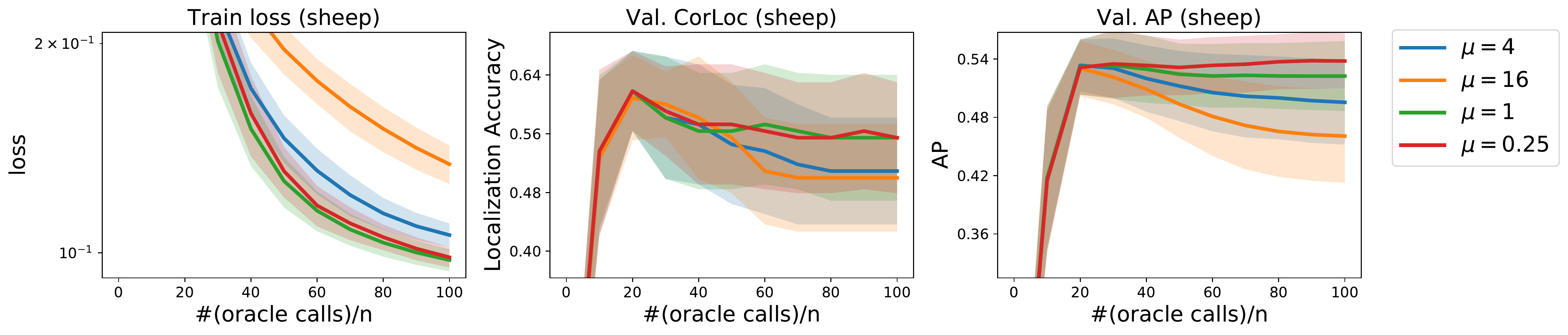}
        \caption{Effect of $\mu$ of the adaptive smoothing strategy.}
        \label{fig:ncvx_smoothing:2}
    \end{subfigure} 

    \begin{subfigure}[b]{\linewidth}
    \centering
        \includegraphics[width=\textwidth]{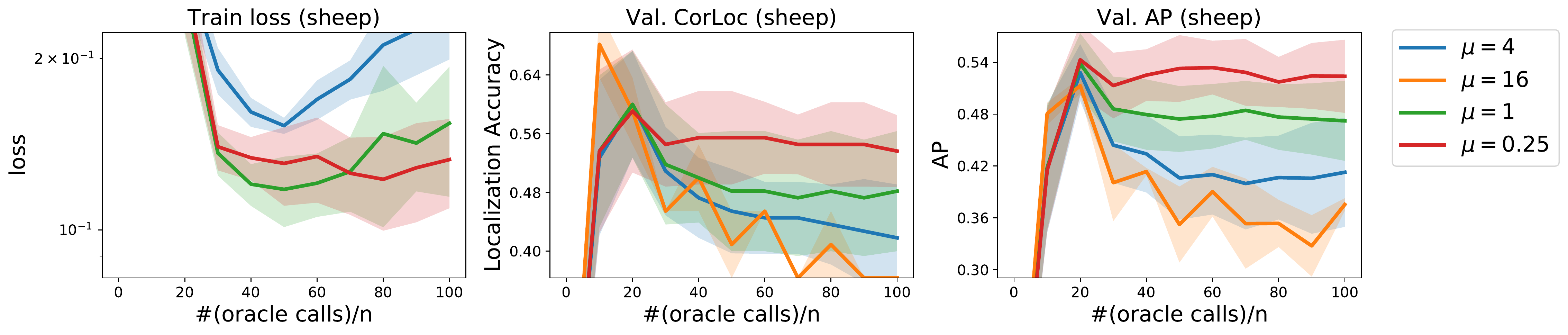}
        \caption{Effect of $\mu$ of the constant smoothing strategy.}
        \label{fig:ncvx_smoothing:3}
    \end{subfigure} 
    \caption{Effect of smoothing on \plcsvrg{} for the task of visual object localization on PASCAL VOC 2007.}\label{fig:ncvx_smoothing}
\end{figure}

\begin{figure}[!thb]
    \centering
    \begin{subfigure}[b]{\linewidth}
    \centering
    	\includegraphics[width=\textwidth]{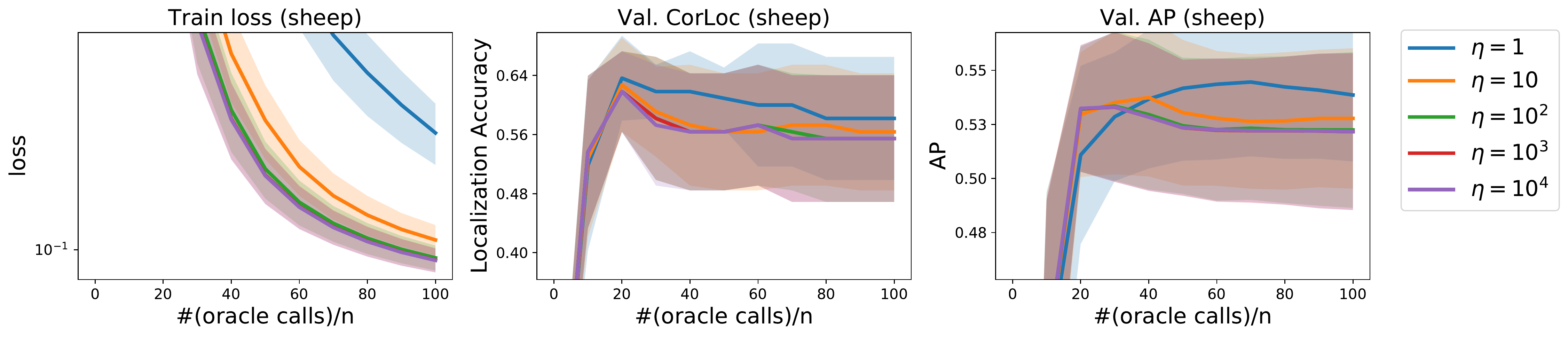}
    	\caption{\small{Effect of the hyperparameter $\eta$.}}\label{fig:ncvx:pl_lr}
    \end{subfigure} 

     \begin{subfigure}[b]{\linewidth}
    \centering
        \includegraphics[width=\textwidth]{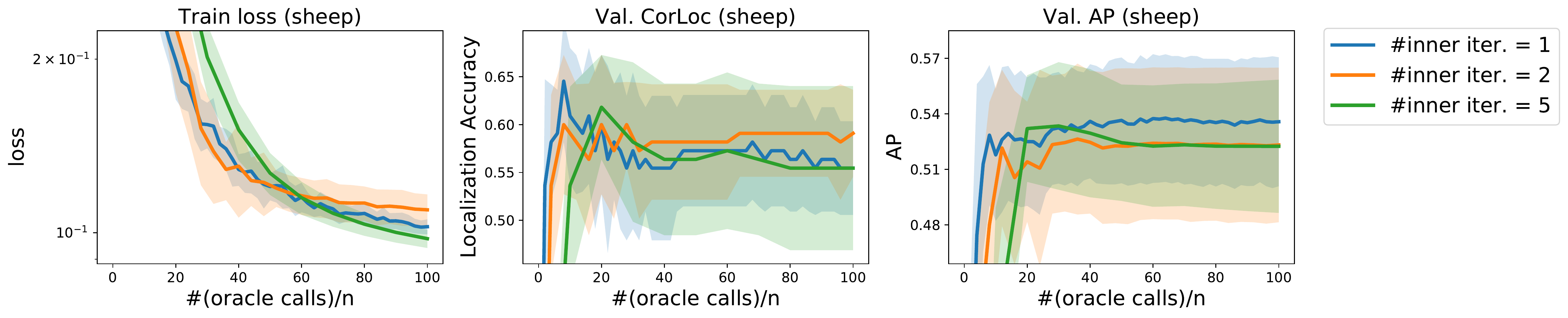}
   		\caption{\small{Effect of the iteration budget of the inner solver.}}\label{fig:ncvx:inner-iter}
    \end{subfigure} 

    \begin{subfigure}[b]{\linewidth}
    \centering
        \includegraphics[width=\textwidth]{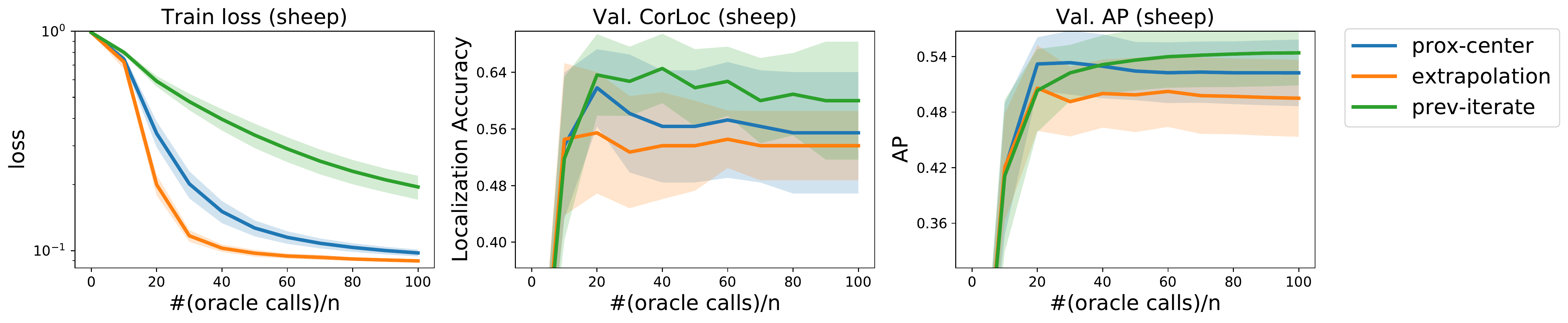}
    	\caption{\small{Effect of the warm start strategy of the inner \nsCatalystExpt{} algorithm.}}
    	\label{fig:ncvx:warm-start}
    \end{subfigure} 
    \caption{Effect of hyperparameters on \plcsvrg{} for the task of visual object localization on PASCAL VOC 2007.}\label{fig:ncvx_hyperparam}
\end{figure}

\paragraph{Effect of Prox-Linear Learning Rate $\eta$} 
Fig.~\ref{fig:ncvx:pl_lr} shows the robustness of the proposed method to the choice of $\eta$.

\paragraph{Effect of Iteration Budget} 
Fig.~\ref{fig:ncvx:inner-iter} also shows the robustness of the proposed method to the choice of iteration budget of the inner solver, \nsCatalystExpt.

\paragraph{Effect of Warm Start of the Inner Solver} 
Fig.~\ref{fig:ncvx:warm-start} studies the effect of the 
warm start strategy used within the inner solver \nsCatalystExpt{} in each inner prox-linear iteration. The results are similar to 
those obtained in the convex case, with {\tt Prox-center} choice being the best compromise between acceleration and compatibility.

\section{Future Directions}
We introduced a general notion of smooth inference oracles in the context of black-box first-order optimization. This allows us to set the scene to extend the scope of fast incremental optimization algorithms to structured prediction problems owing to a careful blend of a smoothing strategy and an acceleration scheme. We illustrated the potential of our framework by proposing a new incremental optimization algorithm to train structural support vector machines both enjoying worst-case complexity bounds and demonstrating competitive performance on two real-world problems. This work paves also the way to faster incremental primal optimization algorithms for deep structured prediction models.

There are several potential venues for future work. When there is no discrete structure that admits efficient inference algorithms, it could be beneficial to not treat inference as a black-box numerical procedure~\citep{meshi2010learning,hazan2010primal,hazan2016blending}. Instance-level improved algorithms along the lines of~\cite{hazan2016blending} could also be interesting to explore.

\paragraph{Acknowledgments}
This work was supported by NSF Award CCF-1740551, the Washington Research Foundation
for innovation in Data-intensive Discovery, and the program ``Learning in Machines and Brains'' of CIFAR. 

\clearpage


\newpage
\bibliography{bib/bib}

\begin{thebibliography}{104}
\providecommand{\natexlab}[1]{#1}
\providecommand{\url}[1]{\texttt{#1}}
\expandafter\ifx\csname urlstyle\endcsname\relax
  \providecommand{\doi}[1]{doi: #1}\else
  \providecommand{\doi}{doi: \begingroup \urlstyle{rm}\Url}\fi

\bibitem[Allen{-}Zhu(2017)]{allen2016katyusha}
Z.~Allen{-}Zhu.
\newblock Katyusha: {T}he {F}irst {D}irect {A}cceleration of {S}tochastic
  {G}radient {M}ethods.
\newblock \emph{Journal of Machine Learning Research}, 18:\penalty0
  221:1--221:51, 2017.

\bibitem[Altun et~al.(2003)Altun, Tsochantaridis, and Hofmann]{altun2003hidden}
Y.~Altun, I.~Tsochantaridis, and T.~Hofmann.
\newblock Hidden {M}arkov {S}upport {V}ector {M}achines.
\newblock In \emph{International Conference on Machine Learning}, pages 3--10,
  2003.

\bibitem[Batra(2012)]{batra2012efficient}
D.~Batra.
\newblock An efficient message-passing algorithm for the {$M$}-best {MAP}
  problem.
\newblock In \emph{Conference on Uncertainty in Artificial Intelligence}, pages
  121--130, 2012.

\bibitem[Batra et~al.(2012)Batra, Yadollahpour, Guzm{\'{a}}n{-}Rivera, and
  Shakhnarovich]{batra2012diverse}
D.~Batra, P.~Yadollahpour, A.~Guzm{\'{a}}n{-}Rivera, and G.~Shakhnarovich.
\newblock Diverse {$M$}-best {S}olutions in {M}arkov {R}andom {F}ields.
\newblock In \emph{European Conference on Computer Vision}, pages 1--16, 2012.

\bibitem[Beck and Teboulle(2012)]{beck2012smoothing}
A.~Beck and M.~Teboulle.
\newblock Smoothing and first order methods: {A} unified framework.
\newblock \emph{SIAM Journal on Optimization}, 22\penalty0 (2):\penalty0
  557--580, 2012.

\bibitem[Belanger and McCallum(2016)]{belanger2016structured}
D.~Belanger and A.~McCallum.
\newblock Structured prediction energy networks.
\newblock In \emph{International Conference on Machine Learning}, pages
  983--992, 2016.

\bibitem[Bellman(1957)]{bellman1957dynamic}
R.~Bellman.
\newblock \emph{Dynamic Programming}.
\newblock Courier Dover Publications, 1957.

\bibitem[Bengio et~al.(1995)Bengio, LeCun, Nohl, and Burges]{bengio1995lerec}
Y.~Bengio, Y.~LeCun, C.~Nohl, and C.~Burges.
\newblock Le{R}ec: A {NN/HMM} {H}ybrid for {O}n-{L}ine {H}andwriting
  {R}ecognition.
\newblock \emph{Neural Computation}, 7\penalty0 (6):\penalty0 1289--1303, 1995.

\bibitem[Bertsekas(1995)]{bertsekas1995dynamic}
D.~P. Bertsekas.
\newblock \emph{Dynamic programming and optimal control}, volume~1.
\newblock Athena scientific Belmont, MA, 1995.

\bibitem[Bertsekas(1999)]{bertsekas1999nonlinear}
D.~P. Bertsekas.
\newblock \emph{Nonlinear programming}.
\newblock Athena Scientific Belmont, 1999.

\bibitem[Bottou and Gallinari(1990)]{bottou1990framework}
L.~Bottou and P.~Gallinari.
\newblock A {F}ramework for the {C}ooperation of {L}earning {A}lgorithms.
\newblock In \emph{Advances in Neural Information Processing Systems}, pages
  781--788, 1990.

\bibitem[Bottou et~al.(1997)Bottou, Bengio, and LeCun]{bottou1997global}
L.~Bottou, Y.~Bengio, and Y.~LeCun.
\newblock {G}lobal {T}raining of {D}ocument {P}rocessing {S}ystems {U}sing
  {G}raph {T}ransformer {N}etworks.
\newblock In \emph{Conference on Computer Vision and Pattern Recognition},
  pages 489--494, 1997.

\bibitem[Bradski(2000)]{opencv_library}
G.~Bradski.
\newblock {The OpenCV Library}.
\newblock \emph{Dr. Dobb's Journal of Software Tools}, 2000.

\bibitem[Burke(1985)]{burke1985descent}
J.~V. Burke.
\newblock Descent methods for composite nondifferentiable optimization
  problems.
\newblock \emph{Mathematical Programming}, 33\penalty0 (3):\penalty0 260--279,
  1985.

\bibitem[Chen et~al.(2013)Chen, Kolmogorov, Zhu, Metaxas, and
  Lampert]{chen2013computing}
C.~Chen, V.~Kolmogorov, Y.~Zhu, D.~N. Metaxas, and C.~H. Lampert.
\newblock Computing the {$M$} {M}ost {P}robable {M}odes of a {G}raphical
  {M}odel.
\newblock In \emph{International Conference on Artificial Intelligence and
  Statistics}, pages 161--169, 2013.

\bibitem[Cheng et~al.(1996)Cheng, Wu, Collins, Hanson, and
  Riseman]{cheng1996maximum}
Y.-Q. Cheng, V.~Wu, R.~Collins, A.~R. Hanson, and E.~M. Riseman.
\newblock Maximum-weight bipartite matching technique and its application in
  image feature matching.
\newblock In \emph{Visual Communications and Image Processing}, volume 2727,
  pages 453--463, 1996.

\bibitem[Collins et~al.(2008)Collins, Globerson, Koo, Carreras, and
  Bartlett]{collins2008exponentiated}
M.~Collins, A.~Globerson, T.~Koo, X.~Carreras, and P.~L. Bartlett.
\newblock Exponentiated gradient algorithms for conditional random fields and
  max-margin markov networks.
\newblock \emph{Journal of Machine Learning Research}, 9\penalty0
  (Aug):\penalty0 1775--1822, 2008.

\bibitem[Collobert et~al.(2011)Collobert, Weston, Bottou, Karlen, Kavukcuoglu,
  and Kuksa]{collobert2011natural}
R.~Collobert, J.~Weston, L.~Bottou, M.~Karlen, K.~Kavukcuoglu, and P.~P. Kuksa.
\newblock Natural language processing (almost) from scratch.
\newblock \emph{Journal of Machine Learning Research}, 12:\penalty0 2493--2537,
  2011.

\bibitem[Cooper(1990)]{cooper1990computational}
G.~F. Cooper.
\newblock The computational complexity of probabilistic inference using
  bayesian belief networks.
\newblock \emph{Artificial Intelligence}, 42\penalty0 (2-3):\penalty0 393--405,
  1990.

\bibitem[Cox et~al.(2014)Cox, Juditsky, and Nemirovski]{cox2014dual}
B.~Cox, A.~Juditsky, and A.~Nemirovski.
\newblock Dual subgradient algorithms for large-scale nonsmooth learning
  problems.
\newblock \emph{Mathematical Programming}, 148\penalty0 (1-2):\penalty0
  143--180, 2014.

\bibitem[Crammer and Singer(2001)]{crammer2001algorithmic}
K.~Crammer and Y.~Singer.
\newblock On the algorithmic implementation of multiclass kernel-based vector
  machines.
\newblock \emph{Journal of Machine Learning Research}, 2\penalty0
  (Dec):\penalty0 265--292, 2001.

\bibitem[Daum{\'{e}}~III and Marcu(2005)]{daume2005learning}
H.~Daum{\'{e}}~III and D.~Marcu.
\newblock Learning as search optimization: approximate large margin methods for
  structured prediction.
\newblock In \emph{International Conference on Machine Learning}, pages
  169--176, 2005.

\bibitem[Davis and Drusvyatskiy(2018)]{davis2018stochastic}
D.~Davis and D.~Drusvyatskiy.
\newblock Stochastic model-based minimization of weakly convex functions.
\newblock \emph{arXiv preprint arXiv:1803.06523}, 2018.

\bibitem[Dawid(1992)]{dawid1992applications}
A.~P. Dawid.
\newblock Applications of a general propagation algorithm for probabilistic
  expert systems.
\newblock \emph{Statistics and Computing}, 2\penalty0 (1):\penalty0 25--36,
  1992.

\bibitem[Defazio(2016)]{defazio2016simple}
A.~Defazio.
\newblock A simple practical accelerated method for finite sums.
\newblock In \emph{Advances in Neural Information Processing Systems}, pages
  676--684, 2016.

\bibitem[Defazio et~al.(2014)Defazio, Bach, and
  Lacoste-Julien]{defazio2014saga}
A.~Defazio, F.~Bach, and S.~Lacoste-Julien.
\newblock S{A}{G}{A}: A fast incremental gradient method with support for
  non-strongly convex composite objectives.
\newblock In \emph{Advances in Neural Information Processing Systems}, pages
  1646--1654, 2014.

\bibitem[Deselaers et~al.(2010)Deselaers, Alexe, and
  Ferrari]{deselaers2010localizing}
T.~Deselaers, B.~Alexe, and V.~Ferrari.
\newblock Localizing objects while learning their appearance.
\newblock In \emph{European Conference on Computer Vision}, pages 452--466,
  2010.

\bibitem[Drusvyatskiy and Paquette(2018)]{drusvyatskiy2016efficiency}
D.~Drusvyatskiy and C.~Paquette.
\newblock Efficiency of minimizing compositions of convex functions and smooth
  maps.
\newblock \emph{Mathematical Programming}, Jul 2018.

\bibitem[Duchi et~al.(2006)Duchi, Tarlow, Elidan, and Koller]{duchi2007using}
J.~C. Duchi, D.~Tarlow, G.~Elidan, and D.~Koller.
\newblock Using {C}ombinatorial {O}ptimization within {M}ax-{P}roduct {B}elief
  {P}ropagation.
\newblock In \emph{Advances in Neural Information Processing Systems}, pages
  369--376, 2006.

\bibitem[Everingham et~al.(2010)Everingham, Van~Gool, Williams, Winn, and
  Zisserman]{everingham2010pascal}
M.~Everingham, L.~Van~Gool, C.~K. Williams, J.~Winn, and A.~Zisserman.
\newblock The {P}ascal {V}isual {O}bject {C}lasses ({V}{O}{C}) challenge.
\newblock \emph{International Journal of Computer Vision}, 88\penalty0
  (2):\penalty0 303--338, 2010.

\bibitem[Flerova et~al.(2016)Flerova, Marinescu, and
  Dechter]{flerova2016searching}
N.~Flerova, R.~Marinescu, and R.~Dechter.
\newblock Searching for the {$M$} {B}est {S}olutions in {G}raphical {M}odels.
\newblock \emph{Journal of Artificial Intelligence Research}, 55:\penalty0
  889--952, 2016.

\bibitem[Fromer and Globerson(2009)]{fromer2009lp}
M.~Fromer and A.~Globerson.
\newblock An {LP} view of the {$M$}-best {MAP} problem.
\newblock In \emph{Advances in Neural Information Processing Systems}, pages
  567--575, 2009.

\bibitem[Frostig et~al.(2015)Frostig, Ge, Kakade, and
  Sidford]{frostig2015regularizing}
R.~Frostig, R.~Ge, S.~Kakade, and A.~Sidford.
\newblock Un-regularizing: approximate proximal point and faster stochastic
  algorithms for empirical risk minimization.
\newblock In \emph{International Conference on Machine Learning}, pages
  2540--2548, 2015.

\bibitem[Girshick et~al.(2014)Girshick, Donahue, Darrell, and
  Malik]{girshick2014rich}
R.~Girshick, J.~Donahue, T.~Darrell, and J.~Malik.
\newblock Rich feature hierarchies for accurate object detection and semantic
  segmentation.
\newblock In \emph{Conference on Computer Vision and Pattern Recognition},
  pages 580--587, 2014.

\bibitem[Greig et~al.(1989)Greig, Porteous, and Seheult]{greig1989exact}
D.~M. Greig, B.~T. Porteous, and A.~H. Seheult.
\newblock Exact maximum a posteriori estimation for binary images.
\newblock \emph{Journal of the Royal Statistical Society. Series B
  (Methodological)}, pages 271--279, 1989.

\bibitem[Hazan and Urtasun(2010)]{hazan2010primal}
T.~Hazan and R.~Urtasun.
\newblock A {P}rimal-{D}ual {M}essage-{P}assing {A}lgorithm for {A}pproximated
  {L}arge {S}cale {S}tructured {P}rediction.
\newblock In \emph{Advances in Neural Information Processing Systems}, pages
  838--846, 2010.

\bibitem[Hazan et~al.(2016)Hazan, Schwing, and Urtasun]{hazan2016blending}
T.~Hazan, A.~G. Schwing, and R.~Urtasun.
\newblock Blending {L}earning and {I}nference in {C}onditional {R}andom
  {F}ields.
\newblock \emph{Journal of Machine Learning Research}, 17:\penalty0
  237:1--237:25, 2016.

\bibitem[He et~al.(2017)He, Lee, Lewis, and Zettlemoyer]{he2017deep}
L.~He, K.~Lee, M.~Lewis, and L.~Zettlemoyer.
\newblock Deep {S}emantic {R}ole {L}abeling: {W}hat {W}orks and {W}hat's
  {N}ext.
\newblock In \emph{Annual Meeting of the Association for Computational
  Linguistics}, pages 473--483, 2017.

\bibitem[He and Harchaoui(2015)]{he2015semi}
N.~He and Z.~Harchaoui.
\newblock Semi-{P}roximal {M}irror-{P}rox for {N}onsmooth {C}omposite
  {M}inimization.
\newblock In \emph{Advances in Neural Information Processing Systems}, pages
  3411--3419, 2015.

\bibitem[Held et~al.(1974)Held, Wolfe, and Crowder]{held1974validation}
M.~Held, P.~Wolfe, and H.~P. Crowder.
\newblock Validation of subgradient optimization.
\newblock \emph{Mathematical Programming}, 6\penalty0 (1):\penalty0 62--88, Dec
  1974.

\bibitem[Hofmann et~al.(2015)Hofmann, Lucchi, Lacoste-Julien, and
  McWilliams]{hofmann2015variance}
T.~Hofmann, A.~Lucchi, S.~Lacoste-Julien, and B.~McWilliams.
\newblock Variance reduced stochastic gradient descent with neighbors.
\newblock In \emph{Advances in Neural Information Processing Systems}, pages
  2305--2313, 2015.

\bibitem[Ishikawa and Geiger(1998)]{ishikawa1998segmentation}
H.~Ishikawa and D.~Geiger.
\newblock Segmentation by grouping junctions.
\newblock In \emph{Conference on Computer Vision and Pattern Recognition},
  pages 125--131, 1998.

\bibitem[Jerrum and Sinclair(1993)]{jerrum1993polynomial}
M.~Jerrum and A.~Sinclair.
\newblock Polynomial-time approximation algorithms for the {I}sing model.
\newblock \emph{SIAM Journal on computing}, 22\penalty0 (5):\penalty0
  1087--1116, 1993.

\bibitem[Joachims et~al.(2009)Joachims, Finley, and Yu]{joachims2009cutting}
T.~Joachims, T.~Finley, and C.-N.~J. Yu.
\newblock Cutting-plane training of structural {S}{V}{M}s.
\newblock \emph{Machine Learning}, 77\penalty0 (1):\penalty0 27--59, 2009.

\bibitem[Johnson(2008)]{johnson2008convex}
J.~K. Johnson.
\newblock \emph{Convex relaxation methods for graphical models: Lagrangian and
  maximum entropy approaches}.
\newblock PhD thesis, Massachusetts Institute of Technology, Cambridge, MA,
  {USA}, 2008.

\bibitem[Johnson and Zhang(2013)]{johnson2013accelerating}
R.~Johnson and T.~Zhang.
\newblock Accelerating stochastic gradient descent using predictive variance
  reduction.
\newblock In \emph{Advances in Neural Information Processing Systems}, pages
  315--323, 2013.

\bibitem[Jojic et~al.(2010)Jojic, Gould, and Koller]{jojic2010accelerated}
V.~Jojic, S.~Gould, and D.~Koller.
\newblock Accelerated dual decomposition for {M}{A}{P} inference.
\newblock In \emph{International Conference on Machine Learning}, pages
  503--510, 2010.

\bibitem[Jurafsky et~al.(2014)Jurafsky, Martin, Norvig, and
  Russell]{jurafsky2014speech}
D.~Jurafsky, J.~H. Martin, P.~Norvig, and S.~Russell.
\newblock \emph{Speech and Language Processing}.
\newblock Pearson Education, 2014.
\newblock ISBN 9780133252934.

\bibitem[Koller and Friedman(2009)]{koller2009probabilistic}
D.~Koller and N.~Friedman.
\newblock \emph{Probabilistic Graphical Models - Principles and Techniques}.
\newblock {MIT} Press, 2009.
\newblock ISBN 978-0-262-01319-2.

\bibitem[Kolmogorov and Zabin(2004)]{kolmogorov2004energy}
V.~Kolmogorov and R.~Zabin.
\newblock What energy functions can be minimized via graph cuts?
\newblock \emph{IEEE Transactions on Pattern Analysis and Machine
  Intelligence}, 26\penalty0 (2):\penalty0 147--159, 2004.

\bibitem[Krizhevsky et~al.(2012)Krizhevsky, Sutskever, and
  Hinton]{krizhevsky2012imagenet}
A.~Krizhevsky, I.~Sutskever, and G.~E. Hinton.
\newblock Imagenet classification with deep convolutional neural networks.
\newblock In \emph{Advances in Neural Information Processing Systems}, pages
  1097--1105, 2012.

\bibitem[Lacoste-Julien et~al.(2012)Lacoste-Julien, Schmidt, and
  Bach]{lacoste2012simpler}
S.~Lacoste-Julien, M.~Schmidt, and F.~Bach.
\newblock A simpler approach to obtaining an {${O} (1/t)$} convergence rate for
  the projected stochastic subgradient method.
\newblock \emph{arXiv preprint arXiv:1212.2002}, 2012.

\bibitem[Lacoste-Julien et~al.(2013)Lacoste-Julien, Jaggi, Schmidt, and
  Pletscher]{lacoste2012block}
S.~Lacoste-Julien, M.~Jaggi, M.~Schmidt, and P.~Pletscher.
\newblock Block-{C}oordinate {F}rank-{W}olfe {O}ptimization for {S}tructural
  {S}{V}{M}s.
\newblock In \emph{International Conference on Machine Learning}, pages 53--61,
  2013.

\bibitem[Lafferty et~al.(2001)Lafferty, McCallum, and
  Pereira]{lafferty2001conditional}
J.~Lafferty, A.~McCallum, and F.~C. Pereira.
\newblock Conditional {R}andom {F}ields: {P}robabilistic {M}odels for
  {S}egmenting and {L}abeling {S}equence {D}ata.
\newblock In \emph{International Conference on Machine Learning}, pages
  282--289, 2001.

\bibitem[Lampert et~al.(2008)Lampert, Blaschko, and Hofmann]{lampert2008beyond}
C.~H. Lampert, M.~B. Blaschko, and T.~Hofmann.
\newblock Beyond sliding windows: {O}bject localization by efficient subwindow
  search.
\newblock In \emph{Conference on Computer Vision and Pattern Recognition},
  pages 1--8, 2008.

\bibitem[Le~Roux et~al.(2012)Le~Roux, Schmidt, and Bach]{roux2012stochastic}
N.~Le~Roux, M.~W. Schmidt, and F.~R. Bach.
\newblock A {S}tochastic {G}radient {M}ethod with an {E}xponential
  {C}onvergence {R}ate for {S}trongly-{C}onvex {O}ptimization with {F}inite
  {T}raining {S}ets.
\newblock In \emph{Advances in Neural Information Processing Systems}, pages
  2672--2680, 2012.

\bibitem[Lewis and Steedman(2014)]{lewis2014ccg}
M.~Lewis and M.~Steedman.
\newblock A* {CCG} parsing with a supertag-factored model.
\newblock In \emph{Conference on Empirical Methods in Natural Language
  Processing}, pages 990--1000, 2014.

\bibitem[Lin et~al.(2015)Lin, Mairal, and Harchaoui]{lin2015universal}
H.~Lin, J.~Mairal, and Z.~Harchaoui.
\newblock A universal catalyst for first-order optimization.
\newblock In \emph{Advances in Neural Information Processing Systems}, pages
  3384--3392, 2015.

\bibitem[Lin et~al.(2018)Lin, Mairal, and Harchaoui]{lin2017catalyst}
H.~Lin, J.~Mairal, and Z.~Harchaoui.
\newblock Catalyst {A}cceleration for {F}irst-order {C}onvex {O}ptimization:
  from {T}heory to {P}ractice.
\newblock \emph{Journal of Machine Learning Research}, 18\penalty0
  (212):\penalty0 1--54, 2018.

\bibitem[Lov{\'a}sz(1983)]{lovasz1983submodular}
L.~Lov{\'a}sz.
\newblock Submodular functions and convexity.
\newblock In \emph{Mathematical Programming The State of the Art}, pages
  235--257. Springer, 1983.

\bibitem[Mairal(2015)]{mairal2013optimization}
J.~Mairal.
\newblock Incremental majorization-minimization optimization with application
  to large-scale machine learning.
\newblock \emph{{SIAM} Journal on Optimization}, 25\penalty0 (2):\penalty0
  829--855, 2015.

\bibitem[Martins and Astudillo(2016)]{martins2016softmax}
A.~F.~T. Martins and R.~F. Astudillo.
\newblock From {S}oftmax to {S}parsemax: {A} {S}parse {M}odel of {A}ttention
  and {M}ulti-{L}abel {C}lassification.
\newblock In \emph{International Conference on Machine Learning}, pages
  1614--1623, 2016.

\bibitem[McEliece et~al.(1998)McEliece, MacKay, and Cheng]{mceliece1998turbo}
R.~J. McEliece, D.~J.~C. MacKay, and J.~Cheng.
\newblock Turbo {D}ecoding as an {I}nstance of {P}earl's "{B}elief
  {P}ropagation" {A}lgorithm.
\newblock \emph{{IEEE} Journal on Selected Areas in Communications},
  16\penalty0 (2):\penalty0 140--152, 1998.

\bibitem[Mensch and Blondel(2018)]{mensch2018differentiable}
A.~Mensch and M.~Blondel.
\newblock Differentiable dynamic programming for structured prediction and
  attention.
\newblock In \emph{International Conference on Machine Learning}, pages
  3459--3468, 2018.

\bibitem[Meshi et~al.(2010)Meshi, Sontag, Jaakkola, and
  Globerson]{meshi2010learning}
O.~Meshi, D.~Sontag, T.~S. Jaakkola, and A.~Globerson.
\newblock Learning {E}fficiently with {A}pproximate {I}nference via {D}ual
  {L}osses.
\newblock In \emph{International Conference on Machine Learning}, pages
  783--790, 2010.

\bibitem[Meshi et~al.(2012)Meshi, Jaakkola, and
  Globerson]{meshi2012convergence}
O.~Meshi, T.~S. Jaakkola, and A.~Globerson.
\newblock Convergence {R}ate {A}nalysis of {MAP} {C}oordinate {M}inimization
  {A}lgorithms.
\newblock In \emph{Advances in Neural Information Processing Systems}, pages
  3023--3031, 2012.

\bibitem[Murphy et~al.(1999)Murphy, Weiss, and Jordan]{murphy1999loopy}
K.~P. Murphy, Y.~Weiss, and M.~I. Jordan.
\newblock Loopy belief propagation for approximate inference: An empirical
  study.
\newblock In \emph{Conference on Uncertainty in Artificial Intelligence}, pages
  467--475, 1999.

\bibitem[Nesterov(1983)]{nesterov1983method}
Y.~Nesterov.
\newblock A method of solving a convex programming problem with convergence
  rate {$O(1/k^2)$}.
\newblock In \emph{Soviet Mathematics Doklady}, volume~27, pages 372--376,
  1983.

\bibitem[Nesterov(2005{\natexlab{a}})]{nesterov2005excessive}
Y.~Nesterov.
\newblock Excessive gap technique in nonsmooth convex minimization.
\newblock \emph{SIAM Journal on Optimization}, 16\penalty0 (1):\penalty0
  235--249, 2005{\natexlab{a}}.

\bibitem[Nesterov(2005{\natexlab{b}})]{nesterov2005smooth}
Y.~Nesterov.
\newblock Smooth minimization of non-smooth functions.
\newblock \emph{Mathematical programming}, 103\penalty0 (1):\penalty0 127--152,
  2005{\natexlab{b}}.

\bibitem[Nesterov(2013)]{nesterov2013introductory}
Y.~Nesterov.
\newblock \emph{Introductory lectures on convex optimization: {A} basic
  course}, volume~87.
\newblock Springer Science \& Business Media, 2013.

\bibitem[Niculae et~al.(2018)Niculae, Martins, Blondel, and
  Cardie]{niculae2018sparsemap}
V.~Niculae, A.~F. Martins, M.~Blondel, and C.~Cardie.
\newblock Sparse{MAP}: {D}ifferentiable {S}parse {S}tructured {I}nference.
\newblock In \emph{International Conference on Machine Learning}, pages
  3796--3805, 2018.

\bibitem[Nilsson(1998)]{nilsson1998efficient}
D.~Nilsson.
\newblock An efficient algorithm for finding the {$M$} most probable
  configurations in probabilistic expert systems.
\newblock \emph{Statistics and Computing}, 8\penalty0 (2):\penalty0 159--173,
  1998.

\bibitem[Osokin et~al.(2016)Osokin, Alayrac, Lukasewitz, Dokania, and
  Lacoste-Julien]{osokin2016minding}
A.~Osokin, J.-B. Alayrac, I.~Lukasewitz, P.~Dokania, and S.~Lacoste-Julien.
\newblock Minding the gaps for block {F}rank-{W}olfe optimization of structured
  {S}{V}{M}s.
\newblock In \emph{International Conference on Machine Learning}, pages
  593--602, 2016.

\bibitem[Palaniappan and Bach(2016)]{palaniappan2016stochastic}
B.~Palaniappan and F.~Bach.
\newblock Stochastic variance reduction methods for saddle-point problems.
\newblock In \emph{Advances in Neural Information Processing Systems}, pages
  1408--1416, 2016.

\bibitem[Paquette et~al.(2018)Paquette, Lin, Drusvyatskiy, Mairal, and
  Harchaoui]{paquette2017catalyst}
C.~Paquette, H.~Lin, D.~Drusvyatskiy, J.~Mairal, and Z.~Harchaoui.
\newblock Catalyst for gradient-based nonconvex optimization.
\newblock In \emph{International Conference on Artificial Intelligence and
  Statistics}, pages 613--622, 2018.

\bibitem[Pearl(1988)]{pearl1988probabilistic}
J.~Pearl.
\newblock \emph{Probabilistic reasoning in intelligent systems: networks of
  plausible inference}.
\newblock Morgan Kaufmann, 1988.

\bibitem[Ratliff et~al.(2007)Ratliff, Bagnell, and
  Zinkevich]{ratliff2007approximate}
N.~D. Ratliff, J.~A. Bagnell, and M.~Zinkevich.
\newblock ({A}pproximate) {S}ubgradient {M}ethods for {S}tructured
  {P}rediction.
\newblock In \emph{International Conference on Artificial Intelligence and
  Statistics}, pages 380--387, 2007.

\bibitem[Rockafellar and Wets(2009)]{rockafellar2009variational}
R.~T. Rockafellar and R.~J.-B. Wets.
\newblock \emph{Variational analysis}, volume 317.
\newblock Springer Science \& Business Media, 2009.

\bibitem[Russakovsky et~al.(2015)Russakovsky, Deng, Su, Krause, Satheesh, Ma,
  Huang, Karpathy, Khosla, Bernstein, Berg, and Fei-Fei]{ILSVRC15}
O.~Russakovsky, J.~Deng, H.~Su, J.~Krause, S.~Satheesh, S.~Ma, Z.~Huang,
  A.~Karpathy, A.~Khosla, M.~Bernstein, A.~C. Berg, and L.~Fei-Fei.
\newblock {ImageNet {L}arge {S}cale {V}isual {R}ecognition {C}hallenge}.
\newblock \emph{International Journal of Computer Vision}, 115\penalty0
  (3):\penalty0 211--252, 2015.

\bibitem[Savchynskyy et~al.(2011)Savchynskyy, Kappes, Schmidt, and
  Schn{\"{o}}rr]{savchynskyy2011study}
B.~Savchynskyy, J.~H. Kappes, S.~Schmidt, and C.~Schn{\"{o}}rr.
\newblock A study of {N}esterov's scheme for {L}agrangian decomposition and
  {MAP} labeling.
\newblock In \emph{Conference on Computer Vision and Pattern Recognition},
  pages 1817--1823, 2011.

\bibitem[Schlesinger(1976)]{schlesinger1976syntactic}
M.~I. Schlesinger.
\newblock Syntactic analysis of two-dimensional visual signals in noisy
  conditions.
\newblock \emph{Kibernetika}, 4\penalty0 (113-130):\penalty0 1, 1976.

\bibitem[Schmidt et~al.(2015)Schmidt, Babanezhad, Ahmed, Defazio, Clifton, and
  Sarkar]{schmidt2015non}
M.~Schmidt, R.~Babanezhad, M.~Ahmed, A.~Defazio, A.~Clifton, and A.~Sarkar.
\newblock Non-uniform stochastic average gradient method for training
  conditional random fields.
\newblock In \emph{International Conference on Artificial Intelligence and
  Statistics}, pages 819--828, 2015.

\bibitem[Schmidt et~al.(2017)Schmidt, Le~Roux, and Bach]{schmidt2017minimizing}
M.~Schmidt, N.~Le~Roux, and F.~Bach.
\newblock Minimizing finite sums with the stochastic average gradient.
\newblock \emph{Mathematical Programming}, 162\penalty0 (1-2):\penalty0
  83--112, 2017.

\bibitem[Schrijver(2003)]{schrijver-book}
A.~Schrijver.
\newblock \emph{Combinatorial Optimization - Polyhedra and Efficiency}.
\newblock Springer, 2003.

\bibitem[Seroussi and Golmard(1994)]{seroussi1994algorithm}
B.~Seroussi and J.~Golmard.
\newblock An algorithm directly finding the {$K$} most probable configurations
  in {B}ayesian networks.
\newblock \emph{International Journal of Approximate Reasoning}, 11\penalty0
  (3):\penalty0 205 -- 233, 1994.

\bibitem[Shalev-Shwartz and Zhang(2013)]{shalev2013stochastic}
S.~Shalev-Shwartz and T.~Zhang.
\newblock Stochastic dual coordinate ascent methods for regularized loss
  minimization.
\newblock \emph{Journal of Machine Learning Research}, 14\penalty0
  (Feb):\penalty0 567--599, 2013.

\bibitem[Shalev-Shwartz and Zhang(2014)]{shalev2014accelerated}
S.~Shalev-Shwartz and T.~Zhang.
\newblock Accelerated proximal stochastic dual coordinate ascent for
  regularized loss minimization.
\newblock In \emph{International Conference on Machine Learning}, pages 64--72,
  2014.

\bibitem[Shalev-Shwartz et~al.(2011)Shalev-Shwartz, Singer, Srebro, and
  Cotter]{shalev2011pegasos}
S.~Shalev-Shwartz, Y.~Singer, N.~Srebro, and A.~Cotter.
\newblock Pegasos: {P}rimal estimated sub-gradient solver for {S}{V}{M}.
\newblock \emph{Mathematical programming}, 127\penalty0 (1):\penalty0 3--30,
  2011.

\bibitem[Song et~al.(2014)Song, Girshick, Jegelka, Mairal, Harchaoui, and
  Darrell]{song2014learning}
H.~O. Song, R.~B. Girshick, S.~Jegelka, J.~Mairal, Z.~Harchaoui, and
  T.~Darrell.
\newblock On learning to localize objects with minimal supervision.
\newblock In \emph{International Conference on Machine Learning}, pages
  1611--1619, 2014.

\bibitem[Taskar et~al.(2004)Taskar, Guestrin, and Koller]{taskar2004max}
B.~Taskar, C.~Guestrin, and D.~Koller.
\newblock Max-margin {M}arkov networks.
\newblock In \emph{Advances in Neural Information Processing Systems}, pages
  25--32, 2004.

\bibitem[Taskar et~al.(2005)Taskar, Lacoste{-}Julien, and
  Klein]{taskar2005discriminative}
B.~Taskar, S.~Lacoste{-}Julien, and D.~Klein.
\newblock A discriminative matching approach to word alignment.
\newblock In \emph{Human Language Technology Conference and Conference on
  Empirical Methods in Natural Language Processing}, pages 73--80, 2005.

\bibitem[Taskar et~al.(2006)Taskar, Lacoste-Julien, and
  Jordan]{taskar2006structured}
B.~Taskar, S.~Lacoste-Julien, and M.~I. Jordan.
\newblock Structured prediction, dual extragradient and {B}regman projections.
\newblock \emph{Journal of Machine Learning Research}, 7\penalty0
  (Jul):\penalty0 1627--1653, 2006.

\bibitem[Teo et~al.(2009)Teo, Vishwanathan, Smola, and Le]{teo2009bundle}
C.~H. Teo, S.~Vishwanathan, A.~Smola, and Q.~V. Le.
\newblock Bundle methods for regularized risk minimization.
\newblock \emph{Journal of Machine Learning Research}, 1\penalty0 (55), 2009.

\bibitem[Tjong Kim~Sang and De~Meulder(2003)]{tjong2003introduction}
E.~F. Tjong Kim~Sang and F.~De~Meulder.
\newblock Introduction to the {C}o{N}{L}{L}-2003 shared task:
  {L}anguage-independent named entity recognition.
\newblock In \emph{Conference on Natural Language Learning}, pages 142--147,
  2003.

\bibitem[Tkachenko and Simanovsky(2012)]{tkachenko2012named}
M.~Tkachenko and A.~Simanovsky.
\newblock Named entity recognition: Exploring features.
\newblock In \emph{Empirical Methods in Natural Language Processing}, pages
  118--127, 2012.

\bibitem[Tsochantaridis et~al.(2004)Tsochantaridis, Hofmann, Joachims, and
  Altun]{tsochantaridis2004support}
I.~Tsochantaridis, T.~Hofmann, T.~Joachims, and Y.~Altun.
\newblock Support vector machine learning for interdependent and structured
  output spaces.
\newblock In \emph{International Conference on Machine Learning}, page 104,
  2004.

\bibitem[Van~de Sande et~al.(2011)Van~de Sande, Uijlings, Gevers, and
  Smeulders]{van2011segmentation}
K.~E. Van~de Sande, J.~R. Uijlings, T.~Gevers, and A.~W. Smeulders.
\newblock Segmentation as selective search for object recognition.
\newblock In \emph{International Conference on Computer Vision}, pages
  1879--1886, 2011.

\bibitem[Viterbi(1967)]{viterbi1967error}
A.~J. Viterbi.
\newblock Error bounds for convolutional codes and an asymptotically optimum
  decoding algorithm.
\newblock \emph{{IEEE} Trans. Information Theory}, 13\penalty0 (2):\penalty0
  260--269, 1967.
\newblock \doi{10.1109/TIT.1967.1054010}.

\bibitem[Wainwright and Jordan(2008)]{wainwright2008graphical}
M.~J. Wainwright and M.~I. Jordan.
\newblock Graphical models, exponential families, and variational inference.
\newblock \emph{Foundations and Trends{\textregistered} in Machine Learning},
  1\penalty0 (1--2):\penalty0 1--305, 2008.

\bibitem[Wainwright et~al.(2005)Wainwright, Jaakkola, and
  Willsky]{wainwright2005map}
M.~J. Wainwright, T.~S. Jaakkola, and A.~S. Willsky.
\newblock {M}{A}{P} estimation via agreement on trees: message-passing and
  linear programming.
\newblock \emph{IEEE transactions on information theory}, 51\penalty0
  (11):\penalty0 3697--3717, 2005.

\bibitem[Woodworth and Srebro(2016)]{woodworth2016tight}
B.~E. Woodworth and N.~Srebro.
\newblock Tight complexity bounds for optimizing composite objectives.
\newblock In \emph{Advances in Neural Information Processing Systems}, pages
  3639--3647, 2016.

\bibitem[Yanover and Weiss(2004)]{yanover2004finding}
C.~Yanover and Y.~Weiss.
\newblock Finding the {$M$} most probable configurations using loopy belief
  propagation.
\newblock In \emph{Advances in Neural Information Processing Systems}, pages
  289--296, 2004.

\bibitem[Zhang et~al.(2014)Zhang, Saha, and Vishwanathan]{zhang2014accelerated}
X.~Zhang, A.~Saha, and S.~Vishwanathan.
\newblock Accelerated training of max-margin markov networks with kernels.
\newblock \emph{Theoretical Computer Science}, 519:\penalty0 88--102, 2014.

\end{thebibliography}
\bibliographystyle{abbrvnat}

\newpage
\appendix

\section{Smoothing}\label{sec:a:smoothing}
We first prove an extension of Lemma 4.2 of \citet{beck2012smoothing}, 
which proves the following statement for the special case of $\mu_2 = 0$.
Recall that we defined $h_{\mu \omega} \equiv h$ when $\mu = 0$.
\begin{proposition} \label{prop:smoothing:difference_of_smoothing}
	Consider the setting of Def.~\ref{defn:smoothing:inf-conv}.
	For $\mu_1 \ge \mu_2 \ge 0$, we have for every $\zv \in \reals^m$ that
	\begin{align*}
		(\mu_1 - \mu_2) \inf_{\uv \in \dom h^*} \omega(\uv) 
		\le 
		h_{\mu_2 \omega}(\zv) - h_{\mu_1 \omega}(\zv) 
		\le 
		(\mu_1 - \mu_2) \sup_{\uv \in \dom h^*} \omega(\uv) \,.
	\end{align*}
\end{proposition}
\begin{proof}
	We successively deduce, 
	\begin{align*}
		h_{\mu_1 \omega}(\zv) 
			&= \sup_{\uv \in \dom h^*} \left\{ \inp{\uv}{\zv} - h^*(\uv) - \mu_1 \omega(\uv) \right\} \\
			&= \sup_{\uv \in \dom h^*} \left\{ \inp{\uv}{\zv} - h^*(\uv) - \mu_2 \omega(\uv) - (\mu_1 - \mu_2) \omega(\uv) \right\} \\
			&\ge \sup_{\uv \in \dom h^*} \left\{ \inp{\uv}{\zv} - h^*(\uv) - \mu_2 \omega(\uv) + 
				\inf_{\uv' \in \dom h^*} \left\{ - (\mu_1 - \mu_2) \omega(\uv') \} \right\} \right\} \\
			&= h_{\mu_2 \omega}(\zv) - (\mu_1 - \mu_2) \sup_{\uv' \in \dom h^*} \omega(\uv') \,,
	\end{align*}
	since $\mu_1 - \mu_2 \ge 0$. The other side follows using instead that 
	\[
	- (\mu_1 - \mu_2) \omega(\uv) \le 
	\sup_{\uv' \in \dom h^*} \left\{ - (\mu_1 - \mu_2) \omega(\uv') \} \right\} \,.
	\]
\end{proof}
Next, we recall the following 
equivalent definition of a matrix norm defined in Eq.~\eqref{eq:matrix_norm_defn}.
\begin{align} \label{eq:matrix_norm_defn_app}
\norma{\beta, \alpha}{\Am} 
	= \sup_{\yv \neq \zerov} \frac{\normad{\beta}{\Am\T \yv}}{\norma{\alpha}{\yv}}
	= \sup_{\xv \neq \zerov} \frac{\normad{\alpha}{\Am \xv}}{\norma{\beta}{\yv}} 
	= \norma{\alpha, \beta}{\Am\T} \,.
\end{align}
Now, we consider the smoothness of a composition of a smooth function with an affine map.
\begin{lemma} \label{lemma:smoothing:composition}
	Suppose $h: \reals^m \to \reals$ is $L$-smooth with respect to $\normad{\alpha}{\cdot}$. 
	Then, for any $\Am \in \reals^{m \times d}$ and $\bv \in \reals^m$, 
	we have that the map $\reals^d \ni \wv \mapsto h(\Am \wv + \bv)$ is 
	$\big( L \normasq{\alpha, \beta}{\Am\T} \big)$-smooth 
	with respect to $\norma{\beta}{\cdot}$.
\end{lemma}
\begin{proof}
	Fix $\Am \in \reals^{m \times d}, b \in \reals^m$ and define $f: \reals^d \to \reals$ as $f(\wv) = h(\Am \wv + \bv)$.
	By the chain rule, we have that $\grad f(\wv) = \Am\T \grad h(\Am \wv + \bv)$. Using smoothness of $h$, 
	we successively deduce,
	\begin{align*}
		\normad{\beta}{\grad f(\wv_1) - \grad f(\wv_2)} 
		&= \normad{\beta}{\Am\T (\grad h(\Am \wv_1 + \bv) - \grad h(\Am \wv_2 + \bv))} \\
		&\stackrel{\eqref{eq:matrix_norm_defn_app}}{\le} 
			\norma{\alpha, \beta}{\Am\T} \norma{\alpha}{\grad h(\Am \wv_1 + \bv) - \grad h(\Am \wv_2 + \bv)} \\
		&\le \norma{\alpha, \beta}{\Am\T} \, L \normad{\alpha}{\Am(\wv_1 - \wv_2)} \\
		&\stackrel{\eqref{eq:matrix_norm_defn_app}}{\le} L \normasq{\alpha, \beta}{\Am\T}  \norma{\beta}{\wv_1 - \wv_2} \,.
	\end{align*}
\end{proof}

Shown in Algo.~\ref{algo:smoothing:top_K_oracle} is the procedure to compute the 
outputs of the top-$K$ oracle from the $K$ best scoring outputs obtained, for instance, 
from the top-$K$ max-product algorithm.
\begin{algorithm}[tb]
   \caption{Top-$K$ oracle from top-$K$ outputs}
   \label{algo:smoothing:top_K_oracle}
\begin{algorithmic}[1]
   \STATE {\bfseries Input:} Augmented score function $\psi$, $\wv \in \reals^d$, $\mu > 0$
        $Y_K = \{\yv_1, \cdots, \yv_K\}$ such that $\yv_k = \max\pow{k}_{\yv\in\mcY} \psi(\yv ; \wv)$.
   \STATE Populate $\zv \in \reals^K$ so that $z_k = \frac{1}{\mu}\psi(\yv_k ; \wv)$.
   \STATE Compute $\uv^* = \argmin_{\uv \in \Delta^{K-1}} \normasq{2}{\uv - \zv}$ by a projection on the simplex.
   \RETURN $s = \sum_{k=1}^K u^*_k \, \psi(\yv_k ; \wv)$ and $\vv = \sum_{k=1}^K u^*_k \, \grad_\wv\psi(\yv_k ; \wv)$.
\end{algorithmic}
\end{algorithm}

\section{Smooth Inference in Trees}\label{sec:a:dp}
A graph $\mcG$ is a tree if it is connected, directed and each node has at most one incoming edge. 
It has one root $r \in \mcV$ with no incoming edge. 
An undirected graph with no loops can be converted to a tree by fixing an arbitrary root and 
directing all edges way from the root.
We say that $\mcG$ is a chain if it is a tree with root $p$ where all edges are of the form $(v+1, v)$. 
For a node $v$ in a tree $\mcG$, we denote by $\rho(v)$ and $C(v)$ 
respectively the parent of $v$ and the children of $v$ in the tree.

Recall first that the height of a node in a rooted tree is the
number of edges on the longest directed path from the node to a leaf where each edge is directed
away from the root. 
We first review the standard max-product algorithm 
for maximum a posteriori (MAP) inference \citep{dawid1992applications} - Algo.~\ref{algo:dp:supp}.
It runs in time $\bigO(p \max_{v\in\mcV} \abs{\mcY_v}^2)$
and requires space $\bigO(p \max_{v\in\mcV} \abs{\mcY_v})$.

\begin{algorithm}[tb]
   \caption{Standard max-product algorithm}
   \label{algo:dp:supp}
\begin{algorithmic}[1]
   \STATE {\bfseries Input:} Augmented score function $\psi(\cdot, \cdot ; \wv)$ defined on 
      tree structured graph $\mcG$ with root $r \in \mcV$.
   \STATE {\bfseries Initialize:} 
   Let $V$ be a list of nodes from $\mcV \backslash \{r\}$ arranged in increasing order of height.
   \FOR{$v$ in $V$}
   		\STATE Set $m_v(y_{\rho(v)}) \leftarrow \max_{y_v \in \mcY_v} \left\{
   			\psi_v(y_v) + \psi_{v, \rho(v)}(y_v, y_{\rho(v)})  + \sum_{v' \in C(v)} m_{v'}(y_v)
   		\right\}$ for each $y_{\rho(v)} \in \mcY_{\rho(v)}$.
   		\STATE Assign to $\delta_v(y_{\rho(v)})$ a maximizing assignment of $y_v$ from above
            for each $y_{\rho(v)} \in \mcY_{\rho(v)}$.
   \ENDFOR
   \STATE $\psi^* \leftarrow \max_{y_r \in \mcY_r} \left\{ \psi_r(y_r) + \sum_{v' \in C(r)} m_{v'}(y_r)  \right\}$.
   \STATE $y_r^* \leftarrow \argmax_{y_r \in \mcY_r} \left\{ \psi_r(y_r) + \sum_{v' \in C(r)} m_{v'}(y_r)  \right\}$.
   \FOR{$v$ in $\mathrm{reverse}(V)$}
   		\STATE $y_v^* = \delta_v(y_{\rho(v)}^*)$.
   \ENDFOR
   \RETURN $\psi^*, \yv^*=(y_1^*, \cdots, y_p^*)$.
\end{algorithmic}
\end{algorithm}

\begin{algorithm}[ptb]
   \caption{Top-$K$ max-product algorithm}
   \label{algo:dp:topK:main}
\begin{algorithmic}[1]
   \STATE {\bfseries Input:} Augmented score function $\psi(\cdot, \cdot ; \wv)$ defined on 
      tree structured graph $\mcG$ with root $r \in \mcV$, and integer $K>0$.
   \STATE {\bfseries Initialize:} 
   Let $V$ be a list of nodes from $\mcV \backslash \{r\}$ arranged in increasing order of height.
   \FOR{$v$ in $V$ and $k=1,\cdots, K$}
         \IF{$v$ is a leaf}
            \STATE $m_v\pow{k}(y_{\rho(v)}) \leftarrow \max\pow{k}_{y_v \in \mcY_v} \left\{
                  \psi_v(y_v) + \psi_{v, \rho(v)}(y_v, y_{\rho(v)}) \right\}$ for each $y_{\rho(v)} \in \mcY_{\rho(v)}$.
         \ELSE
            \STATE Assign for each $y_{\rho(v)} \in \mcY_{\rho(v)}$,  \label{line:algo:dp:topk:message_k}
            \begin{align} \label{eq:dp:topk:algo:update}
               m_v\pow{k}(y_{\rho(v)}) \leftarrow \maxK{k}{\, } \left\{ 
               \begin{matrix}
                  \psi_v(y_v) + \psi_{v, \rho(v)}(y_v, y_{\rho(v)}) \\
                  + \sum_{v' \in C(v)} m_{v'}^{(l_{v'})}(y_v)
               \end{matrix} \,
               \middle| \, 
               \begin{matrix}
                  y_v \in \mcY_v \text{ and } \\ l_{v'} \in [K] \text{ for } v' \in C(v) 
               \end{matrix}
               \right\} \,.
            \end{align}
         \STATE Let $\delta_v\pow{k}(y_{\rho(v)})$  
            and $\kappa_{v'}\pow{k}(y_{\rho(v)})$ for $v'\in C(v)$ 
            store the maximizing assignment of $y_v$ and $l_v'$ from above for each $y_{\rho(v)} \in \mcY_{\rho(v)}$.
         \ENDIF
   \ENDFOR
   \STATE For $k=1,\cdots, K$, set  \label{line:algo:dp:topk:final_score_k}
   \begin{align*}
      \psi\pow{k} \leftarrow \maxK{k}{} \bigg\{  \psi_r(y_r) + \sum_{v' \in C(r)} m_{v'}\pow{l_{v'}}(y_r)
      \, \bigg|  \,
                  y_r \in \mcY_r \text{ and } l_{v'} \in [K] \text{ for } v' \in C(r) 
      \bigg\}
   \end{align*}
   and assign maximizing assignments of $y_r, l_{v'}$ above respectively 
   to $y_r\pow{k}$ and $l_{v'}\pow{k}$ for $v' \in C(r)$.
   \FOR{$v$ in $\mathrm{reverse}(V)$ and $k = 1,\cdots, K$}
      \STATE Set $y_v\pow{k} \leftarrow \delta_v\pow{l\pow{k}_{v}} \big(y_{\rho(v)}\pow{k} \big)$.
      \STATE Set $l_{v'}\pow{k} = \kappa_{v'}\pow{l\pow{k}_v} \big(y_{\rho(v)}\pow{k} \big)$ for all $v' \in C(v)$.
   \ENDFOR
   \RETURN $\left\{ \psi\pow{k}, \yv\pow{k}:=(y_1\pow{k}, \cdots, y_p\pow{k}) \right\}_{k=1}^K$.
\end{algorithmic}
\end{algorithm}

\subsection{Proof of Correctness of Top-$K$ Max-Product}
We now consider the top-$K$ max-product algorithm, shown in full generality 
in Algo.~\ref{algo:dp:topK:main}. The following proposition proves its correctness.

\begin{proposition} \label{eq:prop:dp:topK_guarantee}
   Consider as inputs to Algo.~\ref{algo:dp:topK:main} an augmented score function $\psi(\cdot, \cdot ; \wv)$ defined on 
      tree structured graph $\mcG$, and an integer $K > 0$. Then, the outputs
   of Algo.~\ref{algo:dp:topK:main} satisfy $\psi\pow{k} = \psi(\yv\pow{k}) = \maxK{k}{\yv \in \mcY} \psi(\yv)$.
   Moreover, Algo.~\ref{algo:dp:topK:main} runs in time $\bigO(pK\log K \max_{v\in\mcV} \abs{\mcY_v}^2)$
   and uses space $\bigO(p K \max_{v\in\mcV} \abs{\mcY_v})$.
\end{proposition}
\begin{proof}
   For a node $v \in \mcV$, let $\tau(v)$ denote the sub-tree of $\mcG$ rooted at $v$. Let $\yv_{\tau(v)}$ denote
   $\big( y_{v'} \text{ for } v' \in \tau(v) \big)$. Define $\psi_{\tau(v)}$ as follows:
   if $v$ is a leaf, $\yv_{\tau(v)} = (y_v)$ and $\psi_{\tau(v)}(\yv_{\tau(v)}) := \psi_v(y_v)$.
   For a non-leaf $v$, define recursively
   \begin{align} \label{eq:dp:topk_proof:phi_subtree}
      \psi_{\tau(v)}(\yv_{\tau(v)}) := \psi_v(y_v) + \sum_{v' \in C(v)} \left[ \psi_{v, v'}(y_v, y_{v'}) 
         + \psi_{\tau(v')}(\yv_{\tau(v')}) \right] \,.
   \end{align}
   We will need some identities about choosing the $k$th largest element from a finite collection.
   For finite sets $S_1, \cdots, S_n$ and 
   functions $f_j: S_j \to \reals$, $h: S_1 \times S_2 \to \reals$, we have, 
   \begin{gather}
   \label{eq:dp:topk_proof:bellman1}
         \maxK{k}{u_1 \in S_1, \cdots, u_n \in S_n} \left\{ \sum_{j=1}^n f_j(u_j) \right\} 
         = \quad \maxK{k}{l_1, \cdots, l_n \in [k]} \left\{  
            \sum_{j=1}^n \maxK{l_j}{u_j \in S_j} f_j(u_j) 
         \right\} \,, \\
   \label{eq:dp:topk_proof:bellman2}
         \maxK{k}{u_1 \in S_1, u_2 \in S_2} \{ f_1(u_1) + h(u_1, u_2) \} 
         = \quad \maxK{k}{u_1 \in S_1, l \in [k]} \left\{  
            f_1(u_1) + \maxK{l}{u_2 \in S_2}  h(u_1, u_2) 
         \right\} \,.
   \end{gather}
   The identities above state that for a sum to take its $k$th largest value, 
   each component of the sum must take one of its $k$ largest values. Indeed, if one of the components
   of the sum took its $l$th largest value for $l > k$, replacing it with any of the $k$ largest values cannot 
   decrease the value of the sum.
   Eq.~\eqref{eq:dp:topk_proof:bellman2} is a generalized version of Bellman's principle of 
   optimality (see \citet[Chap. III.3.]{bellman1957dynamic} or \citet[Vol. I, Chap. 1]{bertsekas1995dynamic}).

   For the rest of the proof, $\yv_{\tau(v)} \backslash y_v$ is used as 
   shorthand for $\{y_{v'} \, | \, v' \in \tau(v) \backslash \{v\}  \}$.
   Moreover, $\max_{\yv_{\tau(v)}}$ represents maximization over 
   $\yv_{\tau(v)} \in \bigtimes_{v' \in \tau(v)} \mcY_{v'}$. Likewise for $\max_{\yv_{\tau(v)} \backslash y_v}$.
   Now, we shall show by induction that for all $v \in \mcV$, $y_v \in \mcY_v$ and $k = 1,\cdots, K$, 
   \begin{align} \label{eq:dp:topk_proof:ind_hyp}
      \maxK{k}{\yv_{\tau(v)}\backslash y_v} \psi_{\tau(v)}(\yv_{\tau(v)}) = \psi_v(y_v) + 
      \maxK{k}{} \bigg\{  \sum_{v' \in C(v)} m_{v'}\pow{l_{v'}}(y_{v'}) 
         \bigg| l_{v'} \in [K] \text{ for } v' \in C(v)
      \bigg\} \,.
   \end{align}
   The induction is based on the height of a node. The statement is clearly true for a leaf $v$ since $C(v) = \varnothing$. 
   Suppose \eqref{eq:dp:topk_proof:ind_hyp} holds for all nodes of height $\le h$. For  
   a node $v$ of height $h+1$, we observe that $\tau(v) \backslash v$ can be partitioned 
   into $\{\tau(v') \text{ for } v' \in C(v)\}$ to get, 
   \begin{gather} \nonumber
      \maxK{k}{\yv_{\tau(v)}\backslash y_v}  \psi_{\tau(v)}(\yv_{\tau(v)})
      - \psi_v(y_v) 
      \stackrel{\eqref{eq:dp:topk_proof:phi_subtree}}{=} \maxK{k}{\yv_{\tau(v)}\backslash y_v} 
         \bigg\{ \sum_{v' \in C(v)} \psi_{v, v'}(y_v, y_{v'}) + \psi_{\tau(v')}(\yv_{\tau(v')}) \bigg\} \\
      \label{eq:eq:topk_proof:ind_hyp_todo}
      \stackrel{\eqref{eq:dp:topk_proof:bellman1}}{=} 
         \maxK{k}{} \bigg\{  
            \sum_{v' \in C(v)} 
            \underbrace{
            \maxK{l_{v'}}{\yv_{\tau(v')}} 
            \{ \psi_{v, v'}(y_v, y_{v'}) + \psi_{\tau(v')}(\yv_{\tau(v')}) \} }_{=:\mcT_{v'}(y_v)}
            \, \bigg| \, 
            l_{v'} \in [K] \text{ for } v' \in C(v) 
         \bigg\} \,.
   \end{gather}
   Let us analyze the term in the underbrace, $\mcT_{v'}(y_v)$. We successively deduce,
   with the argument $l$ in the maximization below taking values in $\{1, \cdots, K\}$,
   \begin{align*}
   \mcT_{v'}(y_v)
   &\stackrel{\eqref{eq:dp:topk_proof:bellman2}}{=} 
      \maxK{l_{v'}}{y_{v'}, l} \bigg\{ 
         \psi_{v, v'}(y_v, y_{v'}) + \maxK{l}{\yv_{\tau(v')}\backslash y_{v'}} \psi_{\tau(v')}(\yv_{\tau(v')})  
         \bigg\}  \\
   &\stackrel{\eqref{eq:dp:topk_proof:ind_hyp}}{=}
      \maxK{l_{v'}}{y_{v'}, l} \bigg\{ 
         \begin{matrix}
         \psi_{v'}(y_{v'}) + \psi_{v, v'}(y_v, y_{v'}) +  \\
         \maxK{l}{} \big\{ \sum_{v'' \in C(v')} m_{v''}\pow{l_{v''}}(y_{v'}) 
            \, | \, l_{v''} \in [K] \text{ for } v'' \in C(v')
         \big\}
         \end{matrix}
      \bigg\} \\
   &\stackrel{\eqref{eq:dp:topk_proof:bellman2}}{=}
      \maxK{l_{v'}}{} \bigg\{
               \begin{matrix}
                  \psi_{v'}(y_{v'}) + \psi_{v', v}(y_{v'}, y_v) \\
                  + \sum_{v'' \in C(v')} m\pow{l_{v''}}_{v''}(y_{v'})
               \end{matrix} \,
               \bigg| \, 
               \begin{matrix}
                  y_{v'} \in \mcY_{v'} \text{ and } \\ l_{v''} \in [K] \text{ for } v'' \in C(v) 
               \end{matrix}
      \bigg\} \\
   &\stackrel{\eqref{eq:dp:topk:algo:update}}{=}
      m\pow{l_{v'}}_{v'}(y_v) \, .
   \end{align*}
   Here, the penultimate step followed from applying in reverse the identity \eqref{eq:dp:topk_proof:bellman2}
   with $u_1, u_2$ being by $y_{v'}, \{ l_{v''} \text{ for } v'' \in C(v')\}$ respectively,
   and $f_1$ and $h$ respectively  being $ \psi_{v'}(y_{v'}) + \psi_{v', v}(y_{v'}, y_v)$ 
   and $ \sum_{v''} m\pow{l_{v''}}_{v''}(y_{v'})$.
   Plugging this into \eqref{eq:eq:topk_proof:ind_hyp_todo} completes the induction argument.
   To complete the proof, we repeat the same argument over the root as follows. We note that 
   $\tau(r)$ is the entire tree $\mcG$. Therefore, $\yv_{\tau(r)} = \yv$ and $\psi_{\tau(r)} = \psi$.
   We now apply the identity \eqref{eq:dp:topk_proof:bellman2} with $u_1$ and $u_2$ being 
   $y_r$ and $\yv_{\tau(r) \backslash r}$ respectively and $f_1 \equiv 0$ to get 
   \begin{align*}
      \maxK{k}{\yv \in \mcY}  \psi(\yv)
      & \stackrel{\eqref{eq:dp:topk_proof:bellman2}}{=}
         \maxK{k}{y_r, l} \left\{ \maxK{l}{\yv\backslash y_r}  \psi(\yv) \right\} 
      = \maxK{k}{y_r, l} \left\{ \maxK{l}{\yv_{\tau(r)} \backslash y_r}  \psi_{\tau(r)}(\yv_{\tau(r)}) \right\} \\
      &\stackrel{\eqref{eq:dp:topk_proof:ind_hyp}}{=}
      \maxK{k}{y_r, l} \bigg\{ 
         \begin{matrix}
         \psi_{r}(y_{r}) +  
         \maxK{l}{} \big\{ \sum_{v \in C(r)} m_{v}\pow{l_{v}}(y_r) 
            \, | \, l_{v} \in [K] \text{ for } v \in C(r)
         \big\}
         \end{matrix}
      \bigg\} \\
   &\stackrel{\eqref{eq:dp:topk_proof:bellman2}}{=}
      \maxK{k}{} \bigg\{
               \begin{matrix}
                  \psi_r(y_r) + \\
                     \sum_{v \in C(r)} m\pow{l_{v}}_{v}(y_r)
               \end{matrix} \,
               \bigg| \, 
               \begin{matrix}
                  y_{r} \in \mcY_{r} \text{ and } \\ l_{v} \in [K] \text{ for } v \in C(r) 
               \end{matrix}
      \bigg\} \\
   &\,= \psi\pow{k} \,,
   \end{align*}
   where the last equality follows from Line~\ref{line:algo:dp:topk:final_score_k} of Algo.~\ref{algo:dp:topK:main}.

   The algorithm requires storage of $m_v\pow{k}$, an array of size $\max_{v \in \mcV} \abs{\mcY_v}$
   for each $k = 1,\cdots, K$, and $v \in \mcV$. The backpointers $\delta, \kappa$ are of the same size. 
   This adds up to a total storage of $\bigO(pK \max_{v} \abs{\mcY_v})$.
   To bound the running time, consider Line~\ref{line:algo:dp:topk:message_k} of Algo.~\ref{algo:dp:topK:main}.
   For a fixed $v' \in C(v)$, the computation
   \begin{align*} 
               \maxK{k}{y_v, l_{v'} } \left\{ 
                  \psi_v(y_v) + \psi_{v, \rho(v)}(y_v, y_{\rho(v)})
                  + m_{v'}^{(l_{v'})}(y_v)
               \right\} 
   \end{align*}
   for $k= 1, \cdots, K$ takes time $\bigO(K \log K \max_v \abs{\mcY_v})$.
   This operation is repeated for each $y_v \in \mcY_v$ and once for every $(v, v')\in \mcE$. Since 
   $\abs\mcE = p-1$, the total running time is $\bigO(p K \log K \max_v \abs{\mcY_v}^2)$.
\end{proof}

\subsection{Proof of Correctness of Entropy Smoothing of Max-Product}
Next, we consider entropy smoothing.

\begin{proposition}\label{eq:prop:dp:ent_guarantee}
   Given an augmented score function $\psi(\cdot, \cdot ; \wv)$ defined on 
      tree structured graph $\mcG$ and $\mu > 0$
   as input, Algo.~\ref{algo:dp:supp_exp} correctly computes
      $f_{-\mu H}(\wv)$ 
   and $\grad f_{-\mu H}(\wv)$.
   Furthermore, Algo.~\ref{algo:dp:supp_exp} runs in time $\bigO(p \max_{v\in\mcV} \abs{\mcY_v}^2)$
   and requires space $\bigO(p \max_{v\in\mcV} \abs{\mcY_v})$.
\end{proposition}
\begin{proof}
   The correctness of the function value $f_{- \mu H}$ follows from the bijection 
   $f_{- \mu H}(\wv) = \mu \, A_{\psi/\mu}(\wv)$ (cf. Prop.~\ref{prop:smoothing:exp-crf}),
   where Thm.~\ref{thm:pgm:sum-product} shows correctness of $A_{\psi / \mu}$.
   To show the correctness of the gradient, define the probability distribution $ P_{\psi, \mu}$ 
   as the probability distribution from Lemma~\ref{lemma:smoothing:first-order-oracle}\ref{lem:foo:exp}
   and $P_{\psi, \mu, v}, P_{\psi, \mu, v, v'}$ as its node and edge marginal probabilities respectively as
   \begin{align*}
      P_{\psi, \mu}(\yv ; \wv) 
            &= \frac{
            \exp\left(\tfrac{1}{\mu}\psi(\yv ; \wv)\right)}
            {\sum_{\yv' \in \mcY }\exp\left(\tfrac{1}{\mu}\psi(\yv' ; \wv)\right)} \,, \\
      P_{\psi, \mu, v}(\overline y_v ; \wv) 
            &= \sum_{\substack{ \yv \in \mcY \, : \\ y_v = \overline y_v} }P_{\psi, \mu}(\yv ; \wv)  \quad 
            \text{for } \overline y_v \in \mcY_v, v \in \mcV\,,  \text{ and, } \\
      P_{\psi, \mu, v, v'}(\overline y_v, \overline  y_{v'} ; \wv) 
            &= \sum_{\substack{ \yv \in \mcY : \\ y_v = \overline y_v, \\ y_{v'} = \overline y_{v'} } }P_{\psi, \mu}(\yv ; \wv)  \quad \text{for } \overline y_v \in \mcY_v, \overline y_{v'} \in \mcY_{v'}, (v,v') \in \mcE \,.
   \end{align*}
   Thm.~\ref{thm:pgm:sum-product} again shows that Algo.~\ref{algo:dp:supp_sum-prod} correctly produces 
   marginals  $P_{\psi, \mu, v}$ and $P_{\psi, \mu, v, v'}$.   
   We now start with Lemma~\ref{lemma:smoothing:first-order-oracle}\ref{lem:foo:exp} and invoke \eqref{eq:smoothing:aug_score_decomp}
   to get
   \begin{align*}
      \grad f_{-\mu H}(\wv) =& \sum_{\yv \in \mcY} P_{\psi, \mu}(\yv ; \wv) \grad \psi(\yv ; \wv) \\
         {=}& \sum_{\yv \in \mcY} P_{\psi, \mu}(\yv ; \wv)
            \left(
               \sum_{v\in \mcV} \grad \psi_v(y_v ; \wv) 
               + \sum_{(v, v')\in \mcE} \grad \psi_{v, v'}(y_v, y_{v'} ; \wv)
            \right) \,, \\
         =& \sum_{v \in \mcV} 
            \sum_{\yv \in \mcY}  P_{\psi, \mu}(\yv ; \wv) \grad \psi_v( y_v ; \wv)
            +
            \sum_{(v,v') \in \mcE} \sum_{\yv \in \mcY} P_{\psi, \mu}(\yv ; \wv) \grad \psi_{v, v'}(y_v, y_{v'} ; \wv) \\
         =& \sum_{v \in \mcV} \sum_{\overline y_v \in \mcY_v} 
            \sum_{\yv \in \mcY \, : \,y_v = \overline y_v}  P_{\psi, \mu}(\yv ; \wv) \grad \psi_v(\overline y_v ; \wv)
            \\ &\qquad+ 
            \sum_{(v,v') \in \mcE} \sum_{\overline y_v \in \mcY_v} \sum_{\overline y_{v'} \in \mcY_{v'}} 
            \sum_{\yv \in \mcY \, :\, \substack{ y_v = \overline y_v \\ y_{v'} = \overline y_{v'} } }
            P_{\psi, \mu}(\yv ; \wv) \grad \psi_{v, v'}(\overline y_v,  \overline y_{v'} ; \wv) \\
         =& \sum_{v \in \mcV} \sum_{\overline y_v \in \mcY_v} P_{\psi, \mu, v}(\overline y_v ; \wv) 
            \grad \psi_{v}(\overline y_v ; \wv)
            \\ &\qquad+
            \sum_{(v,v') \in \mcE} \sum_{\overline y_v \in \mcY_v} \sum_{\overline y_{v'} \in \mcY_{v'}} 
            P_{\psi, \mu, v, v'}(\overline y_v, \overline y_{v'} ; \wv) 
            \grad \psi_{v, v'}(\overline y_v, \overline y_{v'} ; \wv) \,.
   \end{align*}
   Here, the penultimate equality followed from breaking the sum over $\yv \in \mcY$ into an outer sum that sums over every
   $\overline y_v \in \mcY_v$ and an inner sum over $\yv \in \mcY : y_v = \overline y_v$, and likewise for the edges.
   The last equality above followed from the definitions of the marginals.
   Therefore, Line~\ref{line:algo:dp:exp:gradient} of Algo.~\ref{algo:dp:supp_exp} correctly computes the gradient.
   The storage complexity of the algorithm is $\bigO(p \max_v \abs{\mcY_v})$ provided that the edge marginals $P_{\psi, \mu, v, v'}$ 
   are computed on the fly as needed.  The time overhead of Algo.~\ref{algo:dp:supp_exp} after Algo.~\ref{algo:dp:supp_sum-prod}
   is $\bigO(p \max_v \abs{\mcY_v}^2)$, by noting that each edge marginal can be computed in constant time
   (Remark~\ref{remark:pgm:sum-prod:fast-impl}).
\end{proof}

\begin{algorithm}[ptbh]
   \caption{Sum-product algorithm}
   \label{algo:dp:supp_sum-prod}
\begin{algorithmic}[1]
   \STATE {\bfseries Procedure:} \textsc{SumProduct}
   \STATE {\bfseries Input:} Augmented score function $\psi$ defined on 
      tree structured graph $\mcG$ with root $r \in \mcV$.
   \STATE {\bfseries Notation:} Let $N(v) = C(v) \cup \{\rho(v)\}$ denote all the neighbors of 
      $v \in \mcV$ if the orientation of the edges were ignored.
   \STATE {\bfseries Initialize:} 
      Let $V$ be a list of nodes from $\mcV$ arranged in increasing order of height.
   \FOR{$v$ in $V \backslash \{r\}$} 
         \STATE Set for each $y_{\rho(v)} \in \mcY_{\rho(v)}$: \label{line:dp:exp_dp:update}
         \[
         m_{v \to \rho(v)}(y_{\rho(v)}) \leftarrow  \sum_{y_v \in \mcY_v} \left[  \exp \left(
            \psi_v(y_v) + \psi_{v, \rho(v)}(y_v, y_{\rho(v)}) \right) \prod_{v' \in C(v)} m_{v' \to v}(y_v)  \right]
          \,.
         \]
   \ENDFOR
   \STATE $A \leftarrow \log \sum_{y_r \in \mcY_r} \left[ \exp\left(
      \psi_r(y_r) \right) \prod_{v' \in C(r)} m_{v' \to r}(y_r) \right] $.
      \label{line:dp:exp_dp:log_part}

   \FOR{$v$ in $\mathrm{reverse}(V)$}
      \FOR{$v' \in C(v)$}
         \STATE Set for each $y_{v'} \in \mcY_{v'}$:
         \[
            m_{v\to v'}(y_{v'}) = \sum_{y_v \in \mcY_v}\left[
               \exp\left(
                  \psi_v(y_v) + \psi_{v', v}(y_{v'}, y_v)
                \right)
               \prod_{v'' \in N(v)\backslash\{v'\}} m_{v'' \to v}(y_v) 
            \right] \,.
         \]
      \ENDFOR
   \ENDFOR

   \FOR{ $v$ in $\mcV$}
      \STATE Set $P_v(y_v) \leftarrow \exp\left( \psi_v(y_v) - A  \right) \prod_{v'' \in N(v)} m_{v''\to v}(y_v)$
         for every $y_v  \in \mcY_v$.
   \ENDFOR

   \FOR{$(v, v')$ in $\mcE$}
      \STATE For every pair $(y_v, y_{v'}) \in \mcY_v \times \mcY_{v'}$, set \label{line:algo:sum-prod:pair}
         \begin{align*}
         P_{v, v'}(y_v, y_{v'}) \leftarrow & 
         \exp\left(\psi_v(y_v) + \psi_{v'}(y_{v'}) + \psi_{v, v'}(y_v, y_{v'}) - A  \right)  \\&
         \prod_{v'' \in N(v) \backslash \{v'\}} m_{v''\to v}(y_v) \prod_{v'' \in N(v') \backslash \{v\}} m_{v''\to v'}(y_{v'}) \,.
         \end{align*}
   \ENDFOR

   \RETURN $A, \{P_v \text{ for } v \in \mcV \}, \{P_{v, v'} \text{ for } (v, v') \in \mcE \}$.
\end{algorithmic}
\end{algorithm}

Given below is the guarantee of the sum-product algorithm (Algo.~\ref{algo:dp:supp_sum-prod}).
See, for instance, \citet[Ch. 10]{koller2009probabilistic} for a proof.
\begin{theorem} \label{thm:pgm:sum-product}
   Consider an augmented score function $\psi$ defined over a tree structured graphical model $\mcG$.
   Then, the output of Algo.~\ref{algo:dp:supp_sum-prod} satisfies
   \begin{align*}
      A &= \log \sum_{\yv \in \mcY} \exp(\psi(\yv)) \,, \\
      P_v(\overline y_v) &= \sum_{\yv \in \mcY\, :\, y_v = \overline y_v} \exp(\psi(\yv) - A) \, 
         \quad \text{for all $\overline y_v \in \mcY_v, v \in \mcV$, and, }\, \\
      P_{v, v'}(\overline y_v, \overline y_{v'}) &= 
      \sum_{\yv \in \mcY\, : \, \substack{ y_v = \overline y_v, \\ y_{v'} = \overline y_{v'} } }
         \exp(\psi(\yv) - A) \, 
         \quad \text{for all $\overline y_v \in \mcY_v, \overline y_{v'} \in \mcY_{v'},  (v,v') \in \mcE$.} \\
   \end{align*}
   Furthermore, Algo.~\ref{algo:dp:supp_sum-prod} runs in time $\bigO(p \max_{v\in\mcV} \abs{\mcY_v}^2)$
   and requires an intermediate storage of $\bigO(p \max_{v\in\mcV} \abs{\mcY_v})$. 
\end{theorem}
\begin{remark} \label{remark:pgm:sum-prod:fast-impl}
Line~\ref{line:algo:sum-prod:pair} of Algo.~\ref{algo:dp:supp_sum-prod} can be 
implemented in constant time by reusing the node marginals $P_v$ and messages $m_{v \to v'},m_{v' \to v}$ as 
   \begin{align*}
         P_{v, v'}(y_v, y_{v'}) = 
         \frac{P_v(y_v) P_{v'}(y_{v'}) \exp(\psi_{v, v'}(y_v, y_{v'}) + A)}{m_{v'\to v}(y_v)  m_{v\to v'}(y_{v'}) }   \,.
   \end{align*}
\end{remark}

\section{Inference Oracles in Loopy Graphs} \label{sec:a:smooth:loopy}

This section presents the missing details and 
recalls from literature the relevant algorithms and results required in Sec.~\ref{subsec:smooth_inference_loopy}.
First, we review the BMMF algorithm of \citet{yanover2004finding}, followed by graph cut
inference and graph matching inference.

We now recall and prove the correctness of the decoding scheme \eqref{eq:max-marg:defn} for completeness.
The result is due to \citet{pearl1988probabilistic,dawid1992applications}.
\begin{theorem} \label{thm:a:loopy:decoding}
	Consider an unambiguous augmented score function $\psi$ , that is,
    $\psi(\yv' ; \wv) \neq \psi(\yv'' ;\wv)$ for all distinct $\yv', \yv'' \in \mcY$.
	Then, the result
	$\widehat\yv$ of the decoding $\widehat y_v = \argmax_{j \in \mcY_v} \psi_{v ; j}$
	satisfies $\widehat\yv = \argmax_{\yv \in \mcY} \psi(\yv)$.
\end{theorem}
\begin{proof}
	Suppose for the sake of contradiction that $\widehat \yv \neq \yv^* := \argmax_{\yv \in \mcY} \psi(\yv)$.
	Let $v \in \mcV$ be such that $y_v = j$ and $y^*_v = j'$ where 
	$j \neq j'$. 
	By the fact that $\yv^*$ has the highest augmented score and unambiguity, we get that 
	\begin{align*}
		 \max_{\yv \in \mcY, \yv_v = j'} \psi(\yv)  = \psi(\yv^*) 
		 > \psi(\widehat\yv) = \max_{\yv \in \mcY, \yv_v = j} \psi(\yv) \,,
	\end{align*}
	which contradicts the definition of $\widehat y_v$.
\end{proof}

\subsection{Review of Best Max-Marginal First} \label{sec:a:bmmf}
If one has access to an algorithm $\mcM$ that can compute max-marginals, 
the top-$K$ oracle is easily implemented via the Best Max Marginal First (BMMF) algorithm of \citet{yanover2004finding}, 
which is recalled in Algo.~\ref{algo:top_k_map:general}. 
This algorithm requires computations of 
two sets of max-marginals 
per iteration, where a {\em set} of max-marginals refers to max-marginals for all variables $y_v$ in $\yv$.

\paragraph{Details}
The algorithm runs by maintaining a partitioning of the search space $\mcY$
and a table $\varphi\pow{k}(v, j)$ that stores the best score 
in partition $k$ (defined by constraints $\mcC\pow{k}$) 
subject to the additional constraint that $y_v = j$. 
In iteration $k$, the algorithm looks at the $k-1$ existing partitions and picks the best 
partition $s_k$ (Line~\ref{alg:bmmf:best_part}). 
This partition is further divided into two parts: 
the max-marginals in the promising partition (corresponding to $y_{v_k} = j_k$)
are computed (Line~\ref{alg:bmmf:line:max-marg}) 
and decoded (Line~\ref{alg:bmmf:line:decoding}) to yield $k$th best scoring $\yv\pow{k}$.
The scores of the less promising partition are updated via a second round of 
max-marginal computations (Line~\ref{alg:bmmf:line:update_score}).

\begin{algorithm}[tb]
   \caption{Best Max Marginal First (BMMF)}
   \label{algo:top_k_map:general}
\begin{algorithmic}[1]
   \STATE {\bfseries Input:} Augmented score function $\psi$, parameters $\wv$, non-negative integer $K$,
         algorithm $\mcM$ to compute max-marginals of $\psi$.
   \STATE {\bfseries Initialization:} $\mcC\pow{1} = \varnothing$ and $\mcU\pow{2} = \varnothing$.
      \FOR{$v \in [p]$}
         \STATE For $j \in \mcY_v$, set
         $\varphi\pow{1}(v;j) = \max \{ \psi(\yv ; \wv) \, | \, \yv \in \mcY \text{ s.t. } y_v = j \}$ using $\mcM$.
      \STATE Set $y\pow{1}_v = \argmax_{j \in \mcY_v} \varphi\pow{1}(v, j)$.
   \ENDFOR
   \FOR{$k = 2, \cdots, K$}
      \STATE Define search space
         $\mcS\pow{k} = \left\{ (v, j, s) \in [p] \times \mcY_v \times  [k-1] \,  \big|  \,
            y\pow{s}_v \neq j, \text{ and } (v, j, s) \notin \mcU\pow{t} \right\}$.
            \label{alg:bmmf:part_search}
      \STATE Find indices $(v_k, j_k, s_k) = \argmax_{(v, j, s) \in \mcS\pow{k}} \varphi\pow{s}(v,j)$
         and set constraints
         $\mcC\pow{k} = \mcC\pow{s_k} \cup \{ y_{v_k} = j_k \}$. \label{alg:bmmf:best_part}
      \FOR{$v\in[p]$}
         \STATE For each $j \in \mcY_v$, use $\mcM$ to set  $\varphi\pow{k}(v, j) = \max \left\{  
            \psi(\yv ; \wv) \, | \, \yv \in \mcY \text{ s.t. constraints }  
            \mcC\pow{k} \text{ hold and } y_v = j \right\}$. \label{alg:bmmf:line:max-marg}
         \STATE Set $y\pow{k}_v  = \argmax_{j \in \mcY_v} \varphi\pow{k}(v, j)$. \label{alg:bmmf:line:decoding}
      \ENDFOR
      \STATE Update $\mcU\pow{k+1} = \mcU\pow{k} \cup \left\{ (v_k, j_k, s_k) \right\}$ and
            $\mcC\pow{s_k} = \mcC\pow{s_k} \cup \{ y_{v_k} \neq j_k \}$ and the max-marginal table
         $\varphi\pow{s_k}(v, j) = \max_{\yv \in \mcY, \mcC\pow{s_k}, y_v = j} \psi(\yv ; \wv)$ using $\mcM$.
         \label{alg:bmmf:line:update_score}
   \ENDFOR
  \RETURN $\left\{ \left(\psi(\yv\pow{k} ; \wv), \yv\pow{k} \right) \right\}_{k=1}^K$.
\end{algorithmic}
\end{algorithm}

\paragraph{Guarantee}
The following theorem shows that Algo.~\ref{algo:top_k_map:general} provably 
implements the top-$K$ oracle as long as the max-marginals can be computed exactly
under the assumption of unambiguity. With approximate max-marginals however, 
Algo.~\ref{algo:top_k_map:general} comes with no guarantees.
\begin{theorem}[\citet{yanover2004finding}] \label{thm:inference:topKmm}
   Suppose the score function $\psi$ is unambiguous, that is,
   $\psi(\yv' ; \wv) \neq \psi(\yv'' ;\wv)$ for all distinct $\yv', \yv'' \in \mcY$.
   Given an algorithm $\mcM$ that can compute the max-marginals of $\psi$ exactly, 
   Algo.~\ref{algo:top_k_map:general} makes at most $2K$ calls to $\mcM$ and its output 
   satisfies $\psi(\yv_k ; \wv) = \max\pow{k}_{\yv \in \mcY} \psi(\yv ; \wv)$.
   Thus, the BMMF algorithm followed by a projection onto the simplex 
	(Algo.~\ref{algo:smoothing:top_K_oracle} in Appendix~\ref{sec:a:smoothing}) 
	is a correct implementation of the top-$K$ oracle.
	It makes $2K$ calls to $\mcM$.
\end{theorem}

\paragraph{Constrained Max-Marginals}
The algorithm requires computation of max-marginals subject to constraints of the form 
$y_v \in Y_v$ for some set $Y_v \subseteq \mcY_v$. This is accomplished by
redefining for a constraint $y_v \in Y_v$:
\[
	\overline \psi(\yv) = 
	\begin{cases}
		\psi(\yv), \, \text{ if } y_v \in Y_v \\
		-\infty, \, \text{ otherwise}
	\end{cases}  \,.
\]

\subsection{Max-Marginals Using Graph Cuts}  \label{sec:a:graph_cuts}

\begin{algorithm}[tb]
   \caption{Max-marginal computation via Graph Cuts}
   \label{algo:top_k_map:graph_cuts}
\begin{algorithmic}[1]
   \STATE {\bfseries Input:} Augmented score function $\psi(\cdot, \cdot ; \wv)$ with $\mcY = \{0, 1\}^p$, 
   		constraints $\mcC$ of the form $y_v = b$ for $b  \in \{0,1\}$.
   \STATE Using artificial source $s$ and sink $t$, set $V' = \mcV \cup \{s, t\}$ and $E' = \varnothing$.
   	\FOR{$v \in [p]$}
   		\STATE Add to $E'$ the (edge, cost) pairs $(s \to y_v, \theta_{v; 0})$ and 
   			$(y_v \to t, \theta_{v; 1})$.
   	\ENDFOR
   	\FOR{$v,v' \in \mcR$ such that $v < v'$}
   		\STATE Add to $E'$ the (edge, cost) pairs
   			$(s \to y_v , \theta_{vv'; 00})$, 
   			$(y_{v'} \to t , \theta_{vv'; 11})$, 
   			$(y_v \to y_{v'} , \theta_{vv'; 10})$,
   			$(y_{v'} \to y_v , \theta_{vv'; 01} - \theta_{vv' ;00} - \theta_{vv' ; 11})$.
	\ENDFOR
	\FOR{constraint $y_v = b$ in $\mcC$}
		\STATE Add to $E'$ the edge $y_v \to t$ if $b=0$ or edge $s \to y_v$ if $b=1$ with cost $+\infty$.
		\label{line:algo:mincut:constr}
	\ENDFOR
	\STATE Create graph $G'=(V', E')$, where parallel edges are merged by adding weights.
	\STATE Compute minimum cost $s, t$-cut of $G'$. Let $C$ be its cost.
	\STATE Create $\widehat \yv \in \{0, 1\}^p$ as follows: for each $v \in \mcV$, 
		set $\widehat y_v = 0$ if the edge $s\to v$ is cut.
		Else $\widehat y_v = 1$.
	\RETURN $-C, \widehat \yv$.
\end{algorithmic}
\end{algorithm}


This section recalls a simple procedure to compute max-marginals using graph cuts.
Such a construction was used, for instance, by \citet{kolmogorov2004energy}.

\paragraph{Notation}
In the literature on graph cut inference, it is customary to work with the energy function, which is defined as
the negative of the augmented score $-\psi$.
For this section, we also assume that the labels are binary, i.e., $\mcY_v = \{0,1\}$ for each $v\in[p]$.
Recall the decomposition~\eqref{eq:smoothing:aug_score_decomp} of the augmented score function over nodes and edges.
Define a reparameterization
\begin{gather*}
\theta_{v;z}(\wv) = -\psi_v(z ; \wv) \, \text{ for } v\in \mcV, z \in \{0,1\} \,\\
\theta_{vv';z,z'}(\wv) = 
		-\psi_{v,v'}(z, z' ;\wv) \,,  \, \text{ if } (v, v') \in \mcE \\ 
	\text{ for } (v, v')\in \mcE, (z,z') \in \{0,1\}^2\,.
\end{gather*}
We then get 
\begin{gather*}
-\psi(\yv) = \sum_{v=1}^p \sum_{z \in \{0,1\}} \theta_{v; z} \ind(y_v=z)
+ \sum_{v=1}^p \sum_{v'=i+1}^p \sum_{z, z' \in \{0,1\}} \theta_{vv'; z z'} \ind(y_v=z) \ind(y_{v'}=z') \ind((v,v') \in \mcE) \,,
\end{gather*}
where we dropped the dependence on $\wv$ for simplicity.
We require the energies to be submodular, i.e., for every $v, v' \in [p]$, we have that
\begin{align} \label{eq:top_k_map:submodular}
	\theta_{vv' ; 00} + \theta_{vv' ; 11} \le \theta_{vv' ; 01} + \theta_{vv' ; 10} \,.
\end{align}
Also, assume without loss of generality that $\theta_{v;z}, \theta_{vv';zz'}$ are non-negative
\citep{kolmogorov2004energy}.

\paragraph{Algorithm and Correctness}
Algo.~\ref{algo:top_k_map:graph_cuts} shows how to compute the max-marginal relative
to a single variable $y_v$. The next theorem shows its correctness.

\begin{theorem}[\citet{kolmogorov2004energy}] \label{thm:top_k_map:graph_cuts}
	Given a binary pairwise graphical model with augmented score function $\psi$
	which satisfies \eqref{eq:top_k_map:submodular},
	and a set of constraints $\mcC$, Algo.~\ref{algo:top_k_map:graph_cuts} 
	returns $\max_{\yv \in \mcY_\mcC} \psi(\yv ; \wv)$, where $\mcY_\mcC$ denotes the
	subset of $\mcY$ that satisfies constraints $\mcC$. 
	Moreover, Algo.~\ref{algo:top_k_map:graph_cuts} requires one maximum flow computation.
\end{theorem} 

\subsection{Max-Marginals Using Graph Matchings}  \label{sec:a:graph_matchings}
The alignment problem that we consider in this section is as follows: 
given two sets $V, V'$, both of equal size (for simplicity), and a weight function 
$\varphi: V \times V' \to \reals$, the task is to find a map $\sigma : V \to V'$ 
so that each $v \in V$ is mapped to a unique $z \in V'$ and the total weight $\sum_{v \in V} \varphi(v, \sigma(v))$
is maximized.
For example, $V$ and $V'$ might represent two natural language sentences and this task is to align the two sentences.

\paragraph{Graphical Model}
This problem is framed as a graphical model as follows. 
Suppose $V$ and $V'$ are of size $p$. Define $\yv = (y_1, \cdots, y_p)$ so that $y_v$ denotes $\sigma(v)$.
The graph $\mcG = (\mcV, \mcE)$ is constructed as the fully connected graph over $\mcV = \{1, \cdots, p\}$. 
The range $\mcY_v$ of each $y_v$ is simply $V'$ in the unconstrained case.
Note that when considering constrained max-marginal computations, $\mcY_v$ might be subset of $V'$.
The score function $\psi$ is defined as node and edge potentials as in 
Eq.~\eqref{eq:smoothing:aug_score_decomp}. Again, we suppress
dependence of $\psi$ on $\wv$ for simplicity.
Define unary and pairwise scores as 
\begin{align*}
	\psi_v(y_v) = \varphi(v, y_v) \quad \text{and} \quad 
	\psi_{v, v'}(y_v, y_{v'}) =
	\begin{cases}
		0, \text{ if } y_v \neq y_{v'} \\
		-\infty, \text{ otherwise }
	\end{cases}
	\, .
\end{align*}

\paragraph{Max Oracle}
The max oracle with $\psi$ defined as above, or equivalently, the inference problem \eqref{eq:pgm:inference} 
(cf. Lemma~\ref{lemma:smoothing:first-order-oracle}\ref{lem:foo:max}) can be cast as a maximum weight bipartite matching, 
see e.g., \citet{taskar2005discriminative}.
Define a fully connected bipartite graph $G = (V \cup V', E)$ with partitions $V, V'$, and directed edges from 
each $v \in V$ to each vertex $z \in V'$ with weight $\varphi(v, z)$.
The maximum weight bipartite matching in this graph $G$ gives the mapping $\sigma$, and thus implements the max oracle. 
It can be written as the following linear program:
\begin{align*}
	\max_{\{\theta_{v,z} \text{ for } (v,z) \in E\}} \, &  \sum_{(v,z) \in E} \varphi(v, z) \theta_{v, z} \,, \\
	\mathrm{s.t.} \quad & 0 \le \theta_{v, z} \le 1 \, \forall (v, z) \in V \times V' \\
	& \sum_{v \in V} \theta_{v, z} \le 1 \, \forall z \in V' \\
	& \sum_{z \in V'} \theta_{v, z} \le 1 \, \forall v \in V \, .
\end{align*}

\paragraph{Max-Marginal}
For the graphical model defined above, the max-marginal $\psi_{\bar v ; \bar z}$ is the constrained maximum weight matching in the 
graph $G$ defined above subject to the constraint that $\bar v$ is mapped to $\bar z$. The linear program above can be
modified to include the constraint $\theta_{\bar v, \bar z} = 1$:
\begin{align} \label{eq:top_k_map:graph_matchings:max-marg:def}
\begin{aligned}
	\max_{\{\theta_{v,z} \text{ for } (v,z) \in E\}} \, &  \sum_{(v,z) \in E} \varphi(v, z) \theta_{v, z} \,, \\
	\mathrm{s.t.} \quad & 0 \le \theta_{v, z} \le 1 \, \forall (v, z) \in V \times V' \\
	& \sum_{v \in V} \theta_{v, z} \le 1 \, \forall z \in V' \\
	& \sum_{z \in V'} \theta_{v, z} \le 1 \, \forall v \in V \\
	& \theta_{\bar v, \bar z} = 1 \, .
\end{aligned}
\end{align}

\paragraph{Algorithm to Compute Max-Marginals}
Algo.~\ref{algo:top_k_map:graph_matchings}, which shows how to compute max-marginals
is due to \citet{duchi2007using}.
Its running time complexity is as follows: the initial 
maximum weight matching computation takes $\bigO(p^3)$ via computation of a maximum flow~\citep[Ch.~10]{schrijver-book}.
Line~\ref{line:top_k_map:graph_matching:all-pairs} of Algo.~\ref{algo:top_k_map:graph_matchings}
can be performed by the all-pairs shortest paths algorithm \citep[Ch.~8.4]{schrijver-book} in time $\bigO(p^3)$.
Its correctness is shown by the following theorem:
\begin{theorem}[\citet{duchi2007using}] \label{thm:top_k_map:graph_matching}
	Given a directed bipartite graph $G$ and weights $\varphi: V \times V' \to \reals$,
	the output $\psi_{v ; z}$ from Algo.~\ref{algo:top_k_map:graph_matchings} 
	are valid max-marginals, i.e., $\psi_{v ; z}$ coincides with the optimal value of the linear program
	\eqref{eq:top_k_map:graph_matchings:max-marg:def}. Moreover, Algo.~\ref{algo:top_k_map:graph_matchings} 
	runs in time $\bigO(p^3)$ where $p = \abs{V} = \abs{V'}$.
\end{theorem}

\begin{algorithm}[tb]
   \caption{Max marginal computation via Graph matchings}
   \label{algo:top_k_map:graph_matchings}
\begin{algorithmic}[1]
   \STATE {\bfseries Input:} Directed bipartite graph $G=(V \cup V', E)$, 
   weights $\varphi: V \times V' \to \reals$.
   \STATE Find a maximum weight bipartite matching $\sigma^*$ in the graph $G$. Let the maximum weight be $\psi^*$.
   \STATE Define a weighted residual bipartite graph $\widehat G = (V \cup V', \widehat E)$, 
   	where the set $\widehat E$ is populated as follows:
   		for $(v,z) \in E$, add an edge $(v,z)$ to $\widehat E$ with weight $1 - \ind(\sigma^*(v) = z)$, 
		 add $(z, v)$ to $\widehat E$ with weights $- \ind(\sigma^*(v) = z)$.
   	
   	\STATE Find the maximum weight path from every vertex $z\in V'$ to every vertex $v \in V$
   		and denote this by $\Delta(z, v)$. \label{line:top_k_map:graph_matching:all-pairs}
   	\STATE Assign the max-marginals $\psi_{v ; z} = \psi^* + \ind(\sigma^*(v) \neq z) \, \left( \Delta(z, v) + \varphi(v, z) \right)$
   		for all $(v, z) \in V \times V'$.
	\RETURN Max-marginals $\psi_{v;z}$ for all $(v, z) \in V \times V'$.
\end{algorithmic}
\end{algorithm}

\subsection{Proof of Proposition~\ref{prop:smoothing:max-marg:all}} \label{sec:a:proof-prop}
\begin{proposition_unnumbered}[\ref{prop:smoothing:max-marg:all}]
   Consider as inputs an augmented score function $\psi(\cdot, \cdot ; \wv)$, 
   an integer $K>0$ and a smoothing parameter $\mu > 0$.
   Further, suppose that $\psi$ is unambiguous, that is, 
   $\psi(\yv' ; \wv) \neq \psi(\yv'' ;\wv)$ for all distinct $\yv', \yv'' \in \mcY$.
   Consider one of the two settings:
   \begin{enumerate}[label={\upshape(\Alph*)}, align=left, leftmargin=*]
   \item the output space $\mcY_v = \{0,1\}$ for each $v \in \mcV$, and the function
      $-\psi$ is submodular (see Appendix~\ref{sec:a:graph_cuts} and, in particular, \eqref{eq:top_k_map:submodular}
      for the precise definition), or, 
   \item the augmented score corresponds to an alignment task where the 
      inference problem~\eqref{eq:pgm:inference} corresponds to a 
      maximum weight bipartite matching (see Appendix~\ref{sec:a:graph_matchings} for a precise definition).
   \end{enumerate}
   In these cases, we have the following:
   \begin{enumerate}[label={\upshape(\roman*)}, align=left, widest=iii, leftmargin=*]
      \item The max oracle can be implemented at a 
         computational complexity of $\bigO(p)$ minimum cut computations in Case~\ref{part:prop:max-marg:cuts}, 
         and in time $\bigO(p^3)$ in Case~\ref{part:prop:max-marg:matching}.
      \item The top-$K$ oracle can be implemented at a 
         computational complexity of $\bigO(pK)$ minimum cut computations in Case~\ref{part:prop:max-marg:cuts}, 
         and in time $\bigO(p^3K)$ in Case~\ref{part:prop:max-marg:matching}.
      \item The exp oracle is \#P-complete in both cases.
   \end{enumerate}
\end{proposition_unnumbered}
\begin{proof}
	A set of max-marginals can be computed by an algorithm $\mcM$ defined as follows:
	\begin{itemize}
	\item In Case~\ref{part:prop:max-marg:cuts}, invoke Algo.~\ref{algo:top_k_map:graph_cuts} a total of $2p$ times, 
	with $y_v =0$, and $y_v = 1$ for each $v \in \mcV$. This takes a total of $2p$ min-cut computations.
	\item In Case~\ref{part:prop:max-marg:matching}, $\mcM$ is simply Algo.~\ref{algo:top_k_map:graph_matchings}, which takes time 
	$\bigO(p^3)$.
	\end{itemize}
	The max oracle can then be implmented by the decoding in Eq.~\eqref{eq:max-marg:defn}, whose correctness is 
	guaranteed by Thm.~\ref{thm:a:loopy:decoding}.
	The top-$K$ oracle is implemented by invoking the BMMF algorithm with $\mcM$ defined above, followed by 
	a projection onto the simplex (Algo.~\ref{algo:smoothing:top_K_oracle} in Appendix~\ref{sec:a:smoothing}) 
	and its correctness is guaranteed by Thm.~\ref{thm:inference:topKmm}.
	Lastly, the result of exp oracle follows from \citet[Thm. 15]{jerrum1993polynomial} in conjunction with 
	Prop.~\ref{prop:smoothing:exp-crf}.
\end{proof}

\subsection{Inference using branch and bound search}  \label{sec:a:bb_search}
Algo.~\ref{algo:top_k:bb} with the input $K=1$ is the standard best-first branch and bound 
search algorithm.
Effectively, the top-$K$ oracle is implemented by simply 
continuing the search procedure until $K$ outputs have been produced - compare 
Algo.~\ref{algo:top_k:bb} with inputs $K=1$ and $K > 1$. We now prove the correctness guarantee.

\begin{algorithm}[tb]
   \caption{Top-$K$ best-first branch and bound search}
   \label{algo:top_k:bb}
\begin{algorithmic}[1]
   \STATE {\bfseries Input:} Augmented score function $\psi(\cdot, \cdot ; \wv)$, integer $K > 0$,
         search space $\mcY$, upper bound $\widehat \psi$, split strategy.
   \STATE {\bfseries Initialization:} Initialize priority queue with 
      single entry $\mcY$ with priority $\widehat \psi(\mcY ; \wv)$, 
      and solution set $\mcS$ as the empty list.
      \WHILE{$\abs{\mcS} < K$}
         \STATE Pop $\widehat \mcY$ from the priority queue. \label{line:algo:bbtopk:pq}
         \IF{${\widehat \mcY} = \{\widehat \yv\}$ is a singleton} \label{line:algo:bbtopk:1}
            \STATE Append $( \widehat \yv, \psi(\widehat \yv ; \wv) )$ to $S$.
         \ELSE
            \STATE $\mcY_1, \mcY_2 \leftarrow \mathrm{split}(\widehat \mcY)$.
            \STATE Add $\mcY_1$ with priority $\widehat \psi(\mcY_1 ; \wv)$
               and $\mcY_2$ with priority $\widehat \psi(\mcY_2 ; \wv)$ to the priority queue.
         \ENDIF
      \ENDWHILE
      \RETURN $\mcS$.
\end{algorithmic}
\end{algorithm}

\begin{proposition_unnumbered}[\ref{prop:smoothing:bb-search}]
   Consider an augmented score function $\psi(\cdot, \cdot, \wv)$, 
   an integer $K > 0$ and a smoothing parameter $\mu > 0$.
   Suppose the upper bound function $\widehat \psi(\cdot, \cdot ; \wv): \mcX \times 2^{\mcY} \to \reals$
   satisfies the following properties:
   \begin{enumerate}[label=(\alph*), align=left, widest=a, leftmargin=*]
      \item $\widehat \psi(\widehat \mcY ; \wv)$ is finite for every $\widehat \mcY \subseteq \mcY$,
      \item $\widehat \psi(\widehat \mcY ; \wv) \ge \max_{\yv \in \widehat \mcY} \psi(\yv ; \wv)$
         for all $\widehat \mcY \subseteq \mcY$, and,
      \item $\widehat \psi(\{\yv\} ; \wv) = \psi(\yv ; \wv)$ for every $\yv \in \mcY$.
   \end{enumerate}
   Then, we have the following:
   \begin{enumerate}[label={\upshape(\roman*)}, align=left, widest=ii, leftmargin=*]
      \item Algo.~\ref{algo:top_k:bb} with $K=1$ is a valid implementation of the max oracle.
      \item Algo.~\ref{algo:top_k:bb} followed by a projection onto the simplex 
         (Algo.~\ref{algo:smoothing:top_K_oracle} in Appendix~\ref{sec:a:smoothing}) is a valid implementation of the top-$K$ oracle.
   \end{enumerate}
\end{proposition_unnumbered}
\begin{proof}
Suppose at some point during the execution of the algorithm, 
we have a $\widehat \mcY = \{\widehat \yv\}$ on Line~\ref{line:algo:bbtopk:1}
and that $\abs{\mcS} = k$ for some $0 \le k < K$.
From the properties of the quality upper bound $\widehat \psi$, 
and using the fact that $ \{\widehat \yv\}$ had the highest priority 
in the priority queue (denoted by $(*)$), we get, 
\begin{align*}
   \psi(\widehat\yv ; \wv) &= \widehat \psi(\{ \widehat \yv\} ; \wv) \\
      &\stackrel{(*)}{\ge} \max_{Y \in \mcP} \widehat \psi(Y ; \wv) \\
      &\ge \max_{Y \in \mcP} \max_{\yv \in Y} \psi(\yv ; \wv) \\
      &\stackrel{(\#)}{=} \max_{\yv \in \mcY - \mcS}  \psi(\yv ; \wv) \,,
\end{align*}
where the equality $(\#)$ followed from the fact that 
any $\yv \in \mcY$ exits the priority queue only if it is added to $\mcS$.
This shows that if a $\widehat \yv$ is added to $\mcS$, it has a score that is no less than
that of any $\yv \in \mcY - \mcS$.  In other words, Algo.~\ref{algo:top_k:bb} returns 
the top-$K$ highest scoring $\yv$'s.
\end{proof}

\section{The Casimir Algorithm and Non-Convex Extensions: Missing Proofs} \label{sec:a:catalyst}
This appendix contains missing proofs from Sections~\ref{sec:cvx_opt} and \ref{sec:ncvx_opt}.
Throughout, we shall assume that $\omega$ is fixed and drop the subscript in $A_\omega, D_\omega$.
Moreover, an unqualified norm $\norm{\cdot}$ refers to the Euclidean norm $\norma{2}{\cdot}$.

\subsection{Behavior of the Sequence $(\alpha_k)_{k \ge 0}$} \label{sec:a:c_alpha_k}

\begin{lemma_unnumbered}[\ref{lem:c:alpha_k}]
	Given a positive, non-decreasing sequence $(\kappa_k)_{k\ge 1}$ and $\lambda \ge 0$, 
	consider the sequence $(\alpha_k)_{k \ge 0}$ defined by \eqref{eq:c:update_alpha}, where
	$\alpha_0 \in (0, 1)$ such that $\alpha_0^2 \ge \lambda / (\lambda + \kappa_1)$.
	Then, we have for every $k \ge 1$ that $0< \alpha_k \le \alpha_{k-1}$ and,
	$
		\alpha_k^2 \ge {\lambda}/({\lambda + \kappa_{k+1}}) \,.
	$
\end{lemma_unnumbered}
\begin{proof}
	It is clear that \eqref{eq:c:update_alpha} always has a positive root, so the update is well defined.
	Define sequences $(c_k)_{k \ge 1}, (d_k)_{k \ge 0}$ as
	\begin{align*}
		c_k = \frac{\lambda + \kappa_k}{\lambda + \kappa_{k+1}}\,, \quad \mbox{and} \quad
		d_k = \frac{\lambda}{\lambda + \kappa_{k+1}} \,.
	\end{align*}
	Therefore, we have that $c_k d_{k-1} = d_k$, $0 < c_k \le 1$ and $0 \le d_k < 1$.
	With these in hand, the rule for $\alpha_k$ can be written as 
	\begin{align} \label{eq:lem:c:alpha_k}
		\alpha_k = \frac{  -(c_k \alpha_{k-1}^2  - d_k ) + \sqrt{ (c_k \alpha_{k-1}^2  - d_k )^2 + 4 c_k \alpha_{k-1}^2 }}{2} \,.
	\end{align}
	We show by induction that that $d_k \le \alpha_k^2 < 1$.
	The base case holds by assumption. Suppose that $\alpha_{k-1}$ satisfies
	the hypothesis for some $k \ge 1$.
	Noting that $\alpha_{k-1}^2 \ge d_{k-1}$ is equivalent to $c_k \alpha_{k-1}^2 - d_k \ge 0$, we get that 
	\begin{align}
	\nonumber
	\sqrt{ (c_k \alpha_{k-1}^2  - d_k )^2 + 4 c_k \alpha_{k-1}^2 } 
	&\le 
	\sqrt{ (c_k \alpha_{k-1}^2  - d_k )^2 + 4 c_k \alpha_{k-1}^2 
				+ 2 (c_k \alpha_{k-1}^2 - d_k) (2\sqrt{c_k} \alpha_{k-1})  } \\
	&= c_k \alpha_{k-1}^2  - d_k + 2\sqrt{c_k} \alpha_{k-1} \,.
	\label{eq:lem:c:alpha_k_helper}
	\end{align}
	We now conclude from \eqref{eq:lem:c:alpha_k} and \eqref{eq:lem:c:alpha_k_helper} that
	\begin{align}
	\nonumber
		\alpha_k &\le \frac{  -(c_k \alpha_{k-1}^2  - d_k ) + (c_k \alpha_{k-1}^2  - d_k + 2\sqrt{c_k} \alpha_{k-1}) }{2} \\
			&= \sqrt{c_k}{\alpha_{k-1}} \le \alpha_{k-1} < 1\,,
		\label{eq:lem:c:alpha_k_dec}
	\end{align}
	since $c_k \le 1$ and $\alpha_{k-1} < 1$. To show the other side, we expand out \eqref{eq:lem:c:alpha_k} 
	and apply \eqref{eq:lem:c:alpha_k_helper} again to get
	\begin{align*}
		\alpha_k^2 - d_k 
		&=  \frac{1}{2}(c_k \alpha_{k-1}^2 - d_k)^2 + (c_k \alpha_{k-1}^2 - d_k) 
			- \frac{1}{2}(c_k \alpha_{k-1}^2 - d_k) \sqrt{(c_k \alpha_{k-1}^2 - d_k)^2 + 4 c_k \alpha_{k-1}^2 } \\
		&= \frac{1}{2}(c_k \alpha_{k-1}^2 - d_k) \left(2 + (c_k \alpha_{k-1}^2 - d_k) 
			- \sqrt{(c_k \alpha_{k-1}^2 - d_k)^2 + 4 c_k \alpha_{k-1}^2 }
			  \right) \\
		&\ge \frac{1}{2}(c_k \alpha_{k-1}^2 - d_k) \left(2 + (c_k \alpha_{k-1}^2 - d_k) 
			- (c_k \alpha_{k-1}^2 - d_k+ 2\sqrt{c_k} \alpha_{k-1})
			  \right) \\
		&= (c_k \alpha_{k-1}^2 - d_k) ( 1- \sqrt{c_k}\alpha_{k-1}) \ge 0 \,.
	\end{align*}
	The fact that $(\alpha_{k})_{k\ge 0}$ is a non-increasing sequence follows from~\eqref{eq:lem:c:alpha_k_dec}.
\end{proof}


\subsection{Proofs of Corollaries to Theorem~\ref{thm:catalyst:outer}} \label{subsec:c:proofs_missing_cor}
We rewrite \eqref{thm:c:main:main} from Theorem~\ref{thm:catalyst:outer} as follows:
	\begin{align} \label{eq:c:app:main}
		F&(\wv_k) - F^* \le 
			\left( \prod_{j=1}^k  \frac{1-\alpha_{j-1}}{1-\delta_j} \right) 
				\left( F(\wv_0) - F^* + \frac{\gamma_0}{2} \normsq{\wv_0 - \wv^*} \right) + \mu_k D_\omega \\ 
			&+ 
			\frac{1}{1-\alpha_k} \left[
			\left( \prod_{j=1}^k  \frac{1-\alpha_j}{1-\delta_j} \right) (1 + \delta_1) \mu_1 D_\omega + 
			\sum_{j=2}^k \left( \prod_{i=j}^k  \frac{1-\alpha_i}{1-\delta_i} \right) 
			\left( \mu_{j-1} - (1-\delta_j)\mu_j \right)D_\omega 
			\right]
			\,, \nonumber
	\end{align}
Next, we have proofs of Corollaries~\ref{cor:c:outer_sc} to~\ref{cor:c:outer_smooth_dec_smoothing}.

\begin{corollary_unnumbered}[\ref{cor:c:outer_sc}]
	Consider the setting of Thm.~\ref{thm:catalyst:outer}. 
	Let $q = \frac{\lambda}{\lambda + \kappa}$. 
	Suppose $\lambda > 0$ and $\mu_k = \mu$, $\kappa_k = \kappa$, for all $k \ge 1$. Choose  $\alpha_0 = \sqrt{q}$ and, 
	$\delta_k = \frac{\sqrt{q}}{2 - \sqrt{q}} \,.$
	Then, we have,
	\begin{align*}
	F(\wv_k) - F^* \le \frac{3 - \sqrt{q}}{1 - \sqrt{q}} \mu D +  
	2 \left( 1- \frac{\sqrt q}{2} \right)^k \left( F(\wv_0) - F^* \right) \,.
	\end{align*}
\end{corollary_unnumbered}
\begin{proof}
	Notice that when $\alpha_0 = \sqrt{q}$, we have, $\alpha_k = \sqrt{q}$ for all $k$. Moreover, for our choice of $\delta_k$, 
	we get, for all $k, j$,  $\frac{1-\alpha_k}{1-\delta_j} = 1 - \frac{\sqrt q}{2}$. 
	Under this choice of $\alpha_0$, we have, $\gamma_0 = \lambda$. So, we get the dependence on initial conditions as
	\begin{align*}
		\Delta_0 = F(\wv_0) - F^* + \frac{\lambda}{2} \normsq{\wv_0 - \wv^*} \le 2( F(\wv_0) - F^*) \,,
	\end{align*}
	by $\lambda$-strong convexity of $F$. The last term of \eqref{eq:c:app:main} is now, 
	\begin{align*}
		\frac{\mu D}{1-\sqrt {q}} \left[ \underbrace{\left( 1 - \frac{\sqrt q}{2} \right)^{k-1} }_{\le 1}
			+ \underbrace{\frac{\sqrt q}{2} \sum_{j=2}^k \left( 1 - \frac{\sqrt q}{2} \right)^{k-j}}_{\stackrel{(*)}{\le} 1  }
			\right] \le \frac{2 \mu D}{1 - \sqrt q} \, ,
	\end{align*}
	where $(*)$ holds since 
	\begin{align*}
		\sum_{j=2}^k \left( 1 - \frac{\sqrt q}{2} \right)^{k-j} \le \sum_{j=0}^\infty \left( 1 - \frac{\sqrt q}{2} \right)^{j} 
		= \frac{2}{\sqrt{q}} \,.
	\end{align*}
\end{proof}

\begin{corollary_unnumbered}[\ref{cor:c:outer_sc:decreasing_mu_const_kappa}]
	Consider the setting of Thm.~\ref{thm:catalyst:outer}. 
	Let $q = \frac{\lambda}{\lambda + \kappa}, \eta = 1 - \frac{\sqrt q}{2}$. 
	Suppose $\lambda > 0$ and 
	$\kappa_k = \kappa$, for all $k \ge 1$. Choose  $\alpha_0 = \sqrt{q}$ and, 
	the sequences $(\mu_k)_{k \ge 1}$ and $(\delta_k)_{k \ge 1}$ as 
	\begin{align*}
	\mu_k = \mu \eta^{{k}/{2}} \,, \qquad \text{and,} \qquad
	\delta_k = \frac{\sqrt{q}}{2 - \sqrt{q}} \,,
	\end{align*}
	where $\mu > 0$ is any constant.
	Then, we have, 
	\begin{align*}
	F(\wv_k) - F^* \le \eta^{{k}/{2}} \left[  
	2 \left( F(\wv_0) - F^* \right) 
	+ \frac{\mu D_\omega}{1-\sqrt{q}} \left(2-\sqrt{q} + \frac{\sqrt{q}}{1 - \sqrt \eta}  \right)
	\right] \, .
	\end{align*}
\end{corollary_unnumbered}
\begin{proof}
    As previously in Corollary~\ref{cor:c:outer_sc}, notice that under the specific parameter choices here, we have,
    $\gamma_0 = \lambda$, $\alpha_k = \sqrt{q}$ for each $k$, and $\frac{1 - \delta}{1 - \alpha} = 1 - \frac{\sqrt q}{2} = \eta$.
    By $\lambda$-strong convexity of $F$ and the fact that $\gamma_0 = \lambda$, the contribution of $\wv_0$ can be upper 
    bounded by $2(F(\wv_0) - F^*)$. Now, we plugging these into \eqref{eq:c:app:main} 
    and collecting the terms dependent on 
    $\delta_k$ separately, we get,
    \begin{align} \label{eq:c:outer:sc:dec_smoothing}
        \nonumber
        F(\wv_k) - F^* \le & \underbrace{2 \eta^k (F(\wv_0) - F^*)}_{=: \mcT_1} + 
            \underbrace{\mu_k D}_{=: \mcT_2}  \\ &+ 
            \frac{1}{1 - \sqrt{q}} \left( 
                \underbrace{\eta^k \mu_1 D}_{=: \mcT_3} + 
                \underbrace{\sum_{j=2}^k \eta^{k-j+1} (\mu_{j-1} - \mu_j)D}_{=:\mcT_4} + 
                \underbrace{\sum_{j=1}^k \eta^{k-j+1} \mu_j \delta_j D}_{=: \mcT_5} 
            \right) \,.
    \end{align}
    We shall consider each of these terms. Since $\eta^k \le \eta^{k/2}$, we get
    $\mcT_1 \le 2\eta^{k/2}(F(\wv_0) - F^*)$ and $\mcT_3 = \eta^k \mu_1 D \le \eta^k \mu D \le \eta^{k/2} \mu D$.
    Moreover, $\mcT_2 = \mu_k D = \eta^{k/2} \mu D$.
    Next, using $ 1- \sqrt \eta \le 1 - \eta = \frac{\sqrt q}{2}$, 
    \begin{align*}
        \mcT_4 &= \sum_{j=2}^k \eta^{k-j+1}(\mu_{j-1} - \mu_j) D 
            = \sum_{j=2}^k \eta^{k-j+1} \mu \eta^{\nicefrac{(j-1)}{2}} (1 - \sqrt\eta) D \\
            &\le \frac{\sqrt{q}}{2} \mu D \sum_{j=2}^k \eta^{k - \frac{j-1}{2}}  
            = \frac{\sqrt{q}}{2} \mu D \eta^{\nicefrac{(k+1)}{2}} \sum_{j=0}^{k-2} \eta^{j/2}
            \le \frac{\sqrt{q}}{2} \mu D \frac{\eta^{\nicefrac{(k+1)}{2}} }{1- \sqrt\eta} \\
            &\le \frac{\sqrt{q}}{2} \mu D \frac{\eta^{\nicefrac{k}{2}} }{1- \sqrt\eta} \, .
    \end{align*}
    Similarly, using $\delta_j = \nicefrac{\sqrt q}{2\eta}$, we have, 
    \begin{align*}
        \mcT_5 &= \sum_{j=1}^k \eta^{k-j+1} \mu \eta^{j/2} D \frac{\sqrt q}{2\eta} 
            = \frac{\sqrt{q}}{2} \mu D\sum_{j=1}^k \eta^{\nicefrac{k-j}{2}}
            \le \frac{\sqrt{q}}{2} \mu D \frac{\eta^{\nicefrac{k}{2}} }{1- \sqrt\eta} \, .
    \end{align*}
    Plugging these into \eqref{eq:c:outer:sc:dec_smoothing} completes the proof.
\end{proof}

\begin{corollary_unnumbered}[\ref{cor:c:outer_smooth}]
	Consider the setting of Thm.~\ref{thm:catalyst:outer}. Suppose  $\mu_k = \mu$, $\kappa_k = \kappa$, for all $k \ge 1$
	and $\lambda = 0$. Choose $\alpha_0 = \frac{\sqrt{5}-1}{2}$ and 
	$\delta_k = \frac{1}{(1 + k)^2} \,.$
	Then, we have, 
	\begin{align*}
	F(\wv_k) - F^* \le  \frac{8}{(k+2)^2} \left( F(\wv_0) - F^* + \frac{\kappa}{2} \normasq{2}{\wv_0 - \wv^*} \right) 
	+ \mu D_\omega\left( 1 +  \frac{12}{k+2} +  \frac{30}{(k+2)^2} \right) \, .
	\end{align*}
\end{corollary_unnumbered}
\begin{proof}
	Firstly, note that $\gamma_0 = \kappa \frac{\alpha_0^2}{1-\alpha_0} = \kappa$. Now, define 
	\begin{align*}
		\mcA_k &= \prod_{i=0}^k (1- \alpha_i) \text{, and, }
		\mcB_k = \prod_{i=1}^k (1-\delta_i) \, .
	\end{align*}
	We have, 
	\begin{align} \label{lem:c:b_k_1}
		\mcB_k = \prod_{i=1}^k \left( 1 - \frac{1}{(i+1)^2} \right) = \prod_{i=1}^k \frac{i(i+2)}{(i+1)^2} = \frac{1}{2} + \frac{1}{2(k+1)}\,.
	\end{align}
	Therefore, 
	\begin{align*}
		F(\wv_k) - F^* \le&  \frac{\mcA_{k-1}}{\mcB_k} \left( F(\wv_0) - F^* 
					+ \frac{\gamma_0}{2} \normsq{\wv_0 - \wv^*} \right) 
			 + \mu D  \\ &+  \frac{\mu D}{1-\alpha_0}  \left( \prod_{j=1}^k \frac{1-\alpha_{j-1}}{1-\delta_k}  \right) (1 + \delta_1) +
			   \mu D \sum_{j=2}^k \left(  \prod_{i=j}^k  \frac{1- \alpha_{i-1}}{1-\delta_i} \right) \frac{\delta_j}{1-\alpha_{j-1}}
			  \\ 
			 \le&  \underbrace{\frac{\mcA_{k-1}}{\mcB_k} \left( F(\wv_0) - F^* + \frac{\gamma_0}{2} \normsq{\wv_0 - \wv^*} \right)}_{=:\mcT_1} 
			 + \mu D \\
			 &+ 
			 \underbrace{ \frac{\tfrac{5}{4}\mu D}{1-\alpha_0} \frac{\mcA_{k-1}}{\mcB_k} }_{=:\mcT_2}+ 
			 	\underbrace{\mu D \sum_{j=2}^k \frac{ \nicefrac{\mcA_{k-1}} {\mcA_{j-2}}} { \nicefrac{\mcB_k}{\mcB_{j-1}}} 
			 		\frac{\delta_j}{1-\alpha_{j-1}}}_{=:\mcT_3} \,.
	\end{align*}
	From Lemma~\ref{lem:c:A_k:const_kappa}, which analyzes the evolution of $(\alpha_k)$ and $(\mcA_k)$,
	we get that $\frac{2}{(k+2)^2} \le \mcA_{k-1} \le \frac{4}{(k+2)^2}$ and $\alpha_k \le \frac{2}{k+3}$ for $k \ge 0$.
	Since $\mcB_k \ge \frac{1}{2}$, 
	\begin{align*}
		\mcT_1 \le \frac{8}{(k+2)^2} \left( F(\wv_0) - F^* + \frac{\gamma_0}{2} \normsq{\wv_0 - \wv^*} \right) \,.
	\end{align*}
	Moreover, since $\alpha_0 \le 2/3$, 
	\begin{align*}
		\mcT_2 \le \frac{30}{(k+2)^2} \,.
	\end{align*}
	Lastly, we have, 
	\begin{align*}
		\mcT_3 &\le \sum_{j=2}^k \frac{4}{(k+2)^2} \times \frac{(j+1)^2}{2} \times 
			2\left( \frac{1}{2}+ \frac{1}{2j} \right) \times \frac{1}{(j+1)^2} \times \frac{1}{1 - \nicefrac{2}{j+2}} \\
			&\le 2\frac{2}{(k+2)^2} \sum_{j=2}^k \frac{j+2}{j} \le \frac{4}{(k+2)^2} \left(k -1 + 2 \log k  \right)
			\le \frac{12}{k+2} \, ,
	\end{align*}
	where we have used the simplifications $\sum_{j=2}^k 1/k \le \log k$ and $ k-1+2\log k \le 3k$.
\end{proof}

\begin{corollary_unnumbered}[\ref{cor:c:outer_smooth_dec_smoothing}]
	Consider the setting of Thm.~\ref{thm:catalyst:outer} with $\lambda = 0$. 
	Choose $\alpha_0 = \frac{\sqrt{5}-1}{2}$, and for some non-negative constants $\kappa, \mu$, 
	define sequences $(\kappa_k)_{k \ge 1}, (\mu_k)_{k \ge 1}, (\delta_k)_{k \ge 1}$ as 
	\begin{align*}
	\kappa_k = \kappa  \, k\,, \quad
	\mu_k = \frac{\mu}{k} \quad \text{and,} \quad
	\delta_k = \frac{1}{(k + 1)^2} \,.
	\end{align*}
	Then, for $k \ge 2$,  we have,
	\begin{align}
	F(\wv_k) - F^* \le  
		\frac{\log(k+1)}{k+1} \left( 
		2(F(\wv_0) - F^*) + \kappa \normasq{2}{\wv_0 - \wv^*} + 27 \mu D_\omega
		\right) \,.
	\end{align}
	For the first iteration (i.e., $k = 1$), this bound is off by a constant factor $1 / \log2$.
\end{corollary_unnumbered}
\begin{proof}
    Notice that $\gamma_0 = \kappa_1 \frac{\alpha_0^2}{1- \alpha_0} = \kappa$.
    As in Corollary~\ref{cor:c:outer_smooth}, define
    \begin{align*}  
        \mcA_k &= \prod_{i=0}^k (1- \alpha_i)\,, \quad \text{and,} \quad
        \mcB_k = \prod_{i=1}^k (1-\delta_i) \, .
    \end{align*}
    From Lemma~\ref{lem:c:A_k:inc_kappa} and \eqref{lem:c:b_k_1} respectively, we have for $k \ge 1$, 
    \begin{align*}
        \frac{1- \frac{1}{\sqrt 2}}{k+1} &\le \mcA_{k} \le \frac{1}{k+2}\,, \quad \text{and,} \quad
        \frac{1}{2} \le \mcB_{k} \le 1\, .
    \end{align*}
    Now, invoking Theorem~\ref{thm:catalyst:outer}, we get, 
    \begin{align} \label{eq:cor:c:nsc:dec_smoothing_eq}
        F(\wv_k) - F^* \le&  \nonumber
        \underbrace{\frac{\mcA_{k-1}}{\mcB_k} \left( F(\wv_0) - F^*
        		 + \frac{\gamma_0}{2} \normsq{\wv_0 - \wv^*} \right)}_{=:\mcT_1} + 
        \underbrace{\mu_k D}_{=:\mcT_2} + 
        \underbrace{\frac{1}{1 - \alpha_0} \frac{\mcA_{k-1}}{\mcB_k} \mu_1 D(1 + \delta_1)}_{=:\mcT_3} +  \\
        &\underbrace{\sum_{j=2}^k \frac{\mcA_{k-1}/\mcA_{j-1}}{\mcB_k / \mcB_{j-1}} (\mu_{j-1} - \mu_j) D}_{=:\mcT_4} + 
        \underbrace{\sum_{j=2}^k \frac{\mcA_{k-1}/\mcA_{j-1}}{\mcB_k / \mcB_{j-1}} \delta_j \mu_j D }_{=:\mcT_5} \,.
    \end{align}
    We shall bound each of these terms as follows.
    \begin{gather*}
        \mcT_1 = \frac{\mcA_{k-1}}{\mcB_k} \left( F(\wv_0) - F^* + \frac{\gamma_0}{2} \normsq{\wv_0 - \wv^*} \right)
             = \frac{2}{k+1} \left( F(\wv_0) - F^* + \frac{\kappa _0}{2} \normsq{\wv_0 - \wv^*} \right) \,, \\
        \mcT_2 = \mu_k D = \frac{\mu D}{k} \le \frac{2\mu D}{k+1} \, , \\
        \mcT_3 = \frac{1}{1 - \alpha_0} \frac{\mcA_{k-1}}{\mcB_k} \mu_1 D(1 + \delta_1)
            \le 3 \times \frac{2}{k+1} \times {\mu} \times \frac{5}{4}D = \frac{15}{2} \frac{\mu D}{k+1} \,,
    \end{gather*}
    where we used the fact that $\alpha_0 \le 2/3$. Next, 
    using $\sum_{j=2}^k {1}/({j-1})  = 1 + \sum_{j=2}^{k-1} {1}/{j} \le 1 + \int_{1}^{k-1}{dx}/{x} = 1 + \log(k-1)$,
    we get,
    \begin{align*}
    \nonumber
        {\mcT_4} 
            &= \sum_{j=2}^k \frac{2}{k+1} \cdot \frac{j}{1- \frac{1}{\sqrt 2}} \left(\frac{\mu}{j-1} - \frac{\mu}{j}\right) D 
            = 2\sqrt2(\sqrt2 + 1) \frac{\mu D}{k+1} \sum_{j=2}^k \frac{1}{j-1}  \nonumber \\
            &\le 2\sqrt2(\sqrt2 + 1) \mu D \left( \frac{1 + \log(k+1)}{k+1} \right) \,.
    \end{align*}
    Moreover, from $\sum_{j=2}^k {1}/{(j+1)^2} \le \int_{2}^{k+1} {dx}/{x^2} \le 1/2$, it follows that 
    \begin{align*}
        \mcT_5
            = \sum_{j=2}^k \frac{2}{k+1} \cdot \frac{j}{1- \frac{1}{\sqrt 2}}  \frac{\mu}{j} \cdot \frac{1}{(j+1)^2} D
            = 2\sqrt2(\sqrt2+1) \frac{\mu D}{k+1} \sum_{j=2}^k \frac{1}{(j+1)^2} 
            \le \sqrt2(\sqrt2 + 1) \frac{\mu D}{k+1} \, .   
    \end{align*}
    Plugging these back into \eqref{eq:cor:c:nsc:dec_smoothing_eq}, we get
    \begin{align*}
        F(\wv_k) - F^* \le& \frac{2}{k+1} \left( F(\wv_k) - F^* + \frac{\kappa}{2} \normsq{\wv_0 - \wv^*} \right)+ \\
            &\frac{\mu D}{k+1} \left(2 + \frac{15}{2} + \sqrt2(1 + \sqrt2) \right) + 
            2\sqrt2(1 + \sqrt2)\mu D \frac{1 + \log(k+1)}{k+1} \,.
    \end{align*}
    To complete the proof, note that $\log(k+1) \ge 1$ for $k\ge 2$ and numerically verify that the coefficient of $\mu D$ is
    smaller than 27.
\end{proof}

\subsection{Inner Loop Complexity Analysis for Casimir} \label{sec:c:proofs:inner_compl}
Before proving Prop.~\ref{prop:c:inner_loop_final}, the following lemmas will be helpful. 
First, we present a lemma from \citet[Lemma 11]{lin2017catalyst}
about the expected number of iterations a randomized linearly convergent first order methods requires
to achieve a certain target accuracy.
\begin{lemma} 
	\label{lem:c:inner_loop}
	Let $\mcM$ be a linearly convergent algorithm and $f \in \mcF_{L, \lambda}$.
	Define $f^* = \min_{\wv \in \reals^d} f(\wv)$. 
	Given a starting point $\wv_0$ and a target accuracy $\eps$, 
	let $(\wv_k)_{k \ge 0}$ be the sequence of iterates generated by $\mcM$.
	Define 
	$T(\eps) = \inf \left\{ k \ge 0 \, | \, f(\wv_k) - f^* \le \eps \right\} \,.$
	We then have, 
	\begin{align}
	\expect[T(\eps)] \le \frac{1}{\tau(L, \lambda)} \log \left( \frac{2C(L, \lambda)}
	{\tau(L,\lambda)\eps} (f(\wv_0) - f^*) \right) + 1 \,.
	\end{align}
\end{lemma}
This next lemma is due to \citet[Lemma 14, Prop.~15]{lin2017catalyst}.
\begin{lemma}
\label{lem:c:inner_loop_restart}
	Consider $F_{\mu\omega, \kappa}(\cdot \, ;\zv)$ defined in Eq.~\eqref{eq:prox_point_algo}  
	and let $\delta \in [0,1)$. Let $\widehat F^* = \min_{\wv \in \reals^d} F_{\mu\omega, \kappa}(\wv ;\zv)$
	and $\widehat \wv^* = \argmin_{\wv \in \reals^d} F_{\mu\omega, \kappa}(\wv ;\zv)$.
	Further let $F_{\mu\omega}(\cdot \, ;\zv)$ be $L_{\mu\omega}$-smooth.
	We then have the following:
	\begin{gather*}
		F_{\mu\omega, \kappa}(\zv ;\zv) - \widehat F^* \le \frac{L_{\mu\omega} + \kappa}{2} \normasq{2}{\zv - \widehat \wv^*} \,,
			\quad \text{and,} \\
		F_{\mu\omega, \kappa}(\widehat\wv ;\zv) - \widehat F^* \le \frac{\delta\kappa}{8} \normasq{2}{\zv - \widehat \wv^*}
			\, \implies \,  
		F_{\mu\omega, \kappa}(\widehat\wv ;\zv) - \widehat F^* \le \frac{\delta\kappa}{2} \normasq{2}{\widehat \wv - \zv} \,.
	\end{gather*}
\end{lemma}
We now restate and prove Prop.~\ref{prop:c:inner_loop_final}.
\begin{proposition_unnumbered}[\ref{prop:c:inner_loop_final}]
	Consider $F_{\mu\omega, \kappa}(\cdot \, ;\zv)$ defined in Eq.~\eqref{eq:prox_point_algo},
	and a linearly convergent algorithm $\mcM$ with parameters $C$, $\tau$. 
	Let $\delta \in [0,1)$. Suppose $F_{\mu\omega}$ is $L_{\mu\omega}$-smooth and 
	$\lambda$-strongly convex. 
	Then the expected number of iterations $\expect[\widehat T]$ of $\mcM$ when started at $\zv$
	in order to obtain $\widehat \wv \in \reals^d$ that satisfies
	\begin{align}\label{eq:inner_stopping_criterion}
		F_{\mu\omega, \kappa}(\widehat\wv;\zv) - \min_\wv  F_{\mu\omega, \kappa}(\wv;\zv)\leq \tfrac{\delta\kappa}{2} \normasq{2}{\wv - \zv}
	\end{align}
	is upper bounded by 
	\begin{align*}
	\expect[\widehat T] \le \frac{1}{\tau(L_{\mu\omega} + \kappa, \lambda + \kappa)} \log\left( 
	\frac{8 C(L_{\mu\omega} + \kappa, \lambda + \kappa)}{\tau(L_{\mu\omega} + \kappa, \lambda + \kappa)} \cdot 
	\frac{L_{\mu\omega} + \kappa}{\kappa \delta} \right)  + 1 \,.
	\end{align*}
\end{proposition_unnumbered}
\begin{proof}
	In order to invoke
	Lemma~\ref{lem:c:inner_loop}, we must appropriately set $\eps$ for
	$\widehat\wv$ to satisfy \eqref{eq:inner_stopping_criterion} and then bound the ratio
	$(F_{\mu\omega, \kappa}(\zv ;\zv) - \widehat F^*) / \eps$.
	Firstly, Lemma~\ref{lem:c:inner_loop_restart} tells us that choosing 
	$\eps = \frac{\delta_k \kappa_k}{8} \normasq{2}{\zv_{k-1} - \widehat \wv^*}$ guarantees 
	that the $\widehat \wv$ so obtained satisfies \eqref{eq:inner_stopping_criterion},
	where $\widehat \wv^* := \argmin_{\wv \in \reals^d} F_{\mu\omega, \kappa}(\wv ;\zv)$, 
	Therefore, $(F_{\mu\omega, \kappa}(\zv ;\zv) - \widehat F^*) / \eps$
	is bounded from above by ${4(L_{\mu\omega} + \kappa)}/{\kappa \delta}$.
\end{proof}

\subsection{Information Based Complexity of \nsCatalystSvrg} \label{sec:c:proofs:total_compl}
Presented below are the proofs of Propositions~\ref{prop:c:total_compl_svrg_sc} to 
\ref{prop:c:total_compl_nsc:dec_smoothing} from Section~\ref{sec:catalyst:total_compl}.
We use the following values of $C, \tau$, see e.g., \citet{hofmann2015variance}.
\begin{align*}
	\tau(L, \lambda) &= \frac{1}{8 \tfrac{L}{\lambda} + n} \ge \frac{1}{8 \left( \tfrac{L}{\lambda} + n \right)}\\
	C(L, \lambda) &=  \frac{L}{\lambda} \left(  1 + \frac{n \tfrac{L}{\lambda}}{8 \tfrac{L}{\lambda} + n} \right)\,.
\end{align*}

\begin{proposition_unnumbered}[\ref{prop:c:total_compl_svrg_sc}]
	Consider the setting of Thm.~\ref{thm:catalyst:outer} with $\lambda > 0$ and 
	fix $\eps > 0$.
	If we run Algo.~\ref{algo:catalyst} with SVRG as the inner solver with parameters:
	$\mu_k = \mu = \eps / {10 D_\omega}$, $\kappa_k = k$ chosen as 
	\begin{align*}
		\kappa = 
	\begin{cases}
		\frac{A}{\mu n} - \lambda \,, \text{ if } \frac{A}{\mu n} > 4 \lambda \\
		\lambda \,, \text{ otherwise}
	\end{cases} \,,
	\end{align*}
	$q = {\lambda}/{(\lambda + \kappa)}$, $\alpha_0 = \sqrt{q}$, and
	$\delta = {\sqrt{q}}/{(2 - \sqrt{q})}$.
	Then, the number of iterations $N$ to obtain $\wv$ such that $F(\wv) - F^* \le \eps$ is 
	bounded in expectation as 
	\begin{align*}
		\expect[N] \le \widetilde \bigO \left( 
				n + \sqrt{\frac{A_\omega D_\omega n}{\lambda \eps}} 
			\right) \,.
	\end{align*}
\end{proposition_unnumbered}
\begin{proof}
	We use shorthand $A:=A_\omega$, $D := D_\omega$, $L_\mu = \lambda + \nicefrac{A}{\mu}$ and 
	$\Delta F_0 = F(\wv_0) - F^*$.
	Let $C, \tau$ be the linear convergence parameters of SVRG.
	From Cor.~\ref{cor:c:outer_sc}, the number of outer iterations $K$ required to obtain 
	$F(\wv_K) - F^* \le \eps$ is 
	\begin{align*}
		K \le \frac{2}{\sqrt{q}} \log\left(\frac{ 2 \Delta F_0}{\eps - c_q \mu D} \right)\, ,
	\end{align*}
	where $c_q = (3 -  \sqrt q)/(1 - \sqrt q)$.
	From Prop.~\ref{prop:c:inner_loop_final}, the number $T_k$ of inner iterations
	for  inner loop $k$ is, from $\delta_k = {\sqrt q}/({2 - \sqrt{q}})$,
	\begin{align*}
		\expect[T_k] &\le \frac{1}{\tau(L_\mu + \kappa, \lambda + \kappa)} \log\left( 
		\frac{8 C(L_\mu + \kappa, \lambda + \kappa)}{\tau(L_\mu + \kappa, \lambda + \kappa)} \cdot 
		\frac{L_\mu + \kappa}{\kappa} \cdot \frac{2 - \sqrt{q}}{\sqrt{q}} \right)  +  1 \\ 
					&\le \frac{2}{\tau(L_\mu + \kappa, \lambda + \kappa)} \log\left( 
		\frac{8 C(L_\mu + \kappa, \lambda + \kappa)}{\tau(L_\mu + \kappa, \lambda + \kappa)} \cdot 
		\frac{L_\mu + \kappa}{\kappa} \cdot \frac{2 - \sqrt{q}}{\sqrt{q}} \right)  \,.
	\end{align*}
	Let the total number $N$ of iterations of SVRG to obtain an iterate $\wv$ that satisfies $F(\wv) - F^* \le \eps$.
	Next, we upper bound $\expect[N] \le \sum_{i=1}^K \expect[T_k]$ as
	\begin{align} \label{eq:c:total_compl_sc}
		\expect[N] \le \frac{4}{\sqrt{q} \tau(L_\mu + \kappa, \lambda +\kappa)} \log \left( 
			\frac{8 C(L_\mu + \kappa, \lambda + \kappa)}{\tau(L_\mu + \kappa, \lambda + \kappa)}
			\frac{L_\mu + \kappa}{\kappa} \frac{2 - \sqrt{q}}{\sqrt q} \right) 
			\log\left( \frac{2(F(\wv_0) - F^*)}{\eps - c_q \mu D}  \right)\,.
	\end{align}
	
	Next, we shall plug in $C, \tau$ for SVRG in two different cases:
	\begin{itemize}
		\item Case 1: $A > 4\mu \lambda n$, in which case $\kappa + \lambda = A / (\mu n)$ and $q < 1/4$.
		\item Case 2: $A \le 4 \mu \lambda n$, in which case, $\kappa = \lambda$ and $q = 1/2$.
	\end{itemize}
	We first consider the term outside the logarithm. It is, up to constants, 
	\begin{align*}
		\frac{1}{\sqrt{q}} \left( n + \frac{A}{\mu(\lambda + \kappa)} \right)
	 	= n \sqrt{\frac{\lambda + \kappa}{\lambda}} + \frac{A}{\mu \sqrt{\lambda(\lambda + \kappa)}} \,.
	\end{align*}
	For Case 1, plug in $\kappa + \lambda = A / (\mu n)$ so this term evaluates to $\sqrt{{ADn}/({\lambda \eps})}$.
	For Case 2, we use the fact that $A \le 4 \mu \lambda n$ so that this term can be upper bounded by, 
	\[
		n\left( \sqrt{\frac{\lambda + \kappa}{\lambda}} + 4 \sqrt{ \frac{\lambda}{\lambda + \kappa}} \right) = 3\sqrt{2}n \,,
	\]
	since we chose $\kappa= \lambda$.
	It remains to consider the logarithmic terms. Noting that $\kappa \ge \lambda$ always, 
	it follows that the first log term of \eqref{eq:c:total_compl_sc} is clearly 
	logarithmic in the problem parameters.
	
	As for the second logarithmic term, we must evaluate $c_q$. For Case 1, we have that $q < 1/4$ so that $c_q < 5$
	and $c_q \mu D < \eps / 2$. For Case 2, we get that $q = 1/2$ and $c_q < 8$ so that $c_q \mu D < 4\eps/5$. Thus, the
	second log term of \eqref{eq:c:total_compl_sc} is also logarithmic in problem parameters.
\end{proof}

\begin{proposition_unnumbered} [\ref{prop:c:total_compl_sc:dec_smoothing_main}]
    Consider the setting of Thm.~\ref{thm:catalyst:outer}. 
    Suppose $\lambda > 0$ and $\kappa_k = \kappa$, for all $k \ge 1$ and 
    that $\alpha_0$, $(\mu_k)_{k \ge 1}$ and $(\delta_k)_{k \ge 1}$ 
    are chosen as in Cor.~\ref{cor:c:outer_sc:decreasing_mu_const_kappa},
    with $q = \lambda/(\lambda + \kappa)$ and $\eta = 1- {\sqrt q}/{2}$.
    If we run Algo.~\ref{algo:catalyst} with SVRG as the inner solver with these parameters,
    the number of iterations $N$ of SVRG required to obtain $\wv$ such that $F(\wv) - F^* \le \eps$ is 
    bounded in expectation as 
    \begin{align*}
        \expect[N] \le \widetilde \bigO \left( n 
            + \frac{A_\omega}{\mu(\lambda + \kappa)\eps} \left( F(\wv_0) - F^* + \frac{\mu D_\omega}{1-\sqrt{q}}  \right)
        \right) \,.
    \end{align*}
\end{proposition_unnumbered}
\begin{proof}
{
	We continue to use shorthand $A:=A_\omega$, $D := D_\omega$.
    First, let us consider the minimum number of outer iterations $K$ required to achieve $F(\wv_K)  - F^* \le \eps$.
    From Cor.~\ref{cor:c:outer_sc:decreasing_mu_const_kappa}, if we have $\eta^{-K/2} \Delta_0 \le \eps$, or, 
    \[
    K \ge K_{\min} := \frac{\log\left( {\Delta_0}/{\eps} \right)}{\log\left({1}/{\sqrt\eta}\right)} \,.
    \]
    For this smallest value, we have,
    \begin{align} \label{eq:c:min_smoother}
        \mu_{K_{\min}} = \mu \eta^{K_{\min}/2} = \frac{\mu \eps}{\Delta_0} \,.
    \end{align}
    Let $C, \tau$ be the linear convergence parameters of SVRG, and 
    define $L_k := \lambda + {A}/{\mu_k}$ for each $k\ge 1$.
    Further, let $\mcT'$ be such that 
    \[
    \mcT' \ge \max_{k\in\{1, \cdots, K_{\min}\}} \log\left( 8
    \frac{C(L_k + \kappa, \lambda + \kappa)}{\tau(L_k + \kappa, \lambda+\kappa)} \frac{L_k + \kappa}{\kappa\delta} \right) \,.
    \]
    Then, the total complexity is, from Prop.~\ref{prop:c:inner_loop_final}, (ignoring absolute constants)
    \begin{align}
        \nonumber
        \expect[N] &\le \sum_{k=1}^{K_{\min}} \left( n + \frac{\lambda + \kappa + \frac{A}{\mu_k}}{\lambda + \kappa} \right) \mcT' \\
            \nonumber
            &= \sum_{k=1}^{K_{\min}}  \left( n+1 + \frac{\nicefrac{A}{\mu}}{\lambda + \kappa} \eta^{-k/2} \right) \mcT' \\
            \nonumber
            &= \left( K_{\min}(n+1) + \frac{\nicefrac{A}{\mu}}{\lambda + \kappa} \sum_{k=1}^{K_{\min}} \eta^{-k/2}  \right) \mcT' \\
            \nonumber
            &\le \left( K_{\min}(n+1) + \frac{\nicefrac{A}{\mu}}{\lambda + \kappa} 
                \frac{\eta^{-K_{\min}/2}}{1 - \eta^{1/2} } \right) \mcT' \\
            &= \left( (n+1)\frac{\log\left( \frac{\Delta_0}{\eps} \right)}{\log(\nicefrac{1}{\sqrt\eta})}  
                 + \frac{\nicefrac{A}{\mu}}{\lambda + \kappa} \frac{1}{1 - \sqrt\eta} \frac{\Delta_0}{\eps}  \right) \mcT' \,.
    \end{align}
    It remains to bound $\mcT'$. Here, we use $\lambda + \frac{A}{\mu} \le L_k \le \lambda + \frac{A}{\mu_K}$ for all $k \le K$
    together with \eqref{eq:c:min_smoother} to 
    note that $\mcT'$ is logarithmic in $\Delta_0/\eps, n, AD, \mu, \kappa, \lambda\inv$.
}
\end{proof}

\begin{proposition_unnumbered}[\ref{prop:c:total_compl_svrg_smooth}]
	Consider the setting of Thm.~\ref{thm:catalyst:outer} and fix $\eps > 0$.
	If we run Algo.~\ref{algo:catalyst} with SVRG as the inner solver with parameters:
	$\mu_k = \mu ={\eps}/{20 D_\omega}$, $\alpha_0 = \tfrac{\sqrt{5} - 1}{2}$, 
	$\delta_k = {1}/{(k+1)^2}$, and $\kappa_k = \kappa = {A_\omega}/{\mu(n+1)}$.
	Then, the number of iterations $N$ to get a point $\wv$ such that $F(\wv) - F^* \le \eps$ is 
	bounded in expectation as 
	\begin{align*}
		\expect[N] \le \widetilde \bigO \left( n\sqrt{\frac{F(\wv_0) - F^*}{\eps}} + 
			\sqrt{A_\omega D_\omega n} \frac{\norma{2}{\wv_0 - \wv^*}}{\eps} \right) \, .
	\end{align*}
\end{proposition_unnumbered}
\begin{proof}
	We use shorthand $A:=A_\omega$, $D := D_\omega$, $L_\mu = \nicefrac{A}{\mu}$ and 
	$\Delta F_0 = F(\wv_0) - F^* + \frac{\kappa}{2} \normsq{\wv_0 -\wv^*}$.
	Further, let $C, \tau$ be the linear convergence parameters of SVRG.
	In Cor.~\ref{cor:c:outer_smooth}, the fact that $K \ge 1$ allows us to bound the contribution of the 
	smoothing as $10 \mu D$. So, we get that the number of outer iterations $K$ required to get
	$F(\wv_K) - F^* \le \eps$ can be bounded as 
	\begin{align*}
		K+1 \le \sqrt{\frac{8\Delta F_0}{\eps - 10 \mu D}} \,.
	\end{align*}
	Moreover, from our choice $\delta_k = 1 / (k+1)^2$, the number of inner iterations $T_k$
	for  inner loop $k$ is, from Prop.~\ref{prop:c:inner_loop_final},
	\begin{align*}
		\expect[T_k] &\le \frac{1}{\tau(L_\mu + \kappa,  \kappa)} \log\left( 
			\frac{8 C(L_\mu + \kappa, \kappa)}{\tau(L_\mu + \kappa, \kappa)} \cdot 
			\frac{L_\mu + \kappa}{\kappa} \cdot (k+1)^2 \right)  + 1\\
		&\le \frac{2}{\tau(L_\mu + \kappa,  \kappa)} \log\left( 
			\frac{8 C(L_\mu + \kappa, \kappa)}{\tau(L_\mu + \kappa, \kappa)} \cdot 
			\frac{L_\mu + \kappa}{\kappa} \cdot {\frac{8\Delta F_0}{\eps - 10 \mu D}} \right) \,.
	\end{align*}

	Next, we consider the total number $N$ of iterations of SVRG to obtain an iterate $\wv$ such that 
	$F(\wv) - F^* \le \eps$. Using the fact that $\expect[N] \le \sum_{i=1}^K \expect[T_k]$, we
	bound it as 
	\begin{align} \label{eq:c:total_compl_smooth}
		\expect[N] \le \frac{1}{\tau(L_\mu + \kappa, \kappa)} 
			\sqrt{\frac{8 \Delta F_0}{\eps - 10\mu D}}
			\log \left( 
			\frac{64 C(L_\mu + \kappa, \kappa)}{\tau(L_\mu + \kappa, \kappa)}
			\frac{L_\mu + \kappa}{\kappa} \frac{\Delta F_0}{\eps - 10\mu D} \right)\,.
	\end{align}
	
	Now, we plug into \eqref{eq:c:total_compl_smooth} the values of $C, \tau$ for SVRG. 
	Note that $\kappa = {L_\mu}/({n+1})$. So we have, 
	\begin{align*}
		\frac{1}{\tau(L_\mu + \kappa, \kappa)} &= 8 \left(  \frac{L_\mu + \kappa }{\kappa} + n \right) = 16(n+1) \,, \text{ and, } \\
		C(L_\mu+\kappa, \kappa) &= \frac{L_\mu + \kappa}{\kappa} \left( 1 + \frac{n \tfrac{L_\mu+\kappa}{\kappa}}{8\tfrac{L+\kappa}{\kappa} + n}  \right) 
			\le (n+2) \left(1 + \tfrac{n}{8} \right)\, .
	\end{align*}
	It now remains to assign $\mu = {\eps}/({20D})$ and plug $C, \tau$ from above into \eqref{eq:c:total_compl_smooth},
	noting that $\kappa = {20A D}/({\eps(n+1)})$.
\end{proof}

\begin{proposition_unnumbered}[\ref{prop:c:total_compl_nsc:dec_smoothing}]
    Consider the setting of Thm.~\ref{thm:catalyst:outer}. 
    Suppose $\lambda = 0$ and that $\alpha_0$, $(\mu_k)_{k\ge 1}$,$ (\kappa_k)_{k\ge 1}$ and $(\delta_k)_{k \ge 1}$ 
    are chosen as in Cor.~\ref{cor:c:outer_smooth_dec_smoothing}.
    If we run Algo.~\ref{algo:catalyst} with SVRG as the inner solver with these parameters,
    the number of iterations $N$ of SVRG required to obtain $\wv$ such that $F(\wv) - F^* \le \eps$ is 
    bounded in expectation as 
    \begin{align*}
        \expect[N] \le \widetilde\bigO \left( \frac{1}{\eps} 
        	\left( F(\wv_0) - F^* + \kappa \normasq{2}{\wv_0 - \wv^*} + \mu D \right)  
        	\left( n + \frac{A_\omega}{\mu \kappa} \right) 
            \right)   \,.
    \end{align*}
\end{proposition_unnumbered}
\begin{proof}
	Define short hand $A:= A_\omega$, $D := D_\omega$ and 
	\begin{gather}
        \label{eq:c:nsc:dec_smoothing_1}
        \Delta_0 := 2(F(\wv_0) - F^*) + \kappa \normsq{\wv_0 - \wv^*} + 27 \mu D \, .
    \end{gather}
    From Cor.~\ref{cor:c:outer_smooth_dec_smoothing},
    the number of iterations $K$ required to obtain $F(\wv_K) - F^* \le \frac{\log(K+1)}{K+1} \Delta_0 \le \eps$ is
    (see Lemma~\ref{lem:c:helper_logx}),
    \begin{align} \label{eq:c:nsc:dec_smoothing}
        K + 1 = \frac{2\Delta_0}{\eps} \log \frac{2\Delta_0}{\eps} \,.
    \end{align}
    Let $C, \tau$ be such that SVRG is linearly convergent with parameters $C, \tau$, and  
    define $L_k := {A}/{\mu_k}$ for each $k \ge 1$.
    Further, let $\mcT'$ be such that 
    \[
    \mcT' \ge \max_{k\in\{1, \cdots, K\}} \log\left( 8
    \frac{C(L_k + \kappa, \kappa)}{\tau(L_k + \kappa, \kappa)} \frac{L_k + \kappa}{\kappa \delta_k} \right) \, .
    \]
    
    Clearly, $\mcT'$ is logarithmic in  $K, n, AD, \mu, \kappa $.
    From Prop.~\ref{prop:c:inner_loop_final}, the minimum total complexity is (ignoring absolute constants)
    \begin{align}
        \expect[N] &= \sum_{k=1}^K \left( n + \frac{\nicefrac{A}{\mu_k} + \kappa_k}{\kappa_k} \right) \mcT' \nonumber \\
            &= \sum_{k=1}^K \left( n + 1 + \frac{A}{\mu_k\kappa_k} \right) \mcT' \nonumber \\
            &= \sum_{k=1}^K \left( n + 1 + \frac{A}{\mu\kappa} \right) \mcT' \nonumber \\
            &\le \left( n+ 1 + \frac{A}{\mu \kappa} \right)K \mcT' \,,
    \end{align}
    and plugging in $K$ from~\eqref{eq:c:nsc:dec_smoothing} completes the proof.
\end{proof}

\subsection{Prox-Linear Convergence Analysis} \label{sec:c:pl_struct_pred}
We first prove Lemma~\ref{lem:pl:struct_pred} that specifies the assumption required by the prox-linear in the case of structured prediction.
\begin{lemma_unnumbered}[\ref{lem:pl:struct_pred}]
	Consider the structural hinge loss $f(\wv) =  \max_{\yv \in \mcY} \psi(\yv ; \wv) = h\circ \gv(\wv)$ 
	where $h, \gv$ are as defined in \eqref{eq:mapping_def}.
	If the mapping $\wv \mapsto \psi(\yv ; \wv)$ is $L$-smooth with respect to $\norma{2}{\cdot}$ for all 
	$\yv \in \mcY$, then it holds for all $\wv, \zv \in \reals^d$ that
	\begin{align*}
	|h(\gv(\wv+\zv)) - h(\gv(\wv) + \grad\gv(\wv) \zv)| \le  \frac{L}{2}\normasq{2}{\zv}\,.
	\end{align*}
\end{lemma_unnumbered}
\begin{proof}
	For any $\Am \in \reals^{m \times d}$ and $\wv \in \reals^d$, and $\norma{2,1}{\Am}$ defined in~\eqref{eq:matrix_norm_defn}, 
	notice that 
	\begin{align} \label{eq:pl-struc-pred-pf:norm}
	\norma{\infty}{\Am\wv} \le \norma{2, 1}{\Am} \norma{2}{\wv} \,.
	\end{align}
	Now using the fact that max function $h$ satisfies $|h(\uv') - h(\uv)| \le \norma{\infty}{\uv' - \uv}$
	and the fundamental theorem of calculus $(*)$, we deduce
	\begin{align}
	|h(\gv(\wv+\zv)) - h(\gv(\wv) + \grad\gv(\wv) \zv)|
	&\le \norma{\infty}{\gv(\wv+\zv)- \left( \gv(\wv) + \grad\gv(\wv) \zv \right) } \nonumber \\
	&\stackrel{(*)}{\le} \norm*{\int_0^1 (\grad\gv(\wv + t\zv) - \grad\gv(\wv) )\zv \, dt }_{\infty} 
	\nonumber \\
	&\stackrel{\eqref{eq:pl-struc-pred-pf:norm}}{\le} 
	\int_0^1 \norma{2,1}{\grad\gv(\wv + t\zv) - \grad\gv(\wv) } \norma{2}{\zv} \, dt \,.
	\label{eq:pl-struc-pred-pf:1}
	\end{align}
	Note that the definition \eqref{eq:matrix_norm_defn} can equivalently be stated as 
	$\norma{2,1}{\Am} = \max_{\norma{1}{\uv}\le 1} \norma{2}{\Am\T \uv}$.
	Given $\uv \in \reals^m$, we index its entries $u_\yv$ by $\yv \in \mcY$. Then, the matrix norm
	in \eqref{eq:pl-struc-pred-pf:1} can be simplified as 
	\begin{align*}
	\norma{2,1}{\grad\gv(\wv + t\zv) - \grad\gv(\wv) } 
	&= \max_{\norma{1}{\uv} \le 1} \bigg\|{\sum_{\yv \in \mcY} u_\yv ( \grad \psi( \yv ; \wv + t\zv) 
		- \grad \psi( \yv ; \wv)) } \bigg\|_2 \\
	&\le \max_{\norma{1}{\uv} \le 1}  \sum_{\yv \in \mcY} |u_\yv| \norma{2}{\grad \psi( \yv ; \wv + t\zv) 
		- \grad \psi( \yv ; \wv)} \\
	&\le L t \norma{2}{\zv} \,,
	\end{align*}
	from the $L$-smoothness of $\psi$. 
	Plugging this back into \eqref{eq:pl-struc-pred-pf:1} completes the proof. The bound on the smothing approximation holds similarly by noticing that if $h$ is $1$-Lipschitz then $h_{\mu \omega}$ too since $\nabla h_{\mu\omega}(\uv) \in \dom h^*$ for any $\uv \in \dom h$.
\end{proof}


\subsection{Information Based Complexity of the Prox-Linear Algorithm with \nsCatalystSvrg} \label{sec:c:pl_proofs}

\begin{proposition_unnumbered}[\ref{prop:pl:total_compl}]
	Consider the setting of Thm.~\ref{thm:pl:outer-loop}. Suppose the sequence $\{\eps_k\}_{k\ge 1}$
	satisfies $\eps_k = \eps_0 / k$ for some $\eps_0 > 0$ and that 
	the subproblem of Line~\ref{line:pl:algo:subprob} of Algo.~\ref{algo:prox-linear} is solved using 
	\nsCatalystSvrg{} with the settings of Prop.~\ref{prop:c:total_compl_svrg_sc}.
	Then, total number of SVRG iterations $N$ required to produce a $\wv$ such that 
	$\norma{2}{\proxgrad_\eta(\wv)} \le \eps$ is bounded as
	\begin{align*}
		\expect[N] \le \widetilde\bigO\left(
			\frac{n}{\eta \eps^2} \left(F(\wv_0) - F^* + \eps_0 \right) + 
			\frac{\sqrt{A_\omega D_\omega n \eps_0\inv}}{\eta \eps^3} \left( F(\wv_0) - F^* + \eps_0 \right)^{3/2}
			\right) \, .
	\end{align*}
\end{proposition_unnumbered}
\begin{proof}
	First note that $\sum_{k=1}^{K} \eps_k \le \eps_0 \sum_{k=1}^K k\inv \le 4 \eps_0 \log K$
for $K \ge 2$. Let $\Delta F_0 := F(\wv_0) - F^*$ and use shorthand $A, D$ for $A_\omega, D_\omega$ respectively.
From Thm.~\ref{thm:pl:outer-loop}, the number $K$ of prox-linear iterations required to find a $\wv$ such that 
$\norma{2}{\proxgrad_\eta(\wv)} \le \eps$ must satisfy
\begin{align*}
	\frac{2}{\eta K} \left( \Delta F_0 + 4\eps_0 \log K \right) \le \eps \,.
\end{align*}
For this, it suffices to have (see e.g., Lemma~\ref{lem:c:helper_logx})
\begin{align*}
	K \ge \frac{4(\Delta F_0 + 4 \eps_0)}{\eta\eps^2} \log\left( \frac{4(\Delta F_0 + 4 \eps_0)}{\eta\eps^2} \right) \,.
\end{align*}
Before we can invoke Prop.~\ref{prop:c:total_compl_svrg_sc}, 
we need to bound the dependence of each inner loop on its warm start:
$F_\eta(\wv_{k-1} ; \wv_{k-1}) - F_\eta(\wv_{k}^* ; \wv_{k-1})$ in terms of problem parameters, 
where $\wv_k^* = \argmin_{\wv} F_\eta(\wv ; \wv_{k-1})$
is the exact result of an exact prox-linear step.
We note that $F_\eta(\wv_{k-1} ; \wv_{k-1}) = F(\wv_{k-1}) \le F(\wv_0)$, by Line~\ref{line:pl:algo:accept} 
of Algo.~\ref{algo:prox-linear}. 
Moreover, from $\eta \le 1/L$ and Asmp.~\ref{asmp:pl:upper-bound}, we have,
\begin{align*}
	F_\eta(\wv_k^* ; \wv_{k-1}) &= \frac{1}{n} \sum_{i=1}^n h \big( \gv\pow{i}(\wv_{k-1}) + \grad \gv\pow{i}(\wv_{k-1})(\wv_k^* - \wv_{k-1}) \big)
		+ \frac{\lambda}{2}\normasq{2}{\wv_k^*} + \frac{1}{2\eta} \normasq{2}{\wv_k^* - \wv_{k-1}} \\
		&\ge \frac{1}{n} \sum_{i=1}^n h \big( \gv\pow{i}(\wv_{k-1}) + \grad \gv\pow{i}(\wv_{k-1})(\wv_k^* - \wv_{k-1}) \big)
		+ \frac{\lambda}{2}\normasq{2}{\wv_k^*} + \frac{L}{2} \normasq{2}{\wv_k^* - \wv_{k-1}} \\
		&\ge \frac{1}{n} \sum_{i=1}^n h \big( \gv\pow{i}(\wv_k^*) \big)
		+ \frac{\lambda}{2}\normasq{2}{\wv_k^*}  \\
		&= F(\wv_k^*) \ge F^* \,.
\end{align*}
Thus, we bound $F_\eta(\wv_{k-1} ; \wv_{k-1}) - F_\eta(\wv_{k}^* ; \wv_{k-1}) \le \Delta F_0$.
We now invoke Prop.~\ref{prop:c:total_compl_svrg_sc} and collect all constants and terms logarithmic in 
$n$, $\eps\inv, \eps_0 \inv$, $\Delta F_0$, $\eta\inv$, $A_\omega D_\omega$ in $\mcT, \mcT', \mcT''$. 
We note that all terms in the logarithm in Prop.~\ref{prop:c:total_compl_svrg_sc} are logarithmic in the problem parameters here.
Letting $N_k$ be the number of 
SVRG iterations required for iteration $k$, we get, 
\begin{align*}
\expect[N] &= \sum_{k=1}^K \expect[N_k] 
\le \sum_{k=1}^K \left( n + \sqrt{\frac{\eta A D n}{\eps_k}} \right) \, \mcT \\
&\le \left[ nK + \sqrt{\frac{\eta A D n}{\eps_0}} \left( \sum_{k=1}^K \sqrt{k}  \right)  \right] \, \mcT \\
&\le \left[ nK + \sqrt{\frac{\eta A D n}{\eps_0}}\,  K^{3/2} \right] \, \mcT' \\
&\le \left[ \frac{n}{\eta \eps^2} (\Delta F_0 + \eps_0)  
+ \sqrt{\frac{\eta A D n}{\eps_0}}\,  \left( \frac{\Delta F_0 + \eps_0}{\eta \eps^2} \right)^{3/2}  
\right] \, \mcT'' \\
&= \left[ \frac{n}{\eta \eps^2} (\Delta F_0 + \eps_0)  
+ \frac{\sqrt{ADn}}{\eta \eps^3} \frac{(\Delta F_0 + \eps_0)^{3/2}}{\sqrt{\eps_0}} 
\right] \, \mcT'' \,.
\end{align*}
\end{proof}
%


\subsection{Some Helper Lemmas} \label{subsec:a:catalyst:helper}
The first lemma is a property of the squared Euclidean norm from \citet[Lemma 5]{lin2017catalyst}, 
which we restate here.
\begin{lemma}\label{lem:c:helper:quadratic}
	For any vectors, $\wv, \zv, \rv \in \reals^d$, we have, for any $\theta > 0$, 
	\begin{align*}
		\normsq{\wv - \zv} \ge (1-\theta) \normsq{\wv - \rv} + \left( 1 - \frac{1}{\theta} \right) \normsq{\rv - \zv} \,.
	\end{align*}
\end{lemma}
The next lemmas consider rates of the sequences $(\alpha_k)$ and $(A_k)$ under different recursions.

\begin{lemma} \label{lem:c:A_k:const_kappa}
	Define a sequence $(\alpha_k)_{k \ge 0}$ as 
	\begin{align*}
		\alpha_0 &= \frac{\sqrt 5 - 1}{2} \\
		\alpha_k^2 &=  (1 - \alpha_k) \alpha_{k-1}^2 \,.
	\end{align*}
	Then this sequence satisfies 
	\begin{align*}
		\frac{\sqrt 2}{k+3} \le \alpha_k \le \frac{2}{k+3} \,.
	\end{align*}
	Moreover, $A_k := \prod_{j=0}^k (1-\alpha_k)$ satisfies
	\begin{align*}
		\frac{2}{(k+3)^2} \le A_k \le \frac{4}{(k+3)^2} \,.
	\end{align*}
\end{lemma}
\begin{proof}
	Notice that $\alpha_0$ satisfies $\alpha_0^2 = 1 - \alpha_0$.
	Further, it is clear from definition that $\alpha_k \in (0, 1)\, \forall k \ge 0$.
	Hence, we can define a sequence $(b_k)_{k\ge 0}$ such that $b_k := 1/\alpha_k$.
	It satisfies the recurrence, $b_k^2 - b_k = b_{k-1}^2$ for $k \ge 1$, 
	or in other words, $b_k = \tfrac{1}{2}\left( 1 + \sqrt{1 + 4 b_{k-1}^2} \right)$.
	Form this we get, 
	\begin{align*}
		b_k &\ge b_{k-1} + \frac{1}{2} \ge b_0 + \frac{k}{2} \ge \frac{3}{2} + \frac{k}{2} \,.
	\end{align*}
	since $b_0 = \frac{\sqrt 5 + 1}{2}$. This gives us the upper bound on $\alpha_k$.
	Moreover, unrolling the recursion, 
	\begin{align} \label{eq:c:helper:2_}
		\alpha_k^2 = (1- \alpha_k) \alpha_{k-1}^2 = A_k \frac{\alpha_0^2}{1 - \alpha_0} = A_k \, .
	\end{align}
	Since $\alpha_k \le 2/(k+3)$, \eqref{eq:c:helper:2_} yields the upper bound on $A_k$.
	The upper bound on $\alpha_k$ again gives us,
	\begin{align*}
		A_k \ge \prod_{i=0}^k \left( 1 - \frac{2}{i+3} \right) = \frac{2}{(k+2)(k+3)} \ge \frac{2}{(k+3)^2} \,,
	\end{align*}
	to get the lower bound on $A_k$. Invoking \eqref{eq:c:helper:2_} again to obtain the lower bound on $\alpha_k$
	completes the proof.
\end{proof}
The next lemma considers the evolution of the sequences $(\alpha_k)$ and $(A_k)$ with a different recursion.
\begin{lemma} \label{lem:c:A_k:inc_kappa}
    Consider a sequence $(\alpha_k)_{k\ge 0}$ defined by $\alpha_0 = \frac{\sqrt{5}- 1}{2}$, and 
    $\alpha_{k+1}$ as the non-negative root of 
    \begin{align*}
        \frac{\alpha_k^2}{1 - \alpha_k} = \alpha_{k-1}^2 \frac{k}{k+1} \,.
    \end{align*}
    Further, define 
    \begin{align*}
        A_k = \prod_{i=0}^k ( 1- \alpha_i) \, .
    \end{align*}
    Then, we have for all $k\ge 0$, 
    \begin{align}
        \frac{1}{k+1} \left(1 - \frac{1}{\sqrt2} \right) \le A_k \le \frac{1}{k+2} \, .
    \end{align}
\end{lemma}
\begin{proof}
    Define a sequence $(b_k)_{k\ge 0}$ such that $b_k = 1/\alpha_k$, for each $k$. This is well-defined because 
    $\alpha_k \neq 0$, which may be verified by induction. This sequence satisfies the recursion for $k\ge 1$:
    $b_k (b_k -1) = \left( \frac{k+1}{k} \right)b_{k-1}$.
    From this recursion, we get, 
    \begin{align}
        \nonumber
        b_k &= \frac{1}{2} \left( 1 + \sqrt{1 + 4 b_{k-1}^2 \left( \frac{k+1}{k} \right)} \right) \\
            \nonumber
            &\ge \frac{1}{2} + b_{k-1} \sqrt\frac{k+1}{k} \\
            \nonumber
            &\ge \frac{1}{2}\left( 1 + \sqrt{\frac{k+1}{k}} + \cdots + \sqrt{\frac{k+1}{2}} \right) + b_0\sqrt{k+1} \\
            \nonumber
            &= \frac{\sqrt{k+1}}{2} \left( 1/\sqrt 2 + \cdots + 1/\sqrt{k+1} \right) + b_0 \sqrt{k+1} \\
            &\stackrel{(*)}{\ge} \sqrt{k+1}\left( \sqrt{k+2} + b_0 - \sqrt 2 \right) 
            = \sqrt{k+1} \left( \sqrt{k+2} + b_0 - \sqrt 2 \right) \,,
    \end{align}
    where $(*)$ followed from noting that $1/\sqrt{2}+\cdots+1/\sqrt{k+1} \ge \int_2^{k+2} \frac{dx}{\sqrt x}
    = 2(\sqrt{k+2}-\sqrt 2)$\,.
    Since $b_0 = 1/\alpha_0 = \frac{\sqrt{5} + 1}{2} > \sqrt 2$, we have, for $k \ge 1$,
    \begin{align}
        \alpha_k \le \frac{1}{\sqrt{k+1}(\sqrt{k+2}+ b_0 - \sqrt{2})} \le \frac{1}{\sqrt{k+1}\sqrt{k+2}} \,.
    \end{align} 
    This relation also clearly holds for $k=0$. Next, we claim that 
    \begin{align}
        A_k = (k+1) \alpha_k^2 \le \frac{k+1}{(\sqrt{k+1}\sqrt{k+2})^2} = \frac{1}{k+2}\, .
    \end{align}
    Indeed, this is true because 
    \begin{align*}
        \alpha_k^2 = (1 - \alpha_k) \alpha_{k-1}^2 \frac{k}{k+1} = A_k \frac{\alpha_0^2}{1 - \alpha_0} \frac{1}{k+1} 
        = \frac{A_k}{k+1} \,.
    \end{align*}
    For the lower bound, we have, 
    \begin{align*}
        A_k = \prod_{i=0}^k (1 - \alpha_i) 
            \ge \prod_{i=0}^k \left(1 - \frac{1}{\sqrt{i+1}\sqrt{i+2}}  \right)
            \ge \left( 1 - \frac{1}{\sqrt2} \right)  \prod_{i=1}^k \left(1 - \frac{1}{i+1}  \right)
            = \frac{1 - \frac{1}{\sqrt{2}}}{k+1} \, .
    \end{align*}
\end{proof}

\begin{lemma} \label{lem:c:helper_logx}
    Fix some $\eps > 0$.
    If $k \ge \frac{2}{\eps} \log \frac{2}{\eps}$, then we have that 
    $\frac{\log k}{k} \le \eps$.
\end{lemma}
\begin{proof}
    We have, since $\log x \le x$ for $x > 0$, 
    \begin{align*}
        \frac{\log k }{k} \le \frac{\log\frac{2}{\eps} + \log\log\frac{2}{\eps}}{\frac{2}{\eps} \log\frac{2}{\eps}}
             = \frac{\eps}{2} \left( 1 + \frac{\log\log\frac{2}{\eps}}{\log\frac{2}{\eps}} \right) \le \eps \,.
    \end{align*}
\end{proof}

\section{Experiments: Extended Evaluation} \label{sec:a:expt}
Given here are plots for all missing classes of PASCAL VOC 2007.
Figures~\ref{fig:plot_all_loc_1} to \ref{fig:plot_all_loc_3} contain the
extension of Figure~\ref{fig:plot_all_loc} while
Figures~\ref{fig:plot_ncvx_loc_1} to \ref{fig:plot_ncvx_loc_3}
contain the extension of Figure~\ref{fig:plot_ncvx_loc} to all classes.

\begin{figure}[!htb]
    \centering
 	\includegraphics[width=0.88\textwidth]{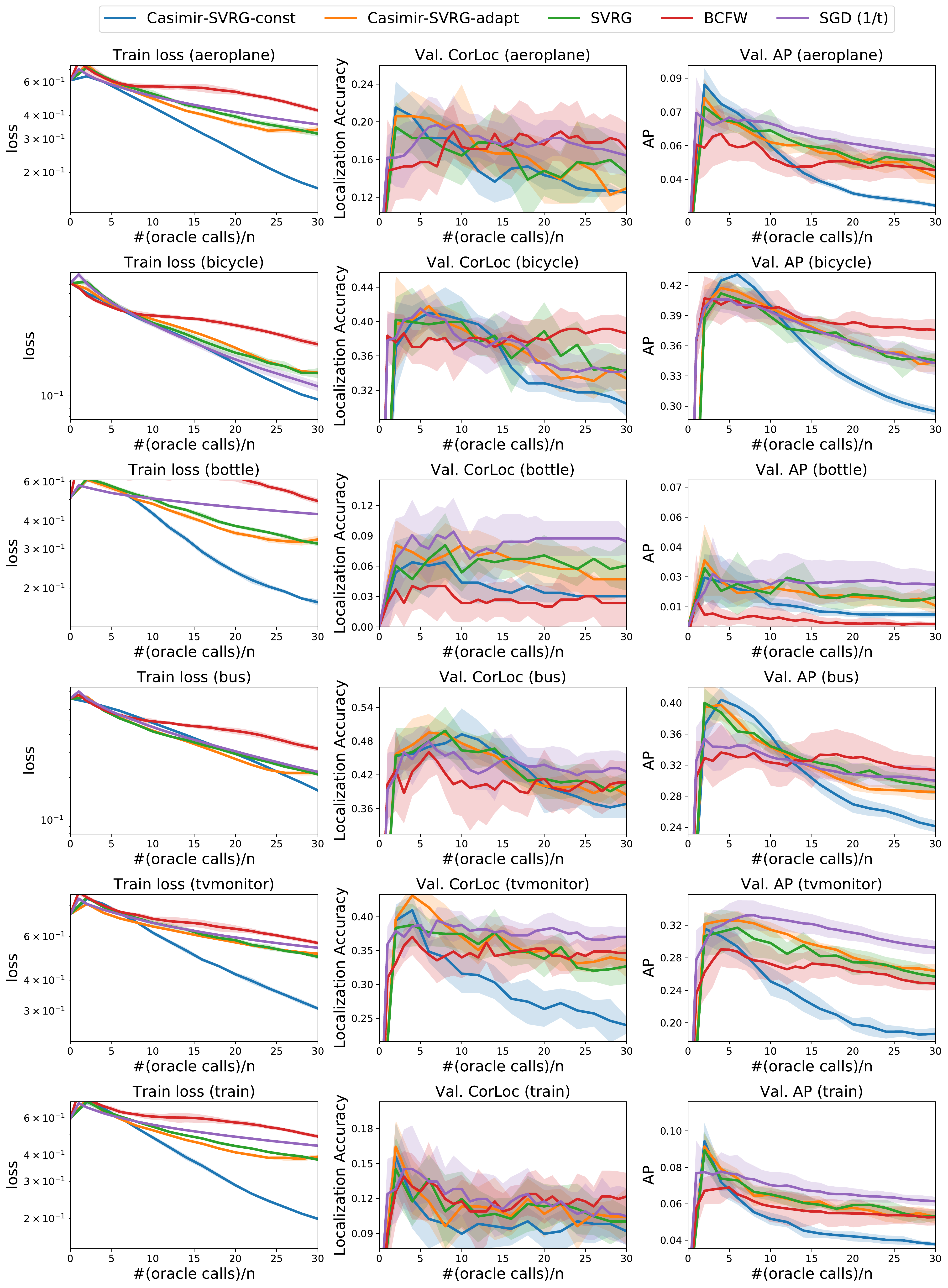}
    \caption{Comparison of convex optimization algorithms 
    	for the task of visual object localization on PASCAL VOC 2007 for $\lambda=10/n$
    	for all other classes (1/3).}\label{fig:plot_all_loc_1}
\end{figure}
\begin{figure}[!htb]
    \centering
    \includegraphics[width=0.88\textwidth]{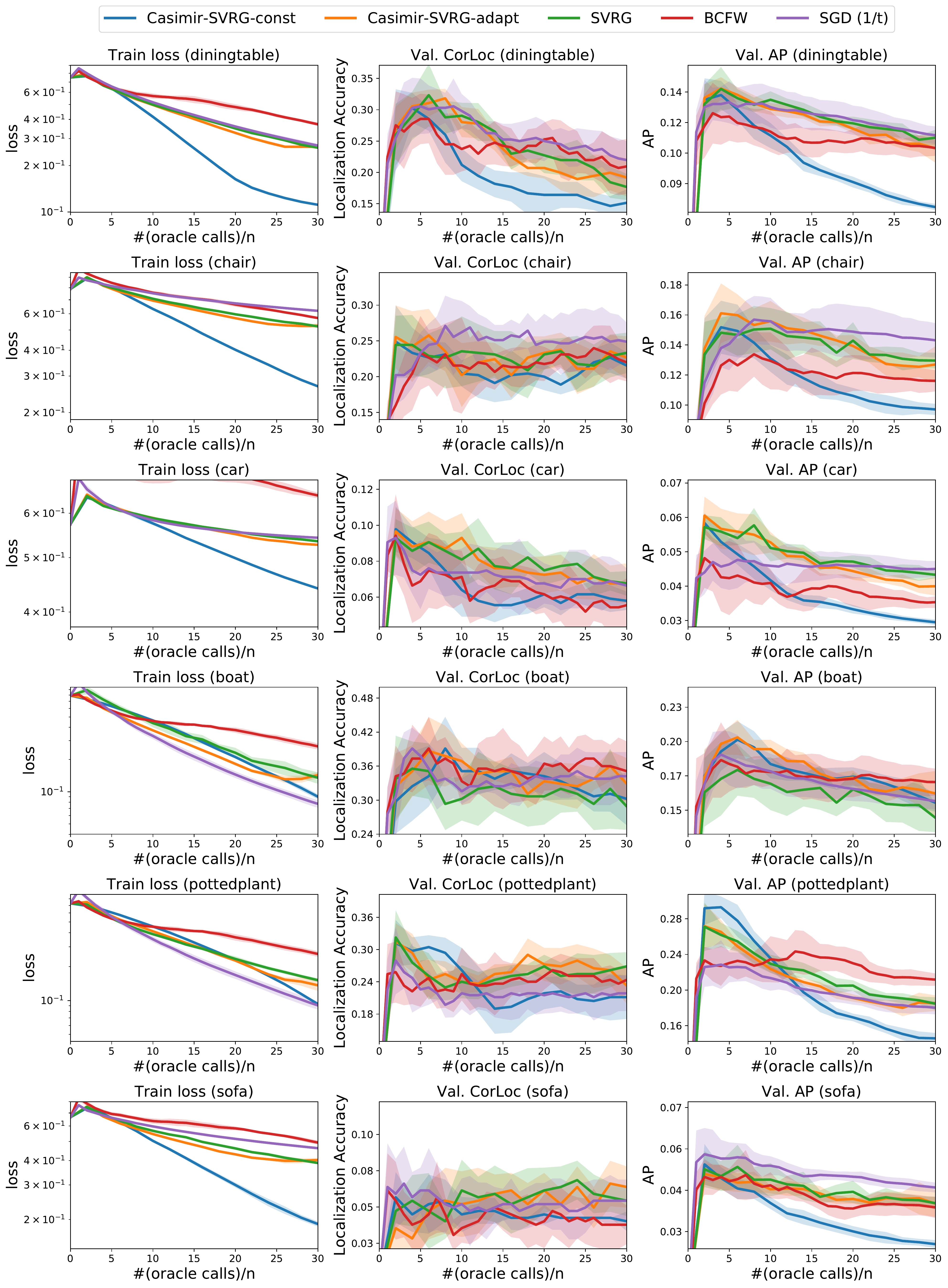}
    \caption{Comparison of convex optimization algorithms 
        for the task of visual object localization on PASCAL VOC 2007 for $\lambda=10/n$
        for all other classes (2/3).}\label{fig:plot_all_loc_2}
\end{figure}
\begin{figure}[!htb]
    \centering
    \includegraphics[width=0.88\textwidth]{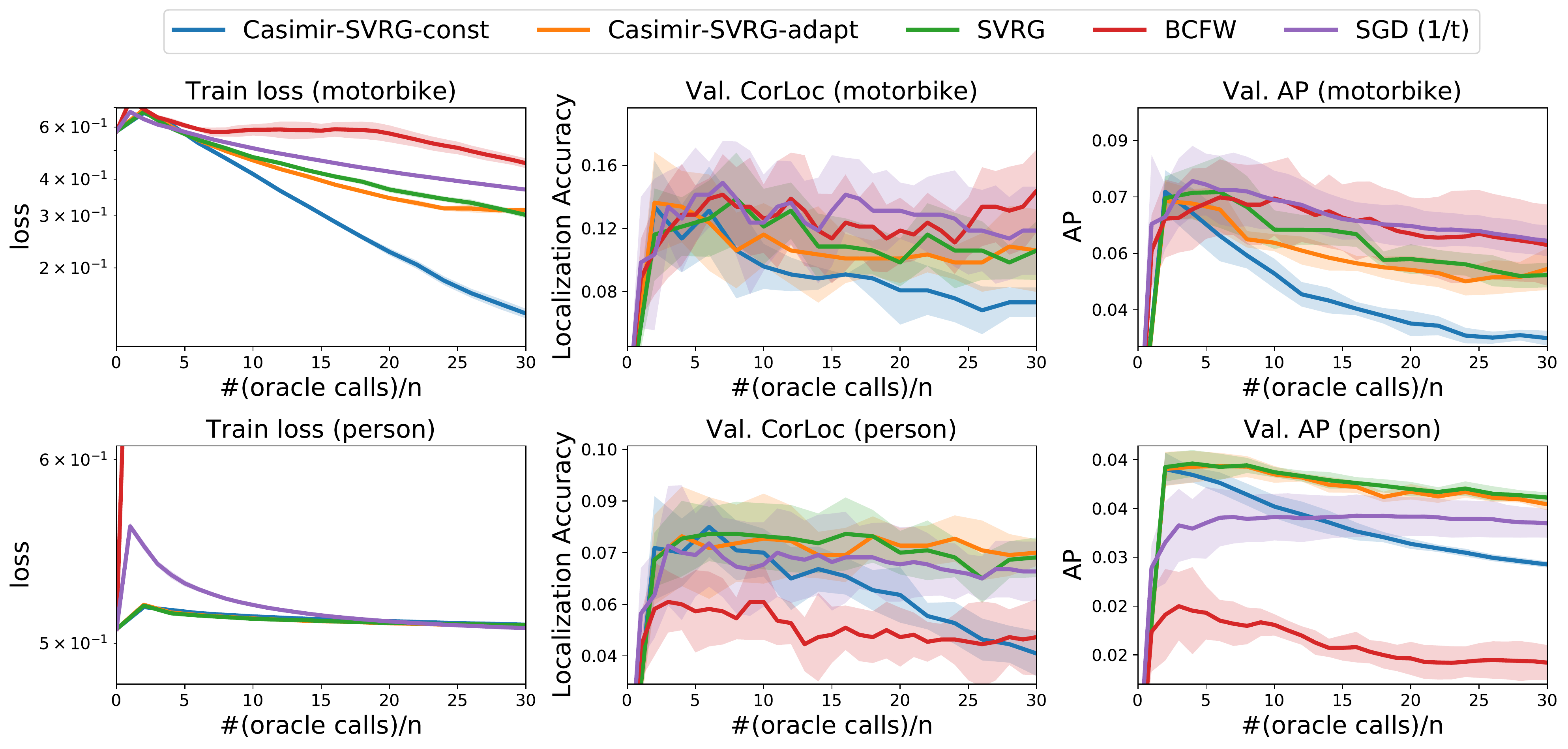}
    \caption{Comparison of convex optimization algorithms 
        for the task of visual object localization on PASCAL VOC 2007 for $\lambda=10/n$
        for all other classes (3/3).}\label{fig:plot_all_loc_3}
\end{figure}

\begin{figure}[!htb]
    \centering
 	\includegraphics[width=0.85\textwidth]{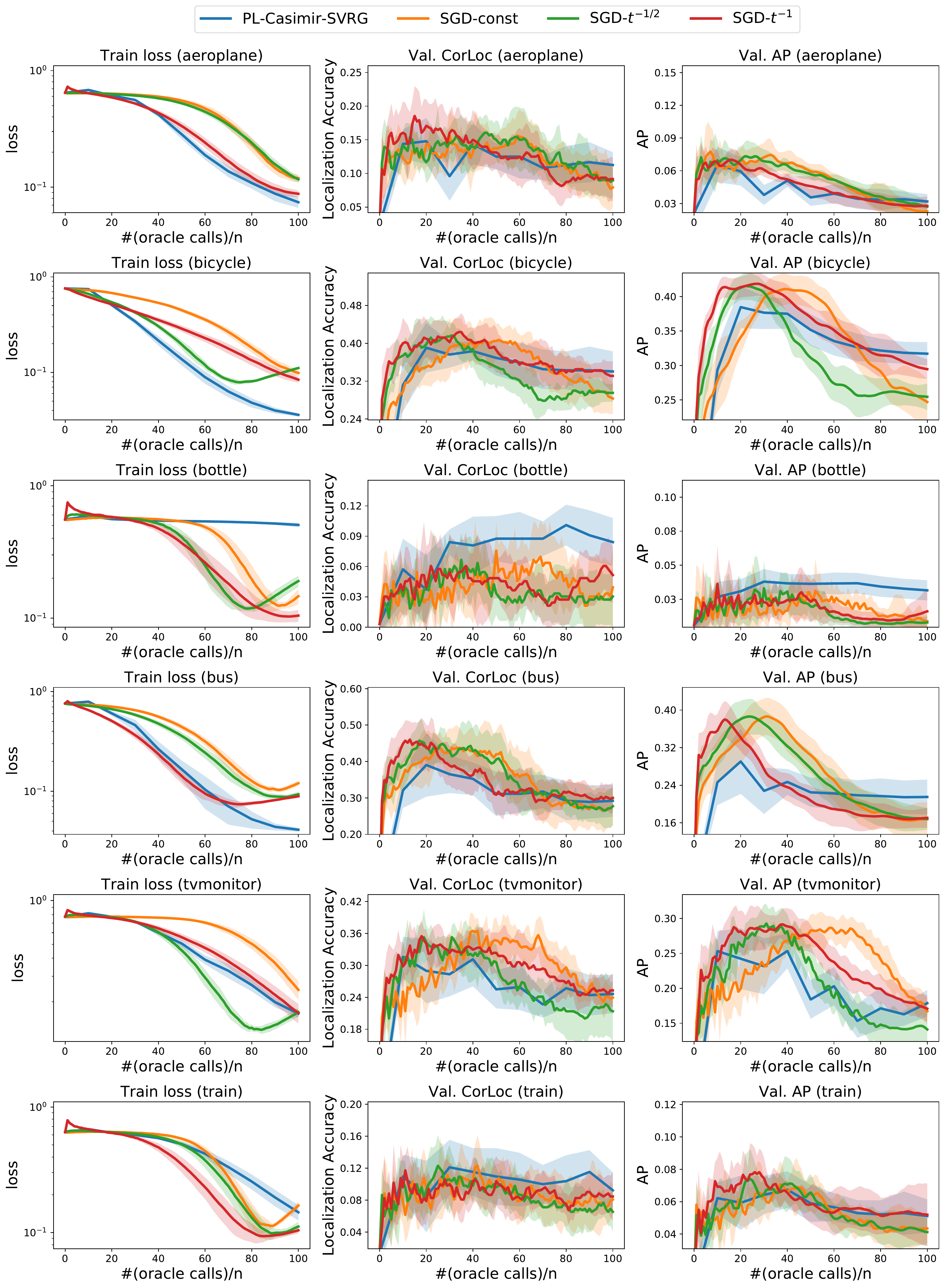}
    \caption{Comparison of non-convex optimization algorithms 
    	for the task of visual object localization on PASCAL VOC 2007 for $\lambda=1/n$
    	for all other classes (1/3).}\label{fig:plot_ncvx_loc_1}
\end{figure}

\begin{figure}[!htb]
    \centering
    \includegraphics[width=0.85\textwidth]{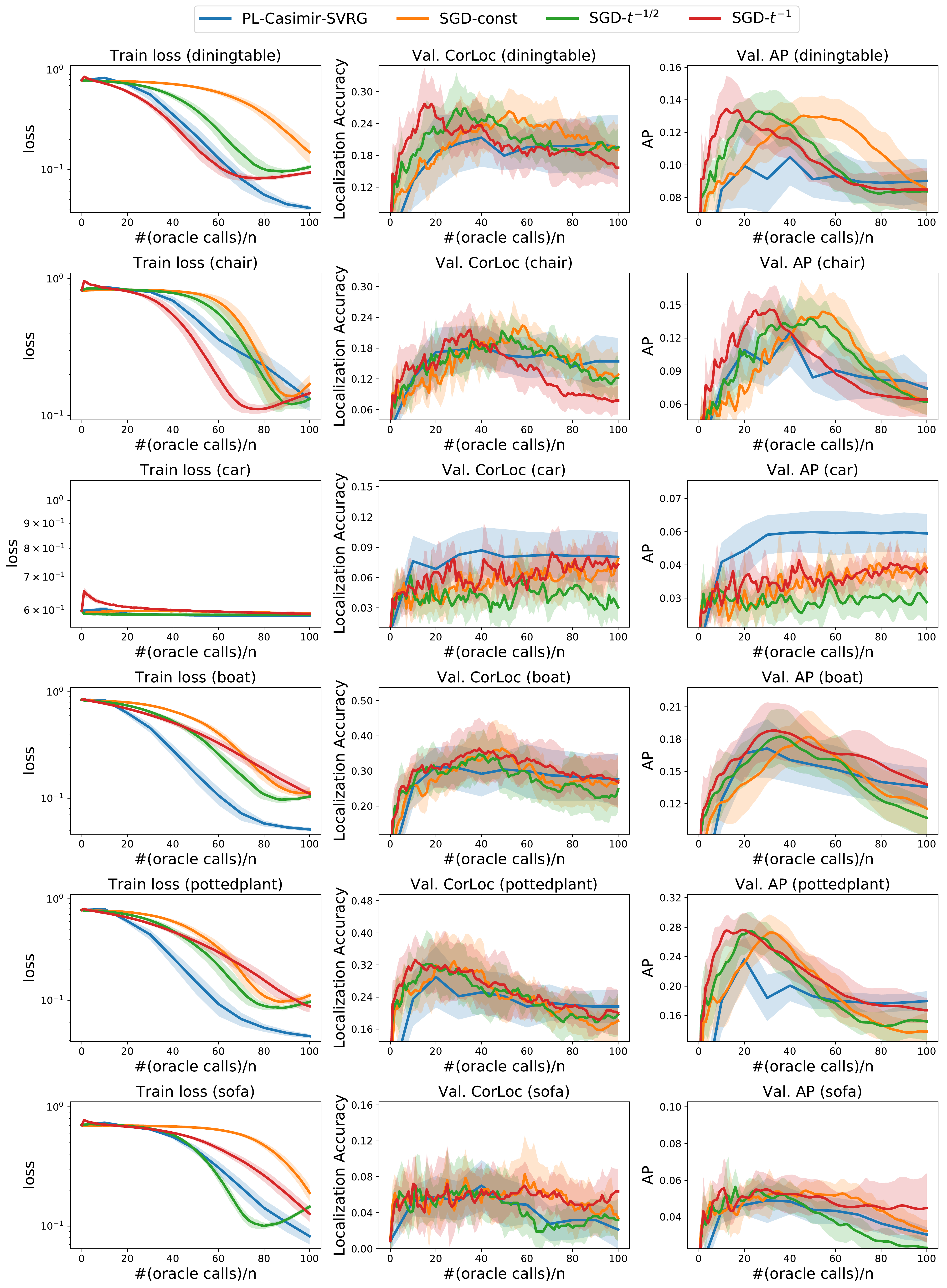}
    \caption{Comparison of non-convex optimization algorithms 
        for the task of visual object localization on PASCAL VOC 2007 for $\lambda=1/n$
        for all other classes (2/3).}\label{fig:plot_ncvx_loc_2}
\end{figure}

\begin{figure}[!htb]
    \centering
    \includegraphics[width=0.85\textwidth]{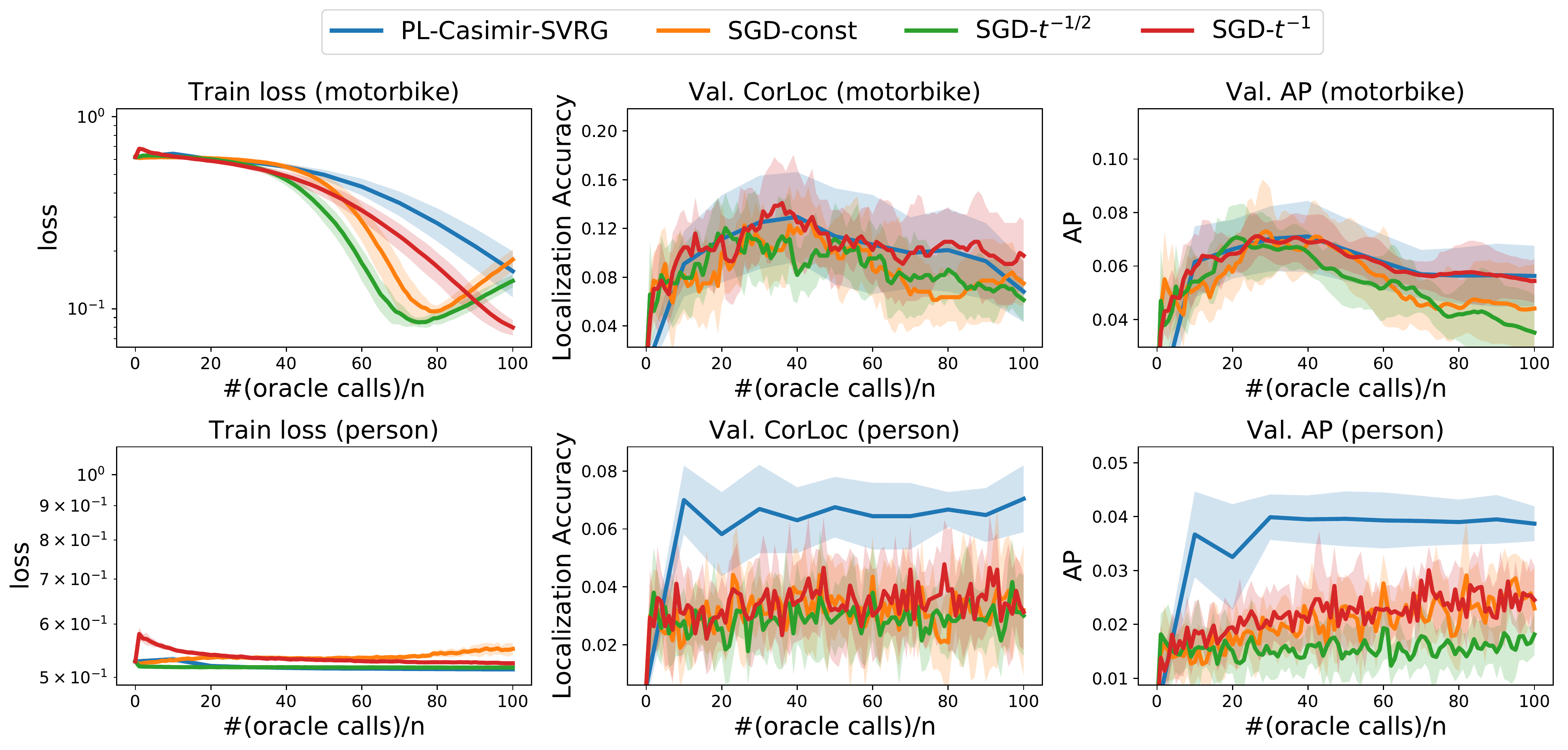}
    \caption{Comparison of non-convex optimization algorithms 
        for the task of visual object localization on PASCAL VOC 2007 for $\lambda=1/n$
        for all other classes (3/3).}\label{fig:plot_ncvx_loc_3}
\end{figure}

\end{document}